\documentclass{article}

\usepackage[utf8]{inputenc} 
\usepackage[T1]{fontenc}    
\usepackage{url}            
\usepackage{booktabs}       
\usepackage{amsfonts}       
\usepackage{nicefrac}       
\usepackage{microtype}      
\usepackage{fullpage}
\usepackage{natbib}

\usepackage{graphicx} 
\usepackage{xcolor}
\usepackage[colorlinks=true,
            allcolors=blue]{hyperref}

\usepackage{enumitem}
\setlist[enumerate,1]{label=(\alph*), ref=(\alph*)}

\usepackage{tikz}
\usepackage{amsmath,amssymb}
\usetikzlibrary{arrows.meta,calc,positioning,decorations.pathmorphing}

\definecolor{lossblue}{RGB}{35,95,175}
\definecolor{alignred}{RGB}{190,55,45}
\definecolor{alignorange}{RGB}{220,125,25}
\definecolor{densegray}{RGB}{110,110,110}
\definecolor{subblue}{RGB}{70,135,195}

\usepackage{makecell}
\usepackage{subcaption}
\usepackage{wrapfig}

\usepackage{amsmath, amsthm, amssymb}
\usepackage{amsfonts,bm,bbm}
\usepackage{apptools}
\usepackage{mathrsfs}

\usepackage{thmtools,thm-restate}
\makeatletter
\renewcommand\thmt@autorefsetup{%
  \@xa\def\csname\thmt@envname autorefname\@xa\endcsname\@xa{\thmt@thmname}%
}
\makeatother

\usepackage{aliascnt}
\newtheorem{theorem}{Theorem}[section]

\newaliascnt{lemma}{theorem}
\newtheorem{lemma}[lemma]{Lemma}
\aliascntresetthe{lemma}

\newaliascnt{proposition}{theorem}
\newtheorem{proposition}[proposition]{Proposition}
\aliascntresetthe{proposition}

\newtheorem{definition}{Definition}
\newtheorem{assumption}{Assumption}
\newtheorem{claim}{Claim}
\newtheorem{remark}{Remark}

\AtAppendix{\counterwithin{claim}{section}}
\AtAppendix{\counterwithin{assumption}{section}}

\def\1{\bm{1}}

\def\eps{{\varepsilon}}

\def\rd{{\textnormal{d}}}

\def\vzero{{\bm{0}}}

\def\vtheta{{\bm{\theta}}}

\def\vxi{{\bm{\xi}}}

\def\va{{\bm{a}}}
\def\vb{{\bm{b}}}

\def\ve{{\bm{e}}}

\def\vh{{\bm{h}}}

\def\vr{{\bm{r}}}
\def\vs{{\bm{s}}}

\def\vu{{\bm{u}}}
\def\vv{{\bm{v}}}
\def\vw{{\bm{w}}}
\def\vx{{\bm{x}}}

\def\vz{{\bm{z}}}

\def\mA{{\bm{A}}}

\def\mI{{\bm{I}}}

\def\mK{{\bm{K}}}

\def\mW{{\bm{W}}}

\DeclareMathAlphabet{\mathsfit}{\encodingdefault}{\sfdefault}{m}{sl}
\SetMathAlphabet{\mathsfit}{bold}{\encodingdefault}{\sfdefault}{bx}{n}
\newcommand{\tens}[1]{\bm{\mathsfit{#1}}}

\def\tH{{\tens{H}}}

\def\tJ{{\tens{J}}}

\def\tT{{\tens{T}}}

\def\gE{{\mathcal{E}}}
\def\gF{{\mathcal{F}}}

\def\sS{{\mathbb{S}}}

\newcommand{\E}{\mathbb{E}}

\newcommand{\R}{\mathbb{R}}

\newcommand{\norm}[1]{\left\|#1\right\|}

\newcommand{\poly}{\mathrm{poly}}
\newcommand{\polylog}{\mathrm{polylog}}

\newcommand{\Rho}{\mathscr{P}}
\renewcommand{\P}{\mathbb{P}}

\newcommand{\ind}{\mathbbm{1}}
\newcommand{\sgn}{\mathrm{sgn}}

\renewcommand{\L}{\mathcal{L}}
\newcommand{\diag}[1]{\mathrm{diag}\left\{#1\right\}}

\newcommand{\BigO}[1]{\mathcal{O}\!\left(#1\right)}
\newcommand{\BigOmega}[1]{\Omega\!\left(#1\right)}
\newcommand{\BigTheta}[1]{\Theta\!\left(#1\right)}
\newcommand{\BigtldO}[1]{\mathcal{\widetilde{O}}\!\left(#1\right)}

\newcommand{\BigtldTheta}[1]{\widetilde{\Theta}\!\left(#1\right)}

\newcommand{\myge}[1]{\stackrel{(\text{#1})}{\ge}}
\newcommand{\myeq}[1]{\stackrel{(\text{#1})}{=}}
\newcommand{\myle}[1]{\stackrel{(\text{#1})}{\le}}
\newcommand{\mylesim}[1]{\stackrel{(\text{#1})}{\lesssim}}

\newcommand{\btt}[1]{{(#1)}}

\newcommand{\bvw}{\overline{\vw}}
\newcommand{\bvv}{\overline{\vv}}
\newcommand{\bvu}{\overline{\vu}}

\newcommand{\bv}{\overline{v}}
\newcommand{\bw}{\overline{w}}
\newcommand{\bu}{\overline{u}}

\newcommand{\ha}{\hat{a}}
\newcommand{\hb}{\hat{b}}
\newcommand{\hsigma}{\widehat{\sigma}}
\newcommand{\Deltamax}{\Delta_{\max}}

\newcommand{\tldmu}{\tilde{\mu}}
\newcommand{\tldE}{\tilde{\E}}
\newcommand{\tldalpha}{\tilde{\alpha}}

\newcommand{\tldtT}{\tilde{\tT}}

\newcommand{\tldvw}{\widetilde{\vw}}
\newcommand{\tldvv}{\widetilde{\vv}}
\newcommand{\tldvu}{\widetilde{\vu}}
\newcommand{\tldv}{\tilde{v}}

\newcommand{\tldmW}{\widetilde{\mW}}

\newcommand{\btldvv}{\overline{\tldvv}}
\newcommand{\btldv}{\overline{\tldv}}

\newcommand{\vvler}{\vv_{\le r}}
\newcommand{\vuler}{\vu_{\le r}}
\newcommand{\vwler}{\vw_{\le r}}
\newcommand{\bvvler}{\bvv_{\le r}}
\newcommand{\bvuler}{\bvu_{\le r}}
\newcommand{\tldvvler}{\tldvv_{\le r}}
\newcommand{\tldvuler}{\tldvu_{\le r}}
\newcommand{\ckvvler}{\ckvv_{\le r}}
\newcommand{\bckvw}{\overline{\ckvw}}

\newcommand{\vvgtr}{\vv_{> r}}
\newcommand{\vugtr}{\vu_{> r}}
\newcommand{\vwgtr}{\vw_{> r}}
\newcommand{\tldvvgtr}{\tldvv_{> r}}
\newcommand{\tldvugtr}{\tldvu_{> r}}
\newcommand{\bvvgtr}{\bvv_{> r}}
\newcommand{\bvugtr}{\bvu_{> r}}

\newcommand{\bbvvler}{\overline{\vvler}}
\newcommand{\bbvwler}{\overline{\vwler}}
\newcommand{\bbvuler}{\overline{\vuler}}
\newcommand{\bbtldvvler}{\overline{\tldvvler}}
\newcommand{\bbtldvuler}{\overline{\tldvuler}}

\newcommand{\bbvvgtr}{\overline{\vvgtr}}
\newcommand{\bbvwgtr}{\overline{\vwgtr}}
\newcommand{\bbvugtr}{\overline{\vugtr}}
\newcommand{\bbtldvvgtr}{\overline{\tldvvgtr}}

\newcommand{\Sname}[1]{S_{\text{#1}}}

\newcommand{\Sgoodi}{\Sname{good,$i$}}
\newcommand{\Spot}{\Sname{pot}}
\newcommand{\Spotone}{\Sname{pot,1}}
\newcommand{\Spoti}{\Sname{pot,$i$}}
\newcommand{\Spotj}{\Sname{pot,$j$}}
\newcommand{\Spotk}{\Sname{pot,$k$}}
\newcommand{\Sdense}{\Sname{dense}}

\newcommand{\epsdir}{\eps_{\text{dir}}}
\newcommand{\epsnm}{\eps_{\text{norm}}}
\newcommand{\errpot}{\err_{\text{pot}}}

\newcommand{\grad}{\mathrm{grad}}
\newcommand{\op}{\mathrm{op}}

\newcommand{\tDelta}{{\bm{\Delta}}}
\newcommand{\vdelta}{{\bm{\delta}}}

\newcommand{\hL}{\widehat{\L}}

\usepackage{cleveref}
\crefname{lemma}{Lemma}{Lemmas}
\crefname{theorem}{Theorem}{Theorems}
\crefname{claim}{Claim}{Claims}
\crefname{figure}{Figure}{Figures}
\crefname{section}{Section}{Sections}
\crefname{remark}{Remark}{Remarks}
\crefname{assumption}{Assumption}{Assumptions}
\crefname{appendix}{Appendix}{Appendices}
\crefname{proposition}{Proposition}{Propositions}
\AddToHook{cmd/appendix/before}{%
    \crefalias{section}{appendix}%
    \crefalias{subsection}{appendix}%
    \crefalias{subsubsection}{appendix}%
}

\newcommand{\bvwler}{\bvw_{\le r}}
\newcommand{\bvwgtr}{\bvw_{> r}}

\newcommand{\tldT}{\widetilde{T}}
\newcommand{\hT}{\hat{T}}
\newcommand{\err}{\mathrm{err}}
\usepackage{bbm}

\newcommand{\vchi}{\bm{\chi}}

\newcommand{\tldvtheta}{\widetilde{\vtheta}}

\newcommand{\ckvw}{\check{\vw}}
\newcommand{\ckvv}{\check{\vv}}

\newcommand{\bckv}{\overline{\check{v}}}
\newcommand{\bckvv}{\overline{\ckvv}}

\newcommand{\ckf}{\check{f}}

\newcommand{\unif}{\mathrm{Unif}}

\newcommand{\tldf}{\widetilde{f}}
\newcommand{\avgrho}{\frac{1}{|\Rho|} \sum_{\rho\in\Rho}}
\newcommand{\id}{\mathrm{Id}}

\author{
  Mo Zhou$^1$\qquad 
  Weihang Xu$^1$\qquad 
  Simon S. Du$^1$\qquad 
  Maryam Fazel$^{1,2}$\\
  $^1$University of Washington
  \qquad
  $^2$Amazon, Inc.\\
  \texttt{\{mozhou17,xuwh,ssdu\}@cs.washington.edu, mfazel@uw.edu}
}

\title{Learning Orthogonal Multi-Index Models Beyond Small Initialization: Incremental Learning, Competitive Dynamics and Symmetry}

\begin{document}

\maketitle

\begin{abstract}
    Recent work has identified incremental learning in shallow networks trained on single-index and multi-index models. However, existing analyses often rely on simplifying settings, such as small initialization, correlation loss, or layer-wise training. These choices reduce neuron interactions and leave some feature learning dynamics under standard initialization unexplored. 
    We study training dynamics for polynomial-width two-layer networks learning orthogonal multi-index targets under standard initialization using polynomially many samples. We first prove that incremental learning still occurs: the loss decreases sequentially according to the Hermite expansion of the target, with lower-order components learned before higher-order components recover the individual target directions. 
    In this standard initialization regime, training also shows a competitive reallocation of parameter mass: after the total mass fits the target mean and stabilizes, mass shifts into the target subspace and then concentrates on aligned neurons. Our theoretical analysis uses slightly modified gradient flow, while vanilla gradient descent empirically exhibits the same qualitative dynamics. Technically, we introduce a symmetry-based finite-width approximation via symmetrized networks, rather than comparing directly with an infinite-width limit. This yields better control of approximation errors and may be of independent interest.
    \end{abstract}

\section{Introduction}
Understanding how neural networks learn structured target functions is a central question in the theory of feature learning. A canonical setting is the learning of low-dimensional target functions, such as single-index and multi-index models, by shallow networks~\citep{arous2021online,damian2022neural,ba2022high,bietti2022learning,mousavi2022neural,glasgow2023sgd,lee2024neural,troiani2024fundamental,collins2024hitting,zhou2024does,montanari2026phase} (see also the recent survey \cite{bruna2025survey}). In these models, the target depends on the input only through a few relevant directions, so feature learning amounts to discovering these hidden directions through gradient-based training.

One particularly interesting phenomenon in this setting is \emph{incremental learning}: the network does not learn all parts of the target at once, but instead learns them in a structured order~\citep{abbe2022merged,abbe2023sgd,dandi2024two,bietti2023learning}. Under Gaussian inputs, this order is naturally described by the Hermite expansion of the target. Lower-order components can be learned first, while higher-order components become important later for identifying the relevant directions.

For multi-index models, a recent line of work \citep{li2020learning,ge2021understanding,oko2024learning,ren2024learning,csimcsek2024learning,ren2025emergence} has studied
targets of the form
\[
    f_*(\vx)=\sum_{i=1}^r a_i^*\sigma(\vw_i^{*\top}\vx),
    \qquad
    \vx\sim N(\vzero,\mI_d),
\]
where the target directions $\{\vw_i^*\}_{i=1}^r$ are orthonormal. Existing analyses often rely on simplifying modifications such as small initialization, correlation loss, or layer-wise training
\citep{damian2022neural,damian2023smoothing,csimcsek2024learning,
zhang2025neural,ren2024learning}. These choices make the dynamics more tractable by reducing the role of the current network output in the gradient signal, so that each neuron approximately amplifies its weak initial correlation with the target. Thus, while these works establish versions of incremental learning, they leave open the technically more challenging standard-initialization regime, where the network output is non-negligible and neurons interact through their collective prediction. These interactions substantially change the learning mechanism, requiring new analysis techniques. 

In this paper, we show that incremental learning persists under standard initialization, but through a different mechanism. The population loss decreases in Hermite order: the network first fits the mean, then recovers the target subspace through the second-order component, and finally uses higher-order components to identify the individual target directions. This process is not simply independent amplification of weakly aligned neurons. After the mean is fitted, the total parameter mass is effectively constrained, and later stages learn by reallocating this mass \emph{competitively}: first from irrelevant directions into the target subspace, and then from diffuse target-subspace neurons to neurons aligned with individual target directions.
\begin{theorem}[Informal version of \cref{thm: main}]
    Under standard random initialization, (modified) gradient flow on a polynomial-width two-layer network learns orthogonal multi-index targets to small population loss using polynomial samples. Moreover, the loss decreases incrementally according to the Hermite expansion of the target and parameter mass undergoes competitive dynamics described in \cref{fig: illustration}.
\end{theorem}
Thus, this identifies competitive mass reallocation as an additional mechanism for feature learning under standard initialization.

\begin{figure}[t]
    \centering
    \includegraphics[width=0.32\textwidth]{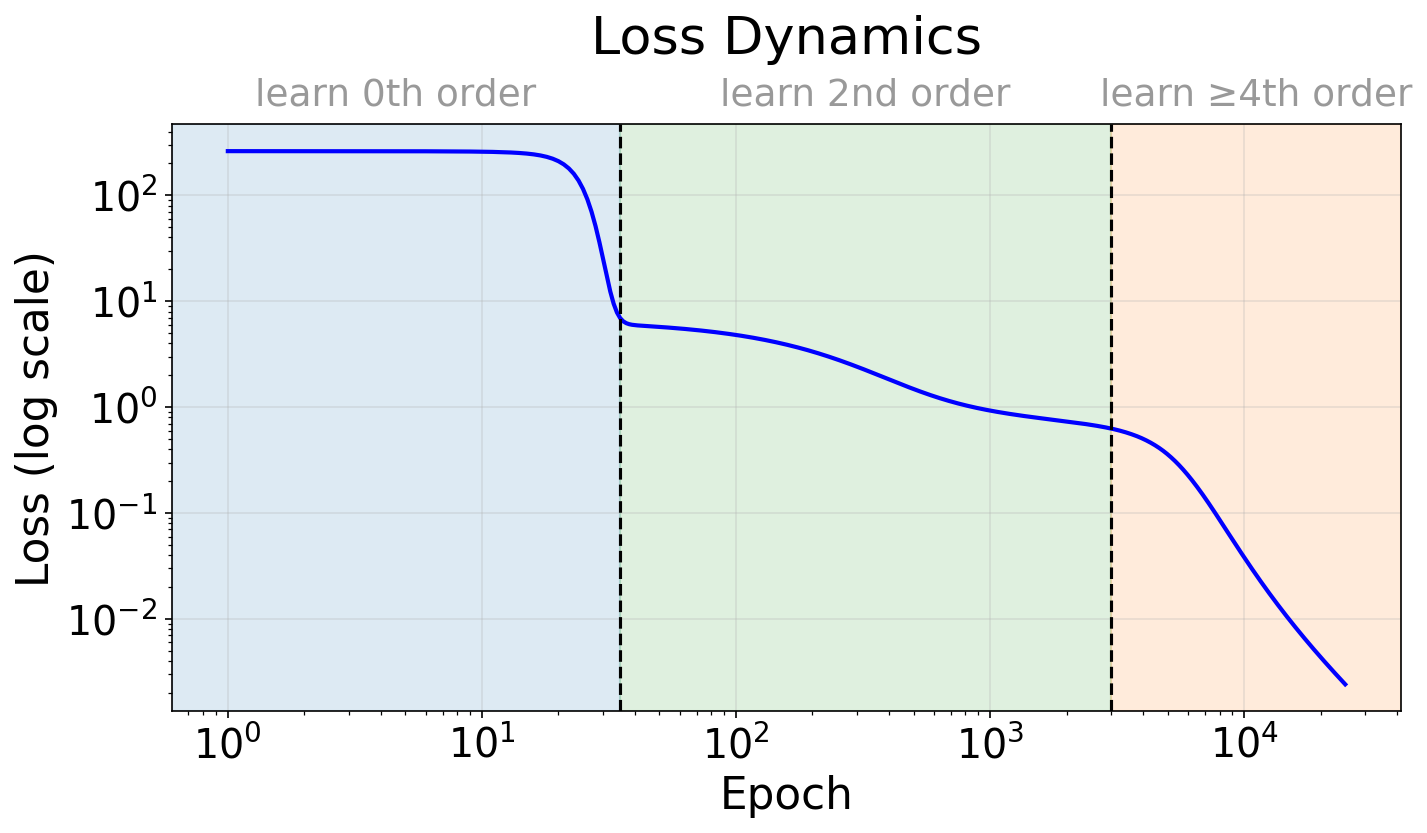}
    \hfill
    \includegraphics[width=0.32\textwidth]{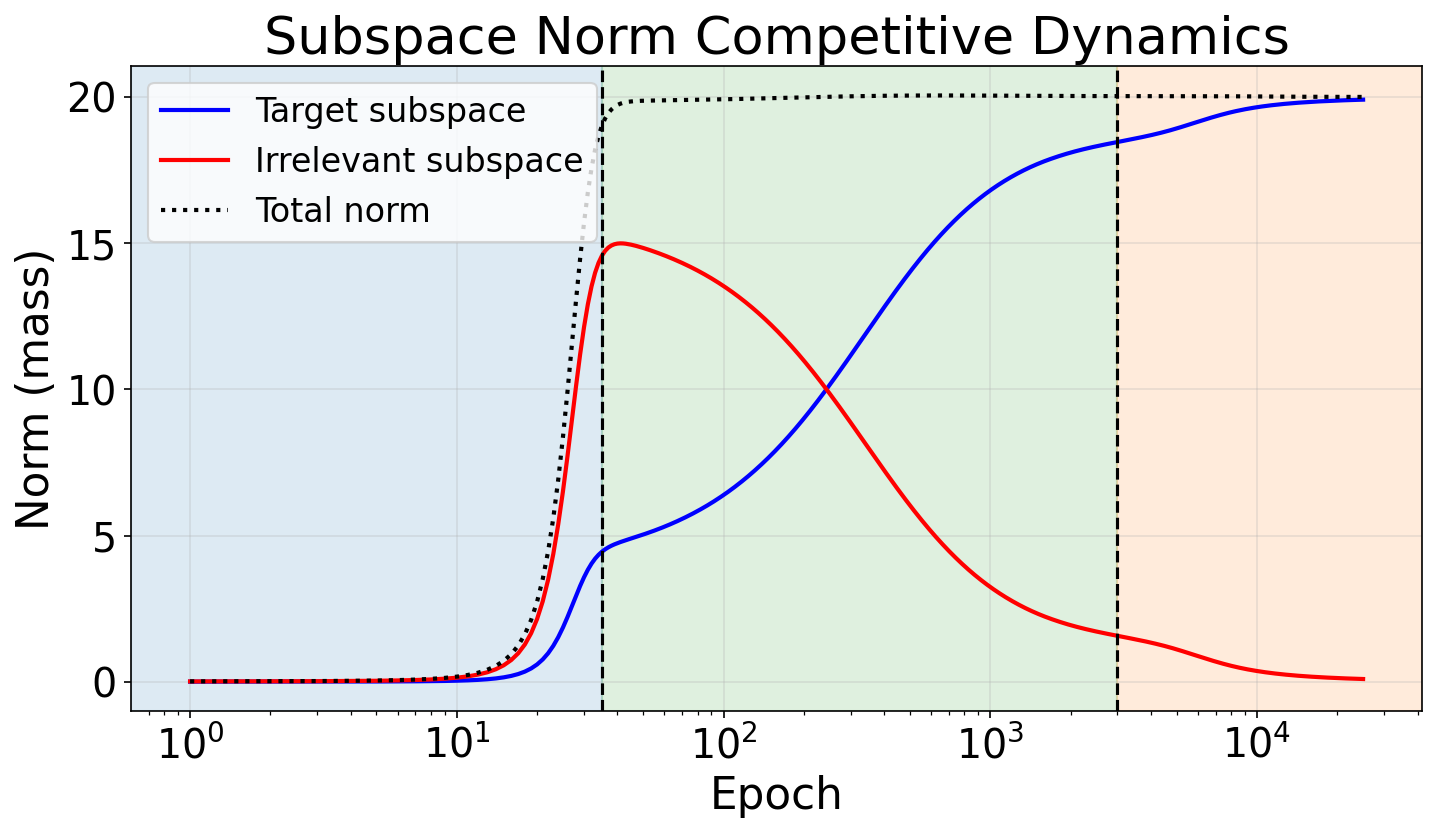}
    \hfill
    \includegraphics[width=0.32\textwidth]{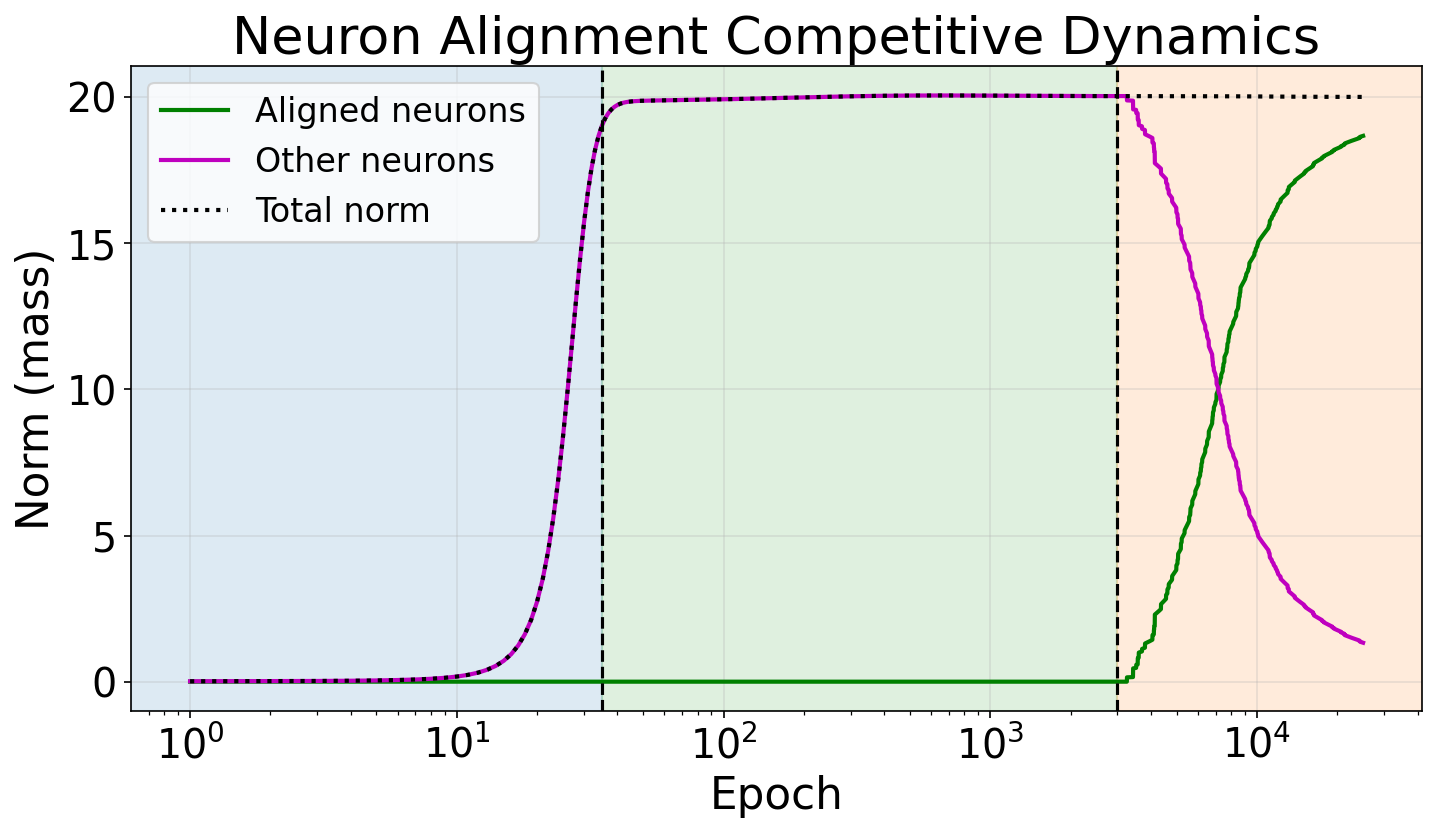}

    \caption{Training dynamics for learning the orthogonal multi-index target \eqref{eq:target-main}. 
    \textbf{Left:} the loss decreases incrementally, with lower-order Hermite components fitted before higher-order components. 
    \textbf{Middle:} after the total mass quickly stabilizes, mass shifts from irrelevant directions to the target subspace. 
    \textbf{Right:} mass concentrates from dense, non-aligned neurons onto target-aligned neurons. 
    ``Mass'' is defined as average squared parameter norm over the indicated subspace or neuron groups.
    We train with vanilla gradient descent using $d=100$, $r=20$, $m=250$, and $n=50000$.
    The vanilla GD trajectory shown here exhibits the same qualitative stagewise behavior as the dynamics characterized in \cref{thm: main} under~\eqref{eq: dynamic of w}.
    }
    \label{fig: illustration}
\end{figure}

\subsection{Related works, key challenges and our contributions}
\paragraph{Incremental learning beyond small initialization.}
A growing line of work has shown that shallow networks can learn low-dimensional targets incrementally. 
In sparse and single-index settings, this appears in saddle-to-saddle dynamics \cite{saxe2013exact,jacot2021saddle,pesme2023saddle}, where higher-order correlations emerge progressively during training 
\citep{damian2023smoothing,bietti2022learning,lee2024neural,berthier2025learning}.
For multi-index models, related works study staircase dynamics \cite{abbe2022merged,abbe2023sgd} and establish subspace or directional recovery, often under slight modifications such as small initialization, correlation loss, or layer-wise training, to achieve strong theoretical guarantees
\citep{damian2022neural,dandi2024two,csimcsek2024learning,ren2024learning,oko2024learning,zhang2025neural,ren2025emergence}. Together, these results show that gradient-based methods can exploit the Hermite structure of the target to recover hidden low-dimensional features.

A common reason these analyses are tractable is that the dynamics become approximately target-driven. Under small initialization with scale $\alpha\ll 1$, a two-layer learner starts with and remains at  $f_t \ll f_*$ during the early feature-learning phase. Each neuron is therefore driven mainly by its correlation with the target, and the many-neuron dynamics are approximately decoupled.
However, this decoupled picture does not apply under standard initialization, where the network output is non-negligible from the start and neurons interact through their shared prediction.

Our first contribution is to prove that incremental learning nevertheless persists in this interacting regime. With standard random initialization and polynomially many neurons and samples, the population loss decreases in Hermite order for the analyzed dynamics. This extends the analysis of incremental learning beyond the small-initialization setting studied in prior work. Related results exist, but often rely on infinite-width limits or nonparametric dynamics \citep{berthier2025learning,bietti2023learning}.

\paragraph{Competitive dynamics in feature learning.}
Many existing analyses of feature-learning dynamics under small initialization largely reduce to an amplification mechanism of the form $\dot x = ax^P$, with $P \ge 1$ and $a>0$
(e.g., \citep{allen2020towards,cao2022benign,ge2021understanding,
csimcsek2024learning,ren2025emergence}). This describes how a useful feature grows from a weak initial correlation with the target. Under standard initialization, however, feature learning can also be competitive: different parts of the parameter population interact, and the growth of one component can occur at the expense of another.

We identify this competitive effect explicitly. After the 0th-order component is fitted, the total parameter mass is effectively constrained. During 2nd-order learning, mass moves from irrelevant directions into the target subspace. During higher-order learning, mass moves from non-aligned neurons to neurons aligned with individual target directions. Thus the mechanism is not only amplification of initially better-aligned neurons, but also competitive mass reallocation across directions and neurons. The resulting effective dynamics have a Lotka--Volterra-type structure \citep{hofbauer1998evolutionary}; see \cref{sec: dynamics main-text}.

Competitive effects have also appeared in phenomena such as grokking
\citep{varma2023explaining,lyu2023dichotomy} and circuit competition in mechanistic studies of large language models
\citep{ortu2024competition,lindsey2025biology}. We give a precise characterization of one such competitive mechanism in an analytically tractable setting.

\paragraph{Symmetry-based finite-width approximation.}
The main technical challenge is to control neuron interactions under standard initialization. Mean-field analyses capture such interactions through infinite-width limits \citep{chizat2018global,mei2018mean,chizat2019lazy}, but turning these limits into quantitative finite-width guarantees can require very large width, many samples, or restrictions to short time horizons
\citep{mousavi2024learning,wang2024mean,takakura2024mean,berthier2025learning, montanari2025dynamical}. Recent works develop finite-width coupling techniques, inspired by propagation of chaos
\citep{sznitman2006topics,chaintron2022propagation}, that obtain polynomial runtime and sample complexity by comparing finite networks with their infinite-width limits
\citep{li2020learning,mahankali2023beyond,glasgow2025propagation}. However, such coupling can still be delicate in the high-accuracy regime, as the required width may depend super-polynomially on the target accuracy $1/\eps$.

A key difficulty is particle-wise coupling near the target directions. In a finite-width network, the neurons that eventually align with the target come from the most favorably initialized particles in a finite sample. In the infinite-width limit, however, there are particles with arbitrarily better initial alignment, which can move faster than any fixed finite counterpart and lead to large coupling errors.

We instead use the symmetry of the target and the population dynamics to construct a finite \emph{symmetrized network}. This network preserves the invariances that make the population dynamics tractable, while retaining a direct finite-particle correspondence with the actual network. We then prove that the actual finite-width empirical dynamics remain close to this symmetrized counterpart; see \cref{sec: coupling main-text}. This avoids projecting neurons onto the sphere and yields a polynomial-width coupling bound, allowing us to control the analyzed finite-width training trajectory under standard initialization.

\paragraph{Comparison with related works on orthogonal multi-index models.}
Several works consider closely related orthogonal additive targets. \cite{li2020learning} study the full-dimensional setting $r=d$. In contrast, for $r\ll d$, we characterize an additional competitive subspace-recovery stage that is absent when $r=d$. \cite{ge2021understanding} analyze analogous incremental dynamics for orthogonal tensor decomposition rather than neural networks, while \cite{oko2024learning} consider nearly orthogonal ridge combinations under layer-wise training. \cite{ben2026learning} provide a sharp characterization of SGD dynamics for quadratic targets, where the relevant signal consists only of second-order Hermite components. \cite{ren2024learning} obtain sharper information-exponent bounds using spherical SGD with correlation loss followed by second-layer fitting, while \cite{csimcsek2024learning} analyze population correlation-loss dynamics in the absence of low-order Hermite components. \cite{ren2025emergence} study online SGD with small initialization in the information-exponent $>2$ regime and precisely characterize the resulting dynamics. \cite{zhu2025gradient} analyze square-loss GD for orthogonal ReLU teachers, but assume small initialization and weak recovery of the target directions. Our focus is instead on interacting finite-width square-loss dynamics under standard initialization, without layer-wise training. In particular, we characterize both incremental loss reduction and competitive mass reallocation arising from the coexistence of low- and high-order components.

\section{Preliminaries and problem setup}\label{sec: setup}

\paragraph{Notation.}
Let $[n]=\{1,\ldots,n\}$. Bold symbols, e.g. $\vv,\mA,\tT$, denote vectors, matrices, and tensors. We use $\mathcal{O},\Theta,\Omega$ to hide universal constants, and $\widetilde{\mathcal{O}},\widetilde{\Theta},\widetilde{\Omega}$ to also hide polylogarithmic factors. We also write $f\lesssim g$ for $f=\mathcal{O}(g)$, and $a=b\pm\delta$ to mean $|a-b|\le |\delta|$. We use $v^{\btt{t}}$ to denote the value of $v$ at time $t$.

\paragraph{Setup.} We consider Gaussian data $\vx\sim N(\vzero,\mI_d)$ and
the orthogonal multi-index target
\begin{equation}\label{eq:target-main}
    f_*(\vx)=\sum_{j=1}^r \sigma(\vw_j^{*\top}\vx),
\end{equation}
where the target directions $\{\vw_j^*\}_{j=1}^r$ are orthonormal. By rotational invariance of the Gaussian distribution, we assume
without loss of generality that $\vw_j^*=\ve_j$ for all $j\le r$.

The learner is a two-layer network with mean-field scaling factor $1/m$ \cite{li2020learning,mahankali2023beyond,glasgow2025propagation}
\begin{equation}\label{eq:network-main}
    f_{\mW}(\vx)
    =\frac{1}{m}\sum_{i=1}^m \norm{\vw_i}_2^2\,
        \sigma(\bvw_i^\top \vx),
\end{equation}
with parameters $\mW=(\vw_1,\ldots,\vw_m)^\top\in\R^{m\times d}$ and $\bvw_i=\vw_i/\norm{\vw_i}_2$ . The 2-homogeneous parametrization has been used in many previous works \cite{li2020learning,zhou2021local,ge2021understanding,ren2025emergence}.

\begin{assumption}[Activation]\label{assum: activation}
The activation $\sigma:\R\to\R$ is even\footnote{This assumption also covers ReLU and related non-even activations under symmetric initialization: $\vw_i\sim N(\vzero,\frac{1}{d}\mI_d)$ for $i\le m/2$ and $\vw_{i+m/2}=-\vw_i$. Each pair acts as a single neuron with the even activation $\sigma(x)=\phi(x)+\phi(-x)$.} and satisfies one of the following conditions. Either
$\sigma(x)=|x|$, or $\sigma$ is smooth with $\sigma(0)=0$,
$\sup_x(|\sigma'(x)|+|\sigma''(x)|)=\BigO{1}$,
$\sum_{k\ge0} k^2\hsigma_{2k}^2=\BigO{1}$, and 
$\hsigma_{2k}\ne0$, $\hsigma_{2k+1}=0$ for all $k\ge0$. Here $\hsigma_k$ is the $k$-th Hermite coefficient of $\sigma$ defined below.
\end{assumption}

This assumption is naturally obtained by symmetrizing activation $\phi$ by $\sigma(x)=\phi(x)+\phi(-x)$ and includes many commonly used activations: ReLU gives
$\sigma(x)=|x|$; smooth ReLU-type activations such as GeLU \cite{hendrycks2016gaussian} and Swish \cite{ramachandran2017searching} give smooth even activations after symmetrization.

We train on the empirical square loss $\hL$ on data $\vx_k\overset{\mathrm{i.i.d.}}{\sim} N(\vzero,\mI_d)$ and evaluate on population loss $\L$:
\begin{align*}
    \L(\mW)
    &=\frac12\E_{\vx\sim N(\vzero,\mI_d)}
        \bigl[(f_{\mW}(\vx)-f_*(\vx))^2\bigr],
    \quad
    \hL(\mW)
    =\frac{1}{2n}\sum_{k=1}^n
        \bigl(f_{\mW}(\vx_k)-f_*(\vx_k)\bigr)^2.
\end{align*}
Training follows gradient flow with weight decay starting from standard random initialization:
\begin{equation}\label{eq: dynamic of w}
    \frac{\rd \vw_i}{\rd t}
    =-\lambda_t\vw_i
      -m\gamma_{t,i}\nabla_{\vw_i}\hL(\mW),
    \qquad
    \vw_i^{(0)}\overset{\mathrm{i.i.d.}}{\sim}N(\vzero,\frac{1}{d}\mI_d)
\end{equation}
The factor $m$ compensates for the mean-field normalization in \eqref{eq:network-main}, yielding
an $\BigO{1}$ particle-level dynamics. Unless otherwise specified, we take
$\lambda_t=\lambda$ and $\gamma_{t,i}=1$, where $\lambda=r\lambda_0$ and
$\lambda_0>0$ is a sufficiently small numerical constant. 
The only exception is a short time interval in which $\gamma_{t,i}$ is modified for technical reasons; see \cref{sec: omit proof main result}.

The initialization above is standard. Each neuron is initialized with $\norm{\vw_i(0)}_2\approx1$. This is in contrast to small-initialization analyses, where $\norm{\vw_i(0)}_2\ll 1$ so that the network output is initially negligible and remains so for a long time~\cite{abbe2022merged,abbe2023sgd}.

\paragraph{Hermite and tensor decomposition.}
Let $\{h_k\}_{k\ge0}$ be the orthonormal Hermite basis under the standard Gaussian measure, and
write $\sigma(z)=\sum_{k\ge0}\hsigma_k h_k(z)$ as the Hermite expansion of $\sigma$, where
$\hsigma_k=\E_{z\sim N(0,1)}[\sigma(z)h_k(z)]$ is the $k$-th Hermite coefficient. For unit vectors
$\vu,\vv$, Gaussian orthogonality gives
$\E_{\vx}[h_k(\vu^\top\vx)h_\ell(\vv^\top\vx)]
=\mathbf{1}\{k=\ell\}\langle \vu,\vv\rangle^k$; see \cite{o2014analysis}.
Therefore the population loss has decomposition below \cite{ge2017learning}, which will play a central role in our proof
\begin{equation}\label{eq:hermite-loss-main}
    \L(\mW)
    =
    \frac12\sum_{k\ge0}\hsigma_k^2
    \norm{\tT_k-\tT_k^*}_F^2,
    \qquad
    \tT_k
    :=
    \frac1m\sum_{i=1}^m \norm{\vw_i}_2^2\bvw_i^{\otimes k},
    \qquad
    \tT_k^*
    :=
    \sum_{j=1}^r \ve_j^{\otimes k}.
\end{equation}
Here $\tT_k$ is the network's $k$th-order moment tensor and $\tT_k^*$ is the corresponding target
tensor, with the convention $\bvw^{\otimes0}=1$ and $\ve_j^{\otimes0}=1$.
Thus, by ``fitting the $k$-th order term,'' we mean $\tT_k\approx\tT_k^*$ in
\eqref{eq:hermite-loss-main}. The corresponding gradient formula is given in \cref{claim: tldw dynamic}.

We also use $\mathrm{P}_k$ for projection onto the degree-$k$ Hermite subspace and
$\mathrm{P}_{\ge k}:=\sum_{\ell\ge k}\mathrm{P}_\ell$. Define the degree-$k$ target energy
$E_k:=\frac12\norm{\mathrm{P}_k f_*}_{L^2}^2
=\frac12\hsigma_k^2\norm{\tT_k^*}_F^2$, so
$E_0=\frac12 r^2\hsigma_0^2$ and $E_k=\frac12 r\hsigma_k^2$ for $k\ge1$. Also write
$E_{\ge k}:=\sum_{\ell\ge k}E_\ell$. 

\section{Main Results}\label{sec: main result}

We now state our main result. We show that, under standard initialization, a polynomial-width two-layer network learns the orthogonal multi-index target with polynomial samples. Moreover, the training trajectory exhibits incremental learning that the population loss decreases in Hermite order.

\begin{theorem}[Main result]\label{thm: main}
    Under \cref{assum: activation}, consider the dynamics given in \eqref{eq: dynamic of w} to learn target $f_*$ with $\log^2 d < r/\log^2 r < r\log^2 r<\log^3 d$. Then for any given target error $\eps>e^{-\log^2 r}$, there exists proper choice of weight decay $\lambda_t$ and stepsize $\gamma_t$, width $d^3\le m\le \poly(d)$ and sample size $n=\BigTheta{d^{3+c}}$ for any fixed small constant $c$, such that with probability $1-1/r$ we have test loss
    \[
        \L(\mW) \le \eps
    \]
    within $T=\BigO{r/\log m}$ time and thus recover target $f_*$.
    If activation\footnote{This is only needed to apply the local convergence result \cite{zhou2021local}. We expect similar results for general activations; see \cref{sec: local convergence} for discussions.} $\sigma(x)=|x|$, then for any given target error $\eps >0$ with $n=\BigTheta{d^{3+c}+\frac{dr^4\log d}{\eps^2}}$ the same dynamics reaches $\L(\mW)\le \eps$ within time $T_\eps=\BigO{r/\eps}$.

    Moreover, the loss decreases sequentially according to the Hermite expansion of the target:
    \begin{center}
\small
\setlength{\tabcolsep}{4pt}
\renewcommand{\arraystretch}{2.2}
\begin{tabular}{p{0.2\linewidth}|p{0.14\linewidth}|p{0.25\linewidth}|p{0.30\linewidth}}
\hline
Learned component
&Transition time
&Loss before transition
&Loss after transition
\\
\hline
$0$th order
& $\tau_0=\BigtldTheta{\frac 1r}$
& \makecell[l]{$\L_t\ge (1-\eta)E_{\ge0}$ \\for $0\le t\le \delta_\tau\tau_0$}
& $\L_{\tau_0}\le \eta E_0+(1+\eta)E_{\ge2}$
\\
\hline
$2$nd order
& $\tau_2=\BigtldTheta{1}$
& \makecell[l]{$\L_t\ge \eta E_0+(1-\eta)E_{\ge2}$ \\for $\tau_0\le t\le (1-\delta_\tau)\tau_2$}
& $\L_{\tau_2}\le \eta E_{\le2}+(1+\eta)E_{\ge4}$
\\
\hline
$\ge4$th order
& $\tau_4=\BigtldTheta{r}$
& \makecell[l]{$\L_t\ge \eta E_{\le2}+(1-\eta)E_{\ge4}$ \\
for $\tau_2\le t\le (1-\delta_\tau)\tau_4$}
& $\L_{\tau_4}\le \eps$
\\
\hline
\end{tabular}
\end{center}
    where $E_{\ge k}=\frac{1}{2}\norm{\mathrm{P}_{\ge k} f_*}_{L^2}^2$, $\eta>0$ is a small constant and $\delta_\tau=1/\sqrt{\log d}=o(1)$.
\end{theorem}

The table should be read as a loss-level characterization of incremental learning. Before each transition, the corresponding Hermite component remains present in the loss. After the transition, its contribution has been reduced to a small residual. Thus \cref{thm: main} identifies not only final target recovery, but also the order in which the target components are learned.

The polynomial width and sample bounds in \cref{thm: main} are not intended to be optimal. They arise from our finite-width coupling arguments in \cref{sec: coupling main-text}. Further discussion of these quantitative requirements and their relation to information-exponent analyses is given in \cref{sec: omit proof main result}.

This distinguishes our result from many small-initialization analyses, where incremental learning is often established through support or subspace recovery rather than through the decrease of the true square loss \citep{abbe2022merged,abbe2023sgd,dandi2024two}. In this sense, our result is closer in spirit to \citep{berthier2025learning,bietti2023learning,ren2025emergence}, although the setting and technical challenges are different. As discussed in \cref{sec: dynamics main-text}, we also reveal a competitive reallocation of parameter mass, a phenomenon \emph{not} captured by the decoupled dynamics typically used in small-initialization analyses.

The theorem also aligns with prior work separating feature learning from fixed-kernel methods in various settings
\citep{wei2019regularization,allen2019learning,li2020learning,mahankali2023beyond}. We obtain a polynomial-complexity guarantee, whereas fixed-kernel methods, including the neural tangent kernel \citep{jacot2018neural}, require super-polynomial sample complexity for such multi-index targets \citep{ghorbani2021linearized}.

The rest of the main text gives the proof strategy. In \cref{sec: symmetry}, we construct the finite symmetrized network. In \cref{sec: dynamics main-text}, we analyze its population dynamics, which exhibits the incremental and competitive dynamics. In \cref{sec: coupling main-text}, we transfer this analysis to the actual empirical network by coupling it to the symmetrized population dynamics.

\section{Symmetry in Dynamics}\label{sec: symmetry}

We next describe the symmetry structure used in the proof. The main idea is to replace the infinite-width reference dynamics with a finite \emph{symmetrized network}, which preserves the invariances of the target and population dynamics while remaining tied to the same finite particles in the actual network.

Throughout the rest of the paper, we use mean-field notation for empirical averages over particles. For the original network, write
$\mu_t:=m^{-1}\sum_{i=1}^m\delta_{\vw_i(t)}$. For the symmetrized network defined below using the group $\Rho$, write
$\tldmu_t:=(m|\Rho|)^{-1}\sum_{i=1}^m\sum_{\rho\in\Rho}
\delta_{\tldvw_{i,\rho}(t)}$. For any test function $g$, we use
$\E_{\mu_t}g:=m^{-1}\sum_i g(\vw_i(t))$ and
$\E_{\tldmu_t}g:=(m|\Rho|)^{-1}\sum_{i,\rho}g(\tldvw_{i,\rho}(t))$.
When the measure and time are clear from context, we abbreviate these as $\E g$ and $\tldE g$.

\paragraph{Symmetry in population dynamics.}
Let $\Rho$ be the finite group of transformations that independently flip coordinate signs and permute the irrelevant coordinates. That is, for every $\rho\in\Rho$,
\[
    \rho(\vx)
    =(s_1x_1,\ldots,s_rx_r,s_{r+1}x_{\pi(r+1)},\ldots,s_dx_{\pi(d)})^\top,
\]
where $s_i\in\{\pm1\}$ and $\pi$ is a permutation of $\{r+1,\ldots,d\}$. Since $\sigma$ is even
and $f_*$ only depends on the first $r$ coordinates, we have
$f_*(\rho(\vx))=f_*(\vx)$ for every $\rho\in\Rho$.

This invariance is preserved by the population dynamics. Indeed, the Gaussian distribution, the
target $f_*$, and the population loss are all invariant under $\Rho$, so the population gradient
field is equivariant with respect to the action of $\Rho$ on the particles. Consequently, if the particle measure
is $\Rho$-invariant at some time, it remains $\Rho$-invariant along population gradient flow.

\begin{claim}[Symmetry preservation]\label{claim:symmetry-preservation-main}
Consider population gradient flow. If the empirical particle measure $\mu_T$ is
$\Rho$-invariant at some time $T$, then $\mu_t$ remains
$\Rho$-invariant for all $t\ge T$. In particular, $f_{\mu_t}$ is $\Rho$-invariant for all
$t\ge T$.
\end{claim}

An immediate consequence of $\Rho$-invariance is conditional sign symmetry: for each coordinate $i$, all moments that are odd in $v_i$ vanish after conditioning on the remaining coordinates. This removes odd monomials from the Hermite-gradient expansion. If $\mu$ is a symmetric particle measure and $\tT_k=\E_{\mu}\bigl[\norm{\vw}_2^2\bvw^{\otimes k}\bigr]$,
then the population dynamics of a particle $\vv$ can be written as (see \cref{claim: loss and gradient})
\begin{align*}
    \frac{\rd \vv}{\rd t}
    =& -\lambda\vv + \norm{\vv} \left(
        2\hsigma_0^2 \left(
            \tT_0^*- \tT_0
            \right)\bvv
        +\sum_{k\ge 2} \hsigma_k^2 \left(
        k (\tT_k^*-\tT_k)(\bvv^{\otimes (k-1)})
        - (k-2)(\tT_k^*-\tT_k)(\bvv^{\otimes k})\bvv
        \right)
    \right),
    \end{align*}
Here $\tT_k^*-\tT_k$ is the residual in the degree-$k$ Hermite component. Symmetry greatly simplifies these tensors. For example,
$\tT_2(\mu)=\E_{\mu}[\vw\vw^\top]$ is diagonal because the off-diagonal entries vanish by conditional sign symmetry. More generally, only tensor entries with even coordinate multiplicities survive. The exact coordinate-wise formula is given in
\cref{claim: tldw dynamic}.

This is the simplification that we will exploit. Although existing works often obtain similar simplifications by considering infinite-width dynamics \citep{li2020learning,mahankali2023beyond}, the source of the simplification is not infinite width itself, but symmetry in the target and in the population dynamics. Instead of enforcing this symmetry through an infinite-width limit, we enforce it directly at finite width through a symmetrized network. We believe the same idea can be adapted to other symmetry groups as well.

\paragraph{Symmetrized networks.}
Motivated by the symmetry preservation above, for any finite network $f$ we define its
symmetrization by
\begin{equation*}
    \tldf(\vx):=\frac1{|\Rho|}\sum_{\rho\in\Rho} f(\rho(\vx)).
\end{equation*}
Equivalently, since $\rho$ is orthogonal, $f(\rho(\vx))$ is obtained by replacing each neuron
$\vw_i$ by $\rho^{-1}(\vw_i)$. Hence $\tldf$ can be viewed as a finite network with $m|\Rho|$
neurons, obtained by copying every neuron along its $\Rho$-transform. By construction, $\tldf$ is
exactly invariant under $\Rho$ at finite width.

A useful consequence is the exact decomposition in \cref{lem: decomp of loss}:
\begin{equation}\label{eq:bias-variance-main}
    \L(f)
    =
    \frac12\E_{\vx}\bigl[(\tldf(\vx)-f_*(\vx))^2\bigr]
    +
    \frac12\E_{\vx}\bigl[(f(\vx)-\tldf(\vx))^2\bigr].
\end{equation}
The cross term vanishes because $\tldf-f_*$ is invariant under $\Rho$, while $f-\tldf$ has zero average under $\Rho$. We therefore first analyze the symmetrized dynamics, which controls the first term in
\eqref{eq:bias-variance-main}; this is done in \cref{sec: dynamics main-text}. We then show that the
original finite-width network remains close to its symmetrized counterpart, so that the second term
is small; this is done in \cref{sec: coupling main-text}. Similar decompositions also appear in settings with
rotational invariance \cite{ren2023depth}.

The symmetrized network is not an infinite-width object. As discussed above, \emph{this distinction is important:} we compare the actual finite-width network directly with this finite symmetrized network, rather than coupling it to an infinite-width limit; see
\cref{sec: coupling main-text}.

\section{Training dynamics of symmetrized network under population loss}\label{sec: dynamics main-text}

We now describe the population dynamics of the symmetrized network $\tldf$, focusing on the incremental and competitive mechanisms. Rigorous statements are deferred to the appendix.

\subsection{A recurring competitive ODE}\label{sec: competitive ode main-text}

Before presenting the training stages, we isolate a simple competitive mechanism that
appears repeatedly. Consider two nonnegative quantities $p(t)$ and $q(t)$ satisfying,
up to lower-order errors,
\begin{equation}\label{eq:competitive-ode-main}
    \dot p
    =
    p(-\lambda+A(1-p-Bq)+\kappa(1-p)),
    \qquad
    \dot q
    =
    q(-\lambda+A(1-p-Bq)),
\end{equation}
\begin{wrapfigure}{r}{0.3\textwidth}
    \centering
    \vspace{-12pt}
    \includegraphics[width=0.22\textwidth]{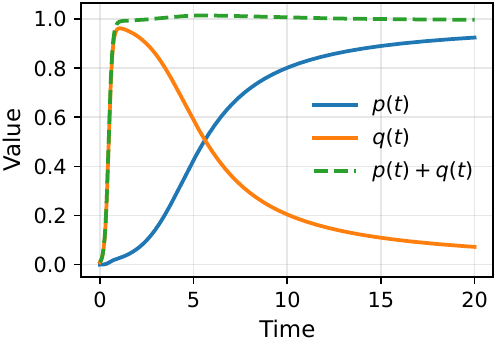}
    \caption{\small Competitive dynamics $A=10$, $B=\kappa=1$, $\lambda=0.1$
    }
    \vspace{-20pt}
\end{wrapfigure}
where $A,B,\kappa>0$. This is a generalized Lotka--Volterra-type competition system \cite{hofbauer1998evolutionary}. The variables $p$ and $q$ share the common growth term $-\lambda+A(1-p-Bq)$, but $p$ has an additional
advantage $\kappa(1-p)$. Note that
$\frac{\rd}{\rd t}\log\frac{q}{p}
    =
    -\kappa(1-p)$. 
As long as $p<1$, the ratio $q/p$ decreases: the favored component $p$ grows relative to the
competing component $q$. In our training dynamics, the common term is controlled by lower-order components, so the total scale can remain nearly constrained while mass is reallocated from
$q$ to $p$. This is the basic mechanism behind both subspace learning and aligned-neuron norm fitting below.

\paragraph{Relation to tensor-power dynamics.}
Many small-initialization analyses largely reduce feature learning to an independent
amplification dynamics, often of the form $\dot x=ax^P$
\citep{allen2020towards,cao2022benign,ge2021understanding,csimcsek2024learning,
ren2025emergence}. This is a one-component limit of the competitive picture: if the
competitor $q$ is absent or nearly frozen, the favored component mainly grows by amplifying its own advantage. The competitive ODE captures the
additional possibility that growth occurs by reallocating mass away from competing components.

\paragraph{Effect of weight decay.}
Weight decay accelerates this reallocation. In \eqref{eq:competitive-ode-main}, when
$0<\lambda<A+\kappa$, the stable equilibrium is
$(p_\lambda,q_\lambda)=(1-\lambda/(A+\kappa),0)$. Near this point, the advantage term
remains positive, $\kappa(1-p_\lambda)=\kappa\lambda/(A+\kappa)$, so the competing
component decays exponentially:
$q(t)/p(t)\approx \exp(-\kappa\lambda t/(A+\kappa))$. When $\lambda=0$, the equilibrium
becomes $(1,0)$, the advantage $\kappa(1-p)$ vanishes as $p\to1$, and the toy model gives
only polynomial decay, typically $q(t)=\Theta(1/t)$ by center-manifold calculation. Thus weight decay keeps a positive
competitive advantage near the limiting state and speeds up the transfer of mass.

\subsection{Training dynamics in stages}

\paragraph{Stage 1.1: learning the mean (0th-order).}
Let $\tldalpha:=\tldE\norm{\vw}^2$ denote the total parameter mass of the symmetrized network,
equivalently the 0th-order component $\tldT_0$. At initialization, $\tldalpha\approx1$,
whereas the 0th-order target mass is $T_0^*=r$. In this stage, the 0th-order Hermite term dominates the dynamics, so $\frac{\rd \tldv_i}{\rd t}\approx \tldv_i(2\hsigma_0^2(r-\tldalpha)-\lambda)$, where $\tldv_i$ denotes the $i$-th coordinate of particle $\tldvv$.
Thus all coordinates of a neuron are scaled by nearly the same factor: directions and relative norms remain essentially fixed, while the total mass grows. The induced dynamics of $\tldalpha$ is
$    \frac{\rd}{\rd t}\tldalpha
    \approx
    4\hsigma_0^2\tldalpha
    (r-\lambda/(2\hsigma_0^2)-\tldalpha).
$
Thus $\tldalpha$ rapidly grows from order one to
$r-\lambda/(2\hsigma_0^2)$ in time $\BigtldO{1/r}$. This gives the first loss drop in \cref{thm: main}. See \cref{sec: stage 1.1}.

\paragraph{Stage 1.2: learning the target subspace (2nd-order).}
After the mean has been fitted, the 2nd-order Hermite term becomes the leading source of movement. Informally,
$\frac{\rd \tldv_i}{\rd t}
    \approx
    \tldv_i(
        -\lambda
        +2\hsigma_0^2(r-\tldalpha)
        +2\hsigma_2^2(\ind_{i\le r}-\tldE w_i^2)
    )$.
By symmetry, $\tldE w_j^2$ is the same for all irrelevant coordinates $j>r$, while the
relevant coordinates remain balanced, with $\tldE w_i^2/\tldE w_j^2\approx1$ for
$i,j\le r$. We therefore summarize the 2nd-order population state by
$\beta_1(t):=\frac1r\sum_{i\le r}\tldE w_i^2$ and $\beta_2(t):=\tldE w_j^2$ for $j>r$, so that $r\beta_1$ is the target-subspace mass and $(d-r)\beta_2$ is the irrelevant mass.

The leading two-dimensional dynamics has the competitive form \eqref{eq:competitive-ode-main}
\begin{align*}
    \dot\beta_1
    &\approx
    2\beta_1(
        -\lambda
        +2\hsigma_0^2(r-r\beta_1-(d-r)\beta_2)
        +2\hsigma_2^2(1-\beta_1)
    ),\\
    \dot\beta_2
    &\approx
    2\beta_2(
        -\lambda
        +2\hsigma_0^2(r-r\beta_1-(d-r)\beta_2)
    ).
\end{align*}
The two masses share the same global mass-control term from the fitted 0th-order component, while $\beta_1$ has the additional advantage
$2\hsigma_2^2(1-\beta_1)$. Thus target-subspace mass grows by displacing irrelevant mass. The stable point is
$(\beta_1^*,\beta_2^*)
    =
    (
        1-\lambda/(2(\hsigma_0^2 r+\hsigma_2^2)),\,0
    )$.
Thus, the scaled ratio 
$(d-r)\beta_2/(r\beta_1)$ decays exponentially, as explained in \cref{sec: competitive ode main-text} and made precise in \cref{prop: stage12-ode}. Since this ratio is initially $\Theta(d/r)$, the transfer of mass into the target subspace takes time $\BigO{\log d}$.
At this point $\tT_2\approx \tT_2^*$, giving the second loss drop in \cref{thm: main}.

At the neuron level, the target-subspace and irrelevant part of each norm track $\beta_1$ and $\beta_2$; see \cref{lem: stage 1.2 convergence of 2nd order}. Hence each neuron becomes almost entirely supported on the target subspace, $\norm{\tldvv_{\le r}}_2/\norm{\tldvv}_2=1-\BigO{1/r}$. Its direction within the target
subspace remains close to its initial direction, so Stage~1.2 recovers the subspace but not the individual directions. See \cref{sec: stage 1.2}.

\paragraph{Stage 2.1: learning the target directions from $\ge4$th-order.}
At this point the target subspace has been learned, but most neurons are still spread
out within that subspace. The higher-order Hermite terms now become the leading source
of directional movement and induce a tensor-power-like dynamics, similar to
\citep{li2020learning,ge2021understanding,ren2025emergence}.

For each target direction $i\le r$, define the set of potential neurons
\begin{equation*}
    \Spoti
    :=\left\{\vv:
        \left(\bbvvler^{(T_1)}\right)_i^2
        \ge (1-\omega^{-0.01})
        \max_\vu\left(\bbvuler^{(T_1)}\right)_i^2
      \right\},
    \qquad
    \Spot:=\bigcup_{i=1}^r\Spoti,
\end{equation*}
where $\omega=\log\log\log r$ absorbs large constants and
$\bbvvler:=\vvler/\norm{\vvler}$ is the normalized projection of $\vv$ onto the target subspace. Thus $\Spoti$ contains neurons with near-maximal alignment with $\ve_i$ inside the target subspace. Overparameterization ensures $\max_\vu(\bbvuler^{(T_1)})_i^2\approx 2\log(m)/r$ with high probability, and that the
sets $\Spoti$ are disjoint; see \cref{lem: stage 2.1 good pot neuron}. We denote the
remaining neurons by $\Sdense$, since they remain small and spread out within the target
subspace.

For a potential neuron in $\Spoti$, let $p_i=(\bbvvler)_i^2$ be its alignment with $\ve_i$. The leading dynamics is
\begin{equation*}
    \frac{\rd}{\rd t} p_i
    \approx
    2(1-p_i)
        \sum_{j\ge2} 2j\hsigma_{2j}^2 p_i^{j},
\end{equation*}
up to lower-order terms controlled by 0th- and 2nd-order components. This resembles a tensor-power-type dynamics: a small advantage is amplified, so a potential neuron aligns with the corresponding target direction. Since the best initial coordinate has size about $2\log(m)/r$, this alignment occurs in time $\BigO{r/\log m}$; see \cref{lem: stage 2.1 neuron main result}. 

The main technical issue is to ensure that lower-order terms do not disrupt this higher-order dynamic. We show that the low-order population terms remain stable and balanced, so dense neurons keep their directions essentially fixed, while for each $i\le r$ at least one neuron in $\Spoti$ reaches alignment at least
$1-\epsdir^2$ with $\ve_i$. Thus Stage 2.1 discovers all individual target directions, although the aligned neurons still carry small total norm and do not yet significantly change the low-order population terms. See \cref{sec: stage 2.1} for details.

\paragraph{Stages 2.2 and 2.3: fitting aligned-neuron norms ($\ge4$th-order).}
Let $G_i$ be the neurons aligned with $\ve_i$ at the end of Stage~2.1, and define
$\ha_i:=(1/m)\sum_{\vv\in G_i}\norm{\vv}_2^2$ and
$\hb_i:=\tldE_{\Sdense} w_i^2$. Here $\ha_i$ is the aligned mass in direction $i$, while
$\hb_i$ is the remaining dense mass in that direction. The target
directions have been found, so the remaining task is to transfer mass from $\hb_i$ to $\ha_i$ and fit the target.

The leading norm dynamics again has the competitive form
\eqref{eq:competitive-ode-main}. The balances $\ha_i\approx\ha_j$ and
$\hb_i\approx\hb_j$ are maintained; see \cref{lem: stage 2.2 log spread}. Thus the
effective dynamics reduce to a two-dimensional competition between aligned and
non-aligned mass:
\begin{align*}
    \dot \ha_i
    &\approx
    2\ha_i(
        -\lambda+2\hsigma_0^2(r-S)
        +2\hsigma_2^2(1-\ha_i-\hb_i)
        +2\hsigma_{\ge4}^2(1-\ha_i)
        ),\\
    \dot \hb_i
    &\approx
    2\hb_i(
        -\lambda+2\hsigma_0^2(r-S)
        +2\hsigma_2^2(1-\ha_i-\hb_i)
        ),
\end{align*}
where $S:=\sum_i(\ha_i+\hb_i)$ and
$\hsigma_{\ge4}^2:=\sum_{j\ge2}\hsigma_{2j}^2$. The shared terms control total mass, while $\ha_i$ receives the additional $\ge4$th-order advantage $2\hsigma_{\ge4}^2(1-\ha_i)$. Thus aligned mass grows by displacing dense non-aligned mass within the target subspace.

Stage~2.2 reaches $|\ha_i-\ha_*|\le \epsnm$ and $\hb_i\le \epsnm$ in time
$\BigO{\log(1/\epsnm)}$, where
$\ha_*:=1-\lambda/(2\hsigma_0^2 r+2\hsigma_{\ge2}^2)$ and
$\epsnm=\exp(-\polylog(r))$; see \cref{sec: stage 2.2}. This fits the target up to the
bias induced by weight decay. Stage~2.3 then sets $\lambda=0$ and reaches loss at most $\eps$; see \cref{sec: stage 2.3}.

\section{Coupling the symmetrized network and finite-width network}\label{sec: coupling main-text}

The previous section analyzed the symmetrized population dynamics. We now explain how to
transfer this analysis to the actual finite-width empirical dynamics. 
Our approach is inspired by finite-width coupling arguments for mean-field dynamics
\citep{li2020learning,mahankali2023beyond,glasgow2025propagation}, but we couple the finite network directly to the finite symmetrized network from \cref{sec: symmetry}, rather than to infinite-width limit.

Let $\vchi_i\sim N(\vzero,\mI_d/d)$ be the initialization of neuron $i$. We couple three processes.
\begin{enumerate}[leftmargin=0pt, itemindent=2em]
    \item \emph{Actual empirical dynamics.} The particle $\vv_i$ follows empirical gradient flow
    \[
        \frac{\rd \vv_i}{\rd t}
        =
        -\lambda \vv_i
        -m\nabla_{\vv_i}\hL(\mW),
        \qquad
        \vv_i(0)=\vchi_i .
    \]

    \item \emph{Symmetrized population dynamics.} For each $\rho\in\Rho$, the particle $\tldvv_{i,\rho}$ follows 
    \[
        \frac{\rd \tldvv_{i,\rho}}{\rd t}
        =
        -\lambda \tldvv_{i,\rho}
        -m|\Rho|\nabla_{\tldvv_{i,\rho}}\L(\tldmW),
        \qquad
        \tldvv_{i,\rho}(0)=\rho(\vchi_i).
    \]

    \item \emph{Intermediate dynamics.} Let
    $\ckvv_i(t):=\tldvv_{i,\id}(t)$ be the identity representative, so
    $\ckvv_i(0)=\vchi_i$ matches the actual initialization. Define
    \[
        \ckf(\vx)
        =
        \frac1m\sum_{i=1}^m
        \norm{\ckvw_i}_2^2
        \sigma(\bckvw_i^\top\vx).
    \]
\end{enumerate}

This leads to the decomposition $f-\tldf=(f-\ckf)+(\ckf-\tldf)$ that separates particle-coupling error from finite-width $\Rho$-sampling error. The term $f-\ckf$ compares the actual empirical particles with their coupled symmetrized representatives, while $\ckf-\tldf$ compares one sampled transform for each particle with the full average over all $\Rho$-transforms.

To quantify the coupling error, define 
$\delta_i:=\norm{\vv_i-\ckvv_i}_2$, $\Delta^2:=\frac1m\sum_{i=1}^m\delta_i^2$, 
$\Deltamax:=\max_i\frac{\delta_i}{\norm{\ckvv_i}_2}$.
The normalization in $\Deltamax$ accounts for the possible growth of relevant particles. The following reduces function-space error to particle coupling and finite-width concentration:
\begin{equation}\label{eq:function-coupling-main}
    \E_{\vx}\bigl[(f(\vx)-\ckf(\vx))^2\bigr]
    \lesssim B\Delta^2+\Delta^4,
    \qquad
    \E_{\vx}\bigl[(\ckf(\vx)-\tldf(\vx))^2\bigr]
    \lesssim \frac{B B_\infty\log m}{m}.
\end{equation}
Here $B$ controls the average particle size. The quantity $B_\infty$ controls the largest single-particle contribution to the $\Rho$-sampling error, i.e., how much one particle can differ from its full average over $\Rho$-transforms. The first bound is a particle-space Lipschitz estimate, and the second is a concentration estimate over the $\Rho$-transforms; see \cref{lem: tldf ckf bound and f ckf bound}.

It remains to control $\Delta$ and $\Deltamax$. We have (see \cref{lem: main text coupling error main,lem: coupling error main})
\begin{lemma}[Finite-width coupling, informal]\label{lem:coupling-main-text-informal}
Under \cref{thm: main}, for all
$t\le T_2$, we have
$\Deltamax(t)\le \epsdir^{1/6}$,
$\Delta(t)\le \epsdir^{1/5}$.
Consequently,
$\E_{\vx}\bigl[(f_t(\vx)-\tldf_t(\vx))^2\bigr]
    \le \epsnm^{1/4}/2$.
\end{lemma}
This coupling allows the actual finite-width empirical network to inherit the incremental and competitive dynamics of the symmetrized population network.

We briefly explain the proof idea. For a network output $g$ and a particle location $\vv$, define the population particle-gradient notation $\grad_{\vv}\L(g)
    :=
    \lambda\vv
    +
    \E_{\vx}\!\left[
        (g(\vx)-f_*(\vx))
        \nabla_{\vv}\bigl(\norm{\vv}_2^2\sigma(\bvv^\top\vx)\bigr)
    \right]$,
and define $\grad_{\vv}\hL(g)$ analogously using
the empirical average. Then 
$\frac{\rd}{\rd t}\vv_i=-\grad_{\vv_i}\hL(f)$, while 
$\frac{\rd}{\rd t}\ckvv_i=-\grad_{\ckvv_i}\L(\tldf)$.
For $\delta=\norm{\vv-\ckvv}_2$, we have
\begin{equation*}
    \frac{\rd}{\rd t}\delta^2
    = A_t+B_t+C_t,
\end{equation*}
where
$A_{t}
    :=
    -2\langle
        \grad_{\vv}\L(\tldf)-\grad_{\ckvv}\L(\tldf),
        \vv-\ckvv
    \rangle$, 
$B_{t}
    :=
    -2\langle
        \grad_{\vv}\L(f)-\grad_{\vv}\L(\tldf),
        \vv-\ckvv
    \rangle$, 
$C_{t}
    :=
    -2\langle
        \grad_{\vv}\hL(f)-\grad_{\vv}\L(f),
        \vv-\ckvv
    \rangle$.
The term $A_t$ is the stability term for the symmetrized population vector field. Since the dynamics is not globally contractive, $A_t$ can be positive during parts of training;
we control it with a stage-wise stability bound tied to particle-norm growth. The term $B_t$ is controlled by the function-space discrepancy in \eqref{eq:function-coupling-main}, and $C_t$ is a finite-sample error controlled by the sample size. A careful stage-wise Gronwall argument then gives \cref{lem:coupling-main-text-informal}, similar to \cite{mahankali2023beyond}.

This coupling hold for time $\omega(\log d)$, so a crude $e^T$ stability bound would require super-polynomial width; see also \citep{glasgow2025propagation}. Our stage-wise symmetrized analysis  keeps the width polynomial.

More broadly, the argument suggests a symmetry-adapted alternative to direct finite-to-infinite-width comparisons. When the target and population dynamics have a useful symmetry, one may first analyze the finite symmetrized dynamics and then couple the actual finite-width dynamics to this symmetry-closed reference system. This perspective may also be useful in combination with recent propagation-of-chaos
techniques \citep{glasgow2025propagation}.

\section{Conclusion}\label{sec: conclusion}

We studied overparameterized two-layer networks learning orthogonal multi-index targets under standard initialization. We show that incremental learning persists beyond small initialization and reveal a competitive reallocation of parameter mass.
Our analysis exploits symmetry in the target and dynamics. Instead of comparing the finite network to an infinite-width limit, we compare it directly with a finite symmetrized counterpart. This symmetry-based finite-width approximation may be useful for other feature-learning problems with exploitable symmetries.

Several questions remain open. The most immediate is whether the technical modification to gradient flow can be removed, allowing us to analyze vanilla gradient flow directly. Another direction is to go beyond the isotropic orthogonal setting, for example to non-isotropic weights or near-orthogonal target directions \citep{oko2024learning,ren2025emergence}. A further question is sample complexity: similar to prior finite-width mean-field analyses
\citep{li2020learning,mahankali2023beyond,glasgow2025propagation}, our bounds separate feature learning from fixed-kernel methods but do not match the optimal rates suggested by information-exponent analyses \citep{arous2021online,ren2024learning}.

\section*{Acknowledgements}
SSD acknowledges the support of  NSF DMS 2134106, NSF IIS 2143493, NSF IIS 2229881, Sloan Fellowship, and the AI2050 program at Schmidt Sciences.
MF, MZ, and WX acknowledge the support of NSF TRIPODS II DMS 2023166. The work of MF was also supported by awards NSF CCF 2212261 and NSF CCF 2312775 and the Moorthy Family Professorship at UW. 

\bibliographystyle{alpha}
\bibliography{ref}

\newpage
\appendix

\section{Omitted results and proofs in \cref{sec: setup,sec: main result,sec: symmetry}}
In this section, we present the omitted results and proofs in  \cref{sec: setup,sec: main result,sec: symmetry}. This includes the exact formula of loss and gradient, and the proof of the main result \cref{thm: main}.

\subsection{Omitted results and proofs in \cref{sec: setup}}
In this section, we collect few useful facts in our proof.

The first is on Hermite coefficient of activation $\sigma(x)=|x|$.
\begin{claim}\label{claim: hermite coeff}
    When $\sigma(x)=|x|$, we know $|\hsigma_k|=\Theta(k^{-5/4})$ and $\hsigma_k=0$ when $k$ is odd.
    As a corollary, $\sum_{k\ge 4} \hsigma_k^2 2k\lesssim 1$.
\end{claim}

\paragraph{Exact Formula on Loss and Dynamics.}
Using the Hermite expansion, the following claim gives us the form of the loss and gradient \cite{ge2017learning}. 
\begin{claim}\label{claim: loss and gradient}
    Recall Hermite expansion $\sigma(x)=\sum_{k=0}^\infty \hsigma_k h_k(x)$.
    The loss can be written as
    \begin{align*}
        L(\mW)
        =& \frac{1}{2}\sum_{k=0}^\infty \hsigma_k^2 \norm{
            \frac{1}{m}\sum_{i=1}^m \norm{\vw_i}_2^2 \bvw_i^{\otimes k}
            -\sum_{i=1}^r a_i^*\norm{\vw_i^*}_2\bvw_i^{*\otimes k} 
            }_F^2  
    \end{align*}
    The dynamic \eqref{eq: dynamic of w} can be written as: for any neuron $\vw_i$
    \begin{align*}
    \frac{\rd \vw_i}{\rd t}
    =& -\lambda_t\vw_i + \gamma_{t,i}\norm{\vw_i} \left(
        2\hsigma_0^2 \left(
            \sum_{j=1}^r \norm{\vw_j^*}_2
            - \frac{1}{m}\sum_{j=1}^m \norm{\vw_j}_2^2\right)\bvw_i
    \right.\\
    &\qquad\qquad \left.
    +\sum_{k\ge 2} \hsigma_k^2 \left(
        k (\tT_k^*-\tT_k)(\bvw_i^{\otimes (k-1)})
        - (k-2)(\tT_k^*-\tT_k)(\bvw_i^{\otimes k})\bvw_i
    \right)
    \right),
    \end{align*}
    where
    \begin{align*}
        \tT_k = \frac{1}{m}\sum_{i=1}^m \norm{\vw_i}_2^2 \bvw_i^{\otimes k}, \quad
        \tT_k^* = \sum_{i=1}^r \norm{\vw_i^*}_2\bvw_i^{*\otimes k}
    \end{align*}
    are $k$-th order tensors.

\end{claim}

We first define the following conditional symmetry. It says the measure is invariant under sign flip.
\begin{definition}[Conditional-symmetry]
We call a measure $\mu$ over $\mathbb{R}^d$ \textit{conditionally-symmetric} if for every $i \in [d]$ and every $v \in \mathbb{R}^d$, the following is true:
\[
    \Pr_{\vw \sim \mu}[w_i = v_i \mid w_{-i} = v_{-i}] = \Pr_{\vw \sim \mu}[w_i = -v_i \mid w_{-i} = v_{-i}].
\]
\end{definition}

Now we can give each cooridinate dynamic under the measure $\mu$ is conditional symmetry.  
\begin{claim}\label{claim: tldw dynamic}
    Suppose $\mu$ is conditional symmetry, then for any neuron $\vv$ and any $i\in[d]$ we have
    \begin{align}\label{eq: dynamic of tldw}
        \frac{\rd v_i}{\rd t}
        =& - \lambda_t v_i- \gamma_{t,\vv}\sum_{j\ge 0} [\nabla_{2j,\vv}]_i , 
    \end{align}
    where $\nabla_{2j,\vv}$ represents the gradient of $2j$-th order loss
    \begin{align*}
        -[\nabla_{0,\vv}]_i 
        =& 2\hsigma_0^2 \left(
                \sum_{j=1}^r a_j^* \norm{\vw_j^*}_2
                - \E_\mu[ \norm{\vw}_2^2 ]\right) v_i,\\
        -[\nabla_{2j,\vv}]_i 
        =& \hsigma_{2j}^2 \left(
            2j \bv_i^{2j-2} \ind_{i \le r}
            - 2j \E_\mu \left[
                \norm{\vw}_2^2 \langle\bvw,\bvv\rangle^{2j-1} \bw_i/\bv_i
                \right]
        - (2j-2) \left(
            \sum_{l\le r}\bv_l^{2j} - \E_\mu \left[\norm{\vw}_2^2\langle \bvw,\bvv\rangle^{2j}\right]
            \right)\right) v_i.
    \end{align*}

    Note that since $\mu$ is conditionally symmetric, which means all odd polynomials vanish under $\mu$, the term
    \[
        \E_\mu \left[
                \norm{\vw}_2^2 \langle\bvw,\bvv\rangle^{2j-1} \bw_i/\bv_i
                \right]
        = \E_\mu\left[\norm{\vw}_2^2
            \sum_{\beta_1+\cdots+\beta_d = j-1} \binom{2j-1}{2\beta_1,2\beta_2,\ldots,1+2\beta_i,\ldots,2\beta_d}
            \prod_{\ell\in[d]}(\bw_{\ell}\bv_{\ell})^{2\beta_\ell} \bw_i^2 \right]
    \]
    can be rewritten as a sum of even polynomials, which removes the denominator $\bv_i$, and is therefore always non-negative.
\end{claim}
\begin{remark}
    The gradient formula above shows that each coordinate $v_i$ evolves multiplicatively: its velocity is proportional to $v_i$ itself. This property is useful in our dynamical analysis, since it helps us control coordinates that are initially small and show that they remain small.

    We also point out that the gradient formula (3.3) in \cite{li2020learning} does not appear to match the stated tensor form. In particular, after expanding the expression and collecting like terms, the coefficients differ from those obtained from the tensor formulation; for the $k$-th order gradient, the discrepancy can be as large as $k^{\BigO{k}}$. 
    As an example, consider the fourth-order term with $i=1$. The expression
    $
        \E_\mu \left[
            \norm{\vw}_2^2 \langle \bvw,\bvv\rangle^{3}
            \frac{\bw_1}{\bv_1}
        \right]
    $
    should be compared with
    $
        \E_\mu\left[
            \norm{\vw}_2^2
            \sum_{j,k\in[d]}
            \prod_{\ell\in\{j,k\}}(\bw_{\ell}\bv_{\ell})^2
            \bw_1^2
        \right].
    $
    The former, which is the expression derived above, gives
    $
        \tldE\left[
            (\bw_1\bv_1)^2
            + 3\sum_{j\ne 1}(\bw_j\bv_j)^2 \bw_1^2
        \right],
    $
    whereas the latter, as written in \cite{li2020learning}, gives
    $
        \tldE\left[
            \sum_{i\in[d]}(\bw_i\bv_i)^2 \bw_1^2
        \right].
    $
    Nonetheless, their results should remain unchanged after appropriate modifications to the proof, especially regarding the dynamics.
\end{remark}

\begin{proof}
    From \cref{claim: loss and gradient} we get the gradient form of $-[\nabla_{2j,\vv}]$.

    In below, we focus on $\E_\mu [\norm{\vw}_2^2\langle\bvw,\bvv\rangle^{2j-1}\bw_i/\bv_i]$. Recall $\mu$ is conditional symmetric, so
    \begin{align*}
        \E_\mu [\norm{\vw}_2^2\langle\bvw,\bvv\rangle^{2j-1}\bw_i/\bv_i]
        =& \tldE\left[\norm{\vw}_2^2
            \sum_{\alpha_1+\cdots+\alpha_d = 2j-1} \binom{2j-1}{\alpha_1,\alpha_2,\ldots,\alpha_d}
            \prod_{\ell\in[d]}(\bw_{\ell}\bv_{\ell})^{\alpha_\ell} \bw_i^2
            \frac{1}{\bw_i\bv_i}
            \right]\\
        \myeq{a}& \E_\mu\left[\norm{\vw}_2^2
            \sum_{\beta_1+\cdots+\beta_d = j-1} \binom{2j-1}{2\beta_1,2\beta_2,\ldots,1+2\beta_i,\ldots,2\beta_d}
            \prod_{\ell\in[d]}(\bw_{\ell}\bv_{\ell})^{2\beta_\ell} \bw_i^2 \right],
    \end{align*}
    where (a) we use $\mu$ is conditional symmetric so odd polynomials are 0 and let $\alpha_i=1+2\beta_i$, $\alpha_\ell=2\beta_\ell$ for $\ell\ne i$.
\end{proof}

Moreover, for symmetric network $\tldf$ and corresponding measure $\tldmu$, since at initialization $\tldmu$ is conditional symmetry, we know from above \cref{claim: tldw dynamic} that $\tldmu$ will keep to be conditional symmetric. Therefore proving \cref{claim:symmetry-preservation-main}.

\subsection{Omitted results and proofs for \cref{sec: main result}}\label{sec: omit proof main result}

We first state the choice of parameters required in \cref{thm: main}.
\paragraph{Choice of parameters.}
    Our results hold for any $\log^{2} d < r/\log^2 r < r\log^2 r < \log^3 d$, and require width $d^{3}\le m\le \poly(d)$ and sample size $n\ge d^{3+c}$. The required exact degree on $d$ is only needed for the coupling section in \cref{sec: coupling proof}. For the dynamic analysis, it suffices to view $m=\poly(d)$. 

    For weight decay, we set $\lambda_t=\lambda$ for $t\le T_{22}$ and $\lambda_t=0$ for $t>T_{22}$, where $\lambda=\lambda_0 r$ and $\lambda_0$ is any sufficiently small constant. For the stepsize $\gamma_t$, we set $\gamma_t=1$ for all times $t$ except when $T_{21}\le t\le T_\gamma$. During this period, we set $\gamma_t=\gamma$ if $\norm{\vv}\ge \sigma_1^2$, and $\gamma_t=1$ otherwise. Here, we choose $\gamma=m$ and $\sigma_1=m^{\Theta(1)}$, with an arbitrarily small constant in the exponent.

    The change in stepsize is needed only because of a technical challenge that arises when analyzing the symmetrized network dynamics. See \cref{sec: stage 2.2} for details and further discussion.

\vspace{3pt}
\noindent\textbf{Discussion of quantitative bounds and assumptions.}
The quantitative bounds in \cref{thm: main} are not intended to be optimal, as our goal is to characterize the finite-width training trajectory rather than optimize sample complexity. For Gaussian single-index models, the information exponent $k^\star$, defined by the lowest nonzero Hermite order carrying directional information, captures the difficulty of learning from the weak correlations available at random initialization. Sharp or near-optimal sample-complexity bounds have been established in several settings \citep{arous2021online,damian2023smoothing,lee2024neural}.

For orthogonal multi-index models, however, a single information exponent does not fully capture the available structure, since different Hermite orders may reveal different aspects of the target. In the regime where the information exponent exceeds $2$, \cite{ren2025emergence} precisely characterize the emergence of individual orthogonal features and the corresponding learning timescales. In our setting, by contrast, the degree-$2$ component is already informative: it reveals the target subspace but cannot distinguish the individual orthogonal directions, which are identified only through higher-order components. \cite{ren2024learning} exploit this structure to obtain sharper sample-complexity bounds by first recovering the target subspace from the degree-$2$ component and then identifying individual directions using higher-order components. Their sharper bounds, however, rely on spherical SGD with correlation loss, followed by separate fitting of the second layer.

Our result captures a related order-by-order progression, but instead focuses on the finite-width square-loss dynamics under standard initialization, where the network output is non-negligible and low- and high-order Hermite components coexist throughout training. Accordingly, the requirements $m\ge d^3$ and $n=\Theta(d^{3+c})$ in \cref{thm: main} should be viewed as sufficient conditions for controlling this trajectory rather than as sharp learnability thresholds. Our finite-width coupling builds on the approach of \cite{mahankali2023beyond}, which similarly yields polynomial but non-optimal width and sample bounds. The current quantitative requirements arise largely from this coupling framework, and obtaining sharper bounds would likely require a corresponding refinement of the finite-width approximation technique. We leave this as an interesting direction for future work.

We now are ready to prove \cref{thm: main}.
    \begin{proof}[Proof of \cref{thm: main}]
We first analyze the symmetrized trajectory $\tldmW$. By the initialization
guarantees in \cref{lem: init} and the stage-wise dynamics in
\cref{lem: stage 1.1 main}, \cref{lem: stage 1.2 main}, and
\cref{lem: stage 2 main}, the symmetrized network $\tldf$ reaches
\[
    \frac12\norm{\tldf-f_*}_{L^2}^2
    =\L(\tldmW)
    \le \frac12 e^{-\log^2 r}
\]
within time $T_2\lesssim r/\log m$. These stage-wise lemmas also give the lower
and upper bounds on the loss before and after each transition time, which yields
the order-by-order loss profile in the table.

It remains to transfer the guarantee from the symmetrized trajectory to the
actual finite-width empirical trajectory. By \cref{lem: main text coupling error main}, we have for all $t\le T_2$ 
\[
    \norm{f-\tldf}_{L^2}\le \epsnm^{1/4}/2
    \le \frac{1}{2}e^{-\log^2 r}.
\]
Hence, at time $T_2$,
\[
    \sqrt{2\L(\mW(T_2))}
    =\norm{f_{\mW(T_2)}-f_*}_{L^2}
    \le
    \norm{\tldf_{T_2}-f_*}_{L^2}
    +\norm{f_{\mW(T_2)}-\tldf_{T_2}}_{L^2}
    \le \sqrt{2}\,e^{-\frac12\log^2 r}.
\]
Thus $\L(\mW(T_2))\le e^{-\log^2 r}\le \eps$, proving the first claim.

For the absolute-value activation, after $T_2$
\cref{lem: local convergence} gives the additional local convergence phase down
to arbitrary target accuracy $\eps$, with the stated sample complexity and time.
\end{proof}

\subsection{Omitted results and proofs in \cref{sec: symmetry}}

The benefits of symmetrized network can be seen in the lemma below. We are able to decompose the loss into 2 separate terms. Intuitively, the first loss is as if we are training on symmetrized network $\tldf$ (or even in the infinite-width limit), and second loss tries to move model $f$ towards its symmetrized version $\tldf$. 
\begin{lemma}\label{lem: decomp of loss}
    We have
    \begin{align*}
        L(\vtheta)
        = \frac{1}{2}\E_\vx [(f_\vtheta(\vx) - f_*(\vx))^2]
        = \frac{1}{2}\E_\vx [(\tldf_\vtheta(\vx) - f_*(\vx))^2]
            + \frac{1}{2}\E_\vx [(f_\vtheta(\vx) - \tldf_\vtheta(\vx))^2].
    \end{align*}
\end{lemma}
\begin{proof}
    We drop the subscript $\vtheta$ for simplicity. To prove the loss decomposition, we only need to show the cross term is 0. Note that $\rho(\vx)\sim N(\vzero,\mI)$ when $\vx\sim N(\vzero,\mI)$ and $\rho\in\Rho$, then we have
    \begin{align*}
    \E_\vx [(\tldf(\vx) - f_*(\vx))(f(\vx) - \tldf(\vx))]
    =& \frac{1}{|\Rho|} \sum_{\rho\in\Rho}
        \E_\vx [(\tldf(\rho(\vx)) - f_*(\rho(\vx)))(f(\rho(\vx)) - \tldf(\rho(\vx)))]\\
    =& \frac{1}{|\Rho|} \sum_{\rho\in\Rho}
        \E_\vx [(\tldf(\vx) - f_*(\vx))(f(\rho(\vx)) - \tldf(\vx)))]
    = 0,
    \end{align*}
    where we use the fact that $\tldf,f_*$ are invariant under $\Rho$ and definition of $\tldf$.
    
\end{proof}

Another interpretation of the above lemma is as a bias--variance decomposition in function space. Here, the symmetrized network $\tldf$ is the average of the model $f$ over the symmetry group $\Rho$. Thus, the first loss term can be viewed as the bias term, while the second loss term can be viewed as the variance term.

It is therefore natural to expect the variance term
$\E_\vx\left[(f_\vtheta(\vx)-\tldf_\vtheta(\vx))^2\right]$
to be small at random initialization when the width $m$ is large. This also suggests, at an intuitive level, that $f$ should remain close to $\tldf$ throughout training.

However, formally justifying this intuition is nontrivial. In fact, the variance term may increase during training because the scale of $f$ itself increases. We develop the bounds in \cref{sec: coupling main-text} to control this term.

\section{Initialization}
The following properties happen at initialization w.h.p. when we have enough overparametrization. 

Throughout the rest of paper, we always condition on the these high probability cases happen at initialization. Especially, for simplicity we set failure probability $\delta=1/r$ so that we could ignore the $\log(1/\delta)$ term in these bounds.
\begin{lemma}\label{lem: init}
    For $e^{\log^2 d}\ge m\ge d$, $1/d\le \delta <1$, $d>e^{\sqrt{\log d}}>r>\log^2 m$, then with probability at least $1-\delta$ (unless state specifically below)
    \begin{enumerate}
        \item Regularity condition for individual neuron: for any neuron $\vv$
        \begin{align*}
            &v_i^2 \lesssim \log(md/\delta)\sigma^2,\ 
            \bar{v}_i^2 \lesssim \frac{\log(md/\delta)}{d},\ 
            \text{ for all $i\in[d]$,} \\
            &\frac{v_i^2}{\norm{\vvler}_2^2}\lesssim \frac{\log(mr/\delta)}{r}, \text{ for all $i\in[r]$},\quad
            \frac{v_i^2}{\norm{\vvgtr}_2^2}\lesssim \frac{\log(md/\delta)}{d}, \text{ for all $i\in[d]\setminus[r]$},\\
            &\norm{\vv}_2 = \sigma\sqrt{d} \pm \BigO{\sigma\sqrt{\log(m/\delta)}},\ 
            \norm{\vvler}_2 = \sigma\sqrt{r} \pm \BigO{\sigma\sqrt{\log(m/\delta)}},\\ 
            &\norm{\bvvler}_2 = \sqrt{\frac{r}{d}} \left(1\pm \BigO{\sqrt{\frac{\log(m/\delta)}{r}}}\right),
            \sum_{i\in[r]}\bv_i^4 \lesssim \frac{r}{d^2},\ 
            \sum_{i\in[r]}\frac{v_i^4}{\norm{\vvler}_2^4}
            \lesssim \frac{1}{r}\\
        \end{align*}
        As a result, for any neuron $\vv,\vu$
        \begin{align*}
            \frac{\norm{\vv}_2}{\norm{\vu}_2} = 1 \pm \BigO{\sqrt{\frac{\log(m/\delta)}{d}}},\quad
            \frac{\norm{\vvler}_2}{\norm{\vuler}_2} = 1 \pm \BigO{\sqrt{\frac{\log(m/\delta)}{r}}},\quad
            \frac{\norm{\vvgtr}_2}{\norm{\vugtr}_2} = 1 \pm \BigO{\sqrt{\frac{\log(m/\delta)}{d}}}
        \end{align*}

        \item Regularity condition for distribution of neuron: 
        \begin{align*}
            \tldE[\bw_i^2] \lesssim \frac{1}{d},\quad
            \tldE[\bw_i^4] \lesssim \frac{1}{d^2},\quad
            \tldE \left[(\bbvwler)_i^2\right]\lesssim\frac{1}{r},\quad
            \tldE \left[(\bbvwler)_i^4\right]\lesssim\frac{1}{r^2}.
        \end{align*}

        \item \label{item: init good pot bad neuron} 
        `potential', `bad' neuron: for any $i\in[r]$, (the set we consider here has some slackness on the threshold we actual use in Stage 2)
        \begin{enumerate}[label=(\roman*)]
            \item scale of best `potential' neuron: with probability $1-\frac{1}{r}$, for all $i\in[r]$ 
            \[
                \max_{\vv}(\bbvvler)_i^2 = \frac{2\log(2m)}{r}\left(1\pm\BigO{\frac{\log r}{\log m}}\right)
            \]
            \item fraction of `potential' neuron is small:  with probability $1-\frac{1}{\sqrt{m}}$
            \[
                \left|\cup_{i\in[r]}\left\{\vv:
                (\bbvvler)_i^2 \ge (1-2\omega^{-0.01})\max_\vu (\bbvuler)_i^2
                \right\}\right|\le m^{5\omega^{-0.01}}.
            \]
            \item gap between coordinates in `potential neurons': for any $\vv$, if ($\vv$ is a potential neuron)
            \[
                (\bbvvler)_i^2 \ge (1-2\omega^{-0.01})\max_\vu (\bbvuler)_i^2
            \]
            then with probability $1-\frac{1}{m}$,
            \[
            (\bbvvler)_i^2 \ge 0.98\max_{j\in[r]\setminus\{i\}} (\bbvvler)_j^2\]
            \item no `bad' neuron: with probability $1-m^{-0.95}$, for any $i\ne j\in[r]$, there is no neuron $\vv$ that
            \[
                (\bbvvler)_i^2 \ge (1-2\omega^{-0.01})\max_\vu (\bbvuler)_i^2,\quad
                (\bbvvler)_j^2 \ge (1-2\omega^{-0.01})\max_\vu (\bbvuler)_j^2
            \]
            (i.e., $\Spoti\cap\Spotj=\emptyset$). 
        \end{enumerate}
    \end{enumerate}
\end{lemma}

\begin{proof}
    We show this one by one.
    \paragraph{item (a)}
    From standard Gaussian bound and union bound over all $d$ dimension, we know $v_i^2 \lesssim \log(md/\delta)\sigma^2$ w.p. $1-\delta$ for all neuron $v$ and $i\in[d]$. Also from standard concentration for Gaussian we know the bound on $\norm{\vv}_2$, $\norm{\vvler}_2$, $\norm{\vvgtr}_2$ (e.g., \cite[Thm. 3.1.1]{vershynin2018high}).

    It is then easy to see above bounds imply the bound on $\bv_i^2$, $\frac{v_i^2}{\norm{\vvler}_2^2}$, $\frac{v_i^2}{\norm{\vvgtr}_2^2}$, and $\norm{\bvvler}_2$.
    
    For the bound on sum of power 4 term like $\sum_{i\in[r]}\bv_i^4$ and $\sum_{i\in[r]}\frac{v_i^4}{\norm{\vvler}_2^4}$, first using concentration on Lipschitz function for gaussian vectors (Theorem 5.2.3 in \cite{vershynin2018high}) on $l_4$ norm $\norm{\cdot}_4$, we have $(\sum_{i\in[r]}v_i^4)^{1/4} \lesssim \sigma r^{1/4}$ w.p. $1-e^{-\BigOmega{\sqrt{r}}}\ge 1-1/m^2$ (using $r>\log^2 m$). Then combining this with the norm bound on $\norm{\vvler}_2$ and $\norm{\vv}_2$, we get the desired bounds by union bound.
    
    \paragraph{item (b)}
    This follows from standard concentration of Gaussian.

    \paragraph{item (c)} We consider any fixed $i\in[r]$ below, then apply union bound.
    
    (i) The scale of best `potential' neuron follows by \cref{lem: max spherical coordinate concentration} (set $\delta=1/r^2$ and take union bound over $i\in[r]$) and bound on $\norm{\vv}_2$.

    (ii) Recall we are interested in neurons that 
    \[
        (\bbvvler)_i^2 \ge (1-2\omega^{-0.01})\max_\vu (\bbvuler)_i^2.
    \]
    Using \cref{lem: max spherical coordinate concentration} with $n=m$, we know RHS
    \[
        (1-2\omega^{-0.01})\max_\vu (\bbvuler)_i^2
        \ge
        (1-2\omega^{-0.01})
        \frac{
            2 \log(2m)
            - 2 \log(C\sqrt{\log m})
            - 2\log\log(1/\delta)
        }{r}.
    \]
    Now consider any fixed neuron $\vv$. It suffices to bound the probability of 
    \[
        (\bbvvler)_i^2
        \ge
        (1-2\omega^{-0.01})
        \frac{
            2 \log(2m)
            - 2 \log(C\sqrt{\log m})
            - 2\log\log(1/\delta)
        }{r}.
    \]
    Using the spherical-coordinate upper tail bound, this probability is at most
    \[
    \begin{aligned}
        &2\exp\left(
            -\frac{r}{2}
            (1-2\omega^{-0.01})
            \frac{
                2 \log(2m)
                - 2 \log(C\sqrt{\log m})
                - 2\log\log(1/\delta)
            }{r}
        \right) 
        \le
        2m^{-1+4\omega^{-0.01}},
    \end{aligned}
    \]
    where we use $\poly(\omega)<\log d$ to absorb the lower-order logarithmic
    factors into the $\omega^{-0.01}$ slack.
    This suggests there are at most $r\cdot 2m^{4\omega^{-0.01}}\le m^{5\omega^{-0.01}}=m^{o(1)}$ potential neurons by Chernoff bound.
    
    (iii) From subitem(i)(ii), we know that if $\vv$ satisfies the condition
    (i.e., is a potential neuron), then
    \[
        (\bbvvler)_i^2
        \ge
        (1-2\omega^{-0.01})
        \frac{
            2 \log(2m)
            - 2 \log(C\sqrt{\log m})
            - 2\log\log(1/\delta)
        }{r}.
    \]
    Thus, to control the probability of
    \[
        (\bbvvler)_i^2 < (1-c')\max_{j\in[r]\setminus\{i\}}(\bbvvler)_j^2,
    \]
    it suffices to control the probability of, for fixed
    $j\in[r]\setminus\{i\}$ and fixed $\vv$,
    \[
        (1-c')(\bbvvler)_j^2, (\bbvvler)_i^2
        \ge
        (1-2\omega^{-0.01})
        \frac{
            2 \log(2m)
            - 2 \log(C\sqrt{\log m})
            - 2\log\log(1/\delta)
        }{r},
    \]
    For $\vz\sim\mathrm{Unif}(\mathbb S^{r-1})$ and $i\ne j$,
    we know $\vz_i^2+\vz_j^2\sim \mathrm{Beta}\left(1,\frac{r-2}{2}\right)$,
    and therefore
    \[
        \P\left(
            \vz_i^2\ge s,\ \vz_j^2\ge t
        \right)
        \le
        \P\left(
            \vz_i^2+\vz_j^2\ge s+t
        \right)
        \le
        \exp\left(-\frac{r-2}{2}(s+t)\right).
    \]
    Applying this with the above choices of $s,t$, the probability for fixed
    $\vv,i,j$ is at most
    \[
        m^{
            -(1-O(\omega^{-0.01}))
            \left(1+\frac{1}{1-c'}\right)
        }.
    \]
    Taking a union bound over all $m$ neurons and all $i\ne j\in[r]$, the failure
    probability is at most
    $mr^2\cdot
        m^{
            -(1-O(\omega^{-0.01}))
            \left(1+\frac{1}{1-c'}\right)
        }$.
    Setting $c'=0.02$ and using $r=m^{o(1)}$, this is at most $1/m$ for
    sufficiently large $m$. 
    
    (iv) Recall neuron of interest, i.e., neuron in both $\Spoti$ and $\Spotj$:
    \[
        (\bbvvler)_i^2 \ge (1-2\omega^{-0.01})\max_\vu(\bbvuler)_i^2,\quad
        (\bbvvler)_j^2 \ge (1-2\omega^{-0.01})\max_\vu(\bbvuler)_j^2.
    \]
    Now consider any fixed neuron $\vv$ and fixed distinct $i,j\in[r]$.
    Similar to above in subitem(ii), it suffices to bound the probability of
    \[
        (\bbvvler)_i^2,(\bbvvler)_j^2
        \ge
        (1-2\omega^{-0.01})
        \frac{
            2 \log(2m)
            - 2 \log(C\sqrt{\log m})
            - 2\log\log(1/\delta)
        }{r}.
    \]
    Using the same joint spherical-coordinate tail bound, we know the probability is at most
    \[
    \begin{aligned}
        m^{-2+O(\omega^{-0.01})}.
    \end{aligned}
    \]
    Hence, taking a union bound over all $m$ neurons and all distinct
    $i,j\in[r]$, the failure probability is at most
    $mr^2\cdot m^{-2+O(\omega^{-0.01})}
        =
        r^2 m^{-1+O(\omega^{-0.01})}
        \le
        m^{-0.95}$.
\end{proof}

\subsection{Concentration}

\begin{lemma}\label{lem: max spherical coordinate concentration}
    Let $\vz_1,\ldots,\vz_n \overset{i.i.d}{\sim} \mathrm{Unif}(\mathbb S^{r-1})$.
    Suppose $r\gtrsim \log^2 n$. Then for any fixed $i\in[r]$ and
    $0<\delta\le e^{-1}$, with probability at least $1-\delta$,
    \begin{align*}
        \frac{
        2 \log(2n)
        - 2 \log(C\sqrt{\log n})
        - 2\log\log(1/\delta)
        }{r}
        \le&
        \max_{a\in[n]} (\vz_a)_i^2 
        \le
        \frac{
        2\log(2n)+2\log(2/\delta)
        }{r},
    \end{align*}
    where $C>0$ is a sufficiently large absolute constant.
\end{lemma}
\begin{proof}
    Fix $i\in[r]$. For $\vz\sim \mathrm{Unif}(\mathbb S^{r-1})$, the random
    variable $(\vz)_i^2$ has density
    \[
        f(x)
        =
        \frac{1}{B(\frac12,\frac{r-1}{2})}
        x^{-1/2}(1-x)^{(r-3)/2},
        \qquad x\in(0,1).
    \]
    Using
    $\frac{1}{B(\frac12,\frac{r-1}{2})}\asymp \sqrt r$,
    we have the standard spherical-coordinate tail estimates: for
    $1/r\lesssim t\lesssim 1$,
    \[
        \P\left( (\vz)_i^2 \ge t \right)
        \le
        2\exp\left(-\frac{rt}{2}\right),
    \]
    and, when additionally $rt^2\lesssim 1$,
    \[
        \P\left( (\vz)_i^2 \ge t \right)
        \ge
        \frac{c}{\sqrt{rt}}
        \exp\left(-\frac{rt}{2}\right),
    \]
    where $c>0$ is an absolute constant. In our application below,
    $t\asymp \frac{\log n}{r}$, and since $r\gtrsim \log^2 n$, indeed
    $rt^2\lesssim 1$.

    For the upper bound, by union bound,
    \[
        \P\left(
            \max_{a\in[n]}(\vz_a)_i^2\ge t
        \right)
        \le
        2n\exp\left(-\frac{rt}{2}\right).
    \]
    Taking
    $t=
        \frac{2\log(2n)+2\log(2/\delta)}{r}$,
    the above probability is at most $\delta$. 
    
    For the lower bound, we have
    \[
    \begin{aligned}
        \P\left(
            \max_{a\in[n]}(\vz_a)_i^2\le t
        \right)
        &=
        \left(
            1-\P\left((\vz)_i^2\ge t\right)
        \right)^n 
        \le
        \exp\left(
            -n\P\left((\vz)_i^2\ge t\right)
        \right).
    \end{aligned}
    \]
    Taking
    $t=
        \frac{
        2 \log(2n)
        - 2 \log(C\sqrt{\log n})
        - 2\log\log(1/\delta)
        }{r}$,
    and using the lower spherical-coordinate tail bound,
    \[
    \begin{aligned}
        n\P\left((\vz)_i^2\ge t\right)
        &\ge
        n\cdot
        \frac{c}{\sqrt{rt}}
        \exp\left(-\frac{rt}{2}\right) 
        \ge
        n\cdot
        \frac{c}{\sqrt{rt}}
        \cdot
        \frac{C\sqrt{\log n}\log(1/\delta)}{2n}
        \ge
        \log(1/\delta).
    \end{aligned}
    \]
    where at the end we use $rt\asymp \log n$ and choosing $C$ sufficiently large.
    Therefore, the failed probability is at most $\delta$.
\end{proof}

\section{Stage 1: learn mean and subspace (0-th and 2nd order term)}\label{sec: stage 1}

As described in \cref{sec: dynamics main-text}, in Stage 1 we will recover 0-th order and 2-st order term and recover subspace in the sense that $\frac{\norm{\tldvvler}_2}{\norm{\tldvv}_2} = 1-o(1)$ for every neuron $\tldvv$.

Recall the definition of 0th-order term
\begin{align*}
    \alpha_* &:= \tT_0^*
             = \sum_{j=1}^r\norm{\vw_j^*}_2
             = r,
    \quad
    \tldalpha := \tldtT_0
             = \tldE[\norm{\vw}_2^2]
             = \frac{1}{m|\Rho|} \sum_{i=1}^m \norm{\tldvw_{i,\rho}}_2^2
             = \frac{1}{m}\sum_{i=1}^m\norm{\tldvw_i}_2^2.
\end{align*}

\paragraph{Outline for this section.}
Stage 1 is further decomposed into two substages:
\begin{enumerate}
    \item \textbf{Stage 1.1: learning the 0th-order component.}
    In this stage, the neuron directions remain essentially unchanged, while their norms grow to fit the mean, i.e., the 0th-order component.

    \item \textbf{Stage 1.2: learning the 2nd-order component.}
    In this stage, the neuron directions within the target subspace and its orthogonal complement remain essentially unchanged. Meanwhile, the norm within the target subspace grows, while the norm in the irrelevant subspace decreases.
\end{enumerate}

\subsection{Stage 1.1: learn 0-th order}\label{sec: stage 1.1}
Denote the end time for Stage 1.1 as the time that learns the 0th-order term
\begin{align*}
    T_{11}=\inf\left\{t\ge 0: 
    \alpha_*-\frac{\lambda}{2\hsigma_0^2} -\tldalpha\le \err_{11}:=\frac{C_{11}}{r^{10}}
    \right\},
\end{align*}
where $C_{11}$ is a large enough constant.

We have the main result for Stage 1.1 summarized in the lemma below. All proofs are deferred to \cref{sec: stage 1.1 proofs}.
\begin{lemma}[Stage 1.1]\label{lem: stage 1.1 main}
    Stage 1.1 ends in time $T_{11}=\BigO{\frac{1}{r}\log r}$. At the end of Stage~1.1, we have for any neuron $\tldvv,\tldvu$ and $i,j\in[d]$
    \begin{enumerate}
        \item Direction of neuron remains the same: $\frac{v_i^{\btt{T_{11}} 2}}{v_j^{\btt{T_{11}} 2}} = \frac{v_i^{\btt{0} 2}}{v_j^{\btt{0} 2}} (1\pm\delta_{dir,11})$, where $\delta_{dir,11}\lesssim\frac{\log d}{r}$.
        Moreover, for $i,j\le r$ we have 
        $\frac{v_i^{\btt{T_{11}} 2}}{v_j^{\btt{T_{11}} 2}} = \frac{v_i^{\btt{0} 2}}{v_j^{\btt{0} 2}} (1\pm\delta_{dir,\le r,11})$, where $\delta_{dir,\le r,11}\lesssim\frac{\log d}{d}$.

        \item Norm of neuron remains relatively:  $\frac{\norm{\tldvv^\btt{T_{11}}}_2}{\norm{\tldvu^\btt{T_{11}}}_2} = \frac{\norm{\tldvv^\btt{0}}_2}{\norm{\tldvu^\btt{0}}_2}
                (1\pm \delta_{norm,11})$,
        $\frac{\norm{\tldvvler^\btt{T_{11}}}_2}{\norm{\tldvuler^\btt{T_{11}}}_2} = \frac{\norm{\tldvvler^\btt{0}}_2}{\norm{\tldvuler^\btt{0}}_2}
                (1\pm \delta_{norm,11})$,
        $\frac{\norm{\tldvvgtr^\btt{T_{11}}}_2}{\norm{\tldvugtr^\btt{T_{11}}}_2} = \frac{\norm{\tldvvgtr^\btt{0}}_2}{\norm{\tldvugtr^\btt{0}}_2}
                (1\pm \delta_{norm,11})$, where $\delta_{norm,11}\lesssim \frac{\log d}{d}$.

        \item population 2nd order ratio: for $i,j\in[r]$
        $\frac{\tldE w_i^{\btt{T_{11}}2}}{\tldE w_j^{\btt{T_{11}}2}} = 1 \pm \frac{1}{r^2}$.

        \item Loss $\L(\tldmW)\ge(1-c)\hsigma_0^2r^2/2+(1-c)\hsigma_{\ge2}^2 r/2$ for $t\le \frac{1}{\log r}T_{11}$, and $\L(\tldmW)\le c\hsigma_0^2r^2/2+(1\pm c)\hsigma_{\ge2}^2 r/2$ for $(1-o(1))T_{11}\le t\le T_{11}$.
    \end{enumerate}
\end{lemma}

To prove the above, we introduce the following induction hypothesis that hold for all Stage 1.1. This will later help us to simplify the analysis.
\begin{restatable}[Induction Hypothesis for Stage 1.1]{lemma}{lemstageaaIH}\label{lem: stage 1.1 IH}
    Denote
    \begin{enumerate}
        \item \label{item: stage 1.1 IH dir}
        Direction of neuron remains the same respectively: For any time $t\le T_{dir,11}$ we have
        for any neuron $\tldvv$ and $i\in[d]$
        \begin{align*}
            \btldv_i^2= \btldv_i^{\btt{0}2} (1 \pm \delta_{dir,11}),
        \end{align*}
        where $\delta_{dir,11}\lesssim\frac{\log d}{r} \le 0.01$.

        \item \label{item: stage 1.1 IH norm}
        Norm of neuron remains the same: For any time $t\le T_{norm,11}$ we have 
        for any neurons $\tldvv,\tldvu$
        \begin{align*}
                \frac{\norm{\tldvv^\btt{t}}_2}{\norm{\tldvu^\btt{t}}_2} = \frac{\norm{\tldvv^\btt{0}}_2}{\norm{\tldvu^\btt{0}}_2}
                (1\pm \delta_{norm,11}),
        \end{align*}
        where $\delta_{norm,11}\lesssim\frac{\log d}{d}\le 0.01$.
    
        \item \label{item stage 1.1 IH reg}
        Regularity condition: For any time $t\le T_{reg,11}$ we have
        for any neurons $\tldvv$ and $i\in[d]$
        \begin{align*}
                |\bv_i| \ge \iota:=e^{-r\log^{1/2} r},
        \end{align*}
    \end{enumerate}        
    Then, $T_{11}\le \min\{T_{dir,11}, T_{norm,11}, T_{reg,11}\}$. That is, current stage ends before any of the above condition breaks.
\end{restatable}

Denote $\tldT_{11}=\min\{T_{11}, T_{dir,11}, T_{norm,11}, T_{reg,11}\}$ as the time that all conditions maintained in Lemma~\ref{lem: stage 1.1 IH} hold. This lemma is effectively the same as induction hypothesis in the discrete time (gradient descent) case that we need to use induction on time step to show such result (that's why we also call this induction hypothesis).

\paragraph{Implication of Induction Hypothesis.}
Given the above induction hypothesis, we first give few useful facts in this stage that will be later used.

The lemma below collects few useful facts that transfer from random initialization given the direction of each neuron does not move much under Induction Hyphthesis (Lemma~\ref{lem: stage 1.1 IH}).
\begin{lemma}\label{lem: stage 1.1 IH implication}
    For any $t\le\tldT_{11}$, the following holds under the initialization Lemma~\ref{lem: init}: for any neuron $\tldvv$ and $i\in[d]$
    \begin{enumerate}
        \item $\norm{\tldvv}_2 \le 1.01\sqrt{r}$.
        \item $\btldv_i^2 \le 1.01\btldv_i^{\btt{0}2}\lesssim \frac{\log d}{d}$, $\tldv_i^2\lesssim \frac{r\log d}{d}$
        \item $\sum_{i\le r}\btldv_i^2\lesssim \frac{r}{d}$, $\sum_{i\le r}\btldv_i^4 \lesssim \frac{r}{d^2}$.
        \item $\tldE[w_i^2] \lesssim \frac{r}{d}$, $\tldE[\bw_i^2] \lesssim \frac{1}{d}$.
    \end{enumerate}
\end{lemma}

Given the induction hypothesis and its implications above, we now are able to simplify the dynamics as below. This shows in this stage, the dynamic is dominated by the 0th-order term.
\begin{lemma}\label{lem: stage 1.1 dynamic}
    In Stage 1.1, for $t\le \tldT_{11}$, the dynamic \eqref{eq: dynamic of tldw} can be simplified as: for any neuron $\tldvv$ and $i\in[d]$
    \begin{align*}
        \frac{\rd \tldv_i}{\rd t}
        =&  \tldv_i \left(
            -\lambda
            + 2\hsigma_0^2\left(r-\tldE\norm{\vw}_2^2\right)
            + 2\hsigma_2^2 \ind_{i\le r} 
            \pm \BigO{\frac{r}{d}}
        \right),\\
        \frac{\rd}{\rd t}\norm{\tldvv}_2
        =& \norm{\tldvv}_2 \left(
            -\lambda
            +2\hsigma_0^2\left(r-\tldE\norm{\vw}_2^2\right)
            \pm \BigO{\frac{r}{d}}
        \right),\\
        \frac{\rd \btldv_i}{\rd t}
        =& \btldv_i \cdot 2\hsigma_2^2 \left(\ind_{i\le r}
            \pm \BigO{\frac{r}{d}}
            \right).
    \end{align*}
\end{lemma}

Give the above simplified dynamics, we now ready to show the convergence of $\tldalpha$ in Stage 1.1.
\begin{lemma}[Dynamic of $\tldalpha$]\label{lem: stage 1.1 dynamic of tldalpha}
    For $t\le \tldT_{11}$, we have
    \begin{align*}
        \frac{\rd \tldalpha}{\rd t}
        = 4\hsigma_0^2\tldalpha \left(r-\frac{\lambda}{2\hsigma_0^2}-\tldalpha \pm \BigO{\frac{r}{d}}\right).
    \end{align*}
    As a corollary, within time $\hT_{11}=\BigO{\frac{1}{r}\log r}$ we have $\alpha_* - \frac{\lambda}{2\hsigma_0^2} - \tldalpha \lesssim \frac{1}{r^{10}}$. Moreover before $\frac{1}{\log r}\hT_{11}\lesssim \frac{1}{r}$, we know $\alpha=o(1)\alpha_*$.
\end{lemma}

\subsubsection{Omitted proofs for Stage 1.1}\label{sec: stage 1.1 proofs}
We present the omitted proofs in this section.

The first is the proof for the main result for this stage. It follows from the induction hypothesis and the dynamic results.
\begin{proof}[Proof of \cref{lem: stage 1.1 main}]
    Item (a)(b) is directly implied by the (proof of) induction hypothesis \cref{lem: stage 1.1 IH} and \cref{lem: stage 1.1 dynamic of tldalpha}. The proof on 
    $\frac{v_i^{\btt{T_{11}} 2}}{v_j^{\btt{T_{11}} 2}}$ for $i,j\le r$ follow the same proof of for any $i,j$, and 
    $\frac{\norm{\tldvvler^\btt{T_{11}}}_2}{\norm{\tldvuler^\btt{T_{11}}}_2}$ and $\frac{\norm{\tldvvgtr^\btt{T_{11}}}_2}{\norm{\tldvugtr^\btt{T_{11}}}_2}$ follow the same proof of $\frac{\norm{\tldvv^\btt{T_{11}}}_2}{\norm{\tldvu^\btt{T_{11}}}_2}$ in \cref{lem: stage 1.1 IH}. We omit for simplicity.

    For item (c), first by \cref{lem: stage 1.1 dynamic} we have for $i\in[r]$
    \begin{align*}
        \frac{\rd}{\rd t} \tldE w_i^2
        =&  \tldE w_i^2 \left(
            -\lambda
            + 2\hsigma_0^2\left(r-\tldE\norm{\vw}_2^2\right)
            + 2\hsigma_2^2 
            \pm \BigO{\frac{r}{d}}
        \right).
    \end{align*}
    Hence, for $i,j\in[r]$
    \begin{align*}
        \frac{\rd}{\rd t} \frac{\tldE w_i^2}{\tldE w_j^2}
        =&  \pm \BigO{\frac{r}{d}} \frac{\tldE w_i^2}{\tldE w_j^2}.
    \end{align*}
    This implies that 
    \begin{align*}
        \frac{\tldE w_i^{\btt{T_{11}}2}}{\tldE w_j^{\btt{T_{11}}2}}
        = \frac{\tldE w_i^{\btt{0}2}}{\tldE w_j^{\btt{0}2}} 
            \left(1+T_{11}\BigO{\frac{r}{d}} \right)
        = 1\pm\frac{1}{r^2},
    \end{align*}
    where last one use \cref{lem: init}.

    Item (d) follows from \cref{lem: stage 1.1 dynamic of tldalpha}.
\end{proof}

Give the dynamics, we now ready to prove the induction hypothesis. The proof mostly uses the dynamics under IH and show that it remains true.
\begin{proof}[Proof of \cref{lem: stage 1.1 IH}]
    For simplicity in the proof, we drop $\sim$ and write $\tldv$ as $v$.

    Recall from Lemma~\ref{lem: stage 1.1 dynamic of tldalpha} that under induction hypothesis within time $\hT_{11}=\BigO{\frac{1}{r}\log d}$ we have $\alpha_* - \tldalpha \lesssim \frac{r}{d}$.
    
    In the following we argue $\tldT_{11}=\min\{T_{dir,11},T_{norm,11},T_{11}\}= T_{11}$, which would imply the desired result $T_{11}\le\min\{T_{dir,11},T_{norm,11}\}$ as well as $T_{11} \le \hT_{11}=\BigO{\frac{1}{r}\log d}$.

    To prove the above claim, assume towards contradiction that either $\tldT_{11}=T_{dir,11}$ or $\tldT_{11} = T_{norm,11}$.
    
    \paragraph{If $\tldT_{11}=T_{dir,11}$} Given the dynamic from Lemma~\ref{lem: stage 1.1 dynamic} we have for any neuron $\vv$, $i,j\in [d]$ and $t\le \tldT_{11}$
    \begin{align*}
        \frac{\rd }{\rd t} \frac{\bv_i^2}{\bv_j^2}
        = \frac{\bv_i^2}{\bv_j^2} \left(\frac{2}{\bv_i}\frac{\rd \bv_i}{\rd t} - \frac{2}{\bv_j} \frac{\rd \bv_j}{\rd t}\right)
        =& \frac{\bv_i^2}{\bv_j^2} \cdot 4\hsigma_2^2 \left(\ind_{i\le r} - \ind_{j\le r}
            \pm \BigO{\frac{r}{d}}
            \right).
    \end{align*}
    This implies 
    \begin{align*}
        \frac{\bv_i^2}{\bv_j^2} = \left(1+\BigO{\frac{\log d}{r}}\right) \frac{\bv_i^{\btt{0}2}}{\bv_j^{\btt{0}2}},
    \end{align*}
    which implies $\bv_i^2 = \frac{v_i^2}{\norm{\vv}_2^2} = \frac{v_i^{\btt{0}2}}{\norm{\vv^\btt{0}}_2^2} (1\pm \BigO{\frac{\log d}{r}}) = \bv_i^{\btt{0}2} (1\pm \BigO{\frac{\log d}{r}})$.
    This contradicts the assumption $\tldT_{11}=T_{dir,11}$ as we can set $\delta_{dir,11}= \BigTheta{\frac{\log d}{r}}$ with large enough hidden constant.

    \paragraph{If $\tldT_{11}=T_{norm,11}$} We have from Lemma~\ref{lem: stage 1.1 dynamic} that for $t\le \tldT_{11}$ and any neuron $\vv,\vu$
    \begin{align*}
        \left|\frac{\rd}{\rd t} \frac{\norm{\vv}_2}{\norm{\vu}_2}\right|
        = \frac{\norm{\vv}_2}{\norm{\vu}_2} \left|\frac{1}{\norm{\vv}}\frac{\rd \norm{\vv}}{\rd t} - \frac{1}{\norm{\vu}} \frac{\rd \norm{\vu}}{\rd t}\right|
        \lesssim& \frac{\norm{\vv}_2}{\norm{\vu}_2} \frac{r}{d}.
    \end{align*}
    Similar as in $T_{dir,11}$ case, we have
    \begin{align*}
        \frac{\norm{\vv}_2}{\norm{\vu}_2}
        = \left(1\pm \BigO{\frac{\log d}{d}}\right) \frac{\norm{\vv^{\btt{0}}}_2}{\norm{\vu^{\btt{0}}}_2},
    \end{align*}
    which contradicts with the assumption that $\tldT_{11}=T_{norm,11}$ as we can set $\delta_{norm,11} = \BigTheta{\frac{\log d}{d}}$ with large enough hidden constant.

    \paragraph{If $\tldT_{11}=T_{reg,11}$}
    From above $\tldT_{11}=T_{reg,11}$ case we know that $\bv_i^2 \ge \bv_i^{\btt{0}2}/2 >\iota^2$, which leads to the contradiction.
    
    In summary, we finish the proof that $\tldT_{11}=\min\{T_{dir,11},T_{norm,11},T_{reg,11},T_{11}\}= T_{11}$ and $T_{11} \lesssim \frac{1}{r}\log d$.
\end{proof}

This lemma shows the implication of IH, so mostly direction follow from the IH.
\begin{proof}[Proof of \cref{lem: stage 1.1 IH implication}]
    We show one by one. We drop $\sim$ for simplicity and write $\tldv$ as $v$.
    \paragraph{item (a)}
    By \cref{lem: stage 1.1 IH}\ref{item: stage 1.1 IH norm} and \cref{lem: init} we know for any neuron $\vv,\vu$ we have $\frac{\norm{\vv}_2}{\norm{\vu}_2} \le 1.01$. Thus, for $t\le \tldT_{11}$ and any neuron $\vv$, we have $\norm{\vv}_2^2 \le 1.01^2\tldE\norm{\vw}_2^2 = 1.01^2\tldalpha \le 1.01^2r$.

    \paragraph{item (b)(c)}
    It suffices to combine \cref{lem: stage 1.1 IH}\ref{item: stage 1.1 IH dir} and item (a) that $\norm{\vv}\le 1.01\sqrt{r}$.
    
    \paragraph{item (d)}
    By item (b) and \cref{lem: init} we know $\tldE[\bw_i^2] \le 1.01 \tldE[\bw_i^{\btt{0}2}]\lesssim \frac{1}{d}$. Then by item (a) we know $\tldE[w_i^2] \lesssim r \tldE[\bw_i^2] \lesssim \frac{r}{d}$.
\end{proof}

The proof of this dynamic of $\vv$ essentially follows from the exact gradient formula \cref{claim: tldw dynamic} and IH that makes the high-order term in gradient as error term.
\begin{proof}[Proof of \cref{lem: stage 1.1 dynamic}]
    For simplicity in the proof, we drop $\sim$ and write $\tldv$ as $v$.
    From \eqref{eq: dynamic of tldw}, we know
    \begin{align*}
        \frac{\rd v_i}{\rd t}
        =& - \sum_{j\ge 0} [\nabla_{2j,\vv}]_i - \lambda v_i, 
    \end{align*}
    where $\nabla_{2j,\vv}$ represents the gradient of $2j$-th order loss
    \begin{align*}
        -[\nabla_{0,\vv}]_i 
        =& 2\hsigma_0^2 \left(
                \sum_{j=1}^r a_j^* \norm{\vw_j^*}_2
                - \tldE[ \norm{\vw}_2^2 ]\right) v_i
        = 2\hsigma_0^2\left(r-\tldE\norm{\vw}_2^2\right) v_i,\\
        -[\nabla_{2,\vv}]_i 
        =& \hsigma_{2}^2 \left(
            2 \ind_{i \le r}
            - 2 \tldE\left[ w_i^2 \right]
            \right) v_i
        = 2\hsigma_{2}^2 \left(
            \ind_{i \le r}
            \pm \BigO{\frac{r}{d}}
            \right) v_i,\\
        -[\nabla_{2j,\vv}]_i 
        =& \hsigma_{2j}^2 \Big(
            2j \bv_i^{2j-2} \ind_{i \le r}
            - 2j \tldE\left[
                \norm{\vw}_2^2 \langle\bvw,\bvv\rangle^{2j-1} \bw_i/\bv_i
                \right]\\
            &- (2j-2) \left(
            \sum_{l\le r}\bv_l^{2j} - 
            \tldE \left[\norm{\vw}_2^2\langle \bvw,\bvv\rangle^{2j}\right]
            \right)\Big) v_i,\quad \text{for $j\ge 2$.}
    \end{align*}    
    
    \paragraph{Dynamic of $\vv$.} To show the dynamic of $v_i$, from above it suffices to bound $|\nabla_{2j,\vv}|$.

    For $j\ge 2$, we have from Lemma~\ref{lem: stage 1.1 IH implication} that
    \begin{align*}
        \bv_i^{2j-2} 
        \le& \bv_i^{2} \lesssim \frac{\log d}{d},\quad
        \sum_{l\le r}\bv_l^{2j}
        \le \sum_{l\le r}\bv_l^{4}
        \lesssim \frac{r}{d^2},\\
        \tldE \left[\norm{\vw}_2^2\langle \bvw,\bvv\rangle^{2j}\right]
        \le& \tldE \left[\norm{\vw}_2^2\langle \bvw,\bvv\rangle^{2}\right]
        = \tldE\left[\norm{\vw}_2^2\sum_{i}\bw_{i}^2\bv_{i}^2 \right]
        \mylesim{a} r\sum_i \tldE [\bw_i^{\btt{0}2}] \bv_i^2
        \mylesim{a} \frac{r}{d},
    \end{align*}
    where (a) we use $\norm{\vw}\lesssim \sqrt{r}$, $\bw_i^2\le 1.01\bw_i^{\btt{0}2}$ from \cref{lem: stage 1.1 IH implication}.

    The remaining term is $\tldE [\norm{\vw}_2^2\langle\bvw,\bvv\rangle^{2j-1}\bw_1/\bv_1]$ (WLOG, consider the case $i=1$).

    For $2\le j\le \sqrt{d}$, by \cref{lem: high order moment bound} we have
    \begin{align*}
        \tldE [\norm{\vw}_2^2\langle\bvw,\bvv\rangle^{2j-1}\bw_1/\bv_1]
        \lesssim& \left\{\begin{array}{ll}
            \frac{r}{d} d\left(\frac{\log d}{d}\right)^2\le \frac{r\log^4 d}{d^{2}}   
            &\ ,\ j=2,3  \\
            \frac{r}{d}d\left(\frac{\log d}{\sqrt{d}}\right)^4
            \le \frac{r\log^4 d}{d^{2}} 
            &\ ,\ j\ge 4 
        \end{array}\right..
    \end{align*}
    
    For $j > \sqrt{d}$, by \cref{lem: high order moment bound} we have
    \begin{align*}
        \tldE [\norm{\vw}_2^2\langle\bvw,\bvv\rangle^{2j-1}\bw_1/\bv_1]
        \le \frac{1}{|\bv_1|}\tldE [\norm{\vw}_2^2\langle\bvw,\bvv\rangle^{2j-2}]
        \lesssim& \frac{r}{|\bv_1|}\left(\frac{\log d}{\sqrt{d}}\right)^{\sqrt{d}}
        \le \frac{1}{d},
    \end{align*}
    where we use \cref{lem: stage 1.1 IH}\ref{item stage 1.1 IH reg} at the end.

    Summing above over $j$ and use Claim~\ref{claim: hermite coeff} that $\sum_j 2j\hsigma_{2j}^2 \lesssim 1$ we get the desired bound.

    \paragraph{Dynamic of $\norm{\vv}$.}
    From item 1, we have
    \begin{align*}
        \frac{\rd}{\rd t}\norm{\vv}_2^2
        =& 2\sum_i v_i \frac{\rd v_i}{\rd t}
        = 2\norm{\vv}_2^2 \left(
            2\hsigma_0^2\left(r-\tldE\norm{\vw}_2^2\right)
            + 2\hsigma_2^2\left(\sum_{i\le r}\bv_i^2 
                - \tldE\left[\sum_{i\in [d]}w_i^2\bv_i^2\right]\right)
            -\lambda
            \pm \BigO{\frac{r}{d}}
        \right)\\
        =& 2\norm{\vv}_2^2 \left(
            2\hsigma_0^2\left(r-\tldE\norm{\vw}_2^2\right)
            -\lambda
            \pm \BigO{\frac{r}{d}}
        \right),
    \end{align*}
    where we use Lemma~\ref{lem: stage 1.1 IH implication}.

    \paragraph{Dynamic of $\bv_i$.}
    From item 1 and 2, we have
    \begin{align*}
        \frac{\rd \bv_i}{\rd t}
        = \bv_i \left(\frac{1}{v_i}\frac{\rd v_i}{\rd t} - \frac{1}{\norm{\vv}_2} \frac{\rd}{\rd t}\norm{\vv}_2\right)
        = \bv_i \cdot 2\hsigma_2^2 \left(\ind_{i\le r}-\tldE w_i^2
            \pm \BigO{\frac{r}{d}}
            \right).
    \end{align*}
    
\end{proof}

This lemma shows the dynamic and convergence of $\tldalpha$. The proof uses the above simplified dynamics of $\vv$ and then derive the convergence.
\begin{proof}[Proof of \cref{lem: stage 1.1 dynamic of tldalpha}]
    Recall $\tldalpha = \tldE \norm{\vw}_2^2$. From Lemma~\ref{lem: stage 1.1 dynamic} we have
    \begin{align*}
        \frac{\rd \tldalpha}{\rd t}
        = 2\tldE \left[\norm{\vw} \frac{\rd \norm{\vw}_2}{\rd t}\right]
        = 2\tldE\left[\norm{\vw}_2^2 \left(
            2\hsigma_0^2\left(r-\tldE\norm{\vw}_2^2\right)
            -\lambda
            \pm \BigO{\frac{r}{d}}
        \right)\right]
        = 4\hsigma_0^2\tldalpha \left(r-\frac{\lambda}{2\hsigma_0^2} - \tldalpha \pm \BigO{\frac{r}{d}}\right). 
    \end{align*}

    Solving above equation we know within time $\BigO{\frac{\log r}{r}}$ we know $r -\frac{\lambda}{2\hsigma_0^2} - \tldalpha \lesssim \frac{1}{r^{10}}$.
        
\end{proof}

\subsection{Stage 1.2: learn target subspace (2nd-order term)}\label{sec: stage 1.2}
In this stage, we show that 2nd order term will become small. Moreover, we learn subspace in the sense that $\norm{\vvler}_2/\norm{\vv}_2$ increases from $\Theta(r/d)$ to $1-o(1)$.

Denote the end time for Stage 1.2 as the time that learns 2nd-order term
\begin{align*}
    T_{12}:= \inf\left\{t\ge 0: \min_{i\in[r]} 1-\frac{\lambda_0}{2\hsigma_0^2} - \tldE w_i^2 \le \err_{12}=\frac{C_{12}\log d}{r}\right\},
\end{align*}
where $C_{12}$ is a large enough constant.

Below is the main result for Stage 1.2. It shows we learn the subspace, the direction of neurons in target subspace and irrelevant subspace remain unchanged.
\begin{lemma}[Main result for Stage 1.2]\label{lem: stage 1.2 main}
    Stage 1.2 ends within time $T_{12}-T_{11} \lesssim \log d$, and the following hold:
    \begin{enumerate}
        \item Learn subspace: for $i\le r$ and any neuron $\tldvv$, $\tldE w_i^2 = 1- \frac{\lambda_0}{2\hsigma_0^2} - \BigTheta{\frac{1}{r}}$ and $\frac{\norm{\tldvvler}^2}{\norm{\tldvv}^2} = 1- \BigO{\frac{1}{r}}$.
        
        \item Balance of direction in individual neuron: for any neuron $\tldvv$, we have
            $\frac{\tldv_i^2}{\tldv_j^2} = \frac{\tldv_i^{\btt{0}2}}{\tldv_j^{\btt{0}2}}\left(1\pm \delta_{dir,12}\right)$ for both $i,j\le r$ or $i,j> r$, where $\delta_{dir,12}\lesssim \frac{\log d\log r}{r}=o(1)$.

        \item Balance of norm: for any neuron $\tldvv$ we have
        $\frac{\norm{\tldvv^\btt{T{12}}}_2}{\norm{\tldvu^\btt{T_{12}}}_2} = \frac{\norm{\tldvv^\btt{0}}_2}{\norm{\tldvu^{\btt{0}}}_2}\left(1\pm \delta_{dir,12}\right)$, $\frac{\norm{\tldvvler^\btt{T_{12}}}_2}{\norm{\tldvuler^\btt{T_{12}}}_2} = \frac{\norm{\tldvvler^\btt{0}}_2}{\norm{\tldvuler^{\btt{0}}}_2}\left(1\pm \delta_{norm,12}\right)$, 
        $\frac{\norm{\tldvvgtr^\btt{T_{12}}}_2}{\norm{\tldvugtr^\btt{T_{12}}}_2} = \frac{\norm{\tldvvgtr^\btt{0}}_2}{\norm{\tldvugtr^{\btt{0}}}_2}\left(1\pm \delta_{norm,12}\right)$, 
        where $\delta_{norm,12}\lesssim \frac{\log d\log r}{r}=o(1)$
            
        \item Balance of 2nd order: for any $i,j\in[r]$, we have 
            $\frac{\tldE w_i^2}{\tldE w_j^2} = 1\pm\BigO{\frac{1}{r}}$.

        \item for $i>r$, we have $\tldE w_i^{\btt{T_{12}}2} \lesssim \frac{1}{d}$.

        \item Incremental learning of 2nd order at $T_{12}$: Loss $\L(\tldmW)\ge c\hsigma_0^2r^2/2+(1\pm c)\hsigma_{\ge2}^2r/2$ for $T_{11}\le t\le(1-\delta_T)T_{12}$ and $\L(\tldmW)\le c(\hsigma_0^2r^2/2 + \hsigma_2^2r/2) + (1\pm c)\hsigma_{\ge4}^2r/2$ at $t=T_{12}$, where $c$ is a small constant and $\delta=1/\sqrt{\log d}$.
    \end{enumerate}
\end{lemma}
\begin{proof}
    Item (a)(d)(e)(f) follows from \cref{lem: stage 1.2 convergence of 2nd order}. Item (b)(c) follow from induction hypothesis \cref{lem: stage 1.2 IH}.  
\end{proof}

To prove the above result, we introduce the induction hypothesis for this stage. All proofs are deferred to \cref{sec: stage 1.2 proofs}.
\begin{restatable}[Induction Hypothesis for Stage 1.2]{lemma}{lemstageabIH}\label{lem: stage 1.2 IH}
    Denote
    \begin{enumerate}
        \item \label{item: stage 1.2 IH neuron dir bal}
        Balance of direction in individual neuron: For any time $t\le T_{dir,12}$ we have
        for any neuron $\tldvv$
        \begin{align*}
            \frac{\tldv_i^{\btt{t}2}}{\tldv_j^{\btt{t}2}} = \frac{\tldv_i^{\btt{0}2}}{\tldv_j^{\btt{0}2}}\left(1\pm \delta_{dir,12}\right) \text{ for both $i,j\le r$ or $i,j> r$},
        \end{align*}
        where $\delta_{dir,12}\lesssim \frac{\log d}{r}$.

        \item \label{item: stage 1.2 IH neuron norm bal}
        Balance of norm between individual neuron: For any time $t\le T_{norm,12}$ we have
        for any neuron $\tldvv,\tldvu$ 
        \begin{align*}
            &\frac{\norm{\tldvv^\btt{t}}_2}{\norm{\tldvu^\btt{t}}_2} = \frac{\norm{\tldvv^\btt{0}}_2}{\norm{\tldvu^{\btt{0}}}_2}\left(1\pm \delta_{norm,12}\right),\quad 
            \frac{\norm{\tldvvler^\btt{t}}_2}{\norm{\tldvuler^\btt{t}}_2} = \frac{\norm{\tldvvler^\btt{0}}_2}{\norm{\tldvuler^{\btt{0}}}_2}\left(1\pm \delta_{norm,12}\right),
        \end{align*}
        where $\delta_{norm,12}\lesssim \frac{\log d\log r}{r}$.

        \item \label{item: stage 1.2 IH reg}
        Regularity condition: For any time $t\le T_{reg,12}$ we have
        \begin{enumerate}
            \item a loose bound on 1st order term: $\tldalpha \le \alpha_*+10 = r+10$.
            \item a loose lower bound on every coordinate: for any neuron $\tldvv$ and $i\in[d]$, $|\btldv_i|\ge \iota=e^{-r\log^{1/2} r}$.
        \end{enumerate}
    \end{enumerate}
    Then, $T_{12}\le \min\{T_{dir,12}, T_{norm,12}, T_{reg,12}\}$. That is, current stage ends before any of the above condition breaks.
\end{restatable}

Denote $\tldT_{12} = \min\{T_{dir,12}, T_{norm,12}, T_{reg,12},T_{12},C\log d\}$ where $C$ is a large enough constant. Thus, before $t\le \tldT_{12}$, the above induction hypothesis hold.

\paragraph{Implication of Induction Hypothesis.}

\begin{lemma}\label{lem: stage 1.2 IH implication}
    For $T_{11}\le t\le \tldT_{12}$, we have for any neuron $\tldvv$ and $i\in[d]$
    \begin{enumerate}
        \item Norm of neurons:
        $\norm{\tldvvler}_2 \le \norm{\tldvv} \le 1.01\sqrt{2r}$
        
        \item Direction of neuron remains bounded respectively in and out subspace: \label{item: stage 1.2 IH implicatino dir of neuron}
            $
            \btldv_i^2
            \le (\bbtldvvler)_i^2 
            = (\bbtldvvler^\btt{0})_i^2 (1\pm 3\delta_{dir,12})
            \lesssim \frac{\log m}{r}$
            for $i<r$ and
            $
            \btldv_i^2
            \le (\bbtldvvgtr)_i^2 
            = (\bbtldvvgtr^\btt{0})_i^2 (1\pm 3\delta_{dir,12})
            \lesssim \frac{\log d}{d}$
            for $i>r$.
            The above also implies $\tldE \btldv_i^2 \lesssim \frac{1}{r}$ for all $i\le d$.

        \item Regularity condition: \label{item: stage 1.2 IH implication reg}
        for neuron $\tldvv$, we have 
        $\sum_{i\le r} \btldv_i^4 \lesssim \frac{1}{r}$

        \item Population 2nd order term bound:
        we have $\tldE w_i^2=\tldE w_j^2 \lesssim \frac{r}{d}$ for any $i,j>r$.
    \end{enumerate}
\end{lemma}

We now ready to give the simplified dynamics under induction hypothesis \cref{lem: stage 1.2 IH}. 
\begin{lemma}\label{lem: stage 1.2 dynamic}
    In Stage 1.2, for $T_{11}\le t\le \tldT_{12}$, the dynamic \eqref{eq: dynamic of tldw} can be simplified as: for any neuron $\tldvv$ and $i\in[d]$
    \begin{align*}
        \frac{\rd \tldv_i}{\rd t}
        =&  \tldv_i \left(
                2\hsigma_0^2 \left(
                r - \tldE[ \norm{\vw}_2^2 ]\right) 
                + 2\hsigma_{2}^2 \left(
                    \ind_{i \le r}
                    - \tldE\left[ w_i^2 \right]
                \right)
                -\lambda
                + \sum_{j\ge 2} \hsigma_{2j}^2 2j \btldv_i^{2j-2}\ind_{i\le r}
                \pm \BigO{\frac{1}{r}}
            \right),\\
        \frac{\rd}{\rd t}\norm{\tldvv}_2
        =& \norm{\tldvv}_2 \left(
            2\hsigma_0^2\left(r-\tldE\norm{\vw}_2^2\right)
            + 2\hsigma_2^2\sum_{i\le r}\btldv_i^2 
                \left( 1 - \tldE\left[w_i^2\right]\right)
            -\lambda
            \pm \BigO{\frac{1}{r}}
            \right),\\
        \frac{\rd}{\rd t}\norm{\tldvvler}_2
        =& \norm{\tldvvler}_2 \left(
            2\hsigma_0^2\left(r-\tldE\norm{\vw}_2^2\right)
            + 2\hsigma_2^2\sum_{i\le r}(\bbtldvvler)_i^2 
                \left( 1 - \tldE\left[w_i^2\right]\right)
            -\lambda
            \pm \BigO{\frac{1}{r}}
            \right).
    \end{align*}
\end{lemma}

As a corollary, we have the dynamic of 0th $\tldalpha$ and 2nd order (per coordinate) $\tldE w_i^2$:
\begin{lemma}[Dynamics of 0th and 2nd order term]\label{lem: stage 1.2 0th 2nd order dynamic}
    For $t\le \tldT_{12}$, we have
    \begin{align*}
        \frac{\rd \tldalpha}{\rd t} 
        =& 4\hsigma_0^2 \tldalpha\left(r-\frac{\lambda}{2\hsigma_0^2}-\tldalpha 
            + \frac{\hsigma_2^2}{\hsigma_0^2}
                \sum_{i\in[r]}\frac{\tldE[w_i^2]}{\tldE\norm{\vw}_2^2} 
                (1  - \tldE\left[w_i^2\right])
            \pm \BigO{\frac{1}{r}}\right),\\
        \frac{\rd}{\rd t} \tldE w_i^2
        =& 4\hsigma_2^2 \tldE[w_i^2] \Bigg(
                \frac{\hsigma_0^2}{\hsigma_2^2}\left(r - \frac{\lambda}{2\hsigma_0^2} - \tldE[ \norm{\vw}_2^2 ]\right)
                + (\ind_{i \le r} - \tldE\left[ w_i^2 \right])
                \pm \BigO{\frac{1}{r}}
            \Bigg),\\
        \frac{\rd}{\rd t}\frac{\tldE w_i^2}{\tldE w_j^2}
        =& 4\hsigma_{2}^2 \frac{\tldE w_i^2}{\tldE w_j^2} \left(
                    \tldE\left[ w_j^2 \right] 
                    - \tldE\left[ w_i^2 \right]
                \pm \BigO{\frac{1}{r}}
           \right)
           \text{, for $i,j\in[r]$}
    \end{align*}
    Moreover, for $T_{11}\le t\le \tldT_{12}$ we have $-\frac{\hsigma_2^2}{\hsigma_0^2} \le r - \frac{\lambda}{2\hsigma_0^2} -\tldalpha \lesssim \frac{1}{r}$, $\tldE w_i^2 \gtrsim r/d$ for $i\in[r]$ and $\frac{\tldE w_i^2}{\tldE w_j^2} = 1\pm\BigO{\frac{\log d}{r}}$ for $i,j\in[r]$.
\end{lemma}

Given the above dynamics and use the fact that $\tldE w_i^2 \approx \tldE w_j^2$ for $i,j\le r$, we can show they converges within $\BigO{\log d}$ time. In particular, the proof uses \cref{prop: stage12-ode} to show the convergences by solve the dynamics of $\tldE w_i^2$.
\begin{lemma}[Convergence of 2nd order term]\label{lem: stage 1.2 convergence of 2nd order}
    We have
    \begin{align*}
        \tldE w_i^2 = \left\{\begin{array}{ll}
           1 - \frac{\lambda_0}{2\hsigma_0^2} - \BigTheta{\frac{1}{r}}  &  \text{ for $i\in[r]$}\\
           \BigTheta{\frac{1}{d}}  & \text{ for $i>r$}
        \end{array}\right.
        ,\quad
        \frac{\norm{\tldvvler}^2}{\norm{\tldvv}^2}
        = 1 - \BigO{\frac{1}{r}}, \text{ for all neuron $\tldvv$}
    \end{align*}
    within time $\hT_{12}-T_{11}=\BigO{\log d}$. 
    Moreover, at time $\hT_{12}$ we have improved bound $\frac{\tldE w_i^2}{\tldE w_j^2} = 1\pm \BigO{\frac{1}{r}}$ for $i,j\le r$ and 
    \[
        \int_0^{\hT_{12}}(1-\tldE w_i^2)\,ds
        =
        \frac{1}{4\hsigma_2^2}\log d
        +
        \BigO{\log\log d}.
    \]
    Lastly, before time $(1-\delta)\hT_{12}$ with $\delta=1/\sqrt{\log d}$, we have $\tldE w_i^2\lesssim 1/r^3$ for $i\le r$ and $\tldE w_i^2=\Theta(\frac{r}{d})$ for $i\ge r$.
\end{lemma}
\begin{proof}
    We show it one by one.
    
    \paragraph{Dynamic and bound of $\tldE w_i^2$}
    We first use the bound of $\frac{\tldE w_i^2}{\tldE w_j^2} = 1\pm\BigO{\frac{\log d}{r}}$ for $i,j\in[r]$ from \cref{lem: stage 1.2 0th 2nd order dynamic}. Let 
    \[
        \beta_1 = \frac{1}{r}\sum_{i\le r} \tldE w_i^2,\quad
        \beta_2 = \tldE w_{r+1}^2 =\cdots = \tldE w_d^2,
    \]
    we have $\beta_1/\tldE w_i^2 = 1\pm \BigO{\frac{\log d}{r}}$ for all $i\le r$ so $\beta_1$ is a good approximation of $\tldE w_i^2$ for $i\le r$. Using \cref{lem: stage 1.2 0th 2nd order dynamic} we know
    \begin{align*}
        \frac{\rd \beta_1}{\rd t} 
        =& 4\hsigma_2^2 \beta_1 \Bigg(
                \frac{\hsigma_0^2}{\hsigma_2^2}\left(r - \frac{\lambda}{2\hsigma_0^2} - r\beta_1 - (d-r)\beta_2 \right)
                + (1 - \beta_1)
                \pm \BigO{\frac{\log d}{r}}
            \Bigg)\\
        =& 4\beta_1 \Bigg(
                (\hsigma_0^2 r + \hsigma_2^2) (\beta_1^*-\beta_1) + \hsigma_0^2 (d-r)(\beta_2^* -\beta_2)
                \pm \BigO{\frac{\log d}{r}}
            \Bigg)\\
        \frac{\rd}{\rd t} \beta_2
        =& 4\hsigma_2^2 \beta_2 \Bigg(
                \frac{\hsigma_0^2}{\hsigma_2^2}\left(r - \frac{\lambda}{2\hsigma_0^2} - r\beta_1 - (d-r)\beta_2\right)
                - \beta_2
                \pm \BigO{\frac{1}{r}}
            \Bigg)\\
        =& 4\beta_2 \Bigg(
                -\lambda_0 \frac{\hsigma_2^2}{\hsigma_0^2 + (\hsigma_2^2/r)}
                + \hsigma_0^2 r\left(\beta_1^* -\beta_1\right)
                + (\hsigma_0^2(d-r)+\hsigma_2^2) (\beta_2^*-\beta_2)
                \pm \BigO{\frac{1}{r}}
            \Bigg),
    \end{align*}
    where $\beta_1^* = 1 - \frac{\lambda_0}{2\hsigma_0^2 + (2\hsigma_2^2/r)}$ and $\beta_2^*=0$, and recall $\lambda=r\lambda_0$. 

    Using \cref{prop: stage12-ode} we know $\tldE w_i^2 = \beta_1 = \beta_1^*-\Theta(1/r)=1-\frac{\lambda_0}{2\hsigma_0^2}-\BigTheta{\frac{1}{r}}$ for $i\le r$, $\tldE w_i^2=\beta_2\lesssim 1/d$ for $i>r$. This happens within time $\hT_{conv}=\BigO{\log d}$. Before $(1-\delta)T_{\rm conv}$ with $\delta=1/\sqrt{\log d}$, we know $\beta_1=\BigO{\frac{1}{r^3}}\beta_1^*$ and $\beta_2 \asymp \beta_2(0)\asymp \frac{r}{d}$.
    And lastly, for $i\le r$
    \[
        \int_0^{\hT_{conv}}(1-\tldE w_i^2)\,ds
        =
        \frac{1}{4\hsigma_2^2}\log d
        +
        \BigO{1}.
    \]
    As we see next, we should set $\hT_{12}=\hT_{conv}+\BigO{\log\log d}$, so the above error term $\BigO{1}$ will become $\BigO{\log\log d}$.

    \paragraph{Bound on $\frac{\tldE w_i^2}{\tldE w_j^2}$}
    From \cref{lem: stage 1.2 0th 2nd order dynamic} we have 
    \begin{align*}
        \frac{\rd}{\rd t}\frac{\tldE w_i^2}{\tldE w_j^2}
        =& 4\hsigma_{2}^2 \frac{\tldE w_i^2}{\tldE w_j^2} \left(
                    \tldE\left[ w_j^2 \right] 
                    - \tldE\left[ w_i^2 \right]
                \pm \BigO{\frac{1}{r}}
           \right).
    \end{align*}
    It is not hard to see after $\tldE w_i^2$ reach above constant, say $1/2$, then $\frac{\rd}{\rd t}\frac{\tldE w_i^2}{\tldE w_j^2}$ can be further improved from $1\pm \BigO{\frac{\log d}{r}}$ in \cref{lem: stage 1.2 0th 2nd order dynamic} to $1 \pm \BigO{\frac{1}{r}}$ in $\BigO{\log\log d}$ time.

    \paragraph{Bound of $\frac{\norm{\vvler}}{\norm{\vv}}$}
    We have when $\beta_1=\beta_1^*-\BigO{\frac{1}{r}}$
    \begin{align*}
        \frac{\norm{\vvler}^2}{\norm{\vv}^2}
        \myeq{a} (1\pm 5\delta_{norm,12})\frac{\tldE \norm{\vwler}^2}{\tldE \norm{\vw}^2}
        = 1\pm \BigO{\frac{\log d \log r}{r}}
    \end{align*}
    where (a) we use \cref{lem: stage 1.2 IH}\ref{item: stage 1.2 IH neuron norm bal};
    (b) \cref{lem: stage 1.2 0th 2nd order dynamic}. 

    From \cref{lem: stage 1.2 dynamic} we have
    \begin{align*}
        \frac{\rd}{\rd t}\frac{\norm{\vvler}^2}{\norm{\vv}^2}
        =& 2\frac{\norm{\vvler}^2}{\norm{\vv}^2}\left(
            2\hsigma_2^2\left(1-\frac{\norm{\vvler}^2}{\norm{\vv}^2}\right)
            \sum_{i\le r}(\bbtldvvler)_i^2 
                \left( 1 - \tldE\left[w_i^2\right]\right)
            \pm \BigO{\frac{1}{r}}
            \right)\\
        \gtrsim& \frac{\norm{\vvler}^2}{\norm{\vv}^2}\left(
                    1-\frac{\norm{\vvler}^2}{\norm{\vv}^2}
                    \pm \BigO{\frac{1}{r}}
                \right),
    \end{align*}
    where we use $\sum_{i\le r}(\bbtldvvler)_i^2 
                \left( 1 - \tldE\left[w_i^2\right]\right)\gtrsim 1$.
    Thus, after additional $\BigO{\log\log d}$ time, $\frac{\norm{\vvler}^2}{\norm{\vv}^2}$ can be improved to $\BigO{\frac{1}{r}}$.

\end{proof}

\begin{proposition}[Convergence for the Stage~1.2 dynamics]
\label{prop: stage12-ode}
    Recall
    \[
    \beta_1^*
    =
    1-\frac{\lambda}{2\hsigma_0^2 r+2\hsigma_2^2}
    =
    1-\frac{\lambda_0}{2\hsigma_0^2+2\hsigma_2^2/r},
    \qquad \lambda=r\lambda_0.
    \]
    
    Suppose that on a time interval containing \([0,T_{\mathrm{conv}}]\), the pair
    \((\beta_1,\beta_2)\) satisfies
    \begin{align}
    \frac{d}{dt}\beta_1
    &=
    4\beta_1\Big(
    (\hsigma_0^2r+\hsigma_2^2)(\beta_1^*-\beta_1)
    -\hsigma_0^2(d-r)\beta_2
    +\varepsilon_1(t)
    \Big),
    \label{eq:beta1-clean-hsigma}
    \\
    \frac{d}{dt}\beta_2
    &=
    4\beta_2\Big(
    (\hsigma_0^2r+\hsigma_2^2)(\beta_1^*-\beta_1)
    -\hsigma_0^2(d-r)\beta_2
    -\hsigma_2^2(1-\beta_1+\beta_2)
    +\varepsilon_2(t)
    \Big),
    \label{eq:beta2-clean-hsigma}
    \end{align}
    with $|\varepsilon_1(t)|\le K_0\frac{\log d}{r}$ and $|\varepsilon_2(t)|\le K_0\frac1r$.
    Assume \(0<\lambda_0<\hsigma_0^2\)
    and $\beta_1(0)\asymp \beta_2(0)\asymp \frac{r}{d}$.
    
    Define
    \[
    q(t):=
    \frac{\hsigma_0^2(d-r)\beta_2(t)}
    {(\hsigma_0^2r+\hsigma_2^2)\beta_1(t)},
    \qquad
    T_{\mathrm{conv}}
    :=
    \inf\Bigl\{t\ge 0:\ q(t)\le r^{-1}\Bigr\}.
    \]
    Then there exist constants \(C,\rho>0\), depending only on
    \(\hsigma_0,\hsigma_2,\lambda_0,K_0\), such that for every
    \(0\le t\le T_{\mathrm{conv}}\),
    \begin{align}
    q(t)
    &\le
    C\frac{d}{r}e^{-\rho t},
    \label{eq:q-decay-hsigma}
    \\
    0\le \beta_1^*-\beta_1(t)
    &\le
    Cq(t)+C\frac{\log d}{r^2}
    \le
    C\frac{d}{r}e^{-\rho t}+C\frac{\log d}{r^2},
    \label{eq:beta1-gap-hsigma}
    \\
    \beta_2(t)
    &\le
    Ce^{-\rho t}.
    \label{eq:beta2-decay-hsigma}
    \end{align}
    In particular,
    \[
    T_{\mathrm{conv}}=O(\log d),\qquad
    \beta_1^*-\beta_1(T_{\mathrm{conv}})\le \frac{C}{r},
    \qquad
    \beta_2(T_{\mathrm{conv}})\le \frac{C}{d}.
    \]
    For $\omega(\frac{\log r}{\log d})=\delta=o(1)$, we further know before time $(1-\delta)T_{\rm conv}$, $\beta_1=\BigO{\frac{1}{r^3}}\beta_1^*$ and $\beta_2 \asymp \beta_2(0)\asymp \frac{r}{d}$.
    
    Moreover, for every \(0\le t\le T_{\mathrm{conv}}\),
    \begin{equation}
    \int_0^t (1-\beta_1(s))\,ds
    \le
    \frac{1}{4\hsigma_2^2}\log\frac{q(0)}{q(t)}
    +
    Ct\frac{\log d}{r}.
    \label{eq:int-any-time-hsigma}
    \end{equation}
    In particular,
    \begin{equation}
    \int_0^{T_{\mathrm{conv}}}(1-\beta_1(s))\,ds
    =
    \frac{1}{4\hsigma_2^2}\log d
    +
    \BigO{1}
    =
    \left(\frac{1}{4\hsigma_2^2}+o(1)\right)\log d.
    \label{eq:int-final-hsigma}
    \end{equation}
\end{proposition}

\subsubsection{Omitted proofs for Stage 1.2}\label{sec: stage 1.2 proofs}
We now present the omitted proofs in this stage.

We first prove the induction hypothesis for this stage. The proof mostly uses the dynamics of $\vv$ derived under IH and show that it remains.
\begin{proof}[Proof of \cref{lem: stage 1.2 IH}]
    For simplicity in the proof, we drop $\sim$ and write $\tldv$ as $v$.

    Recall from Lemma~\ref{lem: stage 1.2 convergence of 2nd order} that under induction hypothesis within time $\hT_{12}=\BigO{\log d}$ we have $\tldE w_i^2 \ge 1-\BigTheta{\frac{\log d}{r}}$ for all $i\in[r]$.
    
    In the following we argue $\tldT_{12}=\min\{T_{dir,12},T_{norm,12},T_{reg,12},T_{12},C\log d \}= T_{12}$, which would imply the desired result $T_{12}\le\min\{T_{dir,12},T_{norm,12},T_{reg,12}\}$ as well as $T_{12} \le \hT_{12}=\BigO{\log d}$.

    To prove the above claim, assume towards contradiction that  $\tldT_{12}$ equals one of $T_{dir,12},T_{norm,12},T_{reg,12}$ ($T_{12}$ will not be $C\log d$ as we can take constant $C$ to be large enough). 

    \paragraph{If $\tldT_{12}=T_{dir,12}$}
    Recall from Lemma~\ref{lem: stage 1.2 dynamic} we have
    \begin{equation}\label{eq: stage 1.2 IH dir}
        \begin{aligned}
            \frac{\rd v_i}{\rd t}
        =& v_i \left(
                2\hsigma_0^2 \left(
                r - \tldE[ \norm{\vw}_2^2 ]\right) 
                + 2\hsigma_{2}^2 \left(
                    \ind_{i \le r}
                    - \tldE\left[ w_i^2 \right]
                \right)
                -\lambda
                + \sum_{j\ge 2} \hsigma_{2j}^2 2j \btldv_i^{2j-2}\ind_{i\le r}
                \pm \BigO{\frac{1}{r}}
            \right)
        \end{aligned}
    \end{equation}

    Hence, for $t\le \tldT_{12}$, we have
    
    Case 1: $i,j\in[r]$.
    When $\tldE w_i^2 \le 1/r^2$ for all $i\le r$ (which takes $\BigO{\log d}$ time, can be seen from the proof of \cref{lem: stage 1.2 0th 2nd order dynamic}. This is the most of the time in this stage), we have 
    \begin{align*}
        \frac{\rd }{\rd t}\frac{v_i^2}{v_j^2}
        = \frac{v_i^2}{v_j^2} \left(
            \frac{1}{v_i^2}\frac{\rd v_i^2}{\rd t}
            - \frac{1}{v_j^2}\frac{\rd v_j^2}{\rd t}
            \right)
        =& \pm  \frac{v_i^2}{v_j^2} \cdot 
            4\hsigma_2^2
                \BigO{\frac{1}{r}}
    \end{align*}
    where we use $\sum_{j\ge 2}\hsigma_{2j}^2 2j\lesssim 1$, \cref{lem: stage 1.2 IH implication}\ref{item: stage 1.2 IH implicatino dir of neuron}, the fact of $\bv_i^2 = \frac{\norm{\vvler}^2}{\norm{\vv}^2}(\bbvvler)_i^2 = \Theta(1) \frac{\tldE\norm{\vwler}^2}{\tldE \norm{\vw}^2}(\bbvvler)_i^2$. 
    Thus, we have $\frac{v_i^2}{v_j^2} - 1 \lesssim \frac{\log d}{r}$.

    For the rest of time, it takes $\BigO{\log r}$ time, we have
    \begin{align*}
        \frac{\rd }{\rd t}\frac{v_i^2}{v_j^2}
        = \frac{v_i^2}{v_j^2} \left(
            \frac{1}{v_i^2}\frac{\rd v_i^2}{\rd t}
            - \frac{1}{v_j^2}\frac{\rd v_j^2}{\rd t}
            \right)
        =& \frac{v_i^2}{v_j^2} \cdot 
            4\hsigma_2^2\left(
                \bv_i^2 - \bv_j^2
                \pm \BigO{\frac{1}{r}}
            \right)
        = \frac{v_i^2}{v_j^2} \cdot 
            4\hsigma_2^2\left(
                \bv_j^2(\frac{\bv_i^2}{\bv_j^2}-1)
                +\BigO{\frac{1}{r}}
            \right),
    \end{align*}
    which implies (assume $\frac{v_i^2}{v_j^2}>
    1$)
    \begin{align*}
        \frac{\rd }{\rd t}\frac{\bv_i^2}{\bv_j^2} - 1
        \lesssim \frac{\log d}{r}\left(
                \frac{\bv_i^2}{\bv_j^2}-1
                \right)
                +\BigO{\frac{1}{r}}
    \end{align*}
     Thus, this implies $\frac{v_i^2}{v_j^2} - 1 \lesssim \frac{\log d}{r}$ since $r>\log m \log r$ and total time for this part is $\BigO{\log r}$.

    Case 2: $i,j\in[d]/[r]$. This follows the same argument. We omit for simplicity.

    This implies $\delta_{dir,12}\lesssim \frac{\log d}{r}$. We know this is a contraction by setting $\delta_{dir,12}=\BigTheta{\frac{\log d}{r}}$ with large enough hidden constant.

    \paragraph{If $\tldT_{12}=T_{norm,12}$}
    Instead of directly proving, we first argue the ratio between different neuron's $\norm{\vvler}$ and $\norm{\vvgtr}$ does not change much.

    We have from \cref{lem: stage 1.2 dynamic} that (using \cref{lem: stage 1.2 0th 2nd order dynamic} to approximate $\tldE w_i$ with $\tldE w_1^2$)
    \begin{align*}
        \frac{\rd}{\rd t}\norm{\vvler}_2^2
        = 2\sum_{i\le r}v_i \frac{\rd v_i}{\rd t}
        =& 2\norm{\vvler}_2^2 \left(
            2\hsigma_0^2\left(r-\tldE\norm{\vw}_2^2\right)
            + 2\hsigma_2^2\sum_{i\le r}(\bbvvler)_i^2 
                \left( 1 - \tldE\left[w_i^2\right]\right)
            -\lambda
            \pm \BigO{\frac{1}{r}}
            \right),\\
        \myeq{a}& 2\norm{\vvler}_2^2 \left(
            2\hsigma_0^2\left(r-\tldE\norm{\vw}_2^2\right)
            + 2\hsigma_2^2 
                \left( 1 - \tldE\left[w_1^2\right]\right)
            -\lambda
            \pm \BigO{\frac{\log d}{r}\tldE w_1^2 + \frac{1}{r}} 
            \right),\\
        \frac{\rd}{\rd t}\norm{\vvgtr}_2^2
        = 2\sum_{i> r}v_i \frac{\rd v_i}{\rd t}
        =& 2\norm{\vvgtr}_2^2 \left(
            2\hsigma_0^2\left(r-\tldE\norm{\vw}_2^2\right)
            -\lambda
            \pm \BigO{\frac{1}{r}}
            \right),
    \end{align*}
    where (a) we follow the same argument as in \eqref{eq: stage 1.2 IH dir}.

    Therefore, for any neuron $\vv,\vu$ we have
    \begin{align*}
        \frac{\rd }{\rd t}\frac{\norm{\vvler}}{\norm{\vuler}}
        =& \frac{\norm{\vvler}}{\norm{\vuler}} \left(
            \frac{1}{\norm{\vvler}}\frac{\rd \norm{\vvler}}{\rd t}
            - \frac{1}{\norm{\vuler}}\frac{\rd \norm{\vuler}}{\rd t}
            \right)
        = \pm \BigO{\frac{\log d}{r}\tldE w_1^2+\frac{1}{r}}\frac{\norm{\vvler}}{\norm{\vuler}}\\
        \frac{\rd }{\rd t}\frac{\norm{\vvgtr}}{\norm{\vugtr}}
        =& \frac{\norm{\vvgtr}}{\norm{\vugtr}} \left(
            \frac{1}{\norm{\vvgtr}}\frac{\rd \norm{\vvgtr}}{\rd t}
            - \frac{1}{\norm{\vugtr}}\frac{\rd \norm{\vugtr}}{\rd t}
            \right)
        = \pm \BigO{\frac{1}{r}}\frac{\norm{\vvgtr}}{\norm{\vugtr}}.
    \end{align*}

    For $\frac{\norm{\vvler}}{\norm{\vuler}}$, we follow the same arguments in "If $\tldT_{12}=T_{dir,12}$" case above. For $\frac{\norm{\vvgtr}}{\norm{\vugtr}}$, it follows from $\hT_{12}\lesssim \log d$.

    We can see this is a contraction by setting $\delta_{norm,12}=\BigTheta{\frac{\log d\log r}{r}}$ with large enough hidden constant.

    \paragraph{If $\tldT_{12}=T_{reg,12}$}
    From Lemma~\ref{lem: stage 1.2 0th 2nd order dynamic} we know $\tldalpha\le r+1.01 < r+10$. 
    
    For $|\bv_i|$ and $i\in[r]$, we have $\frac{\norm{\vvler}^2}{\norm{\vv}^2}
        \myeq{item\ref{item: stage 1.2 IH neuron norm bal}} (1\pm 5\delta_{norm,12})\frac{\tldE \norm{\vwler}^2}{\tldE \norm{\vw}^2} \gtrsim \frac{r}{d}$ from \cref{lem: stage 1.2 0th 2nd order dynamic}. Then $\bv_i^2 = \frac{\norm{\vvler}^2}{\norm{\vv}^2} (\bbvvler)_i^2
        \gtrsim \frac{r}{d} (\bbvvler^\btt{0})_i^2 \ge \iota^2$ using \cref{lem: stage 1.2 IH implication} and \cref{lem: init}.

    For $i>r$, noting $\tldE w_i^2 \ge \frac{r-\sum_{i\le r}\tldE w_i^2}{d} \gtrsim \log d/d$. Then follow the same argument we can show $|\bv_i|\ge \iota$.

    This leads to the contradiction.
    
    \paragraph{Summary}
    Now combining all above we know $\tldT_{12}=T_{12}\le \hT_{12}\lesssim \log d$.
\end{proof}

This gives some implications of IH. The proof mostly directly follows from IH.
\begin{proof}[Proof of \cref{lem: stage 1.2 IH implication}]
    We show one by one. For simplicity in the proof, we drop $\sim$ and write $\tldv$ as $v$.

    \paragraph{item (a)}
    By \cref{lem: stage 1.2 IH}\ref{item: stage 1.2 IH neuron norm bal} and \cref{lem: init} we know for any neuron $\vv,\vu$ we have $\frac{\norm{\vv}_2}{\norm{\vu}_2} \le 1.01$. Thus, for $T_{11} \le t\le \tldT_{12}$ and any neuron $\vv$, we have $\norm{\vv}_2^2 \le 1.01^2\tldE\norm{\vw}_2^2 = 1.01^2\tldalpha \le 1.01^2 2r$, where we use \cref{lem: stage 1.2 IH}\ref{item: stage 1.2 IH reg}.

    \paragraph{item (b)(c)}
    By \cref{lem: stage 1.2 IH}\ref{item: stage 1.2 IH neuron dir bal} we know for any neuron $\vv$ and $i\in[r]$, we have 
    \[
        (\bbvvler)_i^2 
        = \frac{v_i^2}{\norm{\vvler}_2^2}
        = \frac{v_i^{\btt{0}2}}{\norm{\vvler^\btt{0}}_2^2} (1\pm 3\delta_{dir,12})
        = (\bbvvler^\btt{0})_i^2 (1\pm 3\delta_{dir,12})
        \le 1.01 (\bbvvler^\btt{0})_i^2
        \lesssim \frac{\log d}{r}.
    \]
    Thus, we also have $\bv_i^2 \le \frac{v_i^2}{\norm{\vvler}_2^2} \le 1.01 (\bbvvler^\btt{0})_i^2 \lesssim \frac{\log d}{r}$
    and $\sum_{i\le r}\bv_i^4 \lesssim \sum_{i\le r}(\bbvvler^\btt{0})_i^4
    \lesssim \frac{1}{r}$.
    We use the same argument for $i>r$ and get the desired bounds.

    \paragraph{item (d)}

    For $i,j>r$, by the definition of $\tldf$ that it is invariant to the permutation in the last $d-r$ coordinates, we know for $i,j> r$, $\tldE w_i^2=\tldE w_j^2$. Thus, $\frac{\tldE w_i^2}{\sum_{i> r}\tldE w_i^2} =\frac{1}{d-r}$. Hence, $\tldE w_i^2 = \frac{1}{d-r}\sum_{i> r}\tldE w_i^2 \le \frac{2r}{d-r}$ for $i> r$.
\end{proof}

This gives the dynamic of $\vv$. The proof uses the exact gradient formula in \cref{claim: tldw dynamic} and treat the high order term as error term using IH.
\begin{proof}[Proof of \cref{lem: stage 1.2 dynamic}]
    For simplicity in the proof, we drop $\sim$ and write $\tldv$ as $v$.
    From \eqref{eq: dynamic of tldw}, we know
    \begin{align*}
        \frac{\rd v_i}{\rd t}
        =& - \sum_{j\ge 0} [\nabla_{2j,\vv}]_i - \lambda v_i, 
    \end{align*}
    where $\nabla_{2j,\vv}$ represents the gradient of $2j$-th order loss
    \begin{align*}
        -[\nabla_{0,\vv}]_i 
        =& 2\hsigma_0^2 \left(
                \sum_{j=1}^r a_j^* \norm{\vw_j^*}_2
                - \tldE[ \norm{\vw}_2^2 ]\right) v_i
        = 2\hsigma_0^2\left(r-\tldE\norm{\vw}_2^2\right) v_i,\\
        -[\nabla_{2,\vv}]_i 
        =& \hsigma_{2}^2 \left(
            2 \ind_{i \le r}
            - 2 \tldE\left[ w_i^2 \right]
            \right) v_i,\\
        -[\nabla_{2j,\vv}]_i 
        =& \hsigma_{2j}^2 \Bigg(
            2j \bv_i^{2j-2} \ind_{i \le r}
            - 2j \tldE\left[
                \norm{\vw}_2^2 \langle\bvw,\bvv\rangle^{2j-1} \bw_i/\bv_i
                \right]\\
            &- (2j-2) \left(
            \sum_{l\le r}\bv_l^{2j} - 
            \tldE \left[\norm{\vw}_2^2\langle \bvw,\bvv\rangle^{2j}\right]
            \right)\Bigg) v_i,\quad \text{for $j\ge 2$.}
    \end{align*}    

    \paragraph{Dynamic of $v_i$.}
    We first bound each term in $\nabla_{2j,\vv}$ for $j\ge 2$. The term $2j \bv_i^{2j-2} \ind_{i \le r}$ is left as is. We have
    \begin{align*}
        \sum_{l\le r}\bv_l^{2j}
        \le \sum_{j\le r}\bv_j^{4}
        \mylesim{a} \frac{1}{r},
    \end{align*}
    where (a) we use \cref{lem: stage 1.2 IH implication}\ref{item: stage 1.2 IH implication reg}.

    For $\tldE \left[\norm{\vw}_2^2\langle \bvw,\bvv\rangle^{2j}\right]$ and all $j\ge 2$, we have by \cref{lem: high order moment bound}
    \begin{align*}
        \tldE \left[\norm{\vw}_2^2\langle \bvw,\bvv\rangle^{2j}\right]
        \le& \tldE \left[\norm{\vw}_2^2\langle \bvw,\bvv\rangle^{4}\right]
        \lesssim \frac{1}{r}.
    \end{align*}

    The remaining term is $\tldE [\norm{\vw}_2^2\langle\bvw,\bvv\rangle^{2j-1}\bw_1/\bv_1]$ (WLOG, consider the case $i=1$).

    For $2\le j\le r/2$, by \cref{lem: high order moment bound} we have
    \begin{align*}
        \tldE [\norm{\vw}_2^2\langle\bvw,\bvv\rangle^{2j-1}\bw_1/\bv_1]
        \lesssim \frac{1}{r}.
    \end{align*}

    For $j > r/2$, by \cref{lem: high order moment bound} we have
    \begin{align*}
        \tldE [\norm{\vw}_2^2\langle\bvw,\bvv\rangle^{2j-1}\bw_1/\bv_1]
        \le \frac{1}{|\bv_1|}\tldE [\norm{\vw}_2^2\langle\bvw,\bvv\rangle^{2j-2}]
        \lesssim& \frac{r}{|\bv_1|}\left(\frac{2}{r}\right)^{r/2}
        \le \frac{1}{|\bv_1|}\iota^2
        \le \iota
        \le \frac{1}{r^2},
    \end{align*}
    where we use \cref{lem: stage 1.2 IH}\ref{item: stage 1.2 IH reg} at the end.
    
    Summing above over $j$ and use Claim~\ref{claim: hermite coeff} that $\sum_j 2j\hsigma_{2j}^2 \lesssim 1$ we get the desired bound.

    \paragraph{Dynamic of $\norm{\vv}$ and $\norm{\vvler}$.}
    Note that
    \begin{align*}
        \frac{\rd}{\rd t}\norm{\vv}_2^2
        = 2\sum_i v_i \frac{\rd v_i}{\rd t}
    \end{align*}
    so we can apply item 1. 
    When plugging in, further notice that 
    \begin{align*}
        \sum_{i\le r} \sum_{j\ge 2} \hsigma_{2j}^2 2j \btldv_i^{2j}
        \le \sum_{j\ge 2} \hsigma_{2j}^2 2j \sum_{i\le r}\btldv_i^{4}
        \lesssim \sum_{j\ge 2} \hsigma_{2j}^2 2j \sum_{i\le r}(\bbvvler)_i^{4} \frac{\norm{\vvler}^2}{\norm{\vv}^2}
        \mylesim{a} \frac{1}{r}\frac{\norm{\vvler}^2}{\norm{\vv}^2},
    \end{align*}
    where (a) we use \cref{claim: hermite coeff} that $\sum_j 2j\hsigma_{2j}^2 \lesssim 1$ and \cref{lem: stage 1.2 IH implication}\ref{item: stage 1.2 IH implication reg}. This gets the final form. The same gives $\frac{\rd}{\rd t}\norm{\vvler}$.

\end{proof}

The proof uses the dynamic of $\vv$ to derive the dynamic of these population terms.
\begin{proof}[Proof of \cref{lem: stage 1.2 0th 2nd order dynamic}]
    We show one by one.
    \paragraph{Dynamic and bound on $\tldE w_i^2$}
    From Lemma~\ref{lem: stage 1.2 dynamic} and using \cref{lem: stage 1.2 IH implication}\ref{item: stage 1.2 IH implicatino dir of neuron} we know for each $i\in[d]$
    \begin{align*}
        \frac{\rd}{\rd t}\tldE w_i^2
        = 2 \tldE\left[ v_i \frac{\rd v_i}{\rd t} \right]
        =&  2\tldE[w_i^2] \left(
              2\hsigma_0^2 \left(r - \frac{\lambda}{2\hsigma_0^2} - \tldE[ \norm{\vw}_2^2 ]\right)
            + 2\hsigma_{2}^2 \left(
                \ind_{i \le r}
                - \tldE\left[ w_i^2 \right]
            \right)
            \pm \BigO{\frac{1}{r}}\right) \\
        =&   4\hsigma_2^2 \tldE[w_i^2] \Bigg(
                \frac{\hsigma_0^2}{\hsigma_2^2}\left(r - \frac{\lambda}{2\hsigma_0^2} - \tldE[ \norm{\vw}_2^2 ]\right)
                + (\ind_{i \le r} - \tldE\left[ w_i^2 \right])
                \pm \BigO{\frac{1}{r}}
            \Bigg).
    \end{align*}

    For $i<r$ and $j\ge r$, we have
    \begin{align*}
        \frac{\rd}{\rd t} \frac{\tldE w_i^2}{\tldE w_j^2}
        =& 4\hsigma_2^2 \frac{\tldE w_i^2}{\tldE w_j^2} \Bigg(
                (1 - \tldE\left[ w_i^2 \right])
                + \tldE\left[ w_j^2 \right]
                \pm \BigO{\frac{1}{r}}
            \Bigg)
        \ge 4\hsigma_2^2 \frac{\tldE w_i^2}{\tldE w_j^2} \Bigg(
                1 - \tldE\left[ w_i^2 \right]
                - \BigO{\frac{1}{r}}
            \Bigg)
        \ge 0.
    \end{align*}
    This implies $\tldE w_i^2\gtrsim \frac{\tldalpha}{d}\gtrsim r/d$ for $i<r$ for all $t\le \tldT_{12}$.

    For $i,j\le r$, we have (WLOG, let $i=1$ and $j=2$)
    \begin{align*}
        \frac{\rd}{\rd t}\frac{\tldE w_1^2}{\tldE w_2^2}
        =& \frac{\tldE w_1^2}{\tldE w_2^2} \left(
                4\hsigma_{2}^2 \left(
                    \tldE\left[ w_2^2 \right] - \tldE\left[ w_1^2 \right]
                    \right)
                \pm \BigO{\frac{1}{r}}
           \right).
    \end{align*}

    Moreover, this implies an bound on $\frac{\tldE w_i^2}{\tldE w_j^2}$. WLOG, assume $\frac{\tldE w_i^2}{\tldE w_j^2}\ge 1$ (otherwise we can consider $\frac{\tldE w_j^2}{\tldE w_i^2}$). From above we have
    \begin{align*}
        \frac{\rd}{\rd t}\frac{\tldE w_i^2}{\tldE w_j^2}
        \lesssim& \frac{1}{r} \frac{\tldE w_i^2}{\tldE w_j^2}.
    \end{align*}
    Since this stage ends in time $t\le \tldT_{12}\le \BigO{\log d}$, this leads to
    \begin{align*}
        \frac{\tldE w_i^2}{\tldE w_j^2}
        \le \frac{\tldE w_i^{\btt{T_{11}}2}}{\tldE w_j^{\btt{T_{11}}2}}
            \left(1+\BigO{\frac{\log d}{r}}\right)
        \le 1+\BigO{\frac{\log d}{r}}
    \end{align*}
    where last one we use \cref{lem: stage 1.1 main}.

    \paragraph{Dynamic and bound on $\tldalpha$}
    From Lemma~\ref{lem: stage 1.2 dynamic} we know
    \begin{align*}
        \frac{\rd \tldalpha}{\rd t}
        =& \frac{\rd}{\rd t}\tldE \norm{\vv}_2^2
        =  4\hsigma_0^2 \tldalpha\left(r-\frac{\lambda}{2\hsigma_0^2}-\tldalpha \pm \BigO{\frac{1}{r}}\right)
            + 4\hsigma_2^2\left(\sum_{i\in[r]}\tldE[w_i^2] 
                (1 - \tldE\left[w_i^2\right])\right)\\
        =& 4\hsigma_0^2 \tldalpha\left(r-\frac{\lambda}{2\hsigma_0^2}-\tldalpha 
            + \frac{\hsigma_2^2}{\hsigma_0^2}
                \sum_{i\in[r]}\frac{\tldE[w_i^2]}{\tldE\norm{\vw}_2^2} 
                (1  - \tldE\left[w_i^2\right])
            \pm \BigO{\frac{1}{r}}\right).
    \end{align*}
    Moreover, we can see
    \begin{align*}
        \frac{\rd \tldalpha}{\rd t}
        \ge& 4\hsigma_0^2 \tldalpha\left(r-\frac{\lambda}{2\hsigma_0^2}- \tldalpha 
            - \BigO{\frac{1}{r}}\right),
    \end{align*}
    which implies $r - \frac{\lambda}{2\hsigma_0^2} - \tldalpha \lesssim \frac{1}{r}$.
    Similarly, we have
    \begin{align*}
        \frac{\rd \tldalpha}{\rd t}
        \le& 4\hsigma_0^2 \tldalpha\left(r-\frac{\lambda}{2\hsigma_0^2}-\tldalpha 
            +\frac{\hsigma_2^2}{\hsigma_0^2} \frac{r}{4\tldalpha}
            +\BigO{\frac{1}{r}}\right)
        \le 4\hsigma_0^2 \tldalpha\left(r-\frac{\lambda}{2\hsigma_0^2}-\tldalpha 
            +\frac{\hsigma_2^2}{\hsigma_0^2}\right),
    \end{align*}
    which implies $r - \frac{\lambda}{2\hsigma_0^2} - \tldalpha \ge -\frac{\hsigma_2^2}{\hsigma_0^2}$. In above we use $\tldalpha\ge r-\frac{\lambda}{2\hsigma_0^2}-\BigO{1/r}\ge 2r/3$.
\end{proof}

This proof explicitly estimates the convergence time of the ODE given.
\begin{proof}[Proof of \cref{prop: stage12-ode}]
    We divide the proof into following steps.
    
    \medskip
    \noindent\textbf{Step 1: Basic Regularities.}
    
    It is not hard to see $0<\beta_1\le\beta_1^*,0<\beta_2\le 1$. We omit the details.
    Consequently,
    \begin{equation}
    1-\beta_1+\beta_2
    \ge
    1-\beta_1
    \ge
    1-\beta_1^*
    =
    \frac{\lambda}{2\hsigma_0^2r+2\hsigma_2^2}
    =
    \frac{\lambda_0}{2\hsigma_0^2+2\hsigma_2^2/r}.
    \label{eq:one-minus-beta-lower-hsigma}
    \end{equation}
    
    \medskip
    \noindent\textbf{Step 2: Exponential decay of the scaled ratio \(q\).}
    A direct computation from \eqref{eq:beta1-clean-hsigma}--\eqref{eq:beta2-clean-hsigma}
    shows that
    \begin{equation}
    \frac{d}{dt}\log q
    =
    \frac{\dot\beta_2}{\beta_2}-\frac{\dot\beta_1}{\beta_1}
    =
    -4\hsigma_2^2(1-\beta_1+\beta_2)
    +4(\varepsilon_2-\varepsilon_1)
    \le 
    -4\hsigma_2^2\frac{\lambda_0}{2\hsigma_0^2+2\hsigma_2^2/r}
    +
    4K_0\frac{\log d}{r}
    +
    4K_0\frac1r
    \le -\rho.
    \label{eq:logq-exact-hsigma}
    \end{equation}
    where \(\rho>0\) is a constant.
    Therefore
    \[
    q(t)\le q(0)e^{-\rho t}
    \lesssim \frac{\hsigma_0^2(d-r)}{\hsigma_0^2r+\hsigma_2^2}
    \frac{\beta_2(0)}{\beta_1(0)}e^{-\rho t} 
    \asymp
    \frac{d}{r}e^{-\rho t},
    \]
    this proves \eqref{eq:q-decay-hsigma}.
    
    We will also need a lower bound on the rate of variation of \(q\).
    Since \(0<\beta_1\le \beta_1^*<1\) and \(0<\beta_2\le 1\), we have
    \(1-\beta_1+\beta_2\le 2\). Hence by \eqref{eq:logq-exact-hsigma},
    \[
    \frac{d}{dt}\log q
    \ge
    -8\hsigma_2^2
    -4|\varepsilon_1|
    -4|\varepsilon_2|
    \ge -M
    \]
    for some constant \(M>0\). Therefore, for every \(0\le s\le t\le T_{\mathrm{conv}}\),
    \begin{equation}
    q(s)\le q(t)e^{M(t-s)}.
    \label{eq:q-two-sided-speed-hsigma}
    \end{equation}
    
    \medskip
    \noindent\textbf{Step 3: Control of \(\beta_1^*-\beta_1\).}
    Define
    $y(t):=\frac1{\beta_1(t)}-\frac1{\beta_1^*}$.
    Since $\beta_1\le\beta_1^*$, one has \(y(t)\ge 0\) on
    \([0,T_{\mathrm{conv}}]\).
    
    Using
    $\hsigma_0^2(d-r)\beta_2 =
    (\hsigma_0^2r+\hsigma_2^2)q\beta_1$,
    the equation \eqref{eq:beta1-clean-hsigma} becomes
    $\frac{d}{dt}\beta_1
    =
    4\beta_1\Bigl(
    (\hsigma_0^2r+\hsigma_2^2)(\beta_1^*-(1+q)\beta_1)
    +\varepsilon_1
    \Bigr)$.
    Hence
    \begin{equation}
    \frac{d}{dt}y
    =
    -4(\hsigma_0^2r+\hsigma_2^2)\beta_1^*\,y
    +
    4(\hsigma_0^2r+\hsigma_2^2)q
    -
    4\varepsilon_1\Bigl(y+\frac1{\beta_1^*}\Bigr)
    \le
    -\alpha y
    +
    4(\hsigma_0^2r+\hsigma_2^2)q
    +
    C\frac{\log d}{r},
    \qquad
    \alpha:=(\hsigma_0^2r+\hsigma_2^2)\beta_1^*\asymp r.
    \label{eq:y-exact-hsigma}
    \end{equation}
    where we use \(y\ge 0\), \(|\varepsilon_1|\le K_0\frac{\log d}{r}\), and
    \((\hsigma_0^2r+\hsigma_2^2)\beta_1^*\asymp r\), for all sufficiently large
    \(d\) so we may absorb the term \(2|\varepsilon_1|y\) into the damping term.
    By variation of constants,
    \begin{align}
    y(t)
    &\le
    e^{-\alpha t}y(0)
    +
    4(\hsigma_0^2r+\hsigma_2^2)
    \int_0^t e^{-\alpha(t-s)}q(s)\,ds
    +
    C\frac{\log d}{r}\int_0^t e^{-\alpha(t-s)}\,ds.
    \label{eq:y-voc-hsigma}
    \end{align}
    We estimate the three terms on the right-hand side.
    
    First, by \eqref{eq:q-two-sided-speed-hsigma},
    $q(s)\le q(t)e^{M(t-s)}$,
    so
    \begin{align*}
    4(\hsigma_0^2r+\hsigma_2^2)
    \int_0^t e^{-\alpha(t-s)}q(s)\,ds
    &\le
    4(\hsigma_0^2r+\hsigma_2^2)q(t)
    \int_0^t e^{-(\alpha-M)(t-s)}\,ds
    \le
    \frac{4(\hsigma_0^2r+\hsigma_2^2)}{\alpha-M}q(t)
    \le
    Cq(t),
    \end{align*}
    since \(\alpha\asymp r\) while \(M=O(1)\).
    
    Second, \(y(0)\asymp d/r\asymp q(0)\), while
    \eqref{eq:q-two-sided-speed-hsigma} implies
    \(q(t)\ge q(0)e^{-Mt}\). Therefore
    \[
    e^{-\alpha t}y(0)
    \le
    Ce^{-(\alpha-M)t}q(t)
    \le
    Cq(t).
    \]
    
    Third,
    \[
    C\frac{\log d}{r}\int_0^t e^{-\alpha(t-s)}\,ds
    \le
    C\frac{\log d}{r}\cdot \frac1\alpha
    \le
    C\frac{\log d}{r^2}.
    \]
    
    Substituting these estimates into \eqref{eq:y-voc-hsigma} yields
    $y(t)\le Cq(t)+C\frac{\log d}{r^2}$.
    Since
    $
    \beta_1^*-\beta_1
    =
    \beta_1\beta_1^*
    \left(\frac1{\beta_1}-\frac1{\beta_1^*}\right)
    =
    \beta_1\beta_1^* y
    \le y$,
    we obtain \eqref{eq:beta1-gap-hsigma}.
    
    \medskip
    \noindent\textbf{Step 4: Control of \(\beta_2\) and the convergence time.}
    By the definition of \(q\),
    \[
    \beta_2(t)
    =
    \frac{\hsigma_0^2r+\hsigma_2^2}{\hsigma_0^2(d-r)}\,q(t)\beta_1(t)
    \le C\frac{r}{d}q(t)\le Ce^{-\rho t},
    \]
    where we use \(\beta_1(t)\le 1\) and \(d\gg r\).
    This proves \eqref{eq:beta2-decay-hsigma}.
    
    Since \(q(T_{\mathrm{conv}})=r^{-1}\), \eqref{eq:q-decay-hsigma} gives
    $\frac1r=q(T_{\mathrm{conv}})
    \le
    C\frac{d}{r}e^{-\rho T_{\mathrm{conv}}}$,
    hence \(T_{\mathrm{conv}}=O(\log d)\).
    
    At time \(T_{\mathrm{conv}}\), we have \(q(T_{\mathrm{conv}})=r^{-1}\), so
    \eqref{eq:beta1-gap-hsigma} yields
    \[
    \beta_1^*-\beta_1(T_{\mathrm{conv}})
    \le
    \frac{C}{r}
    +
    C\frac{\log d}{r^2}
    \le
    \frac{C}{r}.
    \]
    Also,
    \[
    \beta_2(T_{\mathrm{conv}})
    =
    \frac{\hsigma_0^2r+2\hsigma_2^2}{\hsigma_0^2(d-r)}
    \,q(T_{\mathrm{conv}})\beta_1(T_{\mathrm{conv}})
    \le
    C\frac{r}{d}\cdot \frac1r
    =
    \frac{C}{d}.
    \]
    
    \medskip
    \noindent\textbf{Intermediate time before convergence.}
    Consider $\delta=\delta(d)\to0$ and $\delta\log d\gg \log r$. Define
    $t_\delta:=(1-\delta)T_{\rm conv}$.
    We claim that
    \[
    \beta_1(t_\delta)=o(1)\beta_1^*,
    \qquad
    \beta_2(t_\delta)\asymp \frac{r}{d}.
    \]
    
    Indeed, from Step 2, there exist constants \(0<c<C\) such that
    $\frac{d}{dt}\log q(t)=-\Theta(1)$ for $0\le t\le T_{\rm conv}$.
    Since $T_{\rm conv}=\Theta(\log d)$, we have
    \[
    \log q(t_\delta)
    \ge
    \log q(T_{\rm conv})+c\delta T_{\rm conv}
    =
    -\log r+\Theta(\delta\log d)
    \gg \log r,
    \]
    where we use the choice of \(\delta\). Hence $q(t_\delta)\to\infty$.
    
    It remains to translate this into a bound on \(\beta_1,\beta_2\). As in Step 3, recall $y(t):=\frac1{\beta_1(t)}-\frac1{\beta_1^*}$.
    The same variation-of-constants argument, applied also to the lower
    differential inequality, gives
    $y(t)\asymp q(t)+O\left(\frac{\log d}{r^2}\right)$
    for $0\le t\le T_{\rm conv}$.
    Therefore \(\log q(t_\delta)\gg \log r\) implies \(y(t_\delta)\gtrsim r^3\), and thus
    \[
    \frac{\beta_1(t_\delta)}{\beta_1^*}
    =
    \frac{1}{1+\beta_1^*y(t_\delta)}
    \lesssim \frac{1}{r^3},\qquad
    \beta_2(t_\delta)
    =
    \frac{\hsigma_0^2r+\hsigma_2^2}{\hsigma_0^2(d-r)}
    q(t_\delta)\beta_1(t_\delta)
    \asymp \frac{r}{d}.
    \]
    where we use \(q(t_\delta)\beta_1(t_\delta)\asymp y(t_\delta)\beta_1(t_\delta)\asymp 1\). 
    Thus by time \((1-\delta)T_{\rm conv}\), \(\beta_1\) is still
    \(\BigO{\frac{1}{r^3}}\beta_1^*\), while \(\beta_2\) has only changed by a constant factor.
    
    \medskip
    \noindent\textbf{Step 5: The integral bound and the sharp leading constant.}
    Integrating \eqref{eq:logq-exact-hsigma} from \(0\) to \(t\le T_{\mathrm{conv}}\),
    we obtain
    \begin{equation}
    4\hsigma_2^2
    \int_0^t(1-\beta_1(s)+\beta_2(s))\,ds
    =
    \log\frac{q(0)}{q(t)}
    +
    2\int_0^t(\varepsilon_2(s)-\varepsilon_1(s))\,ds.
    \label{eq:int-exact-hsigma}
    \end{equation}
    Therefore
    \begin{align*}
    \int_0^t (1-\beta_1(s))\,ds
    &=
    \int_0^t (1-\beta_1(s)+\beta_2(s))\,ds
    -
    \int_0^t \beta_2(s)\,ds
    \\
    &\le
    \frac{1}{4\hsigma_2^2}\log\frac{q(0)}{q(t)}
    +
    \frac{1}{2\hsigma_2^2}
    \int_0^t \bigl(|\varepsilon_1(s)|+|\varepsilon_2(s)|\bigr)\,ds
    \\
    &\le
    \frac{1}{4\hsigma_2^2}\log\frac{q(0)}{q(t)}
    +
    Ct\frac{\log d}{r},
    \end{align*}
    which proves \eqref{eq:int-any-time-hsigma}.
    
    Now take \(t=T_{\mathrm{conv}}\). Since \(q(T_{\mathrm{conv}})=r^{-1}\),
    \eqref{eq:int-exact-hsigma} yields
    \[
    \int_0^{T_{\mathrm{conv}}}(1-\beta_1(s))\,ds
    =
    \frac{1}{4\hsigma_2^2}\log\bigl(q(0)r\bigr)
    +
    O\!\left(T_{\mathrm{conv}}\frac{\log d}{r}\right)
    -
    \int_0^{T_{\mathrm{conv}}}\beta_2(s)\,ds
    \le \frac{1}{4\sigma_2^2}\log d + \BigO{1}
    \]
    where we use \eqref{eq:beta2-decay-hsigma}, \(T_{\mathrm{conv}}=O(\log d)\), and \(\log(q(0)r)=\log d+O(1)\). 
\end{proof}

\section{Stage 2: learn target directions (high-order term)}\label{sec: stage 2}
The goal of this section is to show the lemma below: we learn every ground-truth direction $\vw_i^*$, and decrease the loss to $\eps$ at the end. 
\begin{lemma}[Main result for Stage 2]\label{lem: stage 2 main}
    Within time $T_2-T_1\lesssim r/\log m$, we have
    loss $\L(\tldmW)\le \epsnm^{1/2}\le e^{-\log^2 r}$.
    On the other hand, for $t\le (1-\delta)T_2$, we have loss $\L(\tldmW)\ge c(\hsigma_0^2r^2/2 + \hsigma_2^2r/2) + (1\pm c)\hsigma_{\ge4}^2r/2$, where $c$ is a small constant and $\delta=1/\sqrt{\log d}=o(1)$.
\end{lemma}
\begin{proof}
    It follows from \cref{lem: stage 2 IH}, Stage 2.1 in \cref{lem: stage 2.1 main}, Stage 2.2 in \cref{lem: stage 2.2 main}, Stage 2.3 in \cref{lem: stage 2.3 main}. For the loss bound, we note that Stage 2.1 takes $\BigO{r/\log m}$ time, which takes most of time during this stage, and loss only decreases until Stage 2.2.
\end{proof}

In the rest of this section, we describe and present the proof of the above result. First recall few notations that we divided the neurons in to several groups.

Recall that we introduced the set of potential neurons $\Spot:=\cup_{i\in[r]}\Spoti$, which consists of neurons that may converge to the ground-truth directions. Specifically,
\begin{align*}
    \Spoti
    :=
    \left\{
        \vv:
        \left(\bbvvler^{\btt{T_1}}\right)_i^{2}
        \ge
        (1-\omega^{-0.01})
        \max_{\vu}
        \left(\bbvuler^{\btt{T_1}}\right)_i^{2}
    \right\},
\end{align*}
where $\omega=\log\log\log r$ is used to absorb sufficiently large constants. Although we do not show that every neuron in $\Spoti$ converges to the direction $\ve_i$, we show that at least one such neuron does; see \cref{lem: stage 2.1 neuron main result}.

Let $\Sdense$ denote the remaining neurons, i.e., the complement of $\Spot$. We refer to these neurons as small-and-dense neurons, since we will show that their norms remain small and that
$\norm{\bbvvler}_\infty \lesssim \frac{\log d}{r}$;
see \cref{lem: stage 2 IH}\ref{item: stage 2 IH neuron dir}.

Finally, let $G_i$ be the set of neurons that have learned the $i$-th ground-truth direction by time $T_{21}$, the end of Stage 2.1. Let $\ha_i$ denote their total mass:
\[
    \ha_i
    :=
    \frac{1}{m}\sum_{\vv\in G_i}\norm{\vv}_2^2,
    \qquad
    G_i
    :=
    \left\{
        \vv:
        \left(\bv_i^{\btt{T_{21}}}\right)^2
        \ge
        1-\epsdir
    \right\}.
\]
Here we choose $\epsdir=m^{-\Theta(1)}$, with the constant in the exponent sufficiently small. We will show that neurons in $G_i$ remain aligned with the ground-truth direction throughout the subsequent dynamics; see \cref{lem: stage 2 IH}\ref{item: stage 2 IH stability G_i}. The fixed time $T_{21}$ is used only for technical convenience, so that the set $G_i$ does not change with time.

\paragraph{Choice of parameters in Stage 2.}
We recall the parameters from \cref{sec: omit proof main result} that are relevant to Stage 2. The threshold $\sigma_1^2$ is used to identify aligned neurons in Stage 2.1, and we set $\sigma_1=m^{\Theta(1)}$ with a sufficiently small constant in the exponent. Let $\epsdir$ and $\epsnm$ denote the direction and norm errors, respectively. We choose
$\epsdir=\frac{\omega r}{\sigma_1^2}=m^{-\Theta(1)}$, $\epsnm=e^{-\polylog(r)}$.
As above, $\omega=\log\log\log r$ is used to absorb sufficiently large constants throughout the proof.

The weight decay remains fixed, $\lambda_t=\lambda$, until Stage 2.3. The learning-rate multiplier is $\gamma_t=1$ throughout Stage 2.1. During a short interval in Stage 2.2, when a neuron has sufficiently large norm, we instead set $\gamma_t=\gamma=m$. More precisely, for a short time interval $T_{21}\le t\le T_\gamma$, we set $\gamma_t=\gamma$ whenever $\norm{\vv}\ge \sigma_1^2$. We discuss this choice in more detail in \cref{remark: step size} and \cref{sec: stage 2.2}.

\paragraph{Organization of the rest of the section.}
Stage 2 is divided into three substages.
\begin{itemize}
    \item \textbf{Stage 2.1.}
    This stage takes $\BigO{r/\log m}$ time, which accounts for most of Stage 2.
    During this stage, we show that for every direction $i\le r$, there exists at least one potential neuron $\tldvv\in\Spoti$ whose correlation with the ground-truth direction is at least $1-\epsdir^2$, i.e., $\btldv_i^2 > 1-\epsdir^2$.
    Thus, the direction $\ve_i$ is learned. At the same time, the directions of all other neurons remain almost unchanged. Although the directions are recovered in this stage, the corresponding norms have not yet grown significantly. See \cref{sec: stage 2.1} for details.

    \item \textbf{Stage 2.2.}
    This stage is much shorter than Stage 2.1, taking at most
    $\BigO{\log(1/\epsnm)}=\BigO{\log r}$ time.
    During this stage, the norms of the neurons that have learned the ground-truth directions grow rapidly, so that their total mass reaches $\ha_*$. See \cref{sec: stage 2.2} for details.

    \item \textbf{Stage 2.3.}
    In this stage, we set the weight decay to zero, $\lambda_t=0$, to allow full recovery of the target $f_*$. This stage takes only
    $\BigO{\frac{\log(1/\epsnm)}{r}}$ time. Similar to Stage 2.2, the dynamics is driven by the continued growth of the norms of the aligned neurons. See \cref{sec: stage 2.3} for details.
\end{itemize}

\subsection{Induction Hypothesis and Dynamics of Stage 2}
We present the induction hypothesis that we assume hold throughout Stage 2.
\begin{lemma}[Induction Hypothesis for Stage 2]\label{lem: stage 2 IH}
    Denote
    \begin{enumerate}
        \item \label{item: stage 2 IH neuron dir}
        Direction of small and dense neuron barely moves: For any time $T_1\le t\le T_{dir,2}$ we have
        \begin{itemize}
            \item for $i> r$, $\left(\bbtldvvgtr\right)_i^2 \le \frac{\log^2 d}{d}$
            \item for $i\le r$ and any neuron $\tldvv\not\in \Spoti$,  $\left(\bbtldvvler\right)_i^2 \le \omega\left(\bbtldvvler^\btt{0}\right)_i^{2} $
        \end{itemize}
        where $\omega=\log\log\log r$. 

        \item \label{item: stage 2 IH small dense neuron norm}
        Norm of small and dense neuron remains balance between them: For any time $T_1\le t\le T_{norm,sd,2}$ we have for any neuron $\tldvv,\tldvu\in \Sdense$, 
        \[\frac{\norm{\tldvvler}_2}{\norm{\tldvuler}_2} = \frac{\norm{\tldvvler^\btt{0}}_2}{\norm{\tldvuler^{\btt{0}}}_2}\left(1\pm \delta_{norm,sd,2}\right),\quad \frac{\norm{\tldvvgtr}_2}{\norm{\tldvugtr}_2} = \frac{\norm{\tldvvgtr^\btt{0}}_2}{\norm{\tldvugtr^{\btt{0}}}_2}\left(1\pm \delta_{norm,sd,2}\right)\]
        where $\delta_{norm,sd,2}\le \omega$.

        \item \label{item: stage 2 IH norm bound or basis like}
        Norm of potential neuron either bounded or become basis-like: For any time $T_1\le t\le T_{norm,pot,2}$, we have
        for $i\in[r]$, if neuron $\tldvv \in \Spoti$, then 
        \begin{align*}
                &\text{either $\norm{\tldvv}_2 \le \sigma_1$
                or $\btldv_i^2 \ge 1-\epsdir$}, 
        \end{align*}
        where $\epsdir=\frac{\omega r}{\sigma_1^2}$.

        \item \label{item: stage 2 IH stability G_i}
        Stability of aligned neuron: For any time $T_{21}\le t\le T_{stab,2}$ we have for any $\tldvv\in G_i$ and $i\in[r]$
        \[
            \btldv_i^2 \ge 1-\epsdir.
        \]

        \item \label{item: stage 2 IH reg}
        Regularity condition: for any time $T_1\le t\le T_{reg,2}$, we have 
        \begin{enumerate}
            \item $\tldalpha \le \alpha+10 = r+10$
            \item $\ha_i \le 1$ for $i\le r$
            \item $\frac{\norm{\tldvvler}}{\norm{\tldvv}}\ge 1-\frac{\omega^5}{r}$ for all neuron $\tldvv$
            \item $2\ge \tldE w_i^2 \ge 1/2$ for $i\in[r]$
            \item $\tldE \norm{\vwgtr}^2 \le \tldE \norm{\vwgtr^\btt{T_1}}^2\le \omega$
        \end{enumerate}        
        where recall $\ha_i = \frac{1}{m}\sum_{\vv\in G_i}\norm{\vv}_2^2$ and $G_i=\{\vv: \bv_i^2 \ge 1-\epsdir\}$
    \end{enumerate}
    Then, $T_{2}\le \tldT_2:=\min\{T_{dir,2},T_{norm,sd,2},T_{norm,pot,2}, T_{stab,2}, T_{reg,2}\}$. That is, current stage ends before any of the above condition breaks.
\end{lemma}
\begin{proof}
    It follows from \cref{sec: proof stage 2.1 IH} for Stage 2.1, \cref{sec: proof stage 2.2 IH} for Stage 2.2 and \cref{sec: stage 2.3} for Stage 2.3.
\end{proof}

Below are direct implications of induction hypothesis \cref{lem: stage 2 IH}. We collect together here for the ease of use in the later proof.
\begin{lemma}\label{lem: stage 2 IH implication}
    For $T_1\le t \le \tldT_{2}$, we have for any neuron $\tldvv$ and $i\in[d]$
    \begin{enumerate}
        \item \label{item: stage 2 IH implication small and dense norm}
        for $\tldvv\in \Sdense$, $\norm{\tldvv}_2\le 2\omega\sqrt{r}$. Thus, the step size $\gamma_t=1$ for all $\tldvv\in \Sdense$.
        
        \item \label{item: stage 2 IH implication small and dense dir} 
        for $i\in[r]$ and $\tldvv\not\in \Spoti$, $\btldv_i^2 \le \left(\bbtldvvler\right)_i^2\le\omega\left(\bbtldvvler^\btt{0}\right)_i^{2}
        \lesssim \frac{\omega\log m}{r}$. As a result, $\norm{\bbtldvvler}_\infty\lesssim \frac{\omega\log m}{r}$ for $\tldvv \in \Sdense$
        
        \item \label{item: stage 2 IH implication sum v4}
        for $\tldvv \in \Sdense$, $\sum_{i\in[r]}\btldv_i^4 \le \sum_{i\in[r]}(\bbtldvvler)_i^4
        \lesssim \frac{\omega^2}{r}$; 
        for $\tldvv\in \Spoti$, $\sum_{j\in[r]}\btldv_j^4 \le \sum_{j\in[r]}(\bbtldvvler)_j^4
        \lesssim \btldv_i^4 + \frac{\omega^2}{r}$.

        \item Population 2nd order term bound:
        we have $\tldE w_i^2=\tldE w_j^2 \le \frac{2\omega}{d}$ for any $i,j>r$.
    \end{enumerate}
\end{lemma}
\begin{proof}
    We show one by one. For simplicity in the proof, we drop $\sim$ and write $\tldv$ as $v$.

    \paragraph{item (a)}
    By \cref{lem: stage 2 IH}\ref{item: stage 2 IH small dense neuron norm} and \cref{lem: init} we know for any neuron $\vv,\vu\in\Sdense$ we have $\frac{\norm{\vvler}_2}{\norm{\vuler}_2} \le 1+\delta_{norm,sd,2}\le 1+\omega$. Thus, for $T_1 \le t\le \tldT_{2}$ and any neuron $\vv\in\Sdense$, we have $\norm{\vvler}_2^2 \le (1+\omega)^2\tldE\norm{\vw}_2^2 = (1+\omega)^2\tldalpha \le (1+\omega)^2 2r$, where we use \cref{lem: stage 2 IH}\ref{item: stage 2 IH reg}. Similar argument apply to $\norm{\vvgtr}$.

    \paragraph{item (b)}
    This follows by using \cref{lem: stage 2 IH}\ref{item: stage 2 IH neuron dir} and \cref{lem: init}.
    
    \paragraph{item (c)}
    By item(b) we know for any neuron $\vv\in\Sdense$, we have 
    \[
        \sum_{i\le r}\bv_i^4 
        \le \sum_{i\in[r]}(\bbtldvvler)_i^4
        \le \omega^2 \sum_{i\le r}(\bbvvler^\btt{0})_i^4
        \lesssim \frac{\omega^2}{r}    
    \]
    where at the end we use \cref{lem: init}. Similar argument also applies to $\vv\in\Spoti$.

    \paragraph{item (d)}
    For $i,j>r$, by the definition of $\tldf$ that it is invariant to the permutation in the last $d-r$ coordinates, we know $\tldE w_i^2=\tldE w_j^2$. Thus, $\frac{\tldE w_i^2}{\sum_{i> r}\tldE w_i^2} =\frac{1}{d-r}$, and $\tldE w_i^2 = \frac{1}{d-r}\sum_{i> r}\tldE w_i^2 \le \frac{\omega}{d-r}$ for $i> r$.

\end{proof}

\subsubsection{Dynamics}\label{sec: stage 2 dynamics}
The lemmas below describe the dynamics of the small-and-dense neuron, the potential neuron, and the population dynamics, respectively. They are better viewed as meta-lemmas that will be further simplified in later analysis. 

The proofs are calculation-intensive. They primarily use the exact gradient formula in \cref{claim: tldw dynamic}, bound the contributions from terms of order higher than four, and retain only the leading low-order terms.
\begin{lemma}[Dynamics of small-and-dense neuron in Stage 2]\label{lem: stage 2 dynamic small and dense}
    In Stage 2, for $T_{12}\le t\le \tldT_{2}$, the dynamic~\eqref{eq: dynamic of tldw} for small-and-dense neuron can be simplified as: for any neuron $\tldvv\in\Sdense$ and $i\in[d]$
    \begin{align*}
        \frac{\rd \tldv_i}{\rd t}
        =&  \tldv_i \left(
                -\lambda
                + 2\hsigma_0^2 \left(
                r - \tldE[ \norm{\vw}_2^2 ]\right) 
                + 2\hsigma_{2}^2 \left(
                    \ind_{i \le r}
                    - \tldE\left[ w_i^2 \right]
                \right)
                + 4\hsigma_4^2 \btldv_i^2\ind_{i\le r}
                \pm \BigO{\frac{\omega^4}{r}}
            \right)
            \pm \BigO{\frac{\sigma_1^2}{m^{0.9}}+\sqrt{\epsdir}},\\
        \frac{\rd}{\rd t}\norm{\tldvvler}_2
        =& \norm{\tldvvler}_2 \left(
            -\lambda
            + 2\hsigma_0^2\left(r-\tldE\norm{\vw}_2^2\right)
            + 2\hsigma_2^2\sum_{i\le r}(\bbtldvvler)_i^2 
                \left( 1 - \tldE\left[w_i^2\right]\right)
            \pm \BigO{\frac{\omega^4}{r}}
        \right),\\
        \frac{\rd}{\rd t}\norm{\tldvv}_2
        =& \norm{\tldvv}_2 \left(
            -\lambda
            + 2\hsigma_0^2\left(r-\tldE\norm{\vw}_2^2\right)
            + 2\hsigma_2^2\sum_{i\le r}\btldv_i^2 
                \left( 1 - \tldE\left[w_i^2\right]\right)
            \pm \BigO{\frac{\omega^4}{r}}
        \right),\\
        \frac{\rd (\bbtldvvler)_i^2}{\rd t}
        \le& 2(\bbtldvvler)_i^2 \left( 2\hsigma_{2}^2 \left(
                \ind_{i\le r} - \tldE\left[ w_i^2 \right]
                - \sum_{j\le r}(\bbtldvvler)_j^2 
                \left( 1 - \tldE\left[w_j^2\right]\right)\right)
                + 4\hsigma_4^2 \btldv_i^2
                \pm \BigO{\frac{\omega^4}{r}}
                \right).
    \end{align*}
\end{lemma}
\begin{proof}
    For simplicity in the proof, we drop $\sim$ and write $\tldv$ as $v$.
    By \cref{lem: stage 2.1 IH implication} we know step size $\gamma_t=1$. From \eqref{eq: dynamic of tldw}, we know
    \begin{align*}
        \frac{\rd v_i}{\rd t}
        =& - \sum_{j\ge 0} [\nabla_{2j,\vv}]_i -\lambda v_i, 
    \end{align*}
    where $\nabla_{2j,\vv}$ represents the gradient of $2j$-th order loss
    \begin{align*}
        -[\nabla_{0,\vv}]_i 
        =& 2\hsigma_0^2 \left(
                \sum_{j=1}^r a_j^* \norm{\vw_j^*}_2
                - \tldE[ \norm{\vw}_2^2 ]\right) v_i
        = 2\hsigma_0^2\left(r-\tldE\norm{\vw}_2^2\right) v_i,\\
        -[\nabla_{2,\vv}]_i 
        =& \hsigma_{2}^2 \left(
            2 \ind_{i \le r}
            - 2 \tldE\left[ w_i^2 \right]
            \right) v_i,\\
        -[\nabla_{2j,\vv}]_i 
        =& \hsigma_{2j}^2 \Big(
            2j \bv_i^{2j-2} \ind_{i \le r}
            - 2j \tldE\left[
                \norm{\vw}_2^2 \langle\bvw,\bvv\rangle^{2j-1} \bw_i/\bv_i
                \right]
            - (2j-2) \left(
            \sum_{l\le r}\bv_l^{2j} - 
            \tldE \left[\norm{\vw}_2^2\langle \bvw,\bvv\rangle^{2j}\right]
            \right)\Big) v_i,\quad \text{for $j\ge 2$.}
    \end{align*}    

    \paragraph{Dynamic of $v_i$.}
    We first bound the each term in $\nabla_{2j,\vv}$ for $j\ge 2$. 
    
    We have for 
    \begin{align*}
        \bv_i^{2j-2}\ind_{i \le r}
        \le& \bv_i^4\ind_{i \le r} 
        \mylesim{a} \frac{\omega^2\log^2 d}{r^2}
        ,\text{ for $j\ge 3$},
        \qquad
        \sum_{l\le r}\bv_l^{2j}
        \le \sum_{j\le r}\bv_j^{4}
        \mylesim{b} \frac{\omega^2}{r},\text{ for $j\ge 2$}
    \end{align*}
    where (a)(b) we use \cref{lem: stage 2 IH implication}\ref{item: stage 2 IH implication small and dense dir}\ref{item: stage 2 IH implication sum v4}.
    
    By \cref{lem: stage 2 dynamic bound} we have 
    \begin{align*}
        \tldE \left[\norm{\vw}_2^2\langle \bvw,\bvv\rangle^{2j}\right] 
        \lesssim& \frac{\omega^4}{r}, \text{ for $j\ge 2$},
        \quad
        \tldE [\norm{\vw}_2^2\langle\bvw,\bvv\rangle^{2j-1}\bw_i/\bv_i]
        \left\{\begin{array}{ll}
          = \ha_i \bv_i^2\ind_{i\le r} \pm \BigO{\frac{\omega^4}{r}}   & \text{, $j=2$} \\
          \lesssim \frac{\sigma_1^2}{m^{0.9}} \frac{1}{|\bv_i|} + \frac{\sqrt{\epsdir}}{|\bv_i|} +  \frac{\omega^4}{r}  & \text{, $j\ge 3$}
        \end{array}\right.
    \end{align*}
    
    Now  we have the bound for each term. Together with \cref{claim: hermite coeff} that $\sum_j 2j\hsigma_{2j}^2 \lesssim 1$ we have
    \begin{align*}
        \frac{\rd v_i}{\rd t}
        =&  v_i \left(
                -\lambda
                +2\hsigma_0^2 \left(
                r - \tldE[ \norm{\vw}_2^2 ]\right) 
                + 2\hsigma_{2}^2 \left(
                    \ind_{i \le r}
                    - \tldE\left[ w_i^2 \right]
                \right)
                + 4\hsigma_4^2 \bv_i^2\ind_{i\le r}
                \pm \BigO{\frac{\omega^4}{r}}
            \right)
            \pm \BigO{\frac{\sigma_1^2}{m^{0.9}}+\sqrt{\epsdir}}.
    \end{align*}
    
    \paragraph{Dynamic of $\norm{\vvler}$ and $\norm{\vv}$.}
    From the gradient form at beginning, we have
    \begin{align*}
        &\frac{\rd}{\rd t}\norm{\vvler}_2^2
        = 2\sum_{i\in[r]} v_i \frac{\rd v_i}{\rd t}\\
        =& 2\norm{\vvler}_2^2 \Bigg(
            -\lambda
            + 2\hsigma_0^2\left(r-\tldE\norm{\vw}_2^2\right)
            + 2\hsigma_2^2\sum_{i\le r}(\bbvvler)_i^2 
                \left( 1 - \tldE\left[w_i^2\right]\right)\\
            &+ \sum_{j\ge 2}\hsigma_{2j}^2 \Bigg(
                2j \sum_{i\in[r]}\bv_i^{2j-2}(\bbvvler)_i^2
                - 2j \tldE\left[
                    \norm{\vw}_2^2 \langle\bvw,\bvv\rangle^{2j-1} \sum_{i\in[r]}\bw_i\bv_i
                    \frac{\norm{\vv}^2}{\norm{\vvler}^2}
                    \right]
                - (2j-2) \left(
                    \sum_{l\le r}\bv_l^{2j} 
                    - \tldE \left[\norm{\vw}_2^2\langle \bvw,\bvv\rangle^{2j}\right]
                    \right)
                \Bigg)
        \Bigg).\\
    \end{align*}

    We are going to bound the last line above.
    Similar as in item 1, for $j\ge 2$ we have
    \begin{align*}
        &\sum_{i\in[r]}\bv_i^{2j-2}(\bbvvler)_i^2
        \le \sum_{i\in[r]}(\bbvvler)_i^4
        \mylesim{a} \frac{\omega^2}{r},\quad
        \tldE\left[
                    \norm{\vw}_2^2 \langle\bvw,\bvv\rangle^{2j-1} \sum_{i\in[r]}\bw_i\bv_i
                    \frac{\norm{\vv}^2}{\norm{\vvler}^2}
                    \right]
        \myle{b} 4 \tldE \left[\norm{\vw}_2^2\langle \bvw,\bvv\rangle^{2j}\right]
        \mylesim{c} \frac{\omega^4}{r}\\
        &\sum_{l\le r}\bv_l^{2j}
        \le \sum_{j\le r}\bv_j^{4}
        \mylesim{a} \frac{\omega^2}{r},\quad
        \tldE \left[\norm{\vw}_2^2\langle \bvw,\bvv\rangle^{2j}\right] 
        \mylesim{c} \frac{\omega^4}{r},
    \end{align*}
    where (a) \cref{lem: stage 2 IH implication}\ref{item: stage 2 IH implication sum v4};
    (b) \cref{lem: stage 2 IH}\ref{item: stage 2 IH reg} and $\tldmu$ is conditional symmetric so $\tldE\left[
                    \norm{\vw}_2^2 \langle\bvw,\bvv\rangle^{2j-1} \bw_i\bv_i\right] \ge 0$ for $i\in[d]$;
    (c) \cref{lem: stage 2 dynamic bound}.

    Thus, using \cref{claim: hermite coeff} that $\sum_j 2j\hsigma_{2j}^2 \lesssim 1$ we have
    \begin{align*}
        \frac{\rd}{\rd t}\norm{\vvler}_2
        =& \norm{\vvler}_2 \left(
            -\lambda
            + 2\hsigma_0^2\left(r-\tldE\norm{\vw}_2^2\right)
            + 2\hsigma_2^2\sum_{i\le r}(\bbvvler)_i^2 
                \left( 1 - \tldE\left[w_i^2\right]\right)
            \pm \BigO{\frac{\omega^4}{r}}
        \right).
    \end{align*}
    The dynamic of $\norm{\vv}$ follows the same argument.

    \paragraph{Dynamic of $(\bbtldvvler)_i^2$.}
    We have
    \begin{align*}
        \frac{\rd (\bbvvler)_i^2}{\rd t}
        =& (\bbvvler)_i^2 \left(\frac{1}{v_i^2}\frac{\rd v_i^2}{\rd t} - \frac{1}{\norm{\vvler}_2^2} \frac{\rd}{\rd t}\norm{\vvler}_2^2\right)
    \end{align*}

    Note from item 1 that the additional $\pm \BigO{\frac{\sigma_1^2}{m^{0.9}}+r\epsdir}$ term appears due to the existence of $\tldE [\norm{\vw}_2^2\langle\bvw,\bvv\rangle^{2j-1}\bw_i/\bv_i]$ in gradient. By \cref{claim: tldw dynamic} we know $\tldE [\norm{\vw}_2^2\langle\bvw,\bvv\rangle^{2j-1}\bw_i/\bv_i]$ is non-negative, therefore
    \begin{align*}
        \frac{\rd v_i^2}{\rd t}
        \le&  2v_i^2 \left(
                -\lambda
                + 2\hsigma_0^2 \left(
                r - \tldE[ \norm{\vw}_2^2 ]\right) 
                + 2\hsigma_{2}^2 \left(
                    \ind_{i \le r}
                    - \tldE\left[ w_i^2 \right]
                \right)
                + 4\hsigma_4^2 \bv_i^2\ind_{i\le r}
                \pm \BigO{\frac{\omega^4}{r}}
            \right).
    \end{align*}
    This further give the desired bound
    \begin{align*}    
        \frac{\rd (\bbvvler)_i^2}{\rd t}
        \le& 2(\bbvvler)_i^2 \left( 2\hsigma_{2}^2 \left(
                \ind_{i\le r} - \tldE\left[ w_i^2 \right]
                - \sum_{j\le r}(\bbvvler)_j^2 
                \left( 1 - \tldE\left[w_j^2\right]\right)\right)
                + 4\hsigma_4^2 \bv_i^2
                \pm \BigO{\frac{\omega^4}{r}}
                \right).
    \end{align*}

\end{proof}

\begin{lemma}[Dynamics of potential neuron in Stage 2]\label{lem: stage 2 dynamic potential neuron}
    In Stage 2, for $T_{12}\le t\le \tldT_{2}$, the dynamic~\eqref{eq: dynamic of tldw} for potential neuron can be simplified as: for any neuron $\tldvv\in\Spotk$ and $k\in[r]$ (let $j_0=r^2$)
    \begin{align*}
        \frac{\rd \tldv_i}{\rd t}
        =&  \gamma_t v_i \Bigg(
                -\lambda
                +2\hsigma_0^2 \left(
                r - \tldE[ \norm{\vw}_2^2 ]\right) 
                + 2\hsigma_{2}^2 \left(
                    \ind_{i \le r}
                    - \tldE\left[ w_i^2 \right]
                \right)
                + \hsigma_4^2 \left(
                    4(1-\ha_i)\btldv_i^2\ind_{i\le r}
                    - 2(1-\ha_k)\btldv_k^4
                    \right)\\
                &\quad+ \sum_{3\le j\le j_0}\hsigma_{2j}^2\left(
                    2j(1-\ha_i)\btldv_i^{2j-2}\ind_{i=k}
                    -(2j-2)(1-\ha_k)\btldv_k^{2j}
                    \right) 
                \pm \BigO{\frac{\omega^4}{r}}
            \Bigg)
            \pm \BigO{\frac{\sigma_1^2}{m^{0.9}}+r\sqrt{\epsdir}},\\
        \frac{\rd}{\rd t}\norm{\tldvvler}_2
        =& \gamma_t \norm{\tldvvler}_2 \Bigg(
            -\lambda
            +2\hsigma_0^2\left(r-\tldE\norm{\vw}_2^2\right)
            + 2\hsigma_2^2\sum_{i\le r}(\bbtldvvler)_i^2 
                \left( 1 - \tldE\left[w_i^2\right]\right)\\
            &\qquad+ \sum_{2\le j\le j_0}\hsigma_{2j}^2 \left(
                2j\frac{\norm{\tldvv}^2}{\norm{\tldvvler}^2} - (2j-2)
                \right) 
                (1-\ha_k)\btldv_k^{2j}
            \pm \BigO{\frac{\omega^4}{r}}
        \Bigg),\\
        \frac{\rd}{\rd t}\norm{\tldvv}_2
        =& \gamma_t \norm{\tldvv}_2 \Bigg(
            -\lambda
            +2\hsigma_0^2\left(r-\tldE\norm{\vw}_2^2\right)
            + 2\hsigma_2^2\sum_{i\le r}\btldv_i^2 
                \left( 1 - \tldE\left[w_i^2\right]\right)
            + \sum_{2\le j\le j_0}2\hsigma_{2j}^2 
                (1-\ha_k)\btldv_k^{2j}
            \pm \BigO{\frac{\omega^4}{r}}
        \Bigg),\\
        \frac{\rd (\bbtldvvler)_i^2}{\rd t}
        \le& \gamma_t 2(\bbtldvvler)_i^2 \Bigg( 
                2\hsigma_{2}^2 \left(
                    1 - \tldE\left[ w_i^2 \right]
                    - \sum_{j\le r}(\bbtldvvler)_j^2 
                    \left( 1 - \tldE\left[w_j^2\right]\right)\right)
                + 4\hsigma_4^2 (1-\ha_i)\btldv_i^2\ind_{i\le r}
                \pm \BigO{\frac{\omega^4}{r}}
            \Bigg), \text{ for $i\ne k$},\\
        \frac{\rd (\bbtldvvler)_k^2}{\rd t}
        =& \gamma_t 2(\bbvvler)_1^2 \Bigg( 
                2\hsigma_{2}^2 \left(
                    1 - \tldE\left[ w_i^2 \right]
                    - \sum_{j\le r}(\bbvvler)_j^2 
                    \left( 1 - \tldE\left[w_j^2\right]\right)\right)\\
                &\qquad\qquad +\sum_{2\le j\le j_0}\hsigma_{2j}^2 
                2j(1-\ha_k)\bv_k^{2j-2}(1-(\bbvvler)_k^2)
                \pm \BigO{\frac{\omega^4}{r}}
            \Bigg).
    \end{align*}

    As a result, we also have 
    \begin{align*}
        \frac{\rd \tldv_i^2}{\rd t}
        \le&  \gamma_t 2\tldv_i^2 \Bigg(
                -\lambda
                +2\hsigma_0^2 \left(
                r - \tldE[ \norm{\vw}_2^2 ]\right) 
                + 2\hsigma_{2}^2 \left(
                    \ind_{i \le r}
                    - \tldE\left[ w_i^2 \right]
                \right)
                + 4\hsigma_4^2\btldv_i^2\ind_{i\le r}
                + \BigO{\frac{\omega^4}{r}}
            \Bigg),\text{ for $i\ne k$,}\\
        \frac{\rd \tldv_k^2}{\rd t}
        \ge&  \gamma_t 2\tldv_k^2 \Bigg(
                -\lambda
                +2\hsigma_0^2 \left(
                r - \tldE[ \norm{\vw}_2^2 ]\right) 
                + 2\hsigma_{2}^2 \left(
                    1 - \tldE\left[ w_k^2 \right]
                \right)
                +2\hsigma_4^2 \btldv_k^2
                \pm \BigO{\frac{\omega^4}{r}}
            \Bigg),\\
    \end{align*}
\end{lemma}
\begin{proof}
    For simplicity in the proof, we drop $\sim$ and write $\tldv$ as $v$. WLOG, consider the case of $\vv\in\Spotone$.
    
    From \eqref{eq: dynamic of tldw}, we know
    \begin{align*}
        \frac{\rd v_i}{\rd t}
        =& - \sum_{j\ge 0} [\nabla_{2j,\vv}]_i - \lambda v_i, 
    \end{align*}
    where $\nabla_{2j,\vv}$ represents the gradient of $2j$-th order loss
    \begin{align*}
        -[\nabla_{0,\vv}]_i 
        =& 2\hsigma_0^2 \left(
                \sum_{j=1}^r a_j^* \norm{\vw_j^*}_2
                - \tldE[ \norm{\vw}_2^2 ]\right) v_i
        = 2\hsigma_0^2\left(r-\tldE\norm{\vw}_2^2\right) v_i,\\
        -[\nabla_{2,\vv}]_i 
        =& \hsigma_{2}^2 \left(
            2 \ind_{i \le r}
            - 2 \tldE\left[ w_i^2 \right]
            \right) v_i,\\
        -[\nabla_{2j,\vv}]_i 
        =& \hsigma_{2j}^2 \Big(
            2j \bv_i^{2j-2} \ind_{i \le r}
            - 2j \tldE\left[
                \norm{\vw}_2^2 \langle\bvw,\bvv\rangle^{2j-1} \bw_i/\bv_i
                \right]\\
            &\qquad- (2j-2) \left(
            \sum_{l\le r}\bv_l^{2j} - 
            \tldE \left[\norm{\vw}_2^2\langle \bvw,\bvv\rangle^{2j}\right]
            \right)\Big) v_i,\quad \text{for $j\ge 2$.}
    \end{align*}    

    \paragraph{Dynamic of $v_i$.} 
    We first bound the each term in $\nabla_{2j,\vv}$ for $j\ge 2$. 
    
    We have for 
    \begin{align*}
        \bv_i^{2j-2}
        \myeq{a}\left\{\begin{array}{ll}
           \bv_i^{2}     & \text{, $j=2$} \\
           \bv_i^{2j-2}  & \text{, $i=1$}\\
           \BigO{\frac{\omega^2\log^2 d}{r^2}} & \text{, otherwise}
        \end{array}\right.,\quad
        \sum_{l\le r}\bv_l^{2j}
        = \bv_1^{2j} \pm \sum_{2\le l\le r}\bv_l^{2j}
        \myeq{b} \bv_1^{2j} \pm \BigO{\frac{\omega^2}{r}} & \text{, for $j\ge 2$} 
    \end{align*}
    where (a) we use \cref{lem: stage 2 IH implication}\ref{item: stage 2 IH implication small and dense dir}; 
    (b) \cref{lem: stage 2 IH implication}\ref{item: stage 2 IH implication sum v4}.
    
    By \cref{lem: stage 2 dynamic bound} we have 
    \begin{align*}
        \tldE \left[\norm{\vw}_2^2\langle \bvw,\bvv\rangle^{2j}\right] 
        &= \ha_1\bv_1^{2j} + \BigO{\frac{\omega^4}{r}+jr\sqrt{\epsdir}},\\
        \tldE [\norm{\vw}_2^2\langle\bvw,\bvv\rangle^{2j-1}\bw_i/\bv_i]
        &=\left\{\begin{array}{ll}
        \ha_i\bv_i^2 + \BigO{\frac{\omega^4}{r}} &,\ j=2\\
        \ha_i\bv_i^{2j-2}\ind_{i=1} + \BigO{\frac{\sigma_1^2}{m^{0.9}} \frac{1}{|\bv_i|} + \frac{\sqrt{\epsdir}}{|\bv_i|} +  \frac{\omega^4}{r}+j\sqrt{\epsdir}} &,\ j\ge 3 \\
        \le \ha_i\bv_i^{4}\ind_{i=1} + \BigO{\frac{\sigma_1^2}{m^{0.9}} \frac{1}{|\bv_i|} + \frac{\sqrt{\epsdir}}{|\bv_i|} +  \frac{\omega^4}{r}} &,\ j\ge 3 
        \end{array}\right. .
    \end{align*}

    For $j \ge j_0 = r^2$, we treat all gradients of these higher-order terms as error, using a crude bound from \cref{lem: stage 2 dynamic bound} that does not depend on $j$ (due to the fact that $\sum \hsigma_k^2 k^2 = \infty$)
    \begin{align*}
        \tldE \left[\norm{\vw}_2^2\langle \bvw,\bvv\rangle^{2j}\right] 
        &\le \tldE \left[\norm{\vw}_2^2\langle \bvw,\bvv\rangle^{2j_0}\right]
        \lesssim 1,\\
        \tldE [\norm{\vw}_2^2\langle\bvw,\bvv\rangle^{2j-1}\bw_i/\bv_i]
        &\le \ha_i\bv_i^{4}\ind_{i=1} + \BigO{\frac{\sigma_1^2}{m^{0.9}} \frac{1}{|\bv_i|} + \frac{\sqrt{\epsdir}}{|\bv_i|} +  \frac{\omega^4}{r}}
    \end{align*}

    Now  we have the bound for each term. Thus, we have the bound for $[\nabla_{2j,\vv}]_i$ as (recall we WLOG consider $\vv\in\Spotone$)
    \begin{align*}
        -[\nabla_{4,\vv}]_i
        &= \hsigma_4^2 v_i\left(
            4(1-\ha_i)\bv_i^2\ind_{i\le r}
            - 2(1-\ha_1)\bv_1^4
            \pm \BigO{\frac{\omega^4}{r}}
            \right)\\
        \text{for $3\le j\le j_0$, }
        -[\nabla_{2j,\vv}]_i
        &= \hsigma_{2j}^2 v_i\Bigg(
            2j(1-\ha_i)\bv_i^{2j-2}\ind_{i=1}
            -(2j-2)(1-\ha_1)\bv_1^{2j} \\
            &\qquad\qquad\qquad\pm 2j\BigO{\frac{\omega^4}{r}
            + \frac{\sigma_1^2}{m^{0.9}} \frac{1}{|\bv_i|} + \frac{\sqrt{\epsdir}}{|\bv_i|}
            + j\sqrt{\epsdir}}
            \Bigg)\\
        \text{for $j\ge j_0$, }
        |[\nabla_{2j,\vv}]_i|
        &\le \hsigma_{2j}^2 v_i 2j
            \BigO{1
            + \frac{\sigma_1^2}{m^{0.9}} \frac{1}{|\bv_i|} + \frac{\sqrt{\epsdir}}{|\bv_i|}}
    \end{align*}

    Together all of them we have and using \cref{claim: hermite coeff} that $\sum_{j\le j_0} j\hsigma_{2j}^2 \lesssim 1$, $\sum_{j\le j_0} j^2\hsigma_{2j}^2 \lesssim \sqrt{j_0}=r$, $\sum_{j\ge j_0} 2j\hsigma_{2j}^2 \lesssim 1/\sqrt{j_0}=1/r$, we have
    \begin{align*}
        \frac{\rd v_i}{\rd t}
        =&  v_i \Bigg(
                -\lambda
                +2\hsigma_0^2 \left(
                r - \tldE[ \norm{\vw}_2^2 ]\right) 
                + 2\hsigma_{2}^2 \left(
                    \ind_{i \le r}
                    - \tldE\left[ w_i^2 \right]
                \right)
                + \hsigma_4^2 \left(
                    4(1-\ha_i)\bv_i^2\ind_{i\le r}
                    - 2(1-\ha_1)\bv_1^4
                    \right)\\
                &\quad+ \sum_{3\le j\le j_0}\hsigma_{2j}^2\left(
                    2j(1-\ha_i)\bv_i^{2j-2}\ind_{i=1}
                    -(2j-2)(1-\ha_1)\bv_1^{2j}
                    \right) 
                \pm \BigO{\frac{\omega^4}{r}}
            \Bigg)
            \pm \BigO{\frac{\sigma_1^2}{m^{0.9}}+r\sqrt{\epsdir}}.
    \end{align*}

    \paragraph{Dynamic of $\norm{\vvler}$ and $\norm{\vv}$.}
    Similar as in \cref{lem: stage 2 dynamic small and dense}. From the gradient form at beginning, we have
    \begin{align*}
        &\frac{\rd}{\rd t}\norm{\vvler}_2^2
        = 2\sum_{i\in[r]} v_i \frac{\rd v_i}{\rd t}\\
        =& 2\norm{\vvler}_2^2 \Bigg(
            -\lambda
            +2\hsigma_0^2\left(r-\tldE\norm{\vw}_2^2\right)
            + 2\hsigma_2^2\sum_{i\le r}(\bbvvler)_i^2 
                \left( 1 - \tldE\left[w_i^2\right]\right)\\
            &\quad+ \sum_{j\ge 2}\hsigma_{2j}^2 \left(
                \left(
                    2j\frac{\norm{\vv}^2}{\norm{\vvler}^2} - (2j-2) 
                \right)
                \left(
                    \sum_{l\le r}\bv_l^{2j} 
                    - \tldE \left[\norm{\vw}_2^2\langle \bvw,\bvv\rangle^{2j}\right]
                    \right)
                \right)
        \Bigg).
    \end{align*}

    Similar as in item 1, for $j\ge 2$ we have (recall we WLOG consider $\vv\in\Spotone$)
    \begin{align*}
        \sum_{i\in[r]}\bv_i^{2j}
        &= \bv_1^{2j} + \sum_{2\le i \le r}\bv_i^{2j}
        \myeq{a} \bv_1^{2j} +\BigO{\frac{\omega^2}{r}},\\
        \tldE \left[\norm{\vw}_2^2\langle \bvw,\bvv\rangle^{2j}\right] 
        &\myeq{b} \ha_1\bv_1^{2j} + \BigO{\frac{\omega^4}{r}+jr\sqrt{\epsdir}}
    \end{align*}
    where (a) \cref{lem: stage 2 IH implication}\ref{item: stage 2 IH implication sum v4};
    (b) \cref{lem: stage 2 dynamic bound}.

    Then similar to item 1 we cutoff at $j_0$, using \cref{claim: hermite coeff} that $\sum_j 2j\hsigma_{2j}^2 \lesssim 1$ we have
    \begin{align*}
        \frac{\rd}{\rd t}\norm{\vvler}_2
        =& \norm{\vvler}_2 \Bigg(
            -\lambda
            +2\hsigma_0^2\left(r-\tldE\norm{\vw}_2^2\right)
            + 2\hsigma_2^2\sum_{i\le r}(\bbvvler)_i^2 
                \left( 1 - \tldE\left[w_i^2\right]\right)\\
            &\qquad+ \sum_{2\le j\le j_0}\hsigma_{2j}^2 \left(
                2j\frac{\norm{\vv}^2}{\norm{\vvler}^2} - (2j-2)
                \right) 
                (1-\ha_1)\bv_1^{2j}
            \pm \BigO{\frac{\omega^4}{r}}
        \Bigg).
    \end{align*}
    Similar argument gives the dynamic of $\norm{\vv}$.

    \paragraph{Dynamic of $\bv_i^2$.}
    We have
    \begin{align*}
        \frac{\rd (\bbvvler)_i^2}{\rd t}
        =& (\bbvvler)_i^2 \left(\frac{1}{v_i^2}\frac{\rd v_i^2}{\rd t} - \frac{1}{\norm{\vvler}_2^2} \frac{\rd}{\rd t}\norm{\vvler}_2^2\right)
    \end{align*}

    \noindent\textit{Case 1: $i\ne 1$.}
    Note from item 1 that the additional $\pm \BigO{\frac{\sigma_1^2}{m^{0.9}}+r\sqrt{\epsdir}}$ term appears due to the existence of $\tldE [\norm{\vw}_2^2\langle\bvw,\bvv\rangle^{2j-1}\bw_i/\bv_i]$ in gradient. By \cref{claim: tldw dynamic} we know it is non-negative, therefore
    \begin{align*}
        \frac{\rd (\bbvvler)_i^2}{\rd t}
        \le& 2(\bbvvler)_i^2 \Bigg( 
                2\hsigma_{2}^2 \left(
                    1 - \tldE\left[ w_i^2 \right]
                    - \sum_{j\le r}(\bbvvler)_j^2 
                    \left( 1 - \tldE\left[w_j^2\right]\right)\right)
                + 4\hsigma_4^2 (1-\ha_i)\bv_i^2\ind_{i\le r}\\
                &\qquad\qquad -\sum_{2\le j\le j_0}\hsigma_{2j}^2 
                2j\frac{\norm{\vv}^2}{\norm{\vvler}^2}
                (1-\ha_1)\bv_1^{2j}
                \pm \BigO{\frac{\omega^4}{r}}
            \Bigg)\\
        \le& 2(\bbvvler)_i^2 \Bigg( 
                2\hsigma_{2}^2 \left(
                    1 - \tldE\left[ w_i^2 \right]
                    - \sum_{j\le r}(\bbvvler)_j^2 
                    \left( 1 - \tldE\left[w_j^2\right]\right)\right)
                + 4\hsigma_4^2 (1-\ha_i)\bv_i^2\ind_{i\le r}
                \pm \BigO{\frac{\omega^4}{r}}
            \Bigg),
    \end{align*}
    where in the end we use \cref{lem: stage 2 IH}\ref{item: stage 2 IH reg}. 

    \noindent\textit{Case 2: $i=1$.}
    Using item 1 and item 2 we have
    \begin{align*}
        \frac{\rd (\bbvvler)_1^2}{\rd t}
        =& 2(\bbvvler)_1^2 \Bigg( 
                2\hsigma_{2}^2 \left(
                    1 - \tldE\left[ w_i^2 \right]
                    - \sum_{j\le r}(\bbvvler)_j^2 
                    \left( 1 - \tldE\left[w_j^2\right]\right)\right)\\
                &\qquad\qquad +\sum_{2\le j\le j_0}\hsigma_{2j}^2 
                2j(1-\ha_1)\bv_1^{2j-2}(1-(\bbvvler)_1^2)
                \pm \BigO{\frac{\omega^4}{r}}
            \Bigg)
            \pm \BigO{\frac{\sigma_1^2}{m^{0.9}}+r\sqrt{\epsdir}}
    \end{align*}
\end{proof}

\begin{lemma}[Population 0th and 2nd order term dynamic in Stage 2]\label{lem: stage 2 population 0th 2nd dynamic}
    In Stage 2, for $T_{12}\le t\le \tldT_{2}$, the dynamic for population 0th and 2nd order term can be simplified as: 
    \begin{align*}
        \frac{\rd}{\rd t} \tldE[w_i^2]
        =&  2\tldE[w_i^2] \Bigg(
                -\lambda
                + 2\hsigma_0^2 \left(
                r - \tldE[ \norm{\vw}_2^2 ]\right) 
                + 2\hsigma_{2}^2 \left(
                    \ind_{i \le r}
                    - \tldE\left[ w_i^2 \right]
                \right)                                
                \pm \BigO{\frac{\omega^4}{r}}
                \Bigg)
                + 4\hsigma_{\ge 4}^2(1-\ha_i)\ha_i\\
        \frac{\rd \tldalpha}{\rd t}
        =& 2\tldalpha \Bigg(
                -\lambda
                + 2\hsigma_0^2 \left(
                r - \tldalpha\right) 
                + 2\hsigma_{2}^2 \sum_{i\in[r]}\frac{\tldE[w_i^2]}{\tldE \norm{\vw}^2}\left(
                    1 - \tldE\left[ w_i^2 \right]
                \right)      
                + 2\hsigma_{\ge 4}^2\sum_{i\in[r]}\frac{\ha_i}{\tldE\norm{\vw}^2}(1-\ha_i)
                \pm \BigO{\frac{\omega^4}{r}}
                \Bigg),\\
        \text{for $t\ge T_{21}$, }
        \frac{\rd \ha_i}{\rd t}
        = & 2\ha_i \Bigg(
                -\lambda
                +2\hsigma_0^2 \left(
                    r - \tldE[ \norm{\vw}_2^2 ]\right) 
                + 2\hsigma_{2}^2 \left(
                    1 - \tldE\left[ w_i^2 \right]
                \right)                               
                + 2\hsigma_{\ge4}^2(1-\ha_i)\\
                &\qquad- \underbrace{\sum_{j\ge 2}2\hsigma_{2j}^2\tldE_{\Sdense}\left[\norm{\vw}^2\bw_i^{2j}\right]}_{\BigO{\frac{\omega^4}{r^2}\tldE_{\Sdense} \norm{\vw}^2}}
                \pm \BigO{r\sqrt{\epsdir}}
                \Bigg), 
    \end{align*}
    where $\hsigma_{\ge4}^2 = \sum_{j\ge 2}\hsigma_{2j}^2$.
\end{lemma}
\begin{proof}
    We show one by one. The proof is mostly use of \cref{lem: stage 2 dynamic small and dense} and \cref{lem: stage 2 dynamic potential neuron} and simplify them a bit.

    Recall that all neurons can be partitioned into small-and-dense neurons and potential neurons, i.e., $\Sdense \cup \Spot$. For a potential neuron in $\Spot$, according to \cref{lem: stage 2 IH}\ref{item: stage 2 IH norm bound or basis like}, it either belongs to one of the sets $G_i$ (aligned with the ground-truth vector $\ve_i$), or its norm is upper bounded by $\sigma_1^2$. We will frequently use this in the proof.
    
    \paragraph{item: dynamic of $\tldE w_i^2 $.}
    We have
    \begin{align*}
        \tldE w_i^2
        = \frac{|\Spot|}{m}\tldE_{\Spot}[w_i^2]
            + \left(1-\frac{|\Spot|}{m}\right)
            \tldE_{\Sdense}[w_i^2].
    \end{align*}

    For $\Sdense$ part, using \cref{lem: stage 2 dynamic small and dense} we have 
    \begin{align*}
        &\frac{\rd}{\rd t} \tldE_{\Sdense}[w_i^2]
        = \frac{1}{|\Sdense|}\sum_{\tldvv\in\Sdense}
            2\tldv_i\frac{\rd \tldv_i}{\rd t}\\
        =&2\tldE_{\Sdense}[w_i^2] \left(
                -\lambda
                + 2\hsigma_0^2 \left(
                r - \tldE[ \norm{\vw}_2^2 ]\right) 
                + 2\hsigma_{2}^2 \left(
                    \ind_{i \le r}
                    - \tldE\left[ w_i^2 \right]
                \right)
                \pm \BigO{\frac{\omega^4}{r}}
            \right)\\
            &\quad+ 8\hsigma_4^2 \tldE_{\Sdense}[w_i^2\bw_i^2]\ind_{i\le r}
            \pm \BigO{\frac{\sigma_1^2}{m^{0.9}}+r\sqrt{\epsdir}}.
    \end{align*}
    By \cref{lem: stage 2 IH implication}\ref{item: stage 2 IH implication small and dense norm}\ref{item: stage 2 IH implication small and dense dir} we know $\tldE_{\Sdense}[w_i^2\bw_i^2] \lesssim r \omega^2\tldE_{\Sdense}\left[(\bbvwler^\btt{0})_i^4\right]
    \lesssim \frac{\omega^4}{r}$, where we use \cref{lem: init} at the end. Thus, the above can be simplified into
    \begin{align*}
        \frac{\rd}{\rd t} \tldE_{\Sdense}[w_i^2]
        =&2\tldE_{\Sdense}[w_i^2] \left(
                -\lambda
                + 2\hsigma_0^2 \left(
                r - \tldE[ \norm{\vw}_2^2 ]\right) 
                + 2\hsigma_{2}^2 \left(
                    \ind_{i \le r}
                    - \tldE\left[ w_i^2 \right]
                \right)
                \pm \BigO{\frac{\omega^4}{r}}
            \right).
    \end{align*}

    For $\Spot$ part, using \cref{lem: stage 2 dynamic potential neuron} we have for $k\in[r]$ 
    \begin{align*}
        &\frac{\rd}{\rd t} \tldE_{\Spotk}[w_i^2]
        = \frac{1}{|\Spotk|}\sum_{\tldvv\in\Spotk}
            2\tldv_i\frac{\rd \tldv_i}{\rd t}\\
        =&  2\tldE_{\Spotk}[w_i^2] \Bigg(
                -\lambda
                + 2\hsigma_0^2 \left(
                r - \tldE[ \norm{\vw}_2^2 ]\right) 
                + 2\hsigma_{2}^2 \left(
                    \ind_{i \le r}
                    - \tldE\left[ w_i^2 \right]
                \right)                                
                \pm \BigO{\frac{\omega^4}{r}}
                \Bigg)\\
                &\quad+ 2\hsigma_4^2 \left(
                    4(1-\ha_i)\tldE_{\Spotk}[w_i^2\bw_i^2]\ind_{i\le r}
                    - 2(1-\ha_k)\tldE_{\Spotk}[w_i^2\bw_k^4]
                    \right)\\
                &\quad+ 2\sum_{3\le j\le j_0}\hsigma_{2j}^2\left(
                    2j(1-\ha_i)\tldE_{\Spotk}[w_i^2\bw_i^{2j-2}]\ind_{i=k}
                    -(2j-2)(1-\ha_k)\tldE_{\Spotk}[w_i^2\bw_k^{2j}]
                    \right) 
            \pm \BigO{\frac{\sigma_1^2}{m^{0.9}}+r\sqrt{\epsdir}}
    \end{align*}
    We have for $j\ge 0$
    \begin{align*}
        \tldE_{\Spotk}[w_i^2\bw_i^{2j}]
        \myeq{a} \tldE_{G_k}[\norm{\vw}^2\bw_i^{2j+2}]
            \pm \frac{\sigma_1^2}{m^{0.9}}
        \left\{\begin{array}{ll}
        \myeq{b} \ha_k \pm\BigO{ j\sqrt{\epsdir} + \frac{\sigma_1^2}{m^{0.9}}}   &,\ i=k  \\
        = \BigO{\epsdir+\frac{\sigma_1^2}{m^{0.9}}}     & ,\ i\ne k
        \end{array}\right.,
    \end{align*}
    where (a) \cref{lem: stage 2 IH}\ref{item: stage 2 IH norm bound or basis like};
    (b) \cref{lem: <>2j bound}.

    Thus, the above can be simplified into
    \begin{align*}
        \frac{\rd}{\rd t} \tldE_{\Spotk}[w_i^2]
        =&  2\tldE_{\Spotk}[w_i^2] \Bigg(
                -\lambda
                + 2\hsigma_0^2 \left(
                r - \tldE[ \norm{\vw}_2^2 ]\right) 
                + 2\hsigma_{2}^2 \left(
                    \ind_{i \le r}
                    - \tldE\left[ w_i^2 \right]
                \right)                                
                \pm \BigO{\frac{\omega^4}{r}}
                \Bigg)\\
                &\qquad+ 4\hsigma_{\ge 4}^2(1-\ha_i)\ha_i\ind_{i=k} 
            \pm \BigO{\frac{1}{r}},
    \end{align*}
    where $\hsigma_{\ge4}^2 = \sum_{j\ge 2}\hsigma_{2j}^2$ and $\hsigma_{\ge 2j_0}^2\lesssim 1/r$.

    Now combine $\frac{\rd}{\rd t} \tldE_{\Sdense}[w_i^2]$ and $\frac{\rd}{\rd t} \tldE_{\Spotk}[w_i^2]$ for all $k\in[i]$ we have
    \begin{align*}
        \frac{\rd}{\rd t} \tldE[w_i^2]
        =  2\tldE[w_i^2] \Bigg(
                -\lambda
                + 2\hsigma_0^2 \left(
                r - \tldE[ \norm{\vw}_2^2 ]\right) 
                + 2\hsigma_{2}^2 \left(
                    \ind_{i \le r}
                    - \tldE\left[ w_i^2 \right]
                \right)                                
                \pm \BigO{\frac{\omega^4}{r}}
                \Bigg)
                + 4\hsigma_{\ge 4}^2(1-\ha_i)\ha_i,
    \end{align*}
    where we use \cref{lem: stage 2 IH}\ref{item: stage 2 IH reg}.

    \paragraph{item: dynamic of $\tldalpha$.}
    Recall $\tldalpha=\tldE[\norm{\vw}^2]$. We have
    \begin{align*}
        \frac{\rd \tldalpha}{\rd t}
        =& \sum_{i\in[d]}\frac{\rd}{\rd t} \tldE[w_i^2]\\
        =& 2\tldalpha \Bigg(
                -\lambda
                + 2\hsigma_0^2 \left(
                r - \tldalpha\right) 
                + 2\hsigma_{2}^2 \sum_{i\in[r]}\frac{\tldE[w_i^2]}{\tldE \norm{\vw}^2}\left(
                    1 - \tldE\left[ w_i^2 \right]
                \right)      
                + 2\hsigma_{\ge 4}^2\sum_{i\in[r]}\frac{\ha_i}{\tldE\norm{\vw}^2}(1-\ha_i)
                \pm \BigO{\frac{\omega^4}{r}}
                \Bigg).
    \end{align*}

    \paragraph{item: dynamic of $\ha_i$.}

    From the proof of \cref{lem: stage 2 dynamic potential neuron} we can have
    \begin{align*}
        \frac{\rd}{\rd t}\norm{\vv}_2^2
        =& 2\norm{\vv}_2^2 \Bigg(
            -\lambda
            +2\hsigma_0^2\left(r-\tldE\norm{\vw}_2^2\right)
            + 2\hsigma_2^2\sum_{i\le r}\bv_i^2 
                \left( 1 - \tldE\left[w_i^2\right]\right)
            + \sum_{j\ge 2}2\hsigma_{2j}^2 \left(
                    \sum_{l\le r}\bv_l^{2j} 
                    - \tldE \left[\norm{\vw}_2^2\langle \bvw,\bvv\rangle^{2j}\right]
                \right)
        \Bigg).
    \end{align*}
    Since $\vv\in G_k$ we have $\bv_k^2\ge 1-\epsdir$ and using $\tldE \left[\norm{\vw}_2^2\langle \bvw,\bvv\rangle^{2j}\right]= \tldE\left[\norm{\vw}^2\bw_i^{2j}\right]\pm \BigO{jr\sqrt{\epsdir}}=\ha_k + \tldE_{\Sdense}\left[\norm{\vw}^2\bw_i^{2j}\right]+ \BigO{jr\sqrt{\epsdir}}$ from \cref{lem: <>2j bound}. 
    Thus, we can further simplified it into
    \begin{align*}
        \frac{\rd}{\rd t}\norm{\vv}_2^2
        =& 2\norm{\vv}_2^2 \Bigg(
            -\lambda
            +2\hsigma_0^2\left(r-\tldE\norm{\vw}_2^2\right)
            + 2\hsigma_2^2 \left( 1 - \tldE\left[w_i^2\right]\right)
            + 2\hsigma_{\ge4}^2 (1-\ha_k)\\
            &\qquad- \underbrace{\sum_{j\ge 2}2\hsigma_{2j}^2\tldE_{\Sdense}\left[\norm{\vw}^2\bw_i^{2j}\right]}_{\BigO{\frac{\omega^4}{r^2}\tldE_{\Sdense} \norm{\vw}^2}}
            \pm \BigO{r\sqrt{\epsdir}}
        \Bigg),
    \end{align*}
    which gives the dynamic of $\ha_k$.
\end{proof}

\subsection{Stage 2.1}\label{sec: stage 2.1}
Denote the end of Stage 2.1 as the first time that every ground-truth direction are learned.
\begin{align*}
    T_{21}
    := \inf\left\{t\ge T_1: \text{for every $i\in[r]$ there exists at least a neuron $\tldvv$ such that $\btldv_i^2\ge 1-\epsdir^2$}\right\}.
\end{align*}
We note here that the aligned error $\epsdir^2$ is much better than what require to be in $G_i$ that error $\epsdir$. This is intended for later stages to prove the stability of such alignment \cref{lem: stage 2 IH}\ref{item: stage 2 IH stability G_i}. 

The below is the main result for Stage 2.1. It shows that at the end of this stage, we have learned the target direction, but the total mass of such aligned neuron $\ha_i$ is still small.
\begin{lemma}[Main results for Stage 2.1]\label{lem: stage 2.1 main}
    Stage 2.1 ends within time $T_{21}-T_1= \BigTheta{\frac{r}{\log m}}$. The following hold at the end of Stage 2.1:
    \begin{enumerate}
        \item Learn direction: for every $i\in[r]$ we have $\sigma_1^4/m\le \ha_i \le \sigma_1^4/m^{1-o(1)}$. 
        \item Direction of neuron: for $\tldvv\not\in S_{pot,i}$ and $i\le r$, we have $(\bbtldvvler)_i^2 \lesssim \omega^{0.01}(\bbtldvvler^\btt{0})_i^2$.
        \item $\tldE w_i^2\le 1/d^2$ for $i>r$.
        \item All conditions in \cref{lem: stage 2 IH} and \cref{lem: stage 2.1 IH} hold.
        \item Loss does not change: $\L(\tldmW)\ge c(\hsigma_0^2r^2/2 + \hsigma_2^2r/2)+(1\pm c)\hsigma_{\ge4}^2r/2$ with some small constant $c$.
    \end{enumerate}
\end{lemma}
\begin{proof}
    It follows from \cref{lem: stage 2.1 neuron main result}, \cref{lem: stage 2.1 IH} and \cref{lem: stage 2.1 population 0th 2nd dynamic}. The loss bound follows from the observation in \cref{lem: stage 2.1 population 0th 2nd dynamic} that $\tldE w_i^2$ does not move much.
\end{proof}

To prove the above result, we will use the following induction hypothesis throughout this stage.
\begin{lemma}[Induction Hypothesis of Stage 2.1]\label{lem: stage 2.1 IH}
    Stage 2.1 ends within time $\BigO{r/\log m}$ and all condition in Induction Hypothesis \cref{lem: stage 2 IH} hold and the following hold:
    \begin{enumerate}

        \item \label{item: stage 2.1 IH small dense neuron norm}
        Norm of small and dense neuron remains balance between them: For any time $t\le T_{norm,sd,21}$ we have for any neuron $\tldvv,\tldvu\in \Sdense$, 
        \[
            \frac{\norm{\tldvvler}_2}{\norm{\tldvuler}_2} = \frac{\norm{\tldvvler^\btt{0}}_2}{\norm{\tldvuler^{\btt{0}}}_2}\left(1\pm \delta_{norm,sd,21}\right),\quad \frac{\norm{\tldvvgtr}_2}{\norm{\tldvugtr}_2} = \frac{\norm{\tldvvgtr^\btt{0}}_2}{\norm{\tldvugtr^{\btt{0}}}_2}\left(1\pm \delta_{norm,sd,21}\right)\]
        where $\delta_{norm,sd,21}\le \frac{\omega^6}{\log m}$. This is stronger than \cref{lem: stage 2.1 IH}\ref{item: stage 2 IH small dense neuron norm} by having smaller $\delta_{norm,sd}$.

        \item \label{item: stage 2.1 IH population 2nd order balance}
            Population 2nd order balance: For any time $t\le T_{pop2nd,21}$, we have for $i,j\le r$
            \begin{align*}
                \frac{\tldE w_i^2}{\tldE w_j^2}
                = 1 \pm \delta_{pop2nd,21},
            \end{align*}
            where $\delta_{pop2nd,21}\le \frac{\omega^5}{r}$.
        
        \item \label{item: stage 2.1 IH reg} 
            Regularity: For any time $t\le T_{reg,21}$, we have \cref{lem: stage 2 IH}\ref{item: stage 2 IH reg} hold and following stronger condition:
            \begin{enumerate}
                \item $\ha_i \le 1/r$ for $i\in[r]$
                \item $1-\frac{\lambda_0}{4\hsigma_0^2}\ge \tldE w_i^2 \ge 1/2$ for $i\in[r]$
                \item $\tldE_{\Sdense}[\norm{\vw}^2]\bw_i^4\ge \frac{1}{\omega r}$ for $i\le r$.
            \end{enumerate}    
        \end{enumerate}
        Then, $T_{2}\le \tldT_{21}:=\min\{\tldT_2, T_{norm,sd,21}, T_{pop2nd,21}, T_{reg,21}, T_{21}, T_1+\frac{c_{21}r}{\log m}\}$. That is, current stage ends before any of the above condition breaks and ends in $\BigO{\frac{r}{\log m}}$ time.
\end{lemma}

\paragraph{Implications of Induction Hypothesis.}
We collect few useful facts for the analysis. They are implication of the induction hypothesis as well as the previous stages.

The first one is a detailed characterization of the potential neuron. It follows from initialization \cref{lem: init} and the fact that direction within target subspace does not move much as in \cref{lem: stage 1.2 main}.
\begin{lemma}[`potential' neuron]\label{lem: stage 2.1 good pot neuron}
    Under \cref{lem: init}, we have the following at time $T_{1}$ (the start of Stage 2):
    \begin{enumerate}
        \item scale of best `potential' neuron: for $i\in[r]$, then $\max_{\vv\in \Sgoodi}(\bbvvler)_i^2 =\frac{2\log(2m)}{r}\left(1\pm\errpot \right)$,
        where $\errpot=\BigO{\max\{\frac{\log r}{\log m},\frac{\log d}{r}\}}=\BigO{\frac{\log r}{\log m}}=o(1)$.
            
        \item fraction of `potential' neuron is small: $|\Spot|\le m^{o(1)}$.
            
        \item gap between coordinates in `potential neurons': 
        for $\vv\in\Spoti$, $(\bbvvler)_i^2 \ge 0.99\max_{j\in[r]\setminus\{i\}} (\bbvvler)_j^2$
        
        \item no `bad' neuron: $\Spoti\cap\Spotj=\emptyset$ for any $i,j\in[r]$.
    \end{enumerate}
\end{lemma}
\begin{proof}
    First from \cref{lem: stage 1.2 main} we know at time $T_1$ for any neuron $\vv$ and $i\in[r]$, we have 
    \[
        (\bbvvler)_i^2 
        = \frac{v_i^2}{\norm{\vvler}_2^2}
        = \frac{v_i^{\btt{0}2}}{\norm{\vvler^\btt{0}}_2^2} (1\pm 3\delta_{dir,12})
        = (\bbvvler^\btt{0})_i^2 \left(1\pm \BigO{\frac{\log d}{r}}\right).
    \]
    Now we can use \cref{lem: init}\ref{item: init good pot bad neuron} at initialization and get the result we want. Specifically,
    
    (a) it directly follows from \cref{lem: init}\ref{item: init good pot bad neuron}(i).

    (b)(c)(d) It suffices to notice that the set in \cref{lem: init}\ref{item: init good pot bad neuron}(ii)(iii)(iv) contains $\Spot$ (threshold is $2\omega^{-0.01}$ instead of $\omega^{-0.01}$ and $\frac{\log d}{r} \ll \omega^{-0.01}$), so the same bound applies.
\end{proof}

Below are direct implications of induction hypothesis \cref{lem: stage 2.1 IH}. We collect together here for the ease of use in the later proof.
\begin{lemma}\label{lem: stage 2.1 IH implication}
    For $T_1\le t \le \tldT_{2}$, besides \cref{lem: stage 2 IH implication} we have for any neuron $\tldvv$ and $i\in[d]$
    \begin{enumerate}
        \item \label{item: stage 2.1 IH implication small and dense norm}
        for $\tldvv\in \Sdense$, $\norm{\tldvv}_2\le 1.01\sqrt{2r}$.
    \end{enumerate}
\end{lemma}
\begin{proof}
    We show one by one.
    \paragraph{item (a)}
    It follows the same argument as the one in \cref{lem: stage 2 IH implication} with an smaller $\delta_{norm,sd}$ in \cref{lem: stage 2.1 IH}\ref{item: stage 2.1 IH small dense neuron norm}.
\end{proof}

\subsubsection{Dynamics}
The lemmas below describe the dynamics of the small-and-dense neuron, the potential neuron, and the population dynamics, respectively for Stage 2.1. They follows from the meta dynamics in \cref{sec: stage 2 dynamics} and further simplified by Induction Hypothesis \cref{lem: stage 2.1 IH}.

\begin{lemma}[Dynamics of small-and-dense neuron in Stage 2.1]\label{lem: stage 2.1 dynamic small and dense}
    In Stage 2.1, for $T_{12}\le t\le \tldT_{21}$, the dynamic~\eqref{eq: dynamic of tldw} for small-and-dense neuron can be simplified as: for any neuron $\tldvv\in\Sdense$ and $i\in[d]$
    \begin{align*}
        \frac{\rd \tldv_i}{\rd t}
        =&  \tldv_i \left(
                -\lambda
                + 2\hsigma_0^2 \left(
                r - \tldE[ \norm{\vw}_2^2 ]\right) 
                + 2\hsigma_{2}^2 \left(
                    \ind_{i \le r}
                    - \tldE\left[ w_i^2 \right]
                \right)
                + 4\hsigma_4^2 \btldv_i^2\ind_{i\le r}
                \pm \BigO{\frac{\omega^4}{r}}
            \right)
            \pm \BigO{\frac{\sigma_1^2}{m^{0.9}}+\sqrt{\epsdir}},\\
        \frac{\rd}{\rd t}\norm{\tldvvler}_2
        =& \norm{\tldvvler}_2 \left(
            -\lambda
            + 2\hsigma_0^2\left(r-\tldE\norm{\vw}_2^2\right)
            + 2\hsigma_2^2\sum_{i\le r}(\bbtldvvler)_i^2 
                \left( 1 - \tldE\left[w_i^2\right]\right)
            \pm \BigO{\frac{\omega^4}{r}}
        \right),\\
        \frac{\rd (\bbtldvvler)_i^2}{\rd t}
        \le& 8\hsigma_4^2(\bbtldvvler)_i^2 \left( 
                 (\bbtldvvler)_i^2
                \pm \BigO{\frac{\omega^5}{r}}
                \right).
    \end{align*}
\end{lemma}
\begin{proof}
    It follows from \cref{lem: stage 2 dynamic small and dense} and balance of $\tldE w_i^2$ in \cref{lem: stage 2.1 IH}\ref{item: stage 2.1 IH population 2nd order balance}.
\end{proof}

\begin{lemma}[Dynamics of potential neuron in Stage 2.1]\label{lem: stage 2.1 dynamic potential neuron}
    In Stage 2.1, for $T_{12}\le t\le \tldT_{21}$, the dynamic~\eqref{eq: dynamic of tldw} for potential neuron can be simplified as: for any neuron $\tldvv\in\Spotk$ and $k\in[r]$ (let $j_0=r^2$)
    \begin{align*}
        \frac{\rd \tldv_i}{\rd t}
        =&  v_i \Bigg(
                -\lambda
                +2\hsigma_0^2 \left(
                r - \tldE[ \norm{\vw}_2^2 ]\right) 
                + 2\hsigma_{2}^2 \left(
                    \ind_{i \le r}
                    - \tldE\left[ w_i^2 \right]
                \right)
                + \hsigma_4^2 \left(
                    4\btldv_i^2\ind_{i\le r}
                    - 2\btldv_k^4
                    \right)\\
                &\quad+ \sum_{3\le j\le j_0}\hsigma_{2j}^2\left(
                    2j\btldv_i^{2j-2}\ind_{i=k}
                    -(2j-2)\btldv_k^{2j}
                    \right) 
                \pm \BigO{\frac{\omega^4}{r}}
            \Bigg)
            \pm \BigO{\frac{\sigma_1^2}{m^{0.9}}+r\sqrt{\epsdir}},\\
        \frac{\rd}{\rd t}\norm{\tldvvler}_2
        =& \norm{\tldvvler}_2 \Bigg(
            -\lambda
            +2\hsigma_0^2\left(r-\tldE\norm{\vw}_2^2\right)
            + 2\hsigma_2^2\sum_{i\le r}(\bbtldvvler)_i^2 
                \left( 1 - \tldE\left[w_i^2\right]\right)\\
            &\qquad+ \sum_{2\le j\le j_0}\hsigma_{2j}^2 \left(
                2j\frac{\norm{\tldvv}^2}{\norm{\tldvvler}^2} - (2j-2)
                \right) 
                \btldv_k^{2j}
            \pm \BigO{\frac{\omega^4}{r}}
        \Bigg),\\
        \frac{\rd (\bbtldvvler)_i^2}{\rd t}
        \le& 8\hsigma_4^2(\bbtldvvler)_i^2 \Bigg( 
                 \btldv_i^2\ind_{i\le r}
                \pm \BigO{\frac{\omega^5}{r}}
            \Bigg), \text{ for $i\ne k$},\\
        \frac{\rd (\bbtldvvler)_k^2}{\rd t}
        =& 2(\bbtldvvler)_k^2 \Bigg( 
                \sum_{2\le j\le j_0}\hsigma_{2j}^2 
                2j\btldv_k^{2j-2}(1-(\bbtldvvler)_k^2)
                \pm \BigO{\frac{\omega^5}{r}}
            \Bigg).
    \end{align*}

    As a result, we also have
    \begin{align*}
        \frac{\rd \tldv_i^2}{\rd t}
        \le&  2\tldv_i^2 \Bigg(
                -\lambda
                +2\hsigma_0^2 \left(
                r - \tldE[ \norm{\vw}_2^2 ]\right) 
                + 2\hsigma_{2}^2 \left(
                    \ind_{i \le r}
                    - \tldE\left[ w_i^2 \right]
                \right)
                + 4\hsigma_4^2\btldv_i^2\ind_{i\le r}
                + \BigO{\frac{\omega^4}{r}}
            \Bigg),\text{ for $i\ne k$,}\\
        \frac{\rd \tldv_k^2}{\rd t}
        \ge&  2\tldv_k^2 \Bigg(
                -\lambda
                +2\hsigma_0^2 \left(
                r - \tldE[ \norm{\vw}_2^2 ]\right) 
                + 2\hsigma_{2}^2 \left(
                    1 - \tldE\left[ w_k^2 \right]
                \right)
                +2\hsigma_4^2 \btldv_k^2
                \pm \BigO{\frac{\omega^4}{r}}
            \Bigg),\\
    \end{align*}
\end{lemma}
\begin{proof}
    It follows from \cref{lem: stage 2 dynamic potential neuron} and balance of $\tldE w_i^2$ in \cref{lem: stage 2.1 IH}\ref{item: stage 2.1 IH population 2nd order balance}. Due to the choice of step size $\gamma_t$, it stays as 1 in Stage 2.1. 
\end{proof}

\begin{lemma}[Population 0th and 2nd order term dynamic in Stage 2.1]\label{lem: stage 2.1 population 0th 2nd dynamic}
In Stage 2.1, for $T_{12}\le t\le \tldT_{21}$, the dynamic for population 0th and 2nd order term can be simplified as: 
    \begin{align*}
        \text{for $i\in[d]$, }
        \frac{\rd}{\rd t} \tldE[w_i^2]
        =&  2\tldE[w_i^2] \Bigg(
                -\lambda
                +2\hsigma_0^2 \left(
                r - \tldE[ \norm{\vw}_2^2 ]\right) 
                + 2\hsigma_{2}^2 \left(
                    \ind_{i \le r} - \tldE\left[ w_i^2 \right]
                \right)                
                \pm \BigO{\frac{\omega^4}{r}}
                \Bigg)\\
        \frac{\rd \tldalpha}{\rd t}
        =& 2\tldalpha \Bigg(
                -\lambda
                +2\hsigma_0^2 \left(
                r - \tldalpha\right) 
                + 2\hsigma_{2}^2 \sum_{i\in[r]}\frac{\tldE[w_i^2]}{\tldE \norm{\vw}^2}\left(
                    1 - \tldE\left[ w_i^2 \right]
                \right)      
                \pm \BigO{\frac{\omega^4}{r}}
                \Bigg),
    \end{align*}
    where $\hsigma_{\ge4}^2 = \sum_{j\ge 2}\hsigma_{2j}^2$.

    As a result, we have
    \begin{align*}
        \text{for $i,j\in[r]$, }
        \frac{\rd}{\rd t} \frac{\tldE[w_i^2]}{\tldE[w_j^2]}
        =&  4\hsigma_2^2\frac{\tldE[w_i^2]}{\tldE[w_j^2]} \Bigg(
                \left(
                    \tldE\left[w_j^2\right] - \tldE\left[ w_i^2 \right]
                \right)                
                \pm \BigO{\frac{\omega^4}{r}}
                \Bigg).
    \end{align*}
    Moreover, this implies (i) $\frac{\tldE[w_i^2]}{\tldE[w_j^2]}=1\pm\BigO{\frac{\omega^4}{r}}$ that different direction remain balanced, which proves \cref{lem: stage 2.1 IH}\ref{item: stage 2.1 IH population 2nd order balance}; (ii) for $i\le r$
    \[
        -\lambda
        +2\hsigma_0^2 \left(
        r - \tldalpha\right) 
        + 2\hsigma_{2}^2 \left(
            1 - \tldE\left[ w_i^2 \right]
        \right)                
        =\pm\BigO{\frac{\omega^5}{r}},
    \]
    and (iii) $\tldE w_i^2\le 1/d^2$ for $i>r$ at for any time $t$ that $\frac{r}{\log m}\lesssim t \le \tldT_{21}$ and $\tldE w_i^2 \ge 1 - \frac{\lambda_0}{2\hsigma_0^2} - \BigO{\frac{\omega^4}{r}}$ for $i\le r$.
\end{lemma}
\begin{proof}
    The dynamic of $\tldE w_i^2 $ and $\tldalpha$ follow from \cref{lem: stage 2 population 0th 2nd dynamic} and and balance of $\tldE w_i^2$ in \cref{lem: stage 2.1 IH}\ref{item: stage 2.1 IH population 2nd order balance} and $\ha_i$ is small in \cref{lem: stage 2.1 IH}\ref{item: stage 2.1 IH reg}.

    \paragraph{item: dynamic of $\frac{\tldE w_i^2}{\tldE w_j^2}$} This directly follows from the dynamics of $\tldE w_i^2$ above.

    \paragraph{item: dynamic of $\tldE w_i^2+\tldalpha$}
    Using the above dynamics, we have for every $i\le r$
    \begin{align*}
        &\frac{\rd}{\rd t} 2\hsigma_2^2\tldE[w_i^2] + 2\hsigma_0^2\tldalpha\\
        =&  2(2\hsigma_2^2\tldE[w_i^2]+2\hsigma_0^2\tldalpha) \Bigg(
                \frac{2\hsigma_2^2\tldE w_i^2}{2\hsigma_2^2\tldE[w_i^2]+2\hsigma_0^2\tldalpha}\Bigg(
                -\lambda
                +2\hsigma_0^2 \left(
                r - \tldalpha\right) 
                + 2\hsigma_{2}^2 \left(
                    \ind_{i \le r} - \tldE\left[ w_i^2 \right]
                \right)                
                \pm \BigO{\frac{\omega^4}{r}}
                \Bigg)\\
                &\qquad\qquad\qquad\qquad+ \frac{2\hsigma_0^2\tldalpha}{2\hsigma_2^2\tldE[w_i^2]+2\hsigma_0^2\tldalpha}\Bigg(
                    -\lambda
                    +2\hsigma_0^2 \left(
                    r - \tldalpha\right) 
                    + 2\hsigma_{2}^2 \sum_{i\in[r]}\frac{\tldE[w_i^2]}{\tldE \norm{\vw}^2}\left(
                        1 - \tldE\left[ w_i^2 \right]
                    \right)      
                    \pm \BigO{\frac{\omega^4}{r}}
                    \Bigg),\\
            =&2(2\hsigma_2^2\tldE[w_i^2]+2\hsigma_0^2\tldalpha) \Bigg(
                -\lambda
                +2\hsigma_0^2 \left(
                r - \tldalpha\right) 
                + 2\hsigma_{2}^2 \left(
                    1 - \tldE\left[ w_1^2 \right]
                \right)                
                - \BigO{\frac{\omega^5}{r}}
                \Bigg),
    \end{align*}
    where at the end we use \cref{lem: stage 2.1 IH}\ref{item: stage 2.1 IH population 2nd order balance} and \cref{lem: stage 2 IH}\ref{item: stage 2 IH reg}. The showup of $\tldE w_1^2$ can be replaced by any $\tldE w_i^2$ with $i\le r$, and is due to the balance of $\tldE w_i^2$. This implies for $i\le r$
    \[
        -\lambda
        +2\hsigma_0^2 \left(
        r - \tldalpha\right) 
        + 2\hsigma_{2}^2 \left(
            1 - \tldE\left[ w_i^2 \right]
        \right)                
        =\pm\BigO{\frac{\omega^5}{r}}.
    \]

    \paragraph{item: bound on $\tldE w_i^2$ for $i>r$}
    Note that the population dynamic here is the same as the one in Stage 1.2 in \cref{lem: stage 1.2 0th 2nd order dynamic}, and therefore follow the same argument in \cref{lem: stage 1.2 convergence of 2nd order} we know $\tldE w_i^2 \lesssim e^{-\BigTheta{\lambda_0}t}$. Thus, for any $t\gtrsim \frac{r}{\log m}$, we have $\tldE w_i^2 \le 1/d^2$ for $i>r$. Similarly, we have $\tldE w_i^2 \ge 1 - \frac{\lambda_0}{2\hsigma_0^2} - \BigO{\frac{\omega^4}{r}}$.
\end{proof}

\subsubsection{Convergence}
We now are ready to characterize the dynamics of the scale of each neuron coordinate, that is, its correlation with the ground-truth direction. Specifically, we show that by the end of Stage 2.1, every ground-truth direction has at least one neuron with almost perfect correlation. In contrast, nearly all other coordinates have correlation at most $\frac{\omega \log m}{r}$. 

Note from the dynamics in \cref{lem: stage 2.1 dynamic potential neuron} \cref{lem: stage 2.1 dynamic small and dense} we can see the different directions evolve almost independently, mostly thanks to the balance between different 2nd order population term $\tldE w_i$ as \cref{lem: stage 2.1 IH}\ref{item: stage 2.1 IH population 2nd order balance}. The proof, as described in \cref{sec: dynamics main-text}, uses the properties of tensor power type dynamic to show potential neuron align with target direction, while others remain almost unchanged.
\begin{lemma}[Main results for neuron dynamics]\label{lem: stage 2.1 neuron main result}
    For $T_1\le t < \tldT_{21}$, we have within $\BigO{r/\log m}$ time, for every ground-truth direction there exists a neuron that has almost perfect correlation: that is, for every $i\in[r]$ there exists $T_{21,i}'= \Theta(r/\log m)$ such that at most at time $T_1+T_{21,i}'$,  there exists a $\tldvv\in\Spoti$ such that $\btldv_i^2\ge 1-\epsdir^2$. 

    In contrast, for all coordinates except the largest one for each potential neuron, that is for $\tldvv\not\in S_{pot,i}$ and $i\le r$, we have $(\bbtldvvler)_i^2 \lesssim \omega^{0.01}(\bbtldvvler^\btt{0})_i^2\le \frac{\omega\log m}{r}$ for all $T_1\le t\le T_1+ (1+\omega^{-1/2})\max_i T_{21,i}'$. Moreover, for $\tldvv\in\Spot\setminus\Spoti$, we have $(\bbtldvvler)_i^2 \lesssim (\bbtldvvler^\btt{0})_i^2\lesssim \frac{\log m}{r}$ for the same time period.

    Lastly, we have all neurons $\tldvv\in \Spoti$ either $\norm{\tldvv}_2 \le \sigma_1$ or $\btldv_i^2\ge 1-\epsdir$, i.e. \cref{lem: stage 2 IH}\ref{item: stage 2 IH norm bound or basis like} holds. Also $\sigma_1^4/m\le \ha_i \le \sigma_1^4/m^{1-o(1)}$ for any $i\le r$. Moreover, we know $\{\tldvv: \norm{\tldvv}\ge \sigma_1^2\}\cap\Spoti\subset G_i\subset \Spoti$ and $\{\tldvv: \norm{\tldvv}\ge \sigma_1^2\}\cap\Spoti\ne \emptyset$. That is, there are potential neurons with direction $(\bbtldvvler)_i^2\ge 1-\epsdir^2$ and their norm is at least $\sigma_1^2$.
\end{lemma}
\begin{proof}
    We will fix one $i\le r$. The argument apply to all $i\le r$ when consider each respectively. We will bound $\max_i T_{21,i}' - \min_i T_{21,i}'$ at the end. WLOG, we reset the time to 0 to simplify the notation.

    We first upper bound the coordinates except the largest one for each potential neuron (type (ii) below).
        
    We compare the dynamics of two different types of neuron coordinates:
    (i) the largest coordinate of a potential neuron for $i$-th direction, i.e., $\max_{\vv} (\bbvvler)_i^2$;
    (ii) all coordinates except the largest one for each potential neuron. This includes all coordinates of small-and-dense neurons ($\Sdense$), as well as $(\bbvvler)_j^2$ (for $j \ne i$) when $\vv \in \Spoti$.

    Denote $p,q$ as $(\bbvvler)_i^2$ for type (i) and (ii) coordinate respectively. When upperbounding type (ii), we only focus on those are larger than $\omega^6/r$. For those smaller than $\omega^6/r$ can be controlled easily and will be dealt later. We are going to upper bound $q$ and lower bound $p$. Consider 
    \begin{align*}
        \frac{\rd p}{\rd t}
        =& 2 p \left(\sum_{2\le j\le j_0}\hsigma_{2j}^2 2j \cdot p^{j-1}(1-p) \pm \BigO{\frac{\omega^5}{r}}\right),\\
        \frac{\rd q}{\rd t}
        =& 8\hsigma_4^2 q^2 \left(1+\BigO{\frac{1}{\omega}}\right)
    \end{align*}
    Let $p_0= \frac{2\log(2m)}{r}\left(1\pm\BigO{\frac{\log m}{r}}\right)$ and $\frac{\omega^6}{r}\le q_0\le p_0 \left(1-\omega^{-0.01}\right)$ to represent the initial value of these coordinate based on \cref{lem: stage 2.1 good pot neuron} and the definition of potential neuron.

    From \cref{lem: stage 2.1 dynamic small and dense} and \cref{lem: stage 2.1 dynamic potential neuron}, we know $p$ (resp. $q$) always servers as an lower bound (resp. upper bound) of type (i) (resp. type (ii)) coordinates. 

    Ignoring high order term, the above ode has the form of $\frac{\rd x}{\rd t}=a x^2$, which has solution $x=\frac{1}{\frac{1}{x_0}-at}$. Such dynamic has the property that it has a sudden transit around time $\frac{1}{ax_0}$. Intuitively, this gives the upper bound of $q$ and lower bound of $p$. We will formalize it below.

    \paragraph{Lower bound of $p$}
    We first using the lower bound of $\frac{\rd p}{\rd t}$, and estimate it in different phases and show that the total time to reach near 1 is $T_i'\lesssim r/\log m$.
    \paragraph{Phase 1: before $p$ reaches $1/\omega$}
    
    It is easy to see $\frac{\rd x}{\rd t}=a x^2$ has solution $x=\frac{1}{\frac{1}{x_0}-at}$. Thus, this phase total time is at most $T_i'=\frac{1}{8\hsigma_4^2 \left(1-\BigO{\frac{1}{\omega}}\right) \cdot p_0}=\frac{1}{8\hsigma_4^2 \left(1-\BigO{\frac{1}{\omega}}\right) \cdot \frac{2\log (2m)}{r}\left(1-\BigO{\frac{\log d}{r}}\right)}\lesssim \frac{r}{\log m}$.

    \paragraph{Phase 2: before $p$ reaches $1/3$}
    From \cref{lem: stage 2.1 dynamic potential neuron} we have a slightly weaker lower bound on $p$ (due to $1-(\bbtldvvler)_k^2$ term)
    \begin{align*}
        \frac{\rd p}{\rd t}
        =& 4\hsigma_4^2 p^2 
    \end{align*}
    On the other hand, we have have a much stronger correlation starting point of $p=1/\omega$ so this phase takes much shorter time. To see this, notice the ode still admits the same form of solution as in phase 1. Thus, the total time of this phase times at most $T_i''\lesssim \omega$ time. 

    \paragraph{Phase 3: before $p$ reaches $1-\epsdir$}
    Now due to $1-(\bvvler)_k^2$ term in \cref{lem: stage 2.1 dynamic potential neuron} we will not directly look at the dynamics of $(\bvvler)_k^2$, but instead do the following of comparing $\sum_{i\ne k} v_i^2$ and $v_k^2$.

    By \cref{lem: stage 2.1 dynamic small and dense} and \cref{lem: stage 2.1 dynamic potential neuron} and using \cref{lem: stage 2.1 population 0th 2nd dynamic} to simplify the dynamics we have for all coordinates except the largest one for each potential neuron, that is for $\vv\not\in S_{pot,k}$ and $i\neq k$
    \begin{align*}
        \frac{\rd v_i^2}{\rd t}
        \le&  2v_i^2 \Bigg(
                4\hsigma_4^2\bv_i^2\ind_{i\le r}
                + \BigO{\frac{\omega^5}{r}}
            \Bigg),
    \end{align*}
    which suggests 
    \begin{align*}
        \frac{\rd }{\rd t}\sum_{i\ne k}v_i^2
        \le&  \sum_{i\ne k}v_i^2 \cdot \BigO{\frac{\omega^5}{r}}.
    \end{align*}
    This implies $\sum_{i\ne k} v_i^2 \le \sum_{i\ne k}v_i^{\btt{T_{12}}2}2 \le \omega r$, since the total time $\tldT\lesssim r/\log m$.

    On the other hand, we show below that $v_k$ grows up for potential neuron $\vv\in\Spotk$. Again using \cref{lem: stage 2.1 dynamic potential neuron} and \cref{lem: stage 2.1 population 0th 2nd dynamic} we have
    \begin{align*}
        \frac{\rd v_k^2}{\rd t}
        \ge&  2v_k^2 \Bigg(
                2\hsigma_4^2 \bv_k^2
                - \BigO{\frac{\omega^5}{r}}
            \Bigg)
        \gtrsim v_k^2.
    \end{align*}
    This implies within time at most $T_i'''=\BigO{\log(\omega r/\epsdir)}$ so that $v_k^2 \ge \omega^2 r^2/\epsdir^2\ge \sigma_1^4$. Since $1-\bv_k^2=\frac{\sum_{j\ne k}v_j^2}{\sum_{j\ne k}v_j^2+v_k^2}\le \frac{\sum_{j\ne k}v_j^2}{v_k^2}\le \epsilon_{dir}^2$, this further implies $\bv_k^2 \ge 1-\epsdir^2$ and thus $\ha_k>0$.

    In summary, the total time for $\ha_k>0$ is $T_{21,i}'\le T_i'+T_i''+T_i''' = (1+\BigO{\omega^{-1}})T_i'\lesssim r/\log m$ for each $i\le r$. 

    \paragraph{Upper bound of $q$}
    Denote $\tldT_i'=(1+\omega^{-1/2})T_i'$. For $\tldT_{i}'$ amount of time, we can bound $q$ as
    \[
        q\le \frac{1}{\frac{1}{q_0}-8\hsigma_4^2 (1+2\omega^{-1/2})T_i'}.
    \]
    Noticing that $8\hsigma_4^2\left(1+2\omega^{-1/2}\right)T_{i}'= \left(1+\BigO{\omega^{-1/2}}\right)p_0^{-1}$ and
    \[
        \frac{1}{q_0}
        \ge \frac{1}{p_0(1-\omega^{-0.01})}
        \ge 8\hsigma_4^2\left(1+\BigO{\omega^{-1/2}}\right)T_{i}' \cdot \frac{1}{1-\BigO{\omega^{-0.01}}},
    \]
    so we have
    \[
        q \le \frac{1}{
            \frac{\BigO{\omega^{-0.01}}}{q_0}
            }
    \lesssim \omega^{0.01}q_0.\]

    Note that from \cref{lem: stage 2.1 good pot neuron} we know for $\tldvv\in\Spot\setminus\Spoti$ the gap is constant (say 0.01) instead of $\omega^{-0.01}$, so we could have better bound of $q\lesssim q_0$.

    For those type (ii) coordinates smaller than $\omega^6/r$, abuse the notation still using $q$ as their upper bound, we claim $q\le \omega^{0.01}q_0$. We have
    \[
        \frac{\rd q}{\rd t}
        = 8\hsigma_4^2 q \cdot 2\frac{\omega^{6}}{r}.
    \]
    Within $T_i'\le \BigO{r/\log m}$ time, we know $q\le q_0 e^{\BigO{\omega^6/\log m}}\le \omega^{0.01}q_0$.

    \paragraph{Bound on $\norm{\vv}$ and $\ha_i$.}
    We now compare the hitting times across different target directions more carefully. For each $i\in[r]$, let $\vv^{(i)}\in \Spoti$ be the neuron maximizing the $i$-th coordinate at time $T_1$, and denote its normalized coordinate by $p_i(t):=\left(\overline{\vv^{(i)}_{\le r}(t)}\right)_i^2$ and its unnormalized coordinate by $z_i(t):=\left(v_i^{(i)}(t)\right)^2$.
    
    We first compare the time to reach a coarse alignment. Let
    \[
    \tau_i:=\inf\left\{t\ge T_1:p_i(t)\ge 1-\frac1\omega\right\}.
    \]
    By \cref{lem: stage 2.1 good pot neuron}, initially $p_i(T_1)=\frac{2\log(2m)}{r}(1\pm O(\frac{\log r}{\log m}))$, while the error term in its dynamics in \cref{lem: stage 2.1 dynamic potential neuron} is $O(\omega^5/r)$. Applying \cref{lem: stage 2.1 technical ode time gap} with $\epsilon=1/\omega$ gives
    \[
    \max_i\tau_i-\min_i\tau_i
    \lesssim \frac{r\log r}{\log^2m}
    =o(\log m).
    \]
    
    We next compare the scales $z_i$. By the norm balance after Stage 1 and \cref{lem: stage 2.1 good pot neuron}, $z_i(T_1)=\Theta(\log m)$. Moreover, \cref{lem: stage 2.1 dynamic potential neuron,lem: stage 2.1 population 0th 2nd dynamic} imply that $z_i$ is nondecreasing, since its leading higher-order term is at least $\Omega(p_i)\gg O(\omega^5/r)$. The Phase 3 bound above also gives $\sum_{j\neq i}(v_j^{(i)})^2\lesssim \omega r$ throughout this time. Hence, at $p_i(\tau_i)=1-1/\omega$, we have $\log(c\log m)\le \log z_i(\tau_i)\le \log(C\omega^2r)$, and therefore
    \[
    \max_{i,j}\left|\log z_i(\tau_i)-\log z_j(\tau_j)\right|
    =O(\log r)=o(\log m).
    \]
    Let $\tau_*:=\max_i\tau_i$. The region $p_i\ge 1-1/\omega$ is forward invariant: at the boundary, the leading term in \cref{lem: stage 2.1 dynamic potential neuron} is $\Theta(1/\omega)$, while the error is $O(\omega^5/r)=o(1/\omega)$. Moreover, $\frac{\rd}{\rd t}\log z_i=O(1)$. Combining this with $\max_i\tau_i-\min_i\tau_i=o(\log m)$ gives
    \[
    \max_{i,j}\left|\log z_i(\tau_*)-\log z_j(\tau_*)\right|=o(\log m).
    \]
    
    After $\tau_*$, we have $p_i\ge 1-1/\omega$ for every $i$. Thus, using \cref{lem: stage 2.1 dynamic potential neuron,lem: stage 2.1 population 0th 2nd dynamic} again,
    \[
    \frac{\rd}{\rd t}\log z_i
    =
    4\hat\sigma_{\ge4}^2+\rho_i(t),
    \qquad
    |\rho_i(t)|\le \frac{C}{\omega}+o(1)=:\eta_m=o(1).
    \]
    Let $K:=\frac{\omega^2r^2}{\epsilon_{\rm dir}^2}=\sigma_1^4$ and $T_i:=\inf\{t\ge\tau_*:z_i(t)\ge K\}$. Integrating the above dynamics gives
    \[
    \frac{\log K-\log z_i(\tau_*)}
    {4\hat\sigma_{\ge4}^2+\eta_m}
    \le T_i-\tau_*
    \le
    \frac{\log K-\log z_i(\tau_*)}
    {4\hat\sigma_{\ge4}^2-\eta_m}.
    \]
    Since $\log K=O(\log m)$, this implies
    \[
    \max_iT_i-\min_iT_i
    \lesssim
    \max_{i,j}\left|\log z_i(\tau_*)-\log z_j(\tau_*)\right|
    +\eta_m\log K
    =o(\log m).
    \]
    At time $T_i$,
    \[
    1-\left(\overline{\vv^{(i)}}\right)_i^2
    \le
    \frac{\sum_{j\neq i}(v_j^{(i)})^2}{z_i}
    \lesssim
    \frac{\omega r}{K}
    \le \epsilon_{\rm dir}^2.
    \]
    Since $z_i$ remains nondecreasing and $\sum_{j\neq i}(v_j^{(i)})^2\lesssim\omega r$ continues to hold, this alignment is preserved afterwards. Hence, by time $\max_iT_i$, all $r$ target directions have been learned, and the gap between their hitting times is $o(\log m)$.
    
    During this time gap, \cref{lem: stage 2.1 dynamic potential neuron} implies that each neuron norm grows by at most a multiplicative factor $e^{o(\log m)}\le m^{o(1)}$. Therefore, $\norm{\vv}^2\lesssim m^{o(1)}\sigma_1^4$ for all potential neurons. Using $|\Spoti|\le m^{o(1)}$ from \cref{lem: stage 2.1 good pot neuron}, we obtain $\ha_i\le |\Spoti|\cdot m^{o(1)}\sigma_1^4/m\le \sigma_1^4/m^{1-o(1)}$. On the other hand, the construction above gives at least one neuron in each $G_i$ with norm at least $\sigma_1^2$, and hence $\ha_i\ge \sigma_1^4/m$. Moreover, whenever $\norm{\vv}\ge \sigma_1$, the above dynamics ensure that $\vv$ belongs to some $G_i$, i.e., $\bv_i^2\ge 1-\epsdir$ for some $i\in[r]$.

\end{proof}

We show below that $\tldE \norm{\vw}^2\bw_i^4$ is balanced, which will be used in later stage analysis.
\begin{lemma}\label{lem: stage 2.1 E bw^4 bound}
    For $T_{1}\le t\le \tldT_{21}$, we have $\frac{\tldE_{\Sdense} \norm{\vw}^2\bw_i^4}{\tldE_{\Sdense} \norm{\vw}^2\bw_j^4}=1\pm \BigO{\frac{\omega^6}{\log m}}$ and $\tldE_{\Sdense} \norm{\vw}^2\bw_i^4\ge \frac{1}{\omega r}$.
\end{lemma}
\begin{proof}
    Given \cref{lem: stage 2.1 IH}\ref{item: stage 2.1 IH small dense neuron norm}, it suffices to bound the ratio $\frac{\tldE_{\Sdense} \bw_i^4}{\tldE_{\Sdense} \bw_j^4}$. From \cref{lem: stage 2 dynamic small and dense}, we know
    \begin{align*}
        \frac{\rd}{\rd t}\tldE_{\Sdense} \bw_i^4
        = 4\tldE_{\Sdense}\bw_i^4\left(
                4\hsigma_4^2 \frac{\tldE_{\Sdense} \bw_i^6}{\tldE_{\Sdense} \bw_i^4} \pm \BigO{\frac{\omega^5}{r}}
            \right)
        =\pm \tldE_{\Sdense}\bw_i^4 \BigO{\frac{\omega^5}{r}}.
    \end{align*}
    Thus, $\frac{\tldE_{\Sdense} \bw_i^4}{\tldE_{\Sdense} \bw_j^4}=1\pm \BigO{\frac{\omega^5}{\log m}}$, and $\tldE_{\Sdense} \norm{\vw}^2\bw_i^4=\tldE_{\Sdense} \norm{\vw^\btt{0}}^2\bw_i^{\btt{0}4}(1\pm\BigO{\frac{\omega^6}{r}})\ge \frac{1}{\omega r}$
\end{proof}


\subsubsection{Proof of \cref{lem: stage 2.1 IH}: IH of Stage 2.1 (\cref{lem: stage 2.1 IH})}\label{sec: proof stage 2.1 IH}
\begin{proof}
    We show one by one.
    Recall $\tldT_{21}=\min\{\tldT_2, T_{pop2nd,21}, T_{dir,21}, T_{norm,sd,21}, T_{reg,21}, T_{21}, T_1+\frac{c_{21}r}{\log m}\}$ with large enough constant $c_{21}$ and $\tldT_2:=\min\{T_{dir,2},T_{norm,sd,2},T_{norm,pot,2}, T_{reg,2}\}$ defined in \cref{lem: stage 2 IH}.
    
    First, by \cref{lem: stage 2.1 neuron main result} we know $T_{21}< T_1+c_{21}r/\log m$, so $\tldT_{21}=\min\{T_{21}, \tldT_2\}$. In below, we show that $T_1+c_{21}r/\log m < \min\{\tldT_2, T_{pop2nd,21}, T_{reg,21},\}=\min\{T_{dir,21},T_{norm,sd,21},T_{norm,pot,2}, T_{pop2nd,21}, T_{reg,21}\}$ so $\tldT_{21}=T_{21}$.

    Since \cref{lem: stage 2.1 IH} is stronger than \cref{lem: stage 2 IH}, so the current proof also shows that \cref{lem: stage 2 IH} holds upto the time that Stage 2.1 ends.

    We argue each one by one as below:
    
    \paragraph{Case: $T_{dir,2}>T_1+c_{21}r/\log m$. }
    Recall the definition of $T_{dir,21}$: For any time $t\le T_{dir,21}$ we have
        \begin{itemize}
            \item for $i> r$, $\left(\bbtldvvgtr^\btt{s}\right)_i^2 \le \frac{\log^2 d}{d}$
            \item for $i\le r$ and any neuron $\tldvv\not\in \Spoti$,  $\left(\bbtldvvler^\btt{s}\right)_i^2 \le \omega\left(\bbtldvvler^\btt{0}\right)_i^{2} $
        \end{itemize}
        where $\omega=\log\log\log r$ 

    The second line directly follows from \cref{lem: stage 2.1 neuron main result}. 

    The first line we use \cref{lem: stage 2.1 dynamic small and dense} and \cref{lem: stage 2.1 dynamic potential neuron} to get (first compute $\frac{\rd}{\rd t}\norm{\vvgtr}$, then $\frac{\rd}{\rd t}(\vvgtr)_i$ we omit the details)
    \begin{align*}
        \frac{\rd}{\rd t}(\bbvvgtr)_i^2
        \lesssim \frac{\omega^3}{r}(\bbvvgtr)_i^2.
    \end{align*}
    Thus given a time of $\BigO{r/\log m}$ at most increase $(\bbvvgtr)_i^2$ to $\BigO{\frac{\log^2 d}{d}}$.

    \paragraph{Case: $T_{norm,sd,21}>T_1+c_{21}r/\log m$.}
    Recall the definition of $T_{norm,sd,21}$: For any time $t\le T_{norm,sd,21}$ we have for any neuron $\tldvv,\tldvu\in \Sdense$, 
        \[\frac{\norm{\tldvvler}_2}{\norm{\tldvuler}_2} = \frac{\norm{\tldvvler^\btt{0}}_2}{\norm{\tldvuler^{\btt{0}}}_2}\left(1\pm \delta_{norm,sd,21}\right),\quad \frac{\norm{\tldvvgtr}_2}{\norm{\tldvugtr}_2} = \frac{\norm{\tldvvgtr^\btt{0}}_2}{\norm{\tldvugtr^{\btt{0}}}_2}\left(1\pm \delta_{norm,sd,21}\right)\]
        where $\delta_{norm,sd,21}\le \frac{\omega^6}{\log m}$.
        
    Using \cref{lem: stage 2.1 dynamic small and dense} and \cref{lem: stage 2.1 population 0th 2nd dynamic} that $\tldE w_i^2$ are balanced, we have
    for different neuron $\vv,\vu$
    \begin{align*}
        \frac{\rd}{\rd t}\frac{\norm{\tldvvler}_2}{\norm{\tldvuler}_2}
        =& \frac{\norm{\tldvvler}_2}{\norm{\tldvuler}_2}
        \left(
            2\hsigma_2^2\sum_{i\le r}(\bbtldvvler)_i^2 
                \left( 1 - \tldE\left[w_i^2\right]\right)
            - 2\hsigma_2^2\sum_{i\le r}(\bbtldvuler)_i^2 
                \left( 1 - \tldE\left[w_i^2\right]\right)
            \pm \BigO{\frac{\omega^4}{r}}
        \right)\\
        =& \pm \frac{\norm{\tldvvler}_2}{\norm{\tldvuler}_2}
        \BigO{\frac{\omega^5}{r}}.
    \end{align*}
    Thus given a time of $\BigO{r/\log m}$ at most increase $\frac{\norm{\tldvvler}_2}{\norm{\tldvuler}_2}$ to $\frac{\norm{\tldvvler^\btt{0}}_2}{\norm{\tldvuler^\btt{0}}_2}\left(1+\BigO{\frac{\omega^5}{\log m}}\right)$. Setting $\delta_{norm,sd,2}\le \frac{\omega^6}{\log m}$ is enough.

    The bound on $\frac{\norm{\tldvvgtr}_2}{\norm{\tldvugtr}_2}$ follows similarly, we omit the details.
    
    \paragraph{Case: $T_{norm,pot,2}>\min\{T_{2},T_1+c_{21}r/\log m\}$.}
    This follows from \cref{lem: stage 2.1 neuron main result}.

    \paragraph{Case: $T_{pop2nd,21}>T_1+c_{21}r/\log m$.} This follows from \cref{lem: stage 2.1 population 0th 2nd dynamic}.
    
    \paragraph{Case: $T_{reg,21}>T_1+c_{21}r/\log m$.}
    Recall the following hold for any time $t\le T_{reg,21}$, (combine \cref{lem: stage 2.1 IH}\ref{item: stage 2.1 IH reg} and \cref{lem: stage 2 IH}\ref{item: stage 2 IH reg}) 
        \begin{enumerate}
            \item $\tldalpha \le \alpha+10\le r+10$
            \item $\ha_i \le 1/r$ for $i\in[r]$,
            \item $\frac{\norm{\tldvvler}}{\norm{\tldvv}}\ge 1-\frac{\omega^5}{r}$ for all neuron $\tldvv$
            \item $1-\frac{\lambda_0}{4\hsigma_0^2}\ge \tldE w_i^2 \ge 1/2$ for $i\in[r]$
            \item $\tldE \norm{\vwgtr}^2 \le \tldE \norm{\vwgtr^\btt{T_1}}^2\lesssim \log d$
            \item $\tldE_{\Sdense}[\norm{\vw}^2]\bw_i^4\ge \frac{1}{\omega r}$ for $i\le r$.
        \end{enumerate}        
        where recall $\ha_i = \frac{1}{m}\sum_{\vv\in G_i}\norm{\vv}_2^2$ and $G_i=\{\vv: \bv_i^2 \ge 1-\epsdir\}$

    For (a), it follows from \cref{lem: stage 2.1 population 0th 2nd dynamic} as when $\tldalpha=r$, we have $\frac{\rd\tldalpha}{\rd t} <0$;

    For (b), it follows from \cref{lem: stage 2.1 neuron main result} that all neurons norm remains bounded as $\norm{\vv}\le \sigma_1$ so all $\ha_i \le \sigma_1/m\le 1/r$.
    
    For (c), similar to the proof of \cref{lem: stage 1.2 0th 2nd order dynamic}, we have
    \begin{align*}
        \frac{\rd}{\rd t}\frac{\norm{\vvler}^2}{\norm{\vv}^2}
        \ge& 2\frac{\norm{\vvler}^2}{\norm{\vv}^2}\left(
            2\hsigma_2^2\left(1-\frac{\norm{\vvler}^2}{\norm{\vv}^2}\right)
            \sum_{i\le r}(\bbtldvvler)_i^2 
                \left( 1 - \tldE\left[w_i^2\right]\right)
            \pm \BigO{\frac{\omega^4}{r}}
            \right)
        \gtrsim \frac{\norm{\vvler}^2}{\norm{\vv}^2}\left(
                    1-\frac{\norm{\vvler}^2}{\norm{\vv}^2}
                    \pm \BigO{\frac{\omega^4}{r}}
                \right),
    \end{align*}
    where we use $\sum_{i\le r}(\bbtldvvler)_i^2 
                \left( 1 - \tldE\left[w_i^2\right]\right)\gtrsim 1$ from \cref{lem: stage 2.1 IH}\ref{item: stage 2.1 IH reg}.
    Thus, $\frac{\norm{\vvler}^2}{\norm{\vv}^2}$ stays above $1-\BigO{\omega^4/r}\ge 1-\omega^5/r$;
    
    For (d), it follows from \cref{lem: stage 2.1 population 0th 2nd dynamic} and at the boundary the derivative $\frac{\rd}{\rd t}\tldE w_i$ gives the right sign;

    For (e), it follows from \cref{lem: stage 2.1 population 0th 2nd dynamic} that $\frac{\rd}{\rd t}\tldE w_i^2<0$ for $i>r$, and at the end of Stage 1 we know from \cref{lem: stage 1.2 main} that $\tldE \norm{\vwgtr^\btt{T_1}}^2\lesssim \log d$;

    For (f), it follows from \cref{lem: stage 2.1 E bw^4 bound}.
    
    This finishes the proof.
\end{proof}

\subsection{Stage 2.2}\label{sec: stage 2.2}
In Stage 2.2, we further show that the neurons that aligned with ground-truth direction in previous stage now grows up to fit the norm and decrease the loss, up to the error induced by weight decay.

Denote the end of Stage 2.2 as
\[
    T_{22}:=\inf\{t\ge T_{21}: |\ha_i-\ha_*|\le \epsnm, \tldE_{\Sdense}w_i^2 \le \epsnm \text{ for all $i\le r$}\}
\]
where $\epsnm=e^{-\polylog(r)}$ and $\ha_*:=1-\frac{\lambda}{2\hsigma_0^2 r+2\hsigma_{\ge2}^2}$. In words, it says Stage 2.2 ends when we fit the norm using the target-aligned neurons and rest neurons become small. Denote $\tldT_{22}=\min\{\tldT_2, T_{22}, T_{21}+c_{22}\log(1/\epsnm)\}$ with a large enough constant $c_{22}$. We will show $\tldT_{22}=T_2$ so the induction hypothesis \cref{lem: stage 2 IH} holds throughout this stage.

\begin{lemma}\label{lem: stage 2.2 main}
    Stage 2.2 ends within time $\BigO{\log(1/\epsnm)}=\polylog(r)$. The following hold at the end of Stage 2.2:
    \begin{enumerate}
        \item Convergence: for $i\le r$, $|\ha_i-\ha_*|\le \epsnm$ and $\tldE_{\Sdense}w_i^2 \le \epsnm$, where $\epsnm=e^{-\polylog(r)}$ and $\ha_*:=1-\frac{\lambda}{2\hsigma_0^2 r+2\hsigma_{\ge2}^2}$.
        \item All condition in \cref{lem: stage 2 IH} hold.
    \end{enumerate}
\end{lemma}
\begin{proof}
    It follows from \cref{lem: stage 2.2 convergence} and \cref{sec: proof stage 2.2 IH}.
\end{proof}

\subsubsection{Dynamics and convergence}

Recall from \cref{lem: stage 2 population 0th 2nd dynamic} we have
\begin{align*}
    \frac{\rd \ha_i}{\rd t}
        = & \gamma_t 2\ha_i \Bigg(
                -\lambda
                +2\hsigma_0^2 \left(
                    r - \tldE[ \norm{\vw}_2^2 ]\right) 
                + 2\hsigma_{2}^2 \left(
                    1 - \tldE\left[ w_i^2 \right]
                \right)                               
                + 2\hsigma_{\ge4}^2(1-\ha_i)\\
                &\qquad- \underbrace{\sum_{j\ge 2}2\hsigma_{2j}^2\tldE_{\Sdense}\left[\norm{\vw}^2\bw_i^{2j}\right]}_{\BigO{\frac{\omega^4}{r^2}\tldE_{\Sdense} \norm{\vw}^2}}
                \pm \BigO{r\sqrt{\epsdir}}
                \Bigg), \\
    \frac{\rd}{\rd t} \tldE_{\Sdense}[w_i^2]
        =&2\tldE_{\Sdense}[w_i^2] \left(
                -\lambda
                + 2\hsigma_0^2 \left(
                r - \tldE[ \norm{\vw}_2^2 ]\right) 
                + 2\hsigma_{2}^2 \left(
                    \ind_{i \le r}
                    - \tldE\left[ w_i^2 \right]
                \right)
                \pm \BigO{\frac{\omega^4}{r}}
            \right)
\end{align*}
In the following, we will use $\hb_i := \tldE_{\Sdense}[w_i^2]$ and $\err:=\BigO{\frac{\omega^4}{r}}$ to simplify the notation. Further noticing that $\tldE w_i^2 \le 1/d^2$ for $i>r$, the above ODEs become
\begin{equation}\label{eq: stage 2.2 ode}
\begin{aligned}    
    \dot{\ha}_i
    =&
    2\gamma_t\ha_i\bigl(-\lambda+ 2\hsigma_0^2(r-S) + 2\hsigma_2^2(1-\ha_i-\hb_i)+2\hsigma_{\ge 4}^2(1-\ha_i) \pm\err\bigr),\\
    \dot{\hb}_i
    =&
    2\hb_i\bigl(-\lambda+ 2\hsigma_0^2(r-S) + 2\hsigma_2^2(1-\ha_i-\hb_i)\pm\err\bigr),
\end{aligned}
\qquad i=1,\dots,r,
\end{equation}
where
$S(t):=\sum_{i=1}^r (\ha_i(t)+\hb_i(t))$.

The lemma below show the dynamics of a short transit time between $T_{21}$ to $T_{21}+\BigO{\frac{\log m}{\gamma}}$ that we set the step size  $\gamma_t=\gamma=m$ for neurons with $\norm{\vv}\ge \sigma_1^2$ at time $T_{21}$. This is the only time that we change the step size.
\begin{lemma}\label{lem: stage 2.2 short norm fit phase}
    We have \cref{lem: stage 2 IH} hold for $T_{21}\le t \le T_\gamma:= T_{21}+\BigO{\frac{\log m}{\gamma}}$ and at time $T_\gamma$ we have $\ha_i=\BigTheta{1/r}$ and $\frac{\ha_i}{\ha_j}=1\pm\BigO{\frac{\omega^9}{\log m}}$.
\end{lemma}

 \begin{remark}\label{remark: step size}
        We comment on the choice of $\gamma_{t,i}$ in Stage~2.2. This is the only place where we slightly modify the stepsize, and it should be viewed as a technical convenience: we use it to fast-forward a relatively simple norm-fitting phase of the dynamics. By this point, the ground-truth directions have already been identified, and the remaining task is to adjust the neuron norms, increasing the mass of aligned neurons in $G_i$ while decreasing the mass of the non-aligned neurons.

        The main technical reason for this modification is the need to maintain the induction hypothesis in \cref{lem: stage 2 IH}\ref{item: stage 2 IH implication small and dense dir}, which requires $(\bbtldvvler)_i^2$ to remain small for small-and-dense neurons. From the dynamics in \cref{lem: stage 2 dynamic small and dense}, this quantity is affected by the second-order residual term $1-\tldE w_i^2$. This issue is specific to the standard-initialization setting: in previous settings where the second-order term is absent or the initialization is small, one typically has $\tldE w_i^2\approx0$, so this residual does not create the same difficulty.

        In the norm-fitting dynamics \eqref{eq: stage 2.2 ode}, the quantities $\hb_i$ may already be balanced across directions while the aligned masses $\ha_i$ are not yet balanced. During this transient period, the residual $1-\tldE w_i^2=1-\ha_i-\hb_i$ can fluctuate across different directions, which makes it harder to control the small-and-dense directions uniformly. Increasing the learning-rate multiplier in this short phase makes $\ha_i$ rapidly approach the value determined by $\hb_i$, after which the system becomes balanced again and the induction hypothesis is easier to maintain.

        This modification does not change the qualitative dynamics: it only accelerates a short norm-adjustment phase after direction recovery. Empirically, we observe the same behavior with vanilla gradient descent as in \cref{fig: illustration}. Removing this technical simplification is an interesting direction for future work.

    \end{remark}
\begin{proof}
    We begin by first arguing that all neurons with $\gamma_t=1$ (i.e., $\norm{\vv}\le \sigma_1^2$ at time $T_{21}$) stays almost as the same position. This is simply because the total time frame here $\BigO{\frac{\log m}{\gamma}}$ is so small that they barely move. Formally, from \cref{lem: stage 2 dynamic small and dense} and \cref{lem: stage 2 dynamic potential neuron} we know $|\frac{\rd}{\rd t}v_i|\le |v_i| \BigO{1}$. Thus, with in time $\BigO{\frac{\log m}{\gamma}}=\BigO{\frac{\log m}{m}}$ all conditions in \cref{lem: stage 2 IH} remain hold. In particular, the above implies that $\tldE_{\Sdense} w_i^2$ remains the same, and thus remain balance at $\frac{\hb_i}{\hb_j}=\frac{\tldE_{\Sdense} w_i^2}{\tldE_{\Sdense} w_i^2}=1\pm\BigO{\frac{\omega^5}{r}}$.

    In the rest of the proof, we focus on the neurons with large norm and thus using a larger step size $\gamma_t=\gamma$. From \eqref{eq: stage 2.2 ode} we know
    \begin{align*}
        \frac{\rd \ha_i}{\rd t}
        = & \gamma 2\ha_i \Bigg(
                -\lambda
                +2\hsigma_0^2 \left(
                    r - \tldE[ \norm{\vw}_2^2 ]\right) 
                + 2\hsigma_{2}^2 \left(
                    1 - \tldE\left[ w_i^2 \right]
                \right)                               
                + 2\hsigma_{\ge4}^2(1-\ha_i)
                + 2\hsigma_{4}^2\tldE_{\Sdense}\left[\norm{\vw}^2\bw_i^{4}\right]
                \pm \BigO{\frac{\omega^6}{r^2}}
                \Bigg).
    \end{align*}
    To solve the above ODEs, first we could get the dynamics of $\frac{1}{r}\sum_i \ha_i$ by averaging the dynamics over $\ha_i$, which is a logistic type ODE. In fact, this implies that it converges in $\BigO{\frac{\log m}{\gamma}}$ time (since the initial value is $\BigOmega{\frac{\sigma_1^4}{m}}$) to $\frac{1}{r}\sum_j \ha_j=\Theta(1/r)$. We also have for each $\ha_i$ and $\ha_j$ that
    \begin{align*}
        \frac{\rd }{\rd t}\frac{\ha_i}{\ha_j}
        = & \gamma 2\frac{\ha_i}{\ha_j} \Bigg(
                2\hsigma_{\ge2}^2 \left(
                    \ha_j - \ha_i    
                \right)
                +2\hsigma_{4}^2\tldE_{\Sdense}\left[\norm{\vw}^2\bw_i^{4}\right]
                -2\hsigma_{4}^2\tldE_{\Sdense}\left[\norm{\vw}^2\bw_j^{4}\right]
                \pm \BigO{\frac{\omega^6}{r^2}}
                \Bigg)\\
        =& \gamma 2\frac{\ha_i}{\ha_j} \Bigg(
                2\hsigma_{\ge2}^2 \left(
                    \ha_j - \ha_i    
                \right)
                \pm \BigO{\frac{\omega^9}{r\log m}}
                \Bigg).
    \end{align*}
    This further implies that $\ha_i=\frac{1}{r}\sum_j \ha_j\pm \BigO{\frac{\omega^9}{r\log m}}=\Theta(1/r)$, which further suggests that for $i\ne j$ we have $\frac{\ha_i}{\ha_j}=1\pm \BigO{\frac{\omega^9}{\log m}}$. In this way, the spread of $\ha_i$ has been improved.

    As for the the conditions in \cref{lem: stage 2 IH}, we only comment below on maintaining the direction of those aligned neuron as $\bv_i^2\ge 1-\epsdir$. All the rest directly follow from the fact that such time period is short and thus not moving too much.

    From \cref{lem: stage 2.2 align neuron dynamic} we know
    \begin{align*}
        \frac{\rd}{\rd t}\bv_i^2
        \ge& \gamma2\bv_i^2(1-\bv_i^2)\Bigg(
            2\hsigma_2^2\left(
                1-\tldE w_i^2
                - \sum_{\ell\neq i}\frac{\bv_\ell^2}{1-\bv_i^2}(1-\tldE w_\ell^2)
            \right)
            + 4\hsigma_{4}^2 
                \Bigl(
                1-\ha_i-o(1)
                \Bigr)
        \Bigg).
    \end{align*}
    For any time $\ha_i\le 0.01$, we know $\frac{\rd}{\rd t}\bv_i^2\ge 0$. Otherwise, we have the worst case bound $\frac{\rd}{\rd t}\bv_i^2\ge -\Theta(1)$. From \cref{lem: stage 2.1 main} we know the spread of $\ha_i$ is $m^{o(1)}$ so the case $\ha_i> 0.01$ takes at most $o(\frac{1}{\gamma}\log m)$ time and $\bv_i^2$ decreases at most by a factor of $m^{-o(1)}$. Since at the beginning of this stage we know from \cref{lem: stage 2.1 main} that $\bv_i^2\ge 1-\epsdir^2$ and $\epsdir=m^{-\Theta(1)}$, we conclude that $\bv_i^2\ge 1-\epsdir$ throughout this time.
\end{proof}

Here we show the dynamics of aligned neuron. The goal is to show it stabilizes in direction after it aligned with one of the target direction.
\begin{lemma}\label{lem: stage 2.2 align neuron dynamic}
    For $T_{21}\le t\le \tldT_{2}$, for aligned neuron $\tldvv\in G_i$ we have the dynamics
    \begin{align*}
        \frac{\rd}{\rd t}\bv_i^2
        \ge& \gamma_t 2\bv_i^2(1-\bv_i^2)\Bigg(
            2\hsigma_2^2\left(
                1-\tldE w_i^2
                - \sum_{\ell\neq i}\frac{\bv_\ell^2}{1-\bv_i^2}(1-\tldE w_\ell^2)
            \right)
            + 4\hsigma_{4}^2 
                \Bigl(
                1-\ha_i-o(1)
                \Bigr)
        \Bigg).
    \end{align*}
\end{lemma}
\begin{proof}
    From \cref{claim: tldw dynamic} one can directly compute
    \begin{align*}
        \frac{\rd}{\rd t}\bv_i^2
        =& \gamma_t2\bv_i^2\Bigg(
            2\hsigma_2^2\Bigl(1-\tldE w_i^2-\sum_l \bv_l^2(1-\tldE w_l^2)\Bigr)\\
            &\qquad+ \sum_{j\ge 2}\hsigma_{2j}^2 2j \cdot
                \Bigl(
                \bv_i^{2j-2}-\sum_{\ell}\bv_\ell^{2j}
                - \Bigl(
                    \tldE\!\left[\norm{\vw}^2\langle \bvw,\bvv\rangle^{2j-1}\frac{\bw_i}{\bv_i}\right]
                    - \tldE\!\left[\norm{\vw}^2\langle \bvw,\bvv\rangle^{2j}\right]
                \Bigr)
            \Bigr)
        \Bigg)
    \end{align*}
    We have
    \begin{align*}
        1-\tldE \bw_i^2-\sum_l \bv_l^2(1-\tldE \bw_l^2)
        =& (1-\bv_i^2)\left(
            1-\tldE w_i^2
            - \sum_{\ell\neq i}\frac{\bv_\ell^2}{1-\bv_i^2}(1-\tldE w_\ell^2)
        \right),\\
        \bv_i^{2j-2}-\sum_{\ell}\bv_\ell^{2j}
        =& (1-\bv_i^2) \left(
            \bv_i^{2j-2}
            - \sum_{\ell\neq i}\frac{\bv_\ell^2}{1-\bv_i^2}\,\bv_\ell^{2j-2}
        \right).
    \end{align*}

    For the last term, note that
    \begin{align*}
        \tldE\!\left[\norm{\vw}^2\langle \bvw,\bvv\rangle^{2j-1}\frac{\bw_i}{\bv_i}\right]
        =\frac{1}{\bv_i^2}\tldE\!\left[
            \norm{\vw}^2\langle \bvw,\bvv\rangle^{2j-1}\bw_i\bv_i
            \right]
        \le
        \frac{1}{\bv_i^2}\tldE\!\left[
            \norm{\vw}^2\langle \bvw,\bvv\rangle^{2j}
        \right].
    \end{align*}
    Thus, we have
    \begin{align*}
        \tldE\!\left[\norm{\vw}^2\langle \bvw,\bvv\rangle^{2j-1}\frac{\bw_i}{\bv_i}\right] - \tldE\!\left[\norm{\vw}^2\langle \bvw,\bvv\rangle^{2j}\right]
        \le \frac{1-\bv_i^2}{\bv_i^2}\tldE\!\left[
            \norm{\vw}^2\langle \bvw,\bvv\rangle^{2j}
            \right],
    \end{align*}
    where we use the conditional symmetry \cref{claim: tldw dynamic} that this term is even polynomial.
    For $\tldE\!\left[
            \norm{\vw}^2\langle \bvw,\bvv\rangle^{2j}
            \right]$ we further split it as
    \begin{align*}
        \tldE\!\left[
            \norm{\vw}^2\langle \bvw,\bvv\rangle^{2j}
            \right]
        = \tldE_{\Sdense}\left[
            \norm{\vw}^2\langle \bvw,\bvv\rangle^{2j}
            \right]
            + \sum_i\tldE_{\Spoti}\left[
            \norm{\vw}^2\langle \bvw,\bvv\rangle^{2j}
            \right]
        \le \omega\left(\frac{\omega}{r}\right)^{j-1}
            + \ha_i + \BigO{\sqrt{\epsdir}}.
    \end{align*}

    Therefore, for any $j\ge 2$ we have
    \begin{align*}
        &\bv_i^{2j-2}-\sum_{\ell}\bv_\ell^{2j}
                - \Bigl(
                    \tldE\!\left[\norm{\vw}^2\langle \bvw,\bvv\rangle^{2j-1}\frac{\bw_i}{\bv_i}\right]
                    - \tldE\!\left[\norm{\vw}^2\langle \bvw,\bvv\rangle^{2j}\right]
                \Bigr)\\
        \ge& (1-\ha_i)\bv_i^{2j-2} -  \omega\left(\frac{\omega}{r}\right)^{j-1} - \BigO{\sqrt{\epsdir}} - \ha_i (1-\bv_i^{2j-2})
    \end{align*}

    We now have
    \begin{align*}
        \frac{\rd}{\rd t}\bv_i^2
        \ge& \gamma_t2\bv_i^2(1-\bv_i^2)\Bigg(
            2\hsigma_2^2\left(
                1-\tldE w_i^2
                - \sum_{\ell\neq i}\frac{\bv_\ell^2}{1-\bv_i^2}(1-\tldE w_\ell^2)
            \right)
            + 4\hsigma_{4}^2 
                \Bigl(
                1-\ha_i-o(1)
                \Bigr)
            \Bigr)
            - \BigO{1}\sum_j j\hsigma_{2j}^2 (1-\bv_i^{2j}) 
        \Bigg).
    \end{align*}

    For $\sum_j j\hsigma_{2j}^2 (1-\bv_i^{2j}) $, let $j_0$ be chosen later, we have
    \begin{align*}
        \sum_j \hsigma_{2j}^2 (1-\bv_i^{2j}) 
        = \sum_{j\le j_0} \hsigma_{2j}^2 (1-\bv_i^{2j}) 
            + \sum_{j>j_0} \hsigma_{2j}^2 (1-\bv_i^{2j})
        \lesssim \epsdir \sum_{j\le j_0} j^{-1/2} 
            + j_0^{-1/2}.
    \end{align*}
    Choose $j_0=1/\epsdir$,we have the above term $\lesssim \sqrt{\epsdir}=o(1)$.

    Thus, we finally get
    \begin{align*}
        \frac{\rd}{\rd t}\bv_i^2
        \ge& \gamma_t2\bv_i^2(1-\bv_i^2)\Bigg(
            2\hsigma_2^2\left(
                1-\tldE w_i^2
                - \sum_{\ell\neq i}\frac{\bv_\ell^2}{1-\bv_i^2}(1-\tldE w_\ell^2)
            \right)
            + 4\hsigma_{4}^2 
                \Bigl(
                1-\ha_i-o(1)
                \Bigr)
        \Bigg).
    \end{align*}
\end{proof}

From now on the stepsize will keep as $\gamma=1$ again.

Define the initial logarithmic spread
\begin{equation*}
    \Delta_0
    :=
    \max_{1\le p,q\le r}
    \left(
    \left|\log\frac{\ha_i(0)}{\ha_j(0)}\right|
    +
    \left|\log\frac{\hb_i(0)}{\hb_j(0)}\right|
    \right).
\end{equation*}
From \cref{lem: stage 2.1 main} and \cref{lem: stage 2.2 short norm fit phase} we know $\Delta_0 \lesssim \frac{\omega^9}{\log m}$. The lemma below shows that such spread will not increase.
\begin{lemma}[Pairwise logarithmic spread is nonincreasing]\label{lem: stage 2.2 log spread}
    For each pair \(i,j\), define
    \[
    u_{ij}(t):=\log\frac{\ha_i(t)}{\ha_j(t)},
    \qquad
    v_{ij}(t):=\log\frac{\hb_i(t)}{\hb_j(t)},
    \qquad
    M_{ij}(t):=|u_{ij}(t)|+|v_{ij}(t)|.
    \]
    Then, for $T_{21}\le t\le \tldT_{22}$, \(M_{ij}\) is (almost) nonincreasing on \([0,\infty)\). In particular,
    \begin{equation*}
        M_{ij}(t)\le M_{ij}(0)+\err \cdot t
        =
        \left(\left|\log\frac{\ha_i(0)}{\ha_j(0)}\right|
        +
        \left|\log\frac{\hb_i(0)}{\hb_j(0)}\right|\right) + \err \cdot t.
    \end{equation*}
    
    As a result, For all \(i,j\) and all \(T_\gamma\le t\le \tldT_{22}\), we have
    \begin{equation*}
        \left|\log\frac{\ha_i(t)}{\ha_j(t)}\right|\le 2\Delta_0,
        \qquad
        \left|\log\frac{\hb_i(t)}{\hb_j(t)}\right|\le 2\Delta_0.
    \end{equation*}
\end{lemma}

\begin{proof}
    Note that
    \[
        \ha_i-\ha_j=\ha_j(e^{u_{ij}}-1),
        \qquad
        \hb_i-\hb_j=\hb_j(e^{v_{ij}}-1),
    \]
    so
    \[
        x_i-x_j
        =
        (\ha_i-\ha_j)+(\hb_i-\hb_j)
        =
        \ha_j(e^{u_{ij}}-1)+\hb_j(e^{v_{ij}}-1).
    \]
    Thus by \eqref{eq: stage 2.2 ode},
    \begin{align*}    
        \dot u_{ij}
        =&
        \frac{\dot{\ha}_i}{\ha_i}-\frac{\dot{\ha}_j}{\ha_j}
        =
        -2\hsigma_2^2(x_i-x_j)-2\hsigma_{\ge4}^2(\ha_i-\ha_j)\pm\err
        =
        -2\hsigma_{\ge2}^2\ha_j(e^{u_{ij}}-1)-2\hsigma_2^2\hb_j(e^{v_{ij}}-1) \pm \err,\\
        \dot v_{ij}
        =&
        \frac{\dot{\hb}_i}{\hb_i}-\frac{\dot{\hb}_j}{\hb_j}
        =
        -2\hsigma_2^2(x_i-x_j) \pm \err
        =
        -2\hsigma_2^2\ha_j(e^{u_{ij}}-1)-2\hsigma_2^2\hb_j(e^{v_{ij}}-1) \pm \err.
    \end{align*}
    
    We now inspect the four sign quadrants case by case below.
    
    If \(u_{ij}\ge0\) and \(v_{ij}\ge0\), then
    \[
    \frac{d}{dt}M_{ij}
    =
    \dot u_{ij}+\dot v_{ij}
    =
    -3\ha_j(e^{u_{ij}}-1)-2\hb_j(e^{v_{ij}}-1)\le \err.
    \]
    
    If \(u_{ij}\le0\) and \(v_{ij}\le0\), then
    \[
    \frac{d}{dt}M_{ij}
    =
    -\dot u_{ij}-\dot v_{ij}
    =
    3\ha_j(e^{u_{ij}}-1)+2\hb_j(e^{v_{ij}}-1)\le \err.
    \]
    
    If \(u_{ij}\ge0\ge v_{ij}\), then
    \[
    \frac{d}{dt}M_{ij}
    =
    \dot u_{ij}-\dot v_{ij}
    =
    -\ha_j(e^{u_{ij}}-1)\le\err.
    \]
    
    If \(u_{ij}\le0\le v_{ij}\), then
    \[
    \frac{d}{dt}M_{ij}
    =
    -\dot u_{ij}+\dot v_{ij}
    =
    \ha_j(e^{u_{ij}}-1)\le\err.
    \]

    Thus in every quadrant the upper Dini derivative of \(M_{ij}\) is at most $\err$, hence \(M_{ij}\) additive error $\err \cdot t$, which is small when $\err \cdot t<1$. The upper bound then directly follows.
\end{proof}

Finally, we are ready to show the convergence.
\begin{lemma}[Convergence]\label{lem: stage 2.2 convergence}
    For $T_{21}\le t\le \tldT_{22}$, within $\BigO{\frac{1}{\lambda_0}\log(1/\epsnm)}$ time, we have $|\ha_i-\ha_*|\le \epsnm$ and $0\le \hb_i\le \epsnm$, where $\epsnm=e^{-\polylog(r)}$.
\end{lemma}

\begin{proof}
    We focus on the dynamics after $T_\gamma$, since the short interval $T_{21}\le t\le T_\gamma$ is already characterized in \cref{lem: stage 2.2 short norm fit phase}. Recall that $\gamma_t=1$ for $t\ge T_\gamma$, and at $T_\gamma$ we have $\ha_i=\Theta(1/r)$ for all $i\in[r]$.
    
    We first note that $1-\ha_i\ge (1-\ha^*)/2=\Theta(\lambda_0)$ throughout this stage. Indeed, by \cref{lem: stage 2.2 log spread}, $\ha_i/\ha_j=1+o(1)$ for all $i,j$. At the boundary $1-\ha_i=(1-\ha^*)/2$, we have $S\ge r\ha_i(1-o(1))$, and hence from~\eqref{eq: stage 2.2 ode}
    \[
    \frac{\dot{\ha}_i}{2\ha_i}
    \le
    -\lambda+(2\hsigma_0^2r+2\hsigma_{\ge2}^2)(1-\ha_i)+o(\lambda)
    =
    -\frac{\lambda}{2}+o(\lambda)<0,
    \]
    where we use $1-\ha^*=\lambda/(2\hsigma_0^2r+2\hsigma_{\ge2}^2)$. Thus the dynamics points inward at the boundary. Also, $\hat b_i\ge0$ is preserved since its dynamics is multiplicative.
    
    We first show that the dense mass becomes small. Using~\eqref{eq: stage 2.2 ode}, the common terms in the dynamics of $\ha_i$ and $\hat b_i$ cancel, giving
    \[
    \frac{\rd}{\rd t}\log\frac{\hat b_i}{\ha_i}
    =
    -4\hsigma_{\ge4}^2(1-\ha_i)\pm O(\mathrm{err})
    \lesssim -\lambda_0,
    \]
    where $\mathrm{err}=O(\omega^4/r)=o(\lambda_0)$. At $T_\gamma$, we have $\ha_i=\Theta(1/r)$ by \cref{lem: stage 2.2 short norm fit phase} and $\hat b_i=O(1)$ by the regularity condition. Therefore, within time $O(\lambda_0^{-1}\log(r/\epsnm))$, we have $\hat b_i/\ha_i\le \epsnm/r$, and thus $\hat b_i\le\epsnm/r$ for every $i$. Moreover, the ratio $\hat b_i/\ha_i$ continues to decrease, so this bound is preserved afterwards.
    
    We next show that the aligned masses become balanced. Let $T_b$ be the first time that $\max_i\hat b_i\le\epsnm/r$. Then $\sum_i\hat b_i\le\epsnm$ for all $t\ge T_b$. Using the refined error bound in \cref{lem: stage 2 population 0th 2nd dynamic}, the contribution from the dense neurons to the $\ha_i$ dynamics is now $O(\epsnm)$ uniformly over $i$. Also, the regularity condition $\ha_i+\hat b_i\ge1/2$ gives $\ha_i\ge1/3$ for sufficiently small $\epsnm$.
    
    Define $U(t):=\max_i\log\ha_i-\min_i\log\ha_i$. Let $i_+$ and $i_-$ attain the maximum and minimum, respectively. Using~\eqref{eq: stage 2.2 ode} and canceling the common term containing $S$, the upper Dini derivative satisfies
    \[
    D^+U
    \le
    -4\hsigma_{\ge2}^2(\ha_{i_+}-\ha_{i_-})
    +O(\epsnm)
    \le-\Theta(U)+O(\epsnm,
    \]
    where we use $\ha_{i_-}\ge1/3$ and $\ha_{i_+}-\ha_{i_-}=\ha_{i_-}(e^U-1)\gtrsim U$. Hence within an additional $O(\log(1/\epsnm))$ time we have $U=O(\epsnm)$. In particular, writing $\bar a:=\frac1r\sum_i\ha_i$, we have $\ha_i=\bar a\pm O(\epsnm)$ for every $i$, and this balance is preserved afterwards.
    
    It remains to determine the common scale $\bar a$. Since $\sum_i\hat b_i\le\epsnm$, we have $S=r\bar a+O(\epsnm)$. Averaging~\eqref{eq: stage 2.2 ode} over $i$ and using $\ha_i=\bar a\pm O(\epsnm)$ gives
    \[
    \frac{\rd}{\rd t}\bar a
    =
    2(2\hsigma_0^2r+2\hsigma_{\ge2}^2)
    \bar a(\ha^*-\bar a)
    +O(\epsnm),
    \]
    where $\ha^*=1-\lambda/(2\hsigma_0^2r+2\hsigma_{\ge2}^2)$. Since $\bar a\ge1/3$, this is a contracting logistic-type dynamics around $\ha^*$, and therefore within an additional $O(r^{-1}\log(1/\epsnm))$ time we have $|\bar a-\ha^*|=O(\epsnm)$. Combining this with $\ha_i=\bar a\pm O(\epsnm)$ gives $|\ha_i-\ha^*|=O(\epsnm)$ for every $i$.
    
    The total time after $T_\gamma$ is $O(\lambda_0^{-1}\log(r/\epsnm))=O(\log(1/\epsnm))$, where we use that $\lambda_0$ is a fixed numerical constant and $\epsnm=e^{-\polylog(r)}$. Replacing $\epsnm$ by a sufficiently small constant multiple in the intermediate thresholds gives the stated bounds $|\ha_i-\ha^*|\le\epsnm$ and $\hat b_i\le\epsnm$ for every $i$.

\end{proof}

\subsubsection{Proof of Induction Hypothesis \cref{lem: stage 2 IH} in Stage 2.2}\label{sec: proof stage 2.2 IH}
\begin{proof}[Proof of \cref{lem: stage 2 IH} in Stage 2.2]
    We show one by one.
    Recall $\tldT_2:=\min\{T_{dir,2},T_{norm,sd,2},T_{norm,pot,2}, T_{stab,2}, T_{reg,2}\}$ defined in \cref{lem: stage 2 IH} and $\tldT_{22}=\min\{\tldT_2, T_{22}, T_{21}+c_{22}\log(1/\epsnm)\}$ with a large enough constant $c_{22}$.
    
    First, by \cref{lem: stage 2.2 convergence} we know $T_{22}< T_{21}+\BigO{\log\frac{1}{\epsnm}}$, so in below, we show that $T_{21}+c_{22}\log(1/\epsnm) < \tldT_2=\min\{T_{dir,2},T_{norm,sd,2},T_{norm,pot,2}, T_{stab,2}, T_{reg,2}\}$ so $\tldT_{22}=T_{22}$.

    We focus on the time $T_\gamma\le t \le T_{21}+c_{22}\log(1/\epsnm)$, as for the time of $T_{21}\le t \le T_\gamma$ is shown in \cref{lem: stage 2.2 short norm fit phase}.

    We argue each one by one as below.
    We first present one fact that will be frequently used below
    by using \cref{lem: stage 2.2 log spread} that for any $i,j\le r$ and any time $t\ge T_\gamma$
    \begin{align}\label{eq: stage 2.2 IH tld E bound}
            \left( 1 - \tldE\left[w_i^2\right]\right) - 
            \left( 1 - \tldE\left[w_j^2\right]\right) 
            \le  |\ha_i-\ha_j| + |\hb_i-\hb_j|
            \lesssim \Delta_0
            \lesssim \frac{\omega^9}{\log m}.
    \end{align}
            
    \paragraph{Case: $T_{dir,2}>T_{21}+c_{22}\log(1/\epsnm)$. }
    Recall the definition of $T_{dir,2}$: For any time $t\le T_{dir,2}$ we have
    \begin{itemize}
            \item for $i> r$, $\left(\bbtldvvgtr\right)_i^2 \le \frac{\log^2 d}{d}$
            \item for $i\le r$ and any neuron $\tldvv\not\in \Spoti$,  $\left(\bbtldvvler\right)_i^2 \le \omega\left(\bbtldvvler^\btt{0}\right)_i^{2} $
    \end{itemize}
    where $\omega=\log\log\log r$. 

    By \cref{lem: stage 2 dynamic small and dense} and \cref{lem: stage 2 dynamic potential neuron} we know for $\vv\not\in\Spoti$
    \begin{align*}
        \frac{\rd (\bbtldvvler)_i^2}{\rd t}
        \le& 2(\bbtldvvler)_i^2 \Bigg( 
                2\hsigma_{2}^2 \left(
                    1 - \tldE\left[ w_i^2 \right]
                    - \sum_{j\le r}(\bbtldvvler)_j^2 
                    \left( 1 - \tldE\left[w_j^2\right]\right)\right)
                + 4\hsigma_4^2 (1-\ha_i)\btldv_i^2\ind_{i\le r}
                \pm \BigO{\frac{\omega^4}{r}}
            \Bigg),
    \end{align*}
    so using \eqref{eq: stage 2.2 IH tld E bound}
    \begin{align*}
        (\bbtldvvler)_i^2
        \le& (\bbtldvvler^\btt{T_{21}})_i^2 \exp\left(\BigO{\frac{\omega^9\log(1/\epsnm)}{\log m}+\frac{\log d\log(1/\epsnm)}{r}}
            \Bigg)
            \right)
        \lesssim (\bbtldvvler^\btt{T_{21}})_i^2
        \le \omega^{0.5}(\bbtldvvler^\btt{0})_i^2.
    \end{align*}
    Similar argument can show the bound for $(\bbtldvvgtr)_i^2$, we omit the details.

    \paragraph{Case: $T_{norm,sd,2}>T_{21}+c_{22}\log(1/\epsnm)$.}
    Recall the definition of $T_{norm,sd,2}$: For any time $t\le T_{norm,sd,2}$ we have for any neuron $\tldvv,\tldvu\in \Sdense$, 
        \[\frac{\norm{\tldvvler}_2}{\norm{\tldvuler}_2} = \frac{\norm{\tldvvler^\btt{0}}_2}{\norm{\tldvuler^{\btt{0}}}_2}\left(1\pm \delta_{norm,sd,2}\right),\quad \frac{\norm{\tldvvgtr}_2}{\norm{\tldvugtr}_2} = \frac{\norm{\tldvvgtr^\btt{0}}_2}{\norm{\tldvugtr^{\btt{0}}}_2}\left(1\pm \delta_{norm,sd,2}\right)\]
        where $\delta_{norm,sd,2}\le \omega$.

    By \cref{lem: stage 2 dynamic small and dense} we have for any two neurons $\vv,\vu\in \Sdense$
    \begin{align*}
        \frac{\rd}{\rd t} \frac{\norm{\vvler}}{\norm{\vuler}}
        =2\hsigma_2^2 \frac{\norm{\vvler}}{\norm{\vuler}} \left(
            \sum_{i\le r}(\bbtldvvler)_i^2 
                \left( 1 - \tldE\left[w_i^2\right]\right)
            -\sum_{i\le r}(\bbtldvuler)_i^2 
                \left( 1 - \tldE\left[w_i^2\right]\right)
            \pm \BigO{\frac{\omega^4}{r}}
            \right).
    \end{align*}
    Thus, by using \eqref{eq: stage 2.2 IH tld E bound} we have
    \begin{align*}
        \frac{\norm{\vvler}}{\norm{\vuler}}
        \le \frac{\norm{\vvler^\btt{T_{21}}}}{\norm{\vuler^\btt{T_{21}}}}\exp\left(         
            \BigO{\frac{\omega^9\log(1/\epsnm)}{\log m}}
            \right)
        \le 2\frac{\norm{\vvler^\btt{0}}}{\norm{\vuler^\btt{0}}}
    \end{align*}
    Similar argument applies for $\frac{\norm{\vvgtr}}{\norm{\vugtr}}$, we omit the details.

    \paragraph{Case: $T_{norm,pot,2}>T_{21}+c_{22}\log(1/\epsnm)$.}
    Recall the definition of $T_{norm,pot,2}$: For any time $t\le T_{norm,pot,2}$, we have
        for $i\in[r]$, if neuron $\tldvv \in \Spoti$, then 
        \begin{align*}
                &\text{either $\norm{\tldvv}_2 \le \sigma_1$
                or $\btldv_i^2 \ge 1-\epsdir$}, 
        \end{align*}
        where $\epsdir:=\frac{\omega r}{\sigma_1^2}$.

        For potential neuron that are not aligned with any ground-truth, that is $\vv\in \Spoti/G_i$, we bound $\norm{\vv_{\le r, -i}}$ below, where $\vv_{\le r, -i}$ means we only take the first $r$ coordinate but excluding $i$-th coordinate for $\vv$. If we could bound $\norm{\vv_{\le r, -i}}\le  \sqrt{\omega r}$, then we know for any neuron $\norm{\vv}\ge \sigma_1$, it must have $\bv_i^2\ge 1-\epsdir$, which finishes the proof.

        To bound $\norm{\vv_{\le r, -i}}$, note that each coordinate is the same as a coordinate of a small and dense neuron. Thus, with the same calculation for $\frac{\rd}{\rd t}\norm{\vvler}$ in \cref{lem: stage 2.1 dynamic small and dense}, we can show similarly that $\frac{\norm{\vv_{\le r, -i}}}{\norm{\vu_{\le r, -i}}}\le 2$ for any neuron $\vv\not\in G_i$. Thus, we have $\norm{\vv_{\le r, -i}}\le \sqrt{\omega r}$.

    \paragraph{Case: $T_{stab,2}>T_{21}+c_{22}\log(1/\epsnm)$.}
    Recall the definition of $T_{stab,2}$: For any time $T_{21}\le t\le T_{stab,2}$ we have for any $\tldvv\in G_i$ and $i\in[r]$
        \[
            \btldv_i^2 \ge 1-\epsdir.
        \]

    Using \cref{lem: stage 2.2 align neuron dynamic} we know
    \begin{align*}
        \frac{\rd}{\rd t}\bv_i^2
        \ge& \gamma_t 2\bv_i^2(1-\bv_i^2)\Bigg(
            2\hsigma_2^2\left(
                1-\tldE w_i^2
                - \sum_{\ell\neq i}\frac{\bv_\ell^2}{1-\bv_i^2}(1-\tldE w_\ell^2)
            \right)
            + 4\hsigma_{4}^2 
                \Bigl(
                1-\ha_i-o(1)
                \Bigr)
        \Bigg)\\
        \ge& \gamma_t 2\bv_i^2(1-\bv_i^2)\Bigg(
            -\BigO{\frac{\omega^9}{\log m}}
            + \BigTheta{\lambda_0}
        \Bigg)
        \ge 0.
    \end{align*}
    Thus, since $\bv_i^2\ge 1-\epsdir$ at $T_\gamma$ by \cref{lem: stage 2.2 short norm fit phase}, we know it keeps true.
        
    \paragraph{Case: $T_{reg,2}>T_{21}+c_{22}\log(1/\epsnm)$.}
    Recall the definition of $T_{reg,2}$: For any time $t\le T_{reg,2}$, we have 
        \begin{enumerate}
            \item $\tldalpha \le \alpha+10 = r+10$
            \item $\ha_i \le 1$ for $i\le r$
            \item $\frac{\norm{\tldvvler}}{\norm{\tldvv}}\ge 1-\frac{\omega^5}{r}$ for all neuron $\tldvv$
            \item $2\ge \tldE w_i^2 \ge 0.1$ for $i\in[r]$. 
            \item $\tldE \norm{\vwgtr}^2 \le \tldE \norm{\vwgtr^\btt{T_1}}^2\lesssim \log d$
        \end{enumerate}        

    For (a)(b)(d), first from \eqref{eq: stage 2.2 ode} it is easy to see $\ha_i\le 1$ and $\hb_i=\tldE_{\Sdense} w_i\le 1$ by an induction to check the boundary derivative. Thus, we further have $\tldE w_i^2 \le \ha_i+\tldE_{\Sdense} w_i^2 \le 2$. From \cref{lem: stage 2 population 0th 2nd dynamic} we also can see $-\lambda+ 2\hsigma_0^2(r-\tldalpha)\ge - \hsigma_{\ge2}^2$, which implies $\tldalpha \le r-\lambda/(2\hsigma_0^2) + \hsigma_{\ge2}^2/(2\hsigma_0^2)\le r$. Finally, since $-\lambda+ \hsigma_0^2(r-\tldalpha)\ge - \hsigma_{\ge2}^2$, again by \cref{lem: stage 2 population 0th 2nd dynamic} we can see $\tldE w_i^2 \ge 0.1$.

    For (c), for the short time $T_{21}\le t\le T_\gamma$: for neuron in $G_i$, it follows from the stability above that $\bv_i^2\ge 1-\epsdir$; for the rest neurons, it follows from such time period is short so they don't move much. For $t\ge T_\gamma$, it follows the same as in the proof in \cref{sec: proof stage 2.1 IH} for regularity and using $\tldE w_i^2$ are balanced again during this time.

    For (e), it follows from \cref{lem: stage 2 population 0th 2nd dynamic} that $\frac{\rd}{\rd t}\tldE w_i^2<0$ for $i>r$, and at the end of Stage 1.2 we know from \cref{lem: stage 1.2 main} that $\tldE \norm{\vwgtr^\btt{T_1}}^2\lesssim \log d$.

    This finishes the proof.
\end{proof}

\subsection{Stage 2.3: achieve small loss without weight decay}\label{sec: stage 2.3}
We have shown in Stage 2.2 that the aligned neurons norm $\ha_i$ grows up and rest of the neurons decrease to small value. In this section, we further show that by setting weight decay $\lambda=0$, we decrease the loss to a small error.

We now set weight decay $\lambda_t=0$ for $t\ge T_{22}$.
Similar to \eqref{eq: stage 2.2 ode} in Stage 2.2, from \cref{lem: stage 2 population 0th 2nd dynamic} we again have that
\begin{equation}\label{eq: stage 2.3 ode}
\begin{aligned}    
    \dot{\ha}_i
    =&
    2\ha_i\bigl(2\hsigma_0^2(r-S) + 2\hsigma_2^2(1-\ha_i-\hb_i)+2\hsigma_{\ge 4}^2(1-\ha_i) \pm\BigO{\frac{\omega^4}{r^2}\sum_i\hb_i + r\sqrt{\epsdir}}\bigr),\\
    \dot{\hb}_i
    =&
    2\hb_i\bigl(2\hsigma_0^2(r-S) + 2\hsigma_2^2(1-\ha_i-\hb_i)\pm\BigO{\frac{\omega^4}{r}}\bigr),
\end{aligned}
\qquad i=1,\dots,r,
\end{equation}
where $\hb_i=\tldE_{\Sdense} w_i^2$, $S=\sum_i \ha_i+\hb_i$.

Intuitively, given that at the end of Stage 2.2 we know $\ha_i\approx \ha_*$ and $\hb_i\lesssim \epsnm$, following above \eqref{eq: stage 2.3 ode} $\ha_i$ will grow quickly to become 1 while $\hb_i$ still remain small. This leads to we finally achieve small loss. We formalize below
\begin{lemma}[Main result for Stage 2.3]\label{lem: stage 2.3 main}
    Stage 2.3 ends within time $T_{23}-T_{22}\lesssim \frac{\log(1/\epsnm)}{r}$. We have $\ha_i=1\pm \BigO{\epsnm^{1/3}}$ and $\hb_i\lesssim \epsnm^{1/2}$. As a result, we have loss $\L(\tldmW)\lesssim r \epsnm^{2/3}\le \epsnm^{1/2}$.
\end{lemma}
\begin{proof}
    We consider the following induction that $\hb_i\le \epsnm^{1/2}$ before Stage 2.3 ends ($\ha_i=1-\BigTheta{\epsnm^{1/3}}$).
    Then \eqref{eq: stage 2.3 ode} becomes
    \begin{align*}    
    \dot{\ha}_i
    =& 2\ha_i\bigl(2\hsigma_0^2(r-\sum_i \ha_i) + 2\hsigma_{\ge2}^2(1-\ha_i)\pm\BigO{r\epsnm^{1/2}}\bigr),\\
    \dot{\hb}_i
    =&
    2\hb_i\bigl(2\hsigma_0^2(r-\sum_i\ha_i) + 2\hsigma_2^2(1-\ha_i)\pm\BigO{\frac{\omega^4}{r}}\bigr),
    \end{align*}
    
    First, we observe that $\ha_i$ preserves balance:
    \begin{align*}
        \frac{\rd}{\rd t}\frac{\ha_i}{\ha_j}
        = \frac{\ha_i}{\ha_j}\left(
                4\hsigma_{\ge2}^2(\ha_j-\ha_i)
                \pm \BigO{r\epsnm^{1/2}}
            \right).
    \end{align*}
    Since we know at the end of Stage 2.2 that $\ha_i=\ha_*+\BigO{\epsnm}$ from \cref{lem: stage 2.2 main}, we know $\frac{\ha_i}{\ha_j}=1\pm \BigO{r\epsnm^{1/2}}$. This further simplifies the dynamics into
    \begin{align*}
        \dot{\ha}
    =& 2\ha\bigl(2\hsigma_0^2(r-r\ha) + 2\hsigma_{\ge2}^2(1-\ha)\pm\BigO{r^2\epsnm^{1/2}}\bigr),
    \end{align*}
    where $\ha=\frac{1}{r}\sum_i \ha_i$ and $\frac{\ha_i}{\ha}=1\pm \BigO{r\epsnm^{1/2}}$.
    This suggests $\ha=1-\BigO{\epsnm^{1/3}}$ within time $\BigO{\frac{\log(1/\epsnm)}{r}}$.

    During this time, we know
    \[
        \frac{\rd}{\rd t}\hb_i \le \hb_i \left(\frac{\rd}{\rd t}\log \ha \pm\BigO{\frac{\omega^4}{r}}\right), 
    \]
    which implies $\hb_i \lesssim \hb_i(T_{22})\lesssim \epsnm\le \epsnm^{1/2}$. This finishes the induction.
\end{proof}

\begin{proof}[Proof of Induction Hypothesis \cref{lem: stage 2 IH} in Stage 2.3]
    This follows the exact same argument as in \cref{sec: proof stage 2.2 IH}.
\end{proof}
\subsection{Technical lemmas for Stage 2}

The lemma below is used to simplify the dynamics for small-and-dense neuron.
\begin{lemma}\label{lem: stage 2 dynamic bound}
    For $T_1\le t \le \tldT_{2}$, for any neuron $\vv$ we have for $j\ge 2$
    \begin{align*}
        \tldE\left[\norm{\vw}_2^2\langle \bvw,\bvv\rangle^{2j}\right]
        &= \left\{\begin{array}{ll}
          \ha_i\bv_i^{2j}\ind_{\vv\in\Spoti} + \BigO{\frac{\omega^2}{r}+jr\sqrt{\epsdir}}   &  ,\ j\ge2\\
          \le \ha_i\bv_i^{4}\ind_{\vv\in\Spoti} + \BigO{\frac{\omega^2}{r}}.   &,\ j\ge2 
        \end{array}\right.,\\
        \tldE [\norm{\vw}_2^2\langle\bvw,\bvv\rangle^{2j-1}\bw_i/\bv_i]
        &=\left\{\begin{array}{ll}
        \ha_i\bv_i^2 + \BigO{\frac{\omega^2}{r}} &,\ j=2\\
        \ha_i\bv_i^{2j-2}\ind_{\vv\in\Spoti} + \BigO{\frac{\sigma_1^2}{m^{0.9}} \frac{1}{|\bv_i|} + \frac{\sqrt{\epsdir}}{|\bv_i|} +  \frac{\omega^2}{r}+j\sqrt{\epsdir}} &,\ j\ge 3 \\
        \le \ha_i\bv_i^{4}\ind_{\vv\in\Spoti} + \BigO{\frac{\sigma_1^2}{m^{0.9}} \frac{1}{|\bv_i|} + \frac{\sqrt{\epsdir}}{|\bv_i|} +  \frac{\omega^2}{r}} &,\ j\ge 3 
        \end{array}\right. .
    \end{align*}
\end{lemma}
\begin{proof}
    We show one by one. The overall strategy is to consider different types of neurons separately. All neurons can be partitioned into small-and-dense neurons and potential neurons, i.e., $\Sdense \cup \Spot$. For a potential neuron in $\Spot$, according to \cref{lem: stage 2 IH}\ref{item: stage 2 IH norm bound or basis like}, it either belongs to one of the sets $G_i$ (aligned with the ground-truth vector $\ve_i$), or its norm is upper bounded by $\sigma_1^2$.

    The actual proof below may seem cumbersome. Its purpose is simply to obtain the correct order of the bound we need.
    
    \paragraph{Bound $\tldE \left[\norm{\vw}_2^2\langle \bvw,\bvv\rangle^{2j}\right]$}

    We first split the neurons into $\Sdense$ and $\Spot$
    \begin{align*}
        \tldE \left[\norm{\vw}_2^2\langle \bvw,\bvv\rangle^{2j}\right]
        = \frac{|\Spot|}{m}\tldE_{\Spot}\left[\norm{\vw}_2^2\langle \bvw,\bvv\rangle^{2j}\right]
            + \left(1-\frac{|\Spot|}{m}\right)
            \tldE_{\Sdense}\left[\norm{\vw}_2^2\langle \bvw,\bvv\rangle^{2j}\right].
    \end{align*}
    
    By \cref{lem: high order moment bound}, we can bound the $\Sdense$ part for $j\ge 2$ using \cref{lem: stage 2 IH}\ref{item: stage 2 IH neuron dir}
    \begin{align*}
        \tldE_{\Sdense}\left[\norm{\vw}_2^2\langle \bvw,\bvv\rangle^{2j}\right]
        \le \tldE_{\Sdense}\left[\norm{\vw}_2^2\langle \bvw,\bvv\rangle^{4}\right]
        \lesssim \frac{\omega^{2}}{r}.
    \end{align*}

    For $\Spot$ part, we have
    \begin{align*}
        \frac{|\Spot|}{m}\tldE_{\Spot}\left[\norm{\vw}_2^2\langle \bvw,\bvv\rangle^{2j}\right]
        \myeq{a} \sum_{i\in[r]}\frac{|G_i|}{m}\tldE_{G_i}\left[\norm{\vw}_2^2\langle \bvw,\bvv\rangle^{2j}\right]
            \pm \frac{\sigma_1^2}{m^{0.9}}
        \myeq{b} \sum_{i\in[r]}\ha_i\bv_i^{2j} + \BigO{jr\sqrt{\epsdir}+\frac{\sigma_1^2}{m^{0.9}}}.
    \end{align*}
    where (a) recall potential neurons either algins with one of ground-truth (in $G_i$) or the norm is bounded by $\sigma_1$ as \cref{lem: stage 2 IH}\ref{item: stage 2 IH norm bound or basis like}, and $\ha_i = \frac{1}{m}\sum_{\vv\in G_i}\norm{\vv}_2^2$ is the norm in the fitted direction; 
    (b) \cref{lem: <>2j bound} and \cref{lem: stage 2 IH}\ref{item: stage 2 IH reg}. 
    
    Note that $jr\sqrt{\epsdir}$ comes from applying \cref{lem: <>2j bound} on $\tldE_{G_i}[\norm{\vw}^2\langle\bvw,\bvv\rangle^{2j}]$. So
    another bound that does not depend on $j$ is by bounding $\tldE_{G_i}[\norm{\vw}^2\langle\bvw,\bvv\rangle^{2j}]\le \tldE_{G_i}[\norm{\vw}^2\langle\bvw,\bvv\rangle^{4}]$,
    which leads to
    \begin{align*}
        \frac{|\Spot|}{m}\tldE_{\Spot}\left[\norm{\vw}_2^2\langle \bvw,\bvv\rangle^{2j}\right]
        \le \sum_{i\in[r]}\ha_i\bv_i^{4} + \BigO{\frac{\sigma_1^2}{m^{0.9}} + r\sqrt{\epsdir}}.
    \end{align*}

    Hence, for $j\ge 2$, we have
    \begin{align*}
        \tldE\left[\norm{\vw}_2^2\langle \bvw,\bvv\rangle^{2j}\right]
        = \left\{\begin{array}{ll}
          \ha_i\bv_i^{2j}\ind_{\vv\in\Spoti} + \BigO{\frac{\omega^2}{r}+jr\sqrt{\epsdir}}   &  ,\ j\ge2\\
          \le \ha_i\bv_i^{4}\ind_{\vv\in\Spoti} + \BigO{\frac{\omega^2}{r}}.   &,\ j\ge2 
        \end{array}\right..
    \end{align*}

    \paragraph{Bound $\tldE[\norm{\vw}_2^2\langle\bvw,\bvv\rangle^{2j-1}\bw_1/\bv_1]$} (WLOG, consider the case $i=1$).

    \noindent\textit{Case 1: $2\le j\le r/\log r$.} For $2\le j\le r/\log r$, we again first split the neuron into  $\Sdense$ and $\Spot$ 
    \begin{align*}
        \tldE [\norm{\vw}_2^2\langle\bvw,\bvv\rangle^{2j-1}\bw_1/\bv_1]
        = \frac{|\Spot|}{m}\tldE_{\Spot} [\norm{\vw}_2^2\langle\bvw,\bvv\rangle^{2j-1}\bw_1/\bv_1]
            + \left(1-\frac{|\Spot|}{m}\right)
            \tldE_{\Sdense} [\norm{\vw}_2^2\langle\bvw,\bvv\rangle^{2j-1}\bw_1/\bv_1]
    \end{align*}
    
    For $\Sdense$ part, by \cref{lem: high order moment bound} and \cref{lem: stage 2 IH}\ref{item: stage 2 IH neuron dir} we have
    \begin{align*}
        \tldE_{\Sdense} [\norm{\vw}_2^2\langle\bvw,\bvv\rangle^{2j-1}\bw_1/\bv_1]
        \lesssim& \frac{\omega^2}{r}
    \end{align*}

    For $\Spot$ part, we simply have
    \begin{align*}
        \frac{|\Spot|}{m}\tldE_{\Spot} [\norm{\vw}_2^2\langle\bvw,\bvv\rangle^{2j-1}\bw_1/\bv_1]
        =& \frac{|\Spot|}{m}\tldE_{\Spot} \left[\norm{\vw}_2^2
            \langle\bvw,\bvv\rangle^{2j-2}
            \left(\bw_1^2+\sum_{i\ne 1}\bw_i\bw_1\bv_i/\bv_1\right)\right]\\
        \myeq{a}& \ha_1\bv_1^{2j-2} + \BigO{ \frac{\sigma_1^2}{m^{0.9}} \frac{1}{|\bv_1|} + \frac{\sqrt{\epsdir}}{|\bv_1|}+j\sqrt{\epsdir}},
    \end{align*}
    where (a) similar as Case 2 in bounding the other term above using \cref{lem: <>2j bound} and \cref{lem: stage 2 IH}\ref{item: stage 2 IH norm bound or basis like}\ref{item: stage 2 IH reg}.

    For $j=2$, we can slightly improve the above as
    \begin{align*}
        \frac{|\Spot|}{m}\tldE_{\Spot} [\norm{\vw}_2^2\langle\bvw,\bvv\rangle^{3}\bw_1/\bv_1]
        =& \frac{|\Spot|}{m}\tldE_{\Spot}  \left[w_1^2\left((\bw_1\bv_1)^2 + 3\sum_{i\ne 1}(\bw_i\bv_i)^2\right)\right]\\
        \myeq{a}& \frac{3\sigma_1^2}{m^{0.9}}
            + \ha_1 (\bv_1^2\pm\BigO{\epsdir}) 
            \pm \sum_{i\in[r]\setminus\{1\}} \ha_i \BigO{\epsdir}\\
        \myeq{b}& \ha_1 \bv_1^2 + \BigO{\frac{1}{r}} 
    \end{align*}
    where (a) we use \cref{lem: stage 2 IH}\ref{item: stage 2 IH norm bound or basis like};
    (b) use \cref{lem: stage 2 IH}\ref{item: stage 2 IH reg} in Stage 2.1 and $\epsdir$ is small enough in Stage 2.2.

    So, for $2\le j\le r/2$
    \begin{align*}
        \tldE [\norm{\vw}_2^2\langle\bvw,\bvv\rangle^{2j-1}\bw_i/\bv_i]
        =\left\{\begin{array}{ll}
        \ha_i\bv_i^2 + \BigO{\frac{\omega^2}{r}} &,\ j=2\\
        \ha_i\bv_i^{2j-2}\ind_{\vv\in\Spoti} + \BigO{\frac{\sigma_1^2}{m^{0.9}} \frac{1}{|\bv_i|} + \frac{\sqrt{\epsdir}}{|\bv_i|} +  \frac{\omega^2}{r}+j\sqrt{\epsdir}}. &,\ 3\le j\le r/2 
        \end{array}\right. .
    \end{align*}
    
    \noindent\textit{Case 3: $j\ge r/\log r$.} For $j > r/\log r$, similarly 
    \begin{align*}
        \tldE [\norm{\vw}_2^2\langle\bvw,\bvv\rangle^{2j-1}\bw_1/\bv_1]
        = \frac{|\Spot|}{m}\tldE_{\Spot} [\norm{\vw}_2^2\langle\bvw,\bvv\rangle^{2j-1}\bw_1/\bv_1]
            + \left(1-\frac{|\Spot|}{m}\right)
            \tldE_{\Sdense} [\norm{\vw}_2^2\langle\bvw,\bvv\rangle^{2j-1}\bw_1/\bv_1]
    \end{align*}
    
    For $\Sdense$ part, by \cref{lem: high order moment bound} we have
    \begin{align*}
        \tldE_{\Sdense} [\norm{\vw}_2^2\langle\bvw,\bvv\rangle^{2j-1}\bw_1/\bv_1]
        \le \frac{1}{|\bv_1|}\tldE [\norm{\vw}_2^2\langle\bvw,\bvv\rangle^{\Theta(r/\log r)}]
        \lesssim& \frac{r}{|\bv_1|}\left(\frac{\omega}{\log r}\right)^{\Theta(r/\log r)}
        \le \frac{1}{|\bv_1|}\frac{1}{m}.
    \end{align*}

    For $\Spot$ part, similar as in Case 1 we have
    \begin{align*}
        \frac{|\Spot|}{m}\tldE_{\Spot} [\norm{\vw}_2^2\langle\bvw,\bvv\rangle^{2j-1}\bw_1/\bv_1]
        = \ha_1\bv_1^{2j-2} + \BigO{ \frac{\sigma_1^2}{m^{0.9}} \frac{1}{|\bv_1|} + \frac{\sqrt{\epsdir}}{|\bv_1|}+j\sqrt{\epsdir}}.
    \end{align*}

    So, for $j\ge r/\log r$
    \begin{align*}
        \tldE [\norm{\vw}_2^2\langle\bvw,\bvv\rangle^{2j-1}\bw_i/\bv_i]
        = \ha_i\bv_i^{2j-2}\ind_{\vv\in\Spoti} + \BigO{\frac{\sigma_1^2}{m^{0.9}} \frac{1}{|\bv_i|} + \frac{\sqrt{\epsdir}}{|\bv_i|} + j\sqrt{\epsdir}}.
    \end{align*}
    Note that $j\sqrt{\epsdir}$ comes from applying \cref{lem: <>2j bound} on $\tldE_{G_1}[\norm{\vw}^2\langle\bvw,\bvv\rangle^{2j-2}]$ in $\Spot$ part. So
    another bound that does not depend on $j$ is by bounding $\tldE_{G_1}[\norm{\vw}^2\langle\bvw,\bvv\rangle^{2j-2}]\le \tldE_{G_1}[\norm{\vw}^2\langle\bvw,\bvv\rangle^{4}]$,
    which leads to
    \begin{align*}
        \tldE [\norm{\vw}_2^2\langle\bvw,\bvv\rangle^{2j-1}\bw_i/\bv_i]
        \le \ha_i\bv_i^{4}\ind_{\vv\in\Spoti} + \BigO{\frac{\sigma_1^2}{m^{0.9}} \frac{1}{|\bv_i|} + \frac{\sqrt{\epsdir}}{|\bv_i|} +  \frac{\omega^2}{r}}.
    \end{align*}

    \noindent\textit{Summary} For $j\ge 2$, we have
    \begin{align*}
        \tldE [\norm{\vw}_2^2\langle\bvw,\bvv\rangle^{2j-1}\bw_i/\bv_i]
        =\left\{\begin{array}{ll}
        \ha_i\bv_i^2 + \BigO{\frac{\omega^2}{r}} &,\ j=2\\
        \ha_i\bv_i^{2j-2}\ind_{\vv\in\Spoti} + \BigO{\frac{\sigma_1^2}{m^{0.9}} \frac{1}{|\bv_i|} + \frac{\sqrt{\epsdir}}{|\bv_i|} +  \frac{\omega^2}{r}+j\sqrt{\epsdir}} &,\ j\ge 3 \\
        \le \ha_i\bv_i^{4}\ind_{\vv\in\Spoti} + \BigO{\frac{\sigma_1^2}{m^{0.9}} \frac{1}{|\bv_i|} + \frac{\sqrt{\epsdir}}{|\bv_i|} +  \frac{\omega^2}{r}} &,\ j\ge 3 
        \end{array}\right. .
    \end{align*}
\end{proof}

\subsubsection{Technical lemma for Stage 2.1}
In below, we aim to characterize the setting that $x$ as the the dynamic of largest potential neuron for all direction $\max_{k\in[r]}\max_{i\in[m]}\bv_k^2$, and $y$ as the the dynamic of smallest potential neuron for all direction $\min_{k\in[r]}\max_{i\in[m]}\bv_k^2$. Thus, one should view $c_i=\hsigma_i^2$, $\delta_x=\frac{\log(2m)}{r}(1+\BigO{\frac{\log r}{\log m}})$, $\delta_y=\frac{\log(2m)}{r}(1-\BigO{\frac{\log r}{\log m}})$, $b=\BigO{\frac{\omega^5}{r}}$ from \cref{lem: stage 2.1 good pot neuron} and \cref{lem: stage 2.1 dynamic potential neuron}. The goal of lemma below is to show that the time they reach $1-\eps$ does not differ much.
\begin{lemma}[Time Gap between the hitting $1-\eps$ under small perturbations]\label{lem: stage 2.1 technical ode time gap}
    Let
    \[
        F(z):=\sum_{i=1}^{\infty} c_i z^i,
        \qquad c_i>0,
        \qquad c_1>0,
        \qquad S:=\sum_{i=1}^{\infty} c_i<\infty.
    \]
    Fix $a>0$, and assume
    \[
        0<\delta_y<\delta_x\le 2\delta_y\le \frac14,
        \qquad
        0<\varepsilon\le \frac14,
        \qquad
        0<b\le \frac{c_1}{4}\min\{\delta_y,\varepsilon\}.
    \]
    Consider the ODEs
    \begin{equation}\label{eq: ode}
    \begin{aligned}
        \dot x
        &=a x\bigl((1-x)F(x)+b\bigr),
        \qquad
        x(0)=\delta_x,\\
        \dot y
        &=a y\bigl((1-y)F(y)-b\bigr),
        \qquad
        y(0)=\delta_y.    
        \end{aligned}    
    \end{equation}
    Define the hitting times
    \[
        T_x:=\inf\{t\ge 0:\ x(t)=1-\varepsilon\},
        \qquad
        T_y:=\inf\{t\ge 0:\ y(t)=1-\varepsilon\}.
    \]
    Then $T_x,T_y<\infty$, and
    \[
        \frac{\delta_x-\delta_y}{2aS\,\delta_y^2}
        +
        \frac{b}{8aS^2\,\delta_y^2}
        +
        \frac{b}{aS^2\,\varepsilon}
        \;\le\;
        T_y-T_x
        \;\le\;
        \frac{2(\delta_x-\delta_y)}{ac_1\,\delta_y^2}
        +
        \frac{8b}{ac_1^2\,\delta_y^2}
        +
        \frac{32b}{ac_1^2\,\varepsilon}.
    \]
    In particular, if $c_1$ and $S$ are absolute constants, then
    \[
        T_y-T_x
        =
        \BigTheta{
        \frac{\delta_x-\delta_y}{a\delta_y^2}
        +
        \frac{b}{a\delta_y^2}
        +
        \frac{b}{a\varepsilon}
        }.
    \]
\end{lemma}

\begin{proof}
    Set
    \[
        H(z):=(1-z)F(z), \qquad z\in[0,1].
    \]
    Since all $c_i$ are positive and $\sum_i c_i=S<\infty$, for every $z\in[0,1]$ we have
    \[
        c_1 z \le F(z)\le Sz,
    \]
    and hence
    \begin{equation}\label{eq:H-basic}
        c_1 z(1-z)\le H(z)\le S z(1-z).
    \end{equation}

    We claim that
    \begin{equation}\label{eq:H-minus-positive}
        H(z)-b>0
        \qquad\text{for all }z\in[\delta_y,1-\varepsilon].
    \end{equation}
    Indeed, if $z\in[\delta_y,\frac12]$, then by \eqref{eq:H-basic},
    \[
        H(z)\ge c_1 z(1-z)\ge \frac{c_1 z}{2}\ge \frac{c_1\delta_y}{2}\ge 2b.
    \]
    If instead $z\in[\frac12,1-\varepsilon]$, then again by \eqref{eq:H-basic},
    \[
        H(z)\ge c_1 z(1-z)\ge \frac{c_1(1-z)}{2}\ge \frac{c_1\varepsilon}{2}\ge 2b.
    \]
    Thus \eqref{eq:H-minus-positive} holds. In particular, both right-hand sides of \eqref{eq: ode} that 
    $
        a x(H(x)+b), a y(H(y)-b)
    $
    are strictly positive on $[\delta_y,1-\varepsilon]$. Therefore both trajectories are strictly increasing until they hit $1-\varepsilon$, and the hitting times are given by separation of variables:
    \[
        T_x
        =
        \frac1a \int_{\delta_x}^{1-\varepsilon}\frac{dz}{z(H(z)+b)},
        \qquad
        T_y
        =
        \frac1a \int_{\delta_y}^{1-\varepsilon}\frac{dz}{z(H(z)-b)}.
    \]
    
    Subtracting the two formulas gives
    \[
        T_y-T_x
        =
        \frac1a\int_{\delta_y}^{\delta_x}\frac{dz}{z(H(z)-b)}
        +
        \frac1a\int_{\delta_x}^{1-\varepsilon}
        \left(
        \frac{1}{z(H(z)-b)}-\frac{1}{z(H(z)+b)}
        \right)\,dz.
    \]
    Hence
    \[
        T_y-T_x = I + J,
    \]
    where
    \[
        I
        :=
        \frac1a\int_{\delta_y}^{\delta_x}\frac{dz}{z(H(z)-b)},
        \qquad
        J
        :=
        \frac{2b}{a}\int_{\delta_x}^{1-\varepsilon}\frac{dz}{z(H(z)^2-b^2)}.
    \]
    Further decompose $J$ into
        \[
        J = J_1+J_2,
        \]
        where
        \[
        J_1
        :=
        \frac{2b}{a}\int_{\delta_x}^{1/2}\frac{dz}{z(H(z)^2-b^2)},
        \qquad
        J_2
        :=
        \frac{2b}{a}\int_{1/2}^{1-\varepsilon}\frac{dz}{z(H(z)^2-b^2)}.
        \]
    We will estimate $I$ and $J_1,J_2$ separately.
    
    \paragraph{Estimate $I$.}
        On the interval $[\delta_y,\delta_x]$ we have $z\le \delta_x\le \frac14$, so $1-z\ge \frac34$. Therefore, by \eqref{eq:H-basic},
        \[
            H(z) - b\ge c_1 z(1-z) -b\ge \frac{3c_1 z}{4} -b
            \ge \frac{3c_1 z}{4}-\frac{c_1 z}{4}=\frac{c_1 z}{2}.
        \]
        where we use $z\ge \delta_y$ and $b\le \frac{c_1\delta_y}{4}\le \frac{c_1 z}{4}$.
        Hence
        \[
            I
            \le
            \frac{2}{ac_1}\int_{\delta_y}^{\delta_x}\frac{dz}{z^2}
            =
            \frac{2}{ac_1}\left(\frac{1}{\delta_y}-\frac{1}{\delta_x}\right)
            =
            \frac{2(\delta_x-\delta_y)}{ac_1\,\delta_x\delta_y}
            \le
            \frac{2(\delta_x-\delta_y)}{ac_1\,\delta_y^2}.
        \]
        For the lower bound, simply use $H(z)-b\le H(z)\le Sz$ on $[\delta_y,\delta_x]$:
        \[
            I
            \ge
            \frac{1}{aS}\int_{\delta_y}^{\delta_x}\frac{dz}{z^2}
            =
            \frac{\delta_x-\delta_y}{aS\delta_x\delta_y}
            \ge
            \frac{\delta_x-\delta_y}{2aS\delta_y^2}.
        \]

    \paragraph{Estimate $J_1$.}

        For $z\in[\delta_x,\frac12]$, we have $1-z\ge \frac12$, so
        \[
        H(z)\ge c_1 z(1-z)\ge \frac{c_1 z}{2}.
        \]
        Also $b\le \frac{c_1\delta_y}{4}\le \frac{c_1 z}{4}$ because $z\ge \delta_x>\delta_y$. Therefore
        \[
        H(z)-b\ge \frac{c_1 z}{4},
        \qquad
        H(z)+b\ge H(z)\ge \frac{c_1 z}{2}.
        \]
        It follows that
        \[
        H(z)^2-b^2=(H(z)-b)(H(z)+b)\ge \frac{c_1^2 z^2}{8}.
        \]
        Hence
        \[
        J_1
        \le
        \frac{2b}{a}\int_{\delta_x}^{1/2}\frac{8\,dz}{c_1^2 z^3}
        =
        \frac{16b}{ac_1^2}\int_{\delta_x}^{1/2} z^{-3}\,dz
        =
        \frac{8b}{ac_1^2}\left(\frac{1}{\delta_x^2}-4\right)
        \le
        \frac{8b}{ac_1^2\,\delta_y^2}.
        \]
        
        For the lower bound on $J_1$, we use $H(z)^2-b^2\le H(z)^2\le S^2 z^2$ for $z\in[\delta_x,\frac12]$:
        \[
        J_1
        \ge
        \frac{2b}{a}\int_{\delta_x}^{1/2}\frac{dz}{S^2 z^3}
        =
        \frac{b}{aS^2}\left(\frac{1}{\delta_x^2}-4\right)
        \ge
        \frac{b}{2aS^2\,\delta_x^2}
        \ge
        \frac{b}{8aS^2\,\delta_y^2},
        \]
        since $\delta_x\le \frac14$, we have $\delta_x^{-2}\ge 16$.

        \paragraph{Estimate $J_2$.}
        Now let $z\in[\frac12,1-\varepsilon]$. Then $z\ge \frac12$ and $1-z\ge \varepsilon$. By \eqref{eq:H-basic},
        \[
        H(z)\ge c_1 z(1-z)\ge \frac{c_1(1-z)}{2}.
        \]
        Since $1-z\ge \varepsilon$ and $b\le \frac{c_1\varepsilon}{4}\le \frac{c_1(1-z)}{4}$, we get
        \[
        H(z)-b\ge \frac{c_1(1-z)}{4},
        \qquad
        H(z)+b\ge H(z)\ge \frac{c_1(1-z)}{2}.
        \]
        Thus, using $z\ge \frac12$
        \[
        z(H(z)^2-b^2)\ge \frac{c_1^2(1-z)^2}{16}.
        \]
        Therefore
        \[
        J_2
        \le
        \frac{2b}{a}\int_{1/2}^{1-\varepsilon}\frac{16\,dz}{c_1^2(1-z)^2}
        =
        \frac{32b}{ac_1^2}\int_{1/2}^{1-\varepsilon}(1-z)^{-2}\,dz
        =
        \frac{32b}{ac_1^2}\left(\frac{1}{\varepsilon}-2\right)
        \le
        \frac{32b}{ac_1^2\,\varepsilon}.
        \]
        
        For the lower bound, use $H(z)^2-b^2\le H(z)^2\le S^2 z^2(1-z)^2$, so
        \[
        z(H(z)^2-b^2)\le S^2 z^3(1-z)^2\le S^2(1-z)^2.
        \]
        Hence
        \[
        J_2
        \ge
        \frac{2b}{aS^2}\int_{1/2}^{1-\varepsilon}(1-z)^{-2}\,dz
        =
        \frac{2b}{aS^2}\left(\frac{1}{\varepsilon}-2\right)
        \ge \frac{b}{aS^2\,\varepsilon},
        \]
        since $\varepsilon\le \frac14$, we have $\varepsilon^{-1}\ge 4$.

        \paragraph{Combine the estimates.}
        From the bounds on $I_0$, $J_1$, and $J_2$, we obtain
        \[
        T_y-T_x
        =
        I_0+J_1+J_2
        \]
        and hence
        \[
        \frac{\delta_x-\delta_y}{2aS\,\delta_y^2}
        +
        \frac{b}{8aS^2\,\delta_y^2}
        +
        \frac{b}{aS^2\,\varepsilon}
        \le 
        T_y-T_x
        \le
        \frac{2(\delta_x-\delta_y)}{ac_1\,\delta_y^2}
        +
        \frac{8b}{ac_1^2\,\delta_y^2}
        +
        \frac{32b}{ac_1^2\,\varepsilon}.
        \]
\end{proof}

\section{(Optional) Stage 3: local convergence}\label{sec: local convergence}
So far, we have shown at the end of Stage 2 that the loss $\L\le e^{-\log^2 r}$.
In this section, we show the loss eventually goes down to arbitrary small error $\eps$. In particular, we use the local convergence result \cite{zhou2021local} that require the initial loss to be smaller than $\eps_0=1/\poly(r)$. This is satisfied by our results in Stage~2. 

Moreover, the result in this section does not require any symmetrized network as in previous section. The only requirement is loss is smaller than threshold $\eps_0$.
\begin{lemma}\label{lem: local convergence}
    Consider activation $\sigma$ is absolute value function. For given any target error $\eps>0$, let sample size $n\gtrsim\frac{dr^4\log d}{\eps^2}$. There exists threshold $\eps_0=1/\poly(r)$ such that if initially loss $\L(\mW)\le \eps_0$, then loss $\L(\mW)\le \eps$ within time $\BigO{\frac{r}{\eps}}$.
\end{lemma}
We note that the above sample complexity improves upon that in \cite{zhou2021local}, reducing the dependency from $d^2$ to $d$. We also believe the tight convergence rate is $1/t^3$, rather than $1/t$, as suggested in \cite{xu2023over}. Finally, the requirement that the activation function be the absolute value is used to leverage existing local convergence results from \cite{zhou2021local}. We believe similar results can be extended to more general activation functions, such as those considered in this paper. However, achieving these potential improvements would require substantial additional effort in analyzing local convergence, which is beyond the scope of this work.

\begin{proof}
    We use \cite[Theorem~4]{zhou2021local} only for its local gradient lower bound, and first verify that its requirements are satisfied. Here $\sigma(x)=|x|$, the target directions are orthogonal with unit norm, and the width assumption gives $m\ge r$, so the required separation, norm, and width conditions hold. Moreover, by the end of Stage~2 we have $L(W)\le e^{-\log^2 r}\le 1/\poly(r)$, which lies in the required local neighborhood, and we set $\lambda_t=0$ throughout this local phase. It remains only to account for the difference in parametrization, which we do below.
    In \cite{zhou2021local}, the parametrization does not include the $1/m$ scaling factor. It suffices to map each neuron $\vw$ to $\vw'=\vw/\sqrt{m}$. Under this mapping, $\sqrt{m}\nabla_\vw\L=\nabla_{\vw'}\L$ and $\L(\mW)=\L(\mW')$, so we can directly apply their result.

    Using \cite[Thm.4]{zhou2021local} we know
    \[
        \norm{\sqrt{m}\nabla_{\mW} \L(\mW)}_F
        =\norm{\nabla_{\mW'} \L(\mW')}_F
        \gtrsim \frac{1}{\sqrt{r}}\L(\mW')
        = \frac{1}{\sqrt{r}}\L(\mW).
    \]
    Using concentration in \cref{lem: concentration Gamma gradient} and the bound in the proof of \cref{lem: Ct bound} we know
    \[
        \norm{m\nabla \L(\mW) - m\nabla \hL(\mW)}_F^2
        \lesssim \frac{dr^2\log d}{n}\sum_i \norm{\vv_i}^2
        \le \frac{dr^3\log d}{n} m.
    \]

    Therefore, by descent lemma, we have
    \begin{align*}
        \frac{\rd}{\rd t}\L(\mW)
        =& -\langle m\nabla \hL(\mW), \nabla \L(\mW)\rangle
        \le -\norm{\sqrt{m}\nabla \L(\mW)}^2
        + \norm{\sqrt{m}\nabla \hL(\mW) - \sqrt{m}\nabla \L(\mW)}\norm{\sqrt{m}\nabla \L(\mW)}\\
        \lesssim& -\left(\norm{\sqrt{m}\nabla \L(\mW)}-\sqrt{\frac{dr^3\log d}{n}}\right)\norm{\sqrt{m}\nabla \L(\mW)}^2\\
        \lesssim& -\frac{1}{r}\L(\mW)^2,
    \end{align*}
    where as long as $\sqrt{\frac{dr^3\log d}{n}}<\eps/\sqrt{r}$ whenever $\L\ge \eps$.

    This finally implies we achieve $\eps$ loss within time $\BigO{r/\eps}$.
\end{proof}

\section{Coupling the dynamics of symmetrized network with finite-width network}\label{sec: coupling proof}
In this section, we present the detailed results and proofs that outlined in \cref{sec: coupling main-text}. 
Our goal is to couple the dynamics of symmetrized network $\tldf$ and the actual finite-width network $f$. Recall the setup:

Denote $\chi_i\sim N(\vzero,\mI/d)$ as the initialization. 
The actual dynamic on $\vv$, symmetrized dynamic on $\tldvv$ and intermediate dynamic on $\ckvv$ is given by
\begin{equation}\label{eq:coupling dynamic apdx}
\begin{aligned}
    \frac{\rd \vv_i}{\rd t}
    =&
    -\lambda \vv_i
    -m\gamma_{t,i}\nabla_{\vv_i}\hL(\mW),
    \qquad
    \vv_i(0)=\vchi_i,\\
    \frac{\rd \tldvv_{i,\rho}}{\rd t}
    =&
    -\lambda \tldvv_{i,\rho}
    -m|\Rho|\gamma_{t,i}\nabla_{\tldvv_{i,\rho}}\L(\tldmW),
    \qquad
    \tldvv_{i,\rho}(0)=\rho(\vchi_i),\\
    \frac{\rd \ckvv_i}{\rd t}
    =&-\lambda \tldvv_{i,\id}
    -m|\Rho|\gamma_{t,i}\nabla_{\tldvv_{i,\id}}\L(\tldmW),,
    \qquad
    \ckvv_i(0)=\vchi_i .
\end{aligned}    
\end{equation}
Also recall $\ckf(\vx)=\frac1m\sum_{i=1}^m\norm{\ckvw_i}_2^2
    \sigma(\bckvw_i^\top\vx)$ for the network formed by these intermediate particles.

As already motivated in \cref{sec: coupling main-text}, the main goal of this section is to control $\delta_i=\norm{\vw_i-\ckvw_i}_2$ for all $i\in[m]$. Then combine with the the results in \cref{sec: coupling main-text} and the dynamics in \cref{sec: stage 1} and \cref{sec: stage 2} we can get the desired coupling result between $f$ and $\tldf$.

To quantify the coupling error, recall the coupling error
\[
    \delta_i:=\norm{\vw_i-\ckvw_i}_2,
    \qquad
    \Delta^2:=\frac1m\sum_{i=1}^m\delta_i^2,
    \qquad
    \Deltamax:=\max_i\frac{\delta_i}{\norm{\ckvw_i}_2}.
\]

The goal of showing $\tldf\approx f$ now can be decomposed into $\tldf\approx \ckf$ and $\ckf \approx f$ with the intermediate process $\ckf$. The lemma below gives a way to control these two error, as long as the width $m$ is large enough and average coupling error $\Delta$ is small.
\begin{lemma}\label{lem: tldf ckf bound and f ckf bound}
    There exists universal constant $C>0$ such that for all $t\le T$ we have
    \begin{align*}
        \E_{\vx}\left[
        \left(
            \tldf(\vx)-\ckf(\vx)
        \right)^2
        \right]
        \le&
        C
            \frac{B_2^2\log m}{m}
        \le C \frac{B B_\infty\log m}{m},\\
        \E_\vx\left[\left(f(\vx)-\ckf(\vx)\right)^2\right]
        \le&
        C\left(B\Delta^2+\Delta^4\right).
    \end{align*}
    where $ B_i:=
            \norm{\ckvv_i}_2^2
            \left(
                \avgrho 
                \norm{
                    \btldvv_{i,\rho}
                    -
                    \bckvv_i
                }_2
            \right)$, 
        $B_2^2 := \frac{1}{m}\sum_{i=1}^m B_i^2$, 
        $B_\infty := \max_{i\in[m]} B_i$, 
        and $\frac{1}{m}\sum_{i=1}^m \norm{\ckvw_i}_2^2 \le B$.
\end{lemma}
\begin{proof}[Proof Sketch]    
The bound on $f\approx \ckf$ directly follows from Lipschitz argument.
To see $\tldf\approx \ckf$, since $\tldvw$ and $\ckvw$ follow the same dynamics and the only difference is the number of neurons
we have the basic properties:
        (1) $\ckvw_i = \tldvw_{i,\id}$ for all time $t\ge 0$;
        (2) $\tldvw_{i,\rho} = \rho(\tldvw_{i,\id})$, $\norm{\tldvw_{i,\rho}} = \norm{\tldvw_{i,\id}}$ for all time $t\ge 0$.
The above claim suggests that one can view $\ckf$ as random variable whose expectation (over all $\rho\in\Rho$) is $\tldf$. Thus, when $m$ is large enough, then $\ckf\approx \tldf$ by Bernstein-type concentration. See \cref{sec: proof of tldf ckf bound and f ckf bound} for the complete proof.
\end{proof}

Therefore, the remaining task is to bound the average coupling error $\Delta=(1/m)\sum_i \delta_i^2$. It turns out in order to control $\Delta$, we need to control both $\Delta$ and normalized error $\Deltamax=\max_i\delta_i/\norm{\ckvv_i}$. The key ingredient in the proof is \cref{lem: coupling error main} that bound $\Deltamax$ and $\Delta$. We deferred the statement and proof to \cref{sec: proof of coupling main delta}.
\begin{lemma}\label{lem: main text coupling error main}
    Suppose sample size $n>d^{3+c}$ and width $m\ge d^{2+c}$ with any small constant $c$. We have for $t\le T_2$
    \begin{align*}
        \Deltamax\le \epsdir^{1/6},\qquad
        \Delta\le \epsdir^{1/5}.
    \end{align*}
    As a result, we have $\E_{\vx}[(f(\vx)-\tldf(\vx))^2]\le \epsnm^{1/4}/2$.
\end{lemma}
\begin{proof}
    The bound on $\Deltamax$ and $\Delta$ follows directly from \cref{lem: coupling error main}. To bound $\E_{\vx}[(f(\vx)-\tldf(\vx))^2]$, it suffices to note that
    \[
        \E_{\vx}[(f(\vx)-\tldf(\vx))^2]
        \le 2\E_{\vx}[(f(\vx)-\ckf(\vx))^2]
        + 2\E_{\vx}[(\ckf(\vx)-\tldf(\vx))^2]
        \lesssim 
            \frac{rB_\infty\log m}{m}
            + B\Delta^2+\Delta^4,
    \]
    where we use \cref{lem: tldf ckf bound and f ckf bound} 
    and $\tldE \norm{\ckvv}^2\lesssim r$ from \cref{lem: stage 1.1 IH implication} for Stage 1.1, \cref{lem: stage 1.2 IH}\ref{item: stage 1.2 IH reg} for Stage 1.2, \cref{lem: stage 2 IH}\ref{item: stage 2 IH reg} for Stage~2.
    
    Using
    \begin{align*}
        B_\infty = \max_{i\in[m]} B_i
        \le \max\{\sigma_1^2, m\sqrt{\epsdir}\},
    \end{align*}
    where the two cases above are (1) neurons not aligned with ground-truth and thus norm is at most $\sigma_1^2\le m^c$ with small constant $c$, and (2) neurons aligned with ground-truth so norm is at most $m$ as $\ha_i\le 1$ by \cref{lem: stage 2 IH}\ref{item: stage 2 IH reg} and direction unchanged after $\Rho$ upto error $\sqrt{\epsdir}$ by \cref{lem: stage 2 IH}\ref{item: stage 2 IH norm bound or basis like}.

    Therefore, we get
    \[
        \E_{\vx}[(f(\vx)-\tldf(\vx))^2]
        \lesssim r\sqrt{\epsdir}\log m + r\epsdir^{2/5}
        \le \epsnm^{1/4}/2,
    \]
    where recall $\epsnm=e^{-\polylog(r)}\gg \epsdir=m^{-\Theta(1)}$.
\end{proof}

\subsection{Proof of \cref{lem: tldf ckf bound and f ckf bound}}\label{sec: proof of tldf ckf bound and f ckf bound}
To prove \cref{lem: tldf ckf bound and f ckf bound} we first present the following 2 results.

The lemma below shows for any fix time $t$, the difference between $\tldf$ and $\ckf$. The key is to view $\ckf$ as an average of random samples under $\Rho$ whose mean is $\tldf$.
\begin{lemma}[Bounding $\tldf-\ckf$]\label{lem: tldf ckf bound}
    Assume that $\sigma$ satisfies \cref{assum: activation}. 
    Fix any time $t\ge 0$. Consider the dynamics of $\tldvv$ and $\ckvv$ defined above in \eqref{eq:coupling dynamic apdx}.
    Define
    \[
        B_i
        :=
        \norm{\ckvv_i}_2^2
        \left(
            \avgrho 
            d_\pm(
                \btldvv_{i,\rho}
                ,
                \bckvv_i
            )
        \right),
        \quad
        B_2^2 := \frac{1}{m}\sum_{i=1}^m B_i^2,
        \quad
        B_\infty := \max_{i\in[m]} B_i,
        \quad
        B := \frac{1}{m}\sum_i \norm{\ckvv_i}^2
    \]
    where     $d_\pm(\va,\vb):=
        \min\{\norm{\va-\vb}_2,\norm{\va+\vb}_2\}$. Then for any $\delta\in(0,1)$, with probability at least $1-\delta$ over the randomness of initialization, we have
    \[
        \E_{\vx\sim N(\vzero,\mI)}
        \left[
            \left(
                \tldf(\vx)-\ckf(\vx)
            \right)^2
        \right]
        \lesssim
            \frac{B_2^2\log(1/\delta)}{m}
        \le \frac{B B_\infty \log(1/\delta)}{m}.
    \]
\end{lemma}

\begin{proof}
    We fix the time $t$ throughout the proof and drop the explicit dependence on $t$.
    
    We first justify the representation of the symmetrized and checked
    particles. Since $\Rho$ is a finite group of orthogonal transformations and
    $\chi_i\sim N(\vzero,\mI)$, the distribution of $\chi_i$ is invariant under
    $\Rho$.
    
    Let $\gF$ be a measurable fundamental domain for the action of $\Rho$.
    For example, one may choose the unique representative obtained by making
    coordinates nonnegative and sorting within each permutation block. Since
    $\chi_i$ has a continuous distribution, the event that a coordinate is zero
    or that two coordinates in the same permutation block have equal absolute
    value has probability zero. Thus, almost surely, every orbit intersects
    $\gF$ in exactly one point. We may therefore write uniquely
    \[
        \chi_i = \rho_i(\va_i),
        \qquad \va_i\in\gF,\quad \rho_i\in\Rho .
    \]
    Moreover, $\rho_i$ is uniform on $\Rho$ and independent of $\va_i$.
    Indeed, for any measurable $A\subseteq \gF$ and any fixed $\psi\in\Rho$,
    by Gaussian invariance,
    \[
        \P(\va_i\in A,\rho_i=\psi)
        =
        \P(\chi_i\in \psi(A))
        =
        \P(\chi_i\in A)
        =
        \frac{1}{|\Rho|}
        \P(\va_i\in A),
    \]
    where the last one is due to the sets $\{\rho(A):\rho\in\Rho\}$ are disjoint up to measure-zero boundaries and have equal Gaussian probability. This proves that $\rho_i\sim \unif(\Rho)$ and that $\rho_i$ is independent
    of $\va_i$. Since the initializations $\{\chi_i\}_{i=1}^m$ are independent,
    the pairs $\{(\va_i,\rho_i)\}_{i=1}^m$ are independent as well.
    
    Since the dynamics is invariant under $\Rho$, we know if $    \tldvv_{i,\rho}(0)=\rho(\vu_i(0))$ and $\ckvv_i(0)=\rho_i(\vu_i(0))$ start from the canonical representative $\va_i$, i.e. $\vu_i(0)=\va_i$, then at every fixed time $t$,
    \[
        \tldvv_{i,\rho}=\rho(\vu_i(t)),
        \qquad
        \ckvv_i=\rho_i(\vu_i(t)).
    \]
    Moreover, conditional on $\{\vu_i\}_{i=1}^m$, the random variables
    $\{\rho_i\}_{i=1}^m$ remain independent and uniformly distributed over
    $\Rho$, because $\{\vu_i\}_{i=1}^m$ is a deterministic function of the
    orbit representatives $\{\va_i\}_{i=1}^m$ and is independent of
    $\{\rho_i\}_{i=1}^m$.
    
    Recall that $\Rho$ consists of orthogonal transformations. Hence, for any
    $\rho\in\Rho$,
    \[
        \norm{\rho(\vu_i)}_2=\norm{\vu_i}_2,
        \qquad
        \overline{\rho(\vu_i)}=\rho(\bvu_i).
    \]
    Thus
    \[
        \btldvv_{i,\rho}
        =
        \rho(\bvu_i),
        \qquad
        \bckvv_i
        =
        \rho_i(\bvu_i).
    \]
    Using closure of $\Rho$ under composition, we also have
    \[
        \avgrho
        d_\pm\left(
            \btldvv_{i,\rho},
            \bckvv_i
        \right)
        =
        \avgrho
        d_\pm\left(
            \rho(\bvu_i),
            \rho_i(\bvu_i)
        \right)
        =
        \avgrho
        d_\pm\left(
            \rho(\bvu_i),
            \bvu_i
        \right).
    \]
    Therefore, after conditioning on $\{\vu_i\}_{i=1}^m$, the quantities $B_i$,
    $B_2$, and $B_\infty$ are deterministic.
    
    By definition of $\tldf$ and $\ckf$, using $\tldvv_{i,\rho}=\rho(\vu_i)$ and $\ckvv_i=\rho_i(\vu_i)$, we get
    \[
        \tldf(\vx)-\ckf(\vx)
        =
        \frac{1}{m}
        \sum_{i=1}^m h_i(\rho_i;\vx),
    \]
    where
    \[
        h_i(\rho_i;\vx)
        :=
        \norm{\vu_i}_2^2
        \left(
            \avgrho
            \sigma\left(
                \rho(\bvu_i)^\top \vx
            \right)
            -
            \sigma\left(
                \rho_i(\bvu_i)^\top \vx
            \right)
        \right).
    \]
    
    We now view $h_i(\rho_i;\cdot)$ as a random element of the Hilbert space
    $\mathcal H := L_2(\vx\sim N(\vzero,\mI))$.
    Conditional on $\{\vu_i\}_{i=1}^m$, the only randomness in
    $h_i(\rho_i;\cdot)$ comes from $\rho_i$. Since $\rho_i$ is uniform on
    $\Rho$, for every fixed $\vx$,
    $\E_{\rho_i}
        \sigma\left(
            \rho_i(\bvu_i)^\top \vx
        \right)
        =
        \avgrho
        \sigma\left(
            \rho(\bvu_i)^\top \vx
        \right)$.
    Therefore,
    \[
        \E_{\rho_i}
        h_i(\rho_i;\cdot)
        =
        0
    \]
    as an element of $\mathcal H$.
    
    Next we bound the Hilbert norm of each $h_i$. 
    For every fixed $\tau\in\Rho$, Jensen's inequality gives
    \begin{align*}
        \norm{h_i(\tau;\cdot)}_{\mathcal H}
        &\lesssim
        \norm{\vu_i}_2^2
        \avgrho
        \left\|
            \sigma\left(\rho(\bvu_i)^\top \vx\right)
            -
            \sigma\left(\tau(\bvu_i)^\top \vx\right)
        \right\|_{L_2(\vx)}  
        \le
        \norm{\vu_i}^2
        \avgrho
        d_\pm\left(\rho(\bvu_i),\bvu_i\right).
    \end{align*}
    Therefore, for every $\tau\in\Rho$,
    $\norm{h_i(\tau;\cdot)}_{\mathcal H}
        \lesssim B_i$.
    Consequently,
    \[
        \sum_{i=1}^m
        \E_{\rho_i}
        \norm{h_i(\rho_i;\cdot)}_{\mathcal H}^2
        \lesssim
        \sum_{i=1}^m B_i^2
        = m B_2^2 .
    \]
    
    Define
    \[
        S(\rho_1,\ldots,\rho_m)
        :=
        \sum_{i=1}^m h_i(\rho_i;\cdot)
        \in \mathcal H,
        \qquad
        Z(\rho_1,\ldots,\rho_m)
        :=
        \norm{S(\rho_1,\ldots,\rho_m)}_{\mathcal H}.
    \]
    First, since the $h_i$ are conditionally independent and mean zero in
    $\mathcal H$,
    \begin{align*}
        \E_{\boldsymbol\rho} Z
        &\le
        \left(
            \E_{\boldsymbol\rho}
            \norm{
                \sum_{i=1}^m h_i(\rho_i;\cdot)
            }_{\mathcal H}^2
        \right)^{1/2} 
        =
        \left(
            \sum_{i=1}^m
            \E_{\rho_i}
            \norm{h_i(\rho_i;\cdot)}_{\mathcal H}^2
        \right)^{1/2}
        \lesssim \sqrt{m}\,B_2 .
    \end{align*}
    Next, if only the $i$-th coordinate is changed from $\rho_i$ to
    $\rho_i'$, then by the triangle inequality and the previous bound,
    \[
        \left|
            Z(\rho_1,\ldots,\rho_i,\ldots,\rho_m)
            -
            Z(\rho_1,\ldots,\rho_i',\ldots,\rho_m)
        \right|
        \le
        \norm{
            h_i(\rho_i;\cdot)-h_i(\rho_i';\cdot)
        }_{\mathcal H}
        \lesssim B_i .
    \]
    Thus, by the McDiarmid’s inequality,, for every $s\ge 0$,
    \[
        \P_{\boldsymbol\rho}
        \left(
            Z
            \ge
            \E_{\boldsymbol\rho} Z
            +
            C B_2\sqrt{2ms}
        \,\middle|\,
            \{\vu_i\}_{i=1}^m
        \right)
        \le
        e^{-s}.
    \]
    Taking $s=\log(1/\delta)$, note the conditional bound holds for every realization of
    $\{\vu_i\}_{i=1}^m$, it also holds unconditionally. Therefore,
    \[
        \norm{\tldf-\ckf}_{\mathcal H}
        =
        \frac{1}{m}Z
        \lesssim
        B_2
        \sqrt{\frac{\log(e/\delta)}{m}} .
    \]

\end{proof}

The following lemma shows that the difference between $f$ and $\ckf$ is
controlled by the average distance $\Delta$. It follows from standard Lipschitz argument.
\begin{lemma}\label{lem: f ckf bound}
    Recall $\Delta^2=\frac{1}{m}\sum_i \norm{\vv_i-\ckvv_i}^2$. Suppose $\frac{1}{m}\sum_{i=1}^m \norm{\ckvw_i}_2^2 \le B$.
    Then
    \[
        \E_\vx[(f(\vx)-\ckf(\vx))^2]
        \lesssim
        B\Delta^2+\Delta^4 .
    \]
\end{lemma}

\begin{proof}
    For simplicity, write
    \[
        h_i(\vx)
        :=
        \norm{\vw_i}_2^2 \sigma(\bvw_i^\top \vx)
        -
        \norm{\ckvw_i}_2^2 \sigma(\bckvw_i^\top \vx).
    \]
    Then
    \[
        f(\vx)-\ckf(\vx)
        =
        \frac{1}{m}\sum_{i=1}^m h_i(\vx).
    \]
    We first bound the $L_2(\vx)$ norm of each $h_i$. Using $\sigma$ is $\BigO{1}$-Lipschitz
    \[
    \begin{aligned}
        |h_i(\vx)|
        \le&
        \norm{\vw_i}_2^2
        \left|
            \sigma(\bvw_i^\top\vx)
            -
            \sigma(\bckvw_i^\top\vx)
        \right|
        +
        \left|
            \norm{\vw_i}_2^2-\norm{\ckvw_i}_2^2
        \right|
        \left|
            \sigma(\bckvw_i^\top\vx)
        \right|\\
        \lesssim&
        \norm{\vw_i}_2^2 |(\bvw_i-\bckvw_i)^\top\vx|
        +
        \left|
            \norm{\vw_i}_2+\norm{\ckvw_i}_2
        \right| \left|
            \sigma(\bckvw_i^\top\vx)
        \right| \norm{\vw_i-\ckvw_i},,
    \end{aligned}
    \]
    
    Taking the $L_2(\vx)$ norm we get
    \[
    \begin{aligned}
        \norm{h_i}_{L_2(\vx)}
        &\lesssim \norm{\vw_i}_2 \delta_i
            + (\norm{\vw_i}+\norm{\ckvw_i}_2)\delta_i
        \lesssim \delta_i
        \left(
            \norm{\vw_i}_2+\norm{\ckvw_i}_2
        \right)
        \lesssim \delta_i(\norm{\ckvw_i}+\delta_i)
    \end{aligned}
    \]
    Thus, 
    \[
    \begin{aligned}
        \sqrt{\E_\vx[(f(\vx)-\ckf(\vx))^2]}
        = \norm{f-\ckf}_{L_2(\vx)}
        &\le
        \frac{1}{m}\sum_{i=1}^m
        \norm{h_i}_{L_2(\vx)}
        \lesssim
            \frac{1}{m}\sum_{i=1}^m
            \delta_i\norm{\ckvw_i}_2
            +
            \frac{1}{m}\sum_{i=1}^m
            \delta_i^2
        \lesssim \sqrt B \Delta+\Delta^2.
    \end{aligned}
    \]
    This proves the lemma.
\end{proof}

We are now ready to prove \cref{lem: tldf ckf bound and f ckf bound} using the above two results.
\begin{proof}[Proof of \cref{lem: tldf ckf bound and f ckf bound}]

    The second bound follows directly from \cref{lem: f ckf bound}. We focus below on making the first bound uniform over time.
    
    Let $T_{\max}$ be the deterministic upper bound on the complete Stage~1-2 runtime. Recall $\phi_{\vv}(\vx):=\norm{\vv}^2\sigma(\bvv^\top\vx)$. By \cref{assum: activation} and the gradient formula,
    $\norm{\frac{\rd}{\rd t}\phi_{\vv(t)}}_{L_2(\vx)}
    \lesssim \norm{\vv(t)}\norm{\dot{\vv}(t)}$.
    On the regularity event used throughout the stage analysis, for $\vw\in\{\ckvv,\tldvv\}$ we have $\max_{i,t\le T_{\max}}\norm{\vw_i(t)}^2\le m$, $\frac1m\sum_i\norm{\vw_i(t)}^2\le r+10$, and $\norm{\dot{\vw}_i(t)}\lesssim r\gamma_{t,i}\norm{\vw_i(t)}$ with $\gamma_{t,i}\le m$. Therefore, for $g\in\{\ckf,\tldf\}$,
    \[
    \norm{\frac{\rd}{\rd t}g_t}_{L_2}\le L_*,
    \qquad \text{with } L_*\lesssim mr^2.
    \]
    Similarly, directly from the definition of $B_i$, we have $|\frac{\rd}{\rd t}B_i(t)|\le L_B$ with $L_B\lesssim rm^2$. Consequently,
    $|B_2(t)-B_2(s)|\le \left(\frac1m\sum_i|B_i(t)-B_i(s)|^2\right)^{1/2}\le L_B|t-s|$.
    These bounds hold throughout Stage~1-2, including the short interval where $\gamma_{t,i}=m$.
    
    Choose $\eta=m^{-20}$ and $h:=\eta/(8\max\{L_*,L_B\})$, and let $\mathcal T$ be an $h$-net of $[0,T_{\max}]$. Since all relevant quantities are polynomially bounded, we have $N:=|\mathcal T|\le 1+\frac{T_{\max}}{h}=\poly(m,d)$. At every $s\in\mathcal T$, apply \cref{lem: tldf ckf bound} with failure probability $\delta/N$. By a union bound, simultaneously for all $s\in\mathcal T$,
    \[
    \norm{\tldf_s-\ckf_s}_{L_2}^2
    \lesssim \frac{B_2(s)^2\log(N/\delta)}{m}.
    \]
    
    For arbitrary $t\le T_{\max}$, choose $s\in\mathcal T$ with $|t-s|\le h$. The above Lipschitz bounds give
    $\norm{(\tldf_t-\ckf_t)-(\tldf_s-\ckf_s)}_{L_2}\le\eta/4$ and $|B_2(t)-B_2(s)|\le\eta/4$. Therefore,
    \[
    \norm{\tldf_t-\ckf_t}_{L_2}^2
    \lesssim
    \frac{B_2(t)^2\log(N/\delta)}{m}+\eta^2.
    \]
    Since $N=\poly(m,d)$ and $m\ge d^3$, we have $\log(N/\delta)\lesssim \log m+\log(1/\delta)$. Taking $\delta=1/\poly(m)$ and using $B_2^2\lesssim BB_\infty$, uniformly for all $t\le T_{\max}$,
    \[
    \norm{\tldf_t-\ckf_t}_{L_2}^2
    \lesssim
    \frac{BB_\infty\log m}{m}+m^{-40}.
    \]
    The additive $m^{-40}$ term is negligible in all subsequent applications, which gives the desired bound.

\end{proof}

\subsection{Bound $\Deltamax$ and $\Delta$}\label{sec: proof of coupling main delta}
Denote 
\begin{align*}
    \grad_\vv \L(\vtheta)
    &= \E_\vx[(f(\vx)-f_*(\vx)) \nabla_\vv(\norm{\vv}^2\sigma(\bvv^\top \vx))] + \lambda\vv,\\
    \grad_\vv \L(\tldvtheta)
    &= \E_\vx[(\tldf(\vx)-f_*(\vx)) \nabla_\vv(\norm{\vv}^2\sigma(\bvv^\top \vx))] + \lambda\vv,
\end{align*}
as the gradient if there is a neuron at position $\vv$. 
It should note that it is possible that there is in fact no neuron at such position $\vv$, but this is still well-defined. The motivation is from the infinite-width limit where this is often known as the functional derivative, which can be thought as the gradient of distribution in Wasserstein metric. And when a neuron $\vv$ exists at such position, then it goes back to the normal gradient $\nabla_\vv \L(\vtheta)$ and $\nabla_\vv \L(\tldvtheta)$.

As in \cite{mahankali2023beyond} we decompose the dynamic of $\delta$ into
\begin{align}\label{eq: delta decomp ABC}
    \frac{\rd}{\rd t} \delta^2
    = \frac{\rd}{\rd t} \norm{\vv - \ckvv}^2
    = A_t + B_t + C_t,
\end{align}
where 
\begin{align*}
    A_t &= -2\langle \grad_{\vv}\L(\tldvtheta) - \grad_{\ckvv}\L(\tldvtheta),\vv - \ckvv\rangle,\\
    B_t &= -2\langle \grad_{\vv}\L(\vtheta) - \grad_{\vv}\L(\tldvtheta),\vv - \ckvv\rangle,\\
    C_t &= -2\langle \grad_{\vv}\hat{\L}(\vtheta) - \grad_{\vv}\L(\vtheta),\vv - \ckvv\rangle
\end{align*}

Recall
\[
    \Deltamax:= \max_i \frac{\delta_i}{\norm{\ckvv_i}} = \max_i \frac{\norm{\vv_i-\ckvv_i}}{\norm{\ckvv_i}},
    \qquad
    \Delta^2 := \frac{1}{m}\sum_i \delta_i^2
    =\frac{1}{m}\sum_i \norm{\vv_i-\ckvv_i}^2
\]
The lemma below is the main result of this section. The proof uses above decomposition into $A_t,B_t,C_t$ term (bound in \cref{sec: At bound proof,sec: Bt bound proof,sec: Ct bound proof}) to get the final bound using Gronwall-type bound.
\begin{lemma}\label{lem: coupling error main}
    Suppose sample size $n>d^{3+c}$ and width $m\ge d^{2+c}$ with any small constant $c$ ($c$ depends on the exponent $\epsdir=m^{-\BigTheta{1}}$ which is an
    arbitrary small constant). We have for $t\le T_2$
    \begin{align*}
        \Deltamax\le \epsdir^{1/6},\qquad
        \Delta\le \epsdir^{1/5}.
    \end{align*}
\end{lemma}
\begin{proof}
    Recall from \eqref{eq: delta decomp ABC} that
    \begin{align*}
        \frac{\rd}{\rd t} \delta^2
        = \frac{\rd}{\rd t} \norm{\vv - \ckvv}^2
        = A_t + B_t + C_t.
    \end{align*}
    Using \cref{lem: At bound,lem: Bt bound,lem: Ct bound} that bound $A_t,B_t,C_t$, we have the following stage by stage. Typically, within each stage, we first derive the bound for average error $\Delta$ and then to individual error $\Deltamax$, as dynamic of $\delta$ (especially $B_t$) depending on $\Delta$.

    We formalize the following bounds using a stopping-time argument. Define
    \[
    \tau:=T_2\wedge\inf\left\{t\ge0:
    \Delta_{\max}(t)\ge\frac{\log m}{r}
    \text{ or }
    \Delta(t)\ge\frac1{r^4}\right\}.
    \]
    Initially $\Delta(0)=\Delta_{\max}(0)=0$. For all $t<\tau$, we have $\Delta_{\max}<\frac{\log m}{r}=o(1)$ and $\Delta<1/r^4$, so all the estimates in \cref{lem: At bound,lem: Bt bound,lem: Ct bound} apply. In the following, all stage endpoints are implicitly replaced by their minimum with $\tau$.

    \paragraph{Stage 1.} For $t\le T_1$, we have
    \begin{align*}
        \frac{\rd}{\rd t}\Delta^2
        \le& 
            \frac{1}{m}\sum_{i\in[m]}\left(\frac{\rd}{\rd t}\log \norm{\ckvv_i}^2
            + 4\hsigma_2^2\left(1-\frac{1}{r}\sum_{i\le r}\tldE w_i^2\right) 
            + \BigO{\frac{\log m}{r}}\right)\cdot \norm{\vv_i-\ckvv_i}^2
            + \frac{1}{m}\sum_{i\in[m]}\BigO{1}\frac{\norm{\vv-\ckvv_i}^3}{\norm{\ckvv_i}}\\ 
            &+ \BigO{\sqrt{\frac{r^2\log m}{m}}}\Delta + \BigO{r}\Delta^{5/2}
            + \BigO{\sqrt{\frac{dr^2\log d}{n}}}\frac{1}{m}\sum_{i\in[m]}\norm{\vv_i}\delta_i\\
        \le&
            \left(\frac{\rd}{\rd t}\log \tldalpha
            + 4\hsigma_2^2\left(1-\frac{1}{r}\sum_{i\le r}\tldE w_i^2\right) 
            + \BigO{\frac{\log m}{r}}\right)\cdot \Delta^2
            + \BigO{\sqrt{\frac{r^2\log m}{m}}+\sqrt{\frac{dr^3\log d}{n}}}\Delta
    \end{align*}
    where we use $\Deltamax<\frac{\log m}{r}$ and $\Delta<1/r^4$; and $\tldE \norm{\vw}^2\lesssim r$ from \cref{lem: stage 1.1 IH implication}; and dynamic of $\tldalpha$ and $\norm{\vv}$ are the same from \cref{lem: stage 1.1 dynamic} and \cref{lem: stage 1.1 dynamic of tldalpha}.

    Given $\Delta(0)=0$, we can solve the above ode to get for $t\le T_1$
    \[
        \Delta(t)\le \Delta(T_1)
        \lesssim \left(\sqrt{\frac{r^2\log m}{m}}+\sqrt{\frac{dr^3\log d}{n}}\right) \sqrt{\frac{\tldalpha(T_1)}{\tldalpha(0)} d} \polylog(d)
        \lesssim \left(\sqrt{\frac{r^2\log m}{m}}+\sqrt{\frac{dr^3\log d}{n}}\right) \sqrt{dr} \polylog(d),
    \]
    where we use $\int_0^{\hT_{21}}(1-\tldE w_i)\,ds
        =
        \frac{1}{4\hsigma_2^2}\log d
        +
        \BigO{\log\log d}$ from \cref{lem: stage 1.2 0th 2nd order dynamic} and $T_{11}\lesssim \frac{\log d}{r}$ from \cref{lem: stage 1.1 main} and $T_{12}-T_{11}\lesssim \log d$ from \cref{lem: stage 1.2 main}.

    We now move to bound $\Deltamax=\max_i \delta_i/\norm{\ckvv_i}$. We have (omit index $i$ for simplicity)
    \begin{align*}
        &\frac{\rd}{\rd t}\frac{\delta^2}{\norm{\ckvv}^2}
        = \frac{\delta^2}{\norm{\ckvv}^2}\left(
            \frac{1}{\delta^2}\frac{\rd}{\rd t}\delta^2 
            -\frac{1}{\norm{\ckvv}^2}\frac{\rd}{\rd t}\norm{\ckvv}^2\right)\\
        =& \frac{\delta^2}{\norm{\ckvv}^2}\Bigg(
            4\hsigma_2^2\left(1-\frac{1}{r}\sum_{i\le r}\tldE w_i^2
            + \BigO{\frac{\log m}{r}}\right)
            + \BigO{\frac{\delta}{\norm{\ckvv_i}}} 
            + \BigO{\sqrt{r}\Delta+\sqrt{\frac{r\log m}{m}}}\frac{\norm{\vv}}{\delta}
            + \BigO{\sqrt{\frac{dr^2\log d}{n}}}\frac{\norm{\vv}}{\delta}
            \Bigg)\\
        =& \frac{\delta^2}{\norm{\ckvv}^2}\Bigg(
            4\hsigma_2^2\left(1-\frac{1}{r}\sum_{i\le r}\tldE w_i^2
            + \BigO{\frac{\log m}{r}}\right)
            \Bigg)
            + \BigO{\sqrt{r}\Delta(T_1)}\frac{\delta}{\norm{\vv}}
    \end{align*}
    Similar to the $\Delta$ case, given $\delta/\norm{\ckvv}(0) = 0$, we know for $t\le T_1$
    \[
        \Deltamax(t)\le \Deltamax(T_1)
        \le \max_i \frac{\delta_i}{\norm{\ckvv_i}}
        \lesssim \Delta(T_1)\sqrt{dr}\polylog(d)
        \lesssim \left(\sqrt{\frac{r^2\log m}{m}}+\sqrt{\frac{dr^3\log d}{n}}\right) dr \polylog(d).
    \]

    \paragraph{Stage 2.1.}
    For Stage 2.1 that $T_1\le t\le T_{21}$, we have
    \begin{align*}
        \frac{\rd}{\rd t}\Delta^2
        \le& \frac{1}{m}\sum_{i\in[m]}\left(\frac{\rd}{\rd t}\log \norm{\ckvv_i}^2
            + \BigO{\frac{\omega\log m}{r}+ \bckv_i^2\ind_{\ckvv\in \Spoti}}\right)\cdot \norm{\vv_i-\ckvv_i}^2
            + \frac{1}{m}\sum_{i\in[m]}\BigO{1}\frac{\norm{\vv-\ckvv_i}^3}{\norm{\ckvv_i}}\\
            &+ \BigO{\sqrt{\frac{r^2\log m}{m}}}\Delta + \BigO{r}\Delta^{5/2}
            + \BigO{\sqrt{\frac{dr^2\log d}{n}}}\frac{1}{m}\sum_{i\in[m]}\norm{\vv_i}\delta_i\\
        \le& \max_i\left(\frac{\rd}{\rd t}\log \norm{\ckvv_i}^2
            + \BigO{\frac{\omega\log m}{r}+ \bckv_i^2\ind_{\ckvv\in \Spoti}}\right)\cdot \Delta^2
            + \BigO{\sqrt{\frac{r^2\log m}{m}}+\sqrt{\frac{dr^3\log d}{n}}}\Delta
    \end{align*}
    where we use $\Deltamax<\frac{\log m}{r}$ and $\Delta<1/r^4$. 
    Thus, we have
    \begin{align*}
        \Delta(T_{21})
        \lesssim \Delta(T_1) \exp\left\{
            \int_{T_1}^{T_{21}}
            \max_i\left(\frac{\rd}{\rd t}\log \norm{\ckvv_i}^2
            + \BigO{\frac{\omega\log m}{r}+ \bckv_i^2\ind_{\ckvv\in \Spoti}}\right)\rd t
            \right\}.
    \end{align*}
    Using $\tldE w_i^2$ are balanced in \cref{lem: stage 2.1 IH}\ref{item: stage 2.1 IH population 2nd order balance} and the dynamic of $\norm{\vv}$ in \cref{lem: stage 2 dynamic small and dense} and \cref{lem: stage 2 dynamic potential neuron}, we know the above max is taken at neuron $\ckvv$ that maximize $(\bvvler)_i^2$ among all neuron and $i\le r$. Thus, we have
    \begin{align*}
        \int_{T_1}^{T_{21}}
            \frac{\rd}{\rd t}\log \norm{\ckvv_i}^2 \rd t
        = \log \frac{\ckvv_i(T_{21})}{\ckvv_i(T_1)}
        \le \log(\sigma_1^4 m^{o(1)}),
        \qquad
        \int_{T_1}^{T_{21}}
            \BigO{\frac{\omega\log m}{r}} \rd t
        \lesssim \omega,
    \end{align*}
    where we use \cref{lem: stage 2.1 main} and $T_{21}-T_1\lesssim r/\log m$. For the last term, using \cref{lem: stage 2 dynamic potential neuron}
    \[
        \frac{\rd}{\rd t}(\bbvvler)_i^2\ge 2(\bbvvler)_i^2\left((\bbvvler)_i^2-\BigO{\frac{\omega^5}{r}}\right)
    \]
    we have
    \[
        \int_{T_1}^{T_{21}} (\bvvler)_i^2\rd t
        \le \int_{T_1}^{T_{21}} \log ((\bvvler)_i^2) + \BigO{\frac{\omega^5}{r}} \rd t
        \le \log r + \BigO{1}.
    \]
    Therefore, we finally get for $t\le T_{21}$
    \[
        \Delta(t)
        \le \Delta(T_{21})
        \lesssim \Delta(T_1)\sigma_1^4 m^{o(1)}.
    \]

    Now move to the bound on $\Deltamax$. We have
    \begin{align*}
        &\frac{\rd}{\rd t}\frac{\delta^2}{\norm{\ckvv}^2}
        = \frac{\delta^2}{\norm{\ckvv}^2}\left(
            \frac{1}{\delta^2}\frac{\rd}{\rd t}\delta^2 
            -\frac{1}{\norm{\ckvv}^2}\frac{\rd}{\rd t}\norm{\ckvv}^2\right)\\
        =& \frac{\delta^2}{\norm{\ckvv}^2}\Bigg(
            \BigO{\frac{\omega\log m}{r}+\bckv_i^2\ind_{\ckvv\in \Spoti}}
            + \BigO{\frac{\delta}{\norm{\ckvv_i}}} 
            + \BigO{\sqrt{r}\Delta+\sqrt{\frac{r\log m}{m}}}\frac{\norm{\vv}}{\delta}
            + \BigO{\sqrt{\frac{dr^2\log d}{n}}}\frac{\norm{\vv}}{\delta}
            \Bigg)\\
        =& \frac{\delta^2}{\norm{\ckvv}^2}\Bigg(
            \BigO{\frac{\log m}{r}+\bckv_i^2\ind_{\ckvv\in \Spoti}}
            \Bigg)
            + \BigO{\sqrt{r}\Delta(T_{21})}\frac{\delta}{\norm{\vv}}
    \end{align*}
    Similar to the $\Delta$ case, we can bound $\int \bckv_i^2\rd t$. The good thing here is that we do not have the $\frac{\rd}{\rd t}\log\norm{\ckvv}^2$ term, since we look at the normalized error. We have
    \[
        \Deltamax(t)
        \le \Deltamax(T_{21})
        \lesssim \Delta(T_{21})\poly(r)
        \lesssim \Delta(T_1) \sigma_1^4 m^{o(1)}
    \]

    \paragraph{Stage 2.2 for $T_{21}\le t\le T_\gamma$.}
    For Stage 2.2, consider the time $T_{21}\le t\le T_\gamma$. Note that we increase the stepsize to $\gamma=m$ for those neuron with $\norm{\vv}\ge \sigma_1^2$ that are aligned with ground-truth direction ($\ckvv\in G=\cup_k G_k$), thus we focus on the contribution from the average error $\Delta_G=\frac{1}{m}\sum_{i\in\cup_k G_k}\norm{\vv_i-\ckvv_i}^2$. we have
    \begin{align*}
        \frac{\rd}{\rd t}\Delta_G^2
        \le& \frac{1}{m}\sum_{i\in G}\left(\frac{\rd}{\rd t}\log \norm{\ckvv_i}^2
            + \gamma\BigO{\err+\frac{1}{r}}\right)\cdot \norm{\vv_i-\ckvv_i}^2\\
            &+ \gamma\left(\sqrt{r\ha}\Delta_{-G}+\epsdir^{1/4}\ha\log m\right)\Delta_G + \gamma\ha^{3/4}\Delta_G^{5/2}
            + \gamma\BigO{\sqrt{\frac{dr^2\log d}{n}}}\frac{1}{m}\sum_{i\in G}\norm{\vv_i}\delta_i
    \end{align*}
    where $\ha=\sum_i\ha_i$; and for $T_{21}\le t\le T_\gamma$ $\err=\BigO{\frac{\omega^4}{r}}$  before $\max_i \ha_i =1/r^2$ and $\err\lesssim 1$ otherwise (which only takes at most $o(\frac{1}{\gamma}\log m)$ time as the spread of $\ha_i$ is $m^{o(1)}$). We will verify the condition verify condition $\Delta_G\le \ha^{1/2}$ in \cref{lem: Bt bound} later.

    Before $\max_i \ha_i =1/r^2$ (denote this time as $T_\gamma'$),  the above becomes
    \begin{align*}
        \frac{\rd}{\rd t}\Delta_G^2
        \le& \left(\frac{\rd}{\rd t}\log \ha
            + \BigO{\gamma \frac{\omega^4}{r}}\right)\cdot \Delta_G^2
            + \gamma\BigO{\sqrt{r\ha}\Delta(T_{21})+\epsdir^{1/4}\ha\log m+\sqrt{\frac{dr^3\log d}{n}\ha}}\Delta_G
    \end{align*}
    where we use dynamic of $\ha_i$ in \cref{lem: stage 2 population 0th 2nd dynamic} and dynamic of $\norm{\vv}$ in \cref{lem: stage 2 dynamic potential neuron} that are same. Solving the ode and using $\ha_i\ge \sigma_1^2/m$ at $T_{21}$ from \cref{lem: stage 2.1 main} and $T_\gamma'-T_{21}\lesssim \frac{\log m}{\gamma}$, we get
    \begin{equation}\label{eq: Delta G before T gamma '}
    \begin{aligned}
        \Delta_G(t)
        \le \Delta_G(T_\gamma')
        \lesssim& \Delta_G(T_{21})\sqrt{\frac{\ha(T_\gamma')}{\ha(T_{21})}}
        + \int_{T_{21}}^{T_\gamma'} \sqrt{\frac{\ha(t)}{\ha(s)}}\left(
            \sqrt{r\ha(s)}\Delta(T_{21})+\epsdir^{1/4}\ha(s)\log m+\sqrt{\frac{dr^3\log d}{n}\ha(s)}
        \right)
        \rd s\\
        \lesssim& \epsdir^{1/4},
    \end{aligned}
    \end{equation}
    where we use $\Delta_G(T_{21})\le \frac{|G|}{m}\Deltamax(T_{21})$ and $\frac{\ha(T_\gamma')}{\ha(T_{21})}\le m/\sigma_1^4$.

    After $\max_i \ha_i =1/r^2$, it only takes $o(\frac{1}{\gamma}\log m)$ as mentioned above, thus using the crude bound of $\err$ term as $\BigO{1}$ and similar ode above, we get
    \begin{align*}
        \Delta_G(t)
        \le \Delta_G(T_\gamma)
        \lesssim \Delta_G(T_\gamma') m^{o(1)} + \epsdir^{1/4}\log^2 m
        \le \epsdir^{1/4}m^{o(1)},
    \end{align*}
    where recall $\epsdir=m^{-\Theta(1)}$ with small enough constant.
    
    Since $\Delta_{-G}$ only moves $\frac{\log m}{\gamma}=\frac{\log m}{m}$ during this time, we finally know
    \[
        \Delta(t)
        \le \Delta(T_\gamma)
        \lesssim \Delta_G(T_\gamma)
        \le \epsdir^{1/4}m^{o(1)}.
    \]

    Finally, to verify the condition $\Delta_G\le \ha^{1/2}$ in \cref{lem: Bt bound}. For $T_{21}\le t\le T_\gamma'$, using \eqref{eq: Delta G before T gamma '} and replacing $T_\gamma'$ with any time $t$, we can see such condition is true. For $T_\gamma'\le t\le T_\gamma$, it follows directly since $\Delta_G$ is small.
    
    Now move to bound $\Deltamax$. Again, we only focus on those neurons with stepsize $\gamma$, as the total time is short so others does not move much (by at most $\BigO{\frac{\log m}{\gamma}}=\frac{\log m}{m}$). We have
    \begin{align*}
        \frac{\rd}{\rd t}\frac{\delta^2}{\norm{\ckvv}^2}
        =& \frac{\delta^2}{\norm{\ckvv}^2}\left(
            \frac{1}{\delta^2}\frac{\rd}{\rd t}\delta^2 
            -\frac{1}{\norm{\ckvv}^2}\frac{\rd}{\rd t}\norm{\ckvv}^2\right)\\
        =& \gamma\frac{\delta^2}{\norm{\ckvv}^2}\Bigg(
            \BigO{\frac{\omega\log m}{r}+\err}
            + \BigO{\frac{\delta}{\norm{\ckvv_i}}} 
            + \BigO{\sqrt{r}\Delta+\epsdir^{1/4}}\frac{\norm{\vv}}{\delta}
            + \BigO{\sqrt{\frac{dr^2\log d}{n}}}\frac{\norm{\vv}}{\delta}
            \Bigg)\\
        =& \gamma\frac{\delta^2}{\norm{\ckvv}^2}\Bigg(
            \BigO{\frac{\omega\log m}{r}+\err}
            \Bigg)
            + \gamma\BigO{\sqrt{r}\Delta(T_\gamma)}\frac{\delta}{\norm{\vv}}
    \end{align*}
    where for $T_{21}\le t\le T_\gamma$ $\err=\BigO{\frac{\omega^4}{r}}$  before $\max_i \ha_i =1/r^2$ and $\err\lesssim 1$ otherwise (which only takes at most $o(\frac{1}{\gamma}\log m)$ time as the spread of $\ha_i$ is $m^{o(1)}$).

    Thus, we have for $T_{21}\le t\le T_\gamma$, we have
    \begin{align*}
        \Deltamax(t)
        \le \Deltamax(T_\gamma)
        \lesssim \Delta(T_\gamma)+\epsdir^{1/4}m^{o(1)}
        \lesssim \epsdir^{1/4}m^{o(1)}.
    \end{align*}

    \paragraph{Stage 2.2 for $T_\gamma\le t\le T_{22}$ and Stage 2.3 for $T_{22}\le t\le T_2$.}
    For $T_\gamma \le t\le T_{2}$, we have
    \begin{align*}
        \frac{\rd}{\rd t}\Delta^2
        \le& \frac{1}{m}\sum_{i\in[m]}\left(\frac{\rd}{\rd t}\log \norm{\ckvv_i}^2
            + \BigO{\frac{\omega^9}{\log m}+ \bckv_i^2\ind_{\ckvv\in \Spoti}}\right)\cdot \norm{\vv_i-\ckvv_i}^2
            + \frac{1}{m}\sum_{i\in[m]}\BigO{1}\frac{\norm{\vv-\ckvv_i}^3}{\norm{\ckvv_i}}\\
            &+ \BigO{r\epsdir^{1/4}}\Delta + \BigO{r}\Delta^{5/2}
            + \BigO{\sqrt{\frac{dr^2\log d}{n}}}\frac{1}{m}\sum_{i\in[m]}\norm{\vv_i}\delta_i\\
        \le& \left(\frac{\rd}{\rd t}\log (\sum_i\ha_i)
            + \BigO{\frac{\omega^9}{\log m}+ \bckv_i^2\ind_{\ckvv\in \Spoti}}\right)\cdot \Delta^2
            + \BigO{r\epsdir^{1/4}+\sqrt{\frac{dr^3\log d}{n}}}\Delta
    \end{align*}
    where we use dynamic of $\norm{\ckvv}^2$ in \cref{lem: stage 2 dynamic potential neuron} are the same as dynamic of $\ha_i$ in \cref{lem: stage 2 population 0th 2nd dynamic}. 
    Thus, solve the above ode we have for $T_\gamma\le t\le T_2$
    \begin{align*}
        \Delta(t)
        \le \Delta(T_2)
        \lesssim \Delta(T_\gamma)\frac{\sum_i\ha_i(T_2)}{\sum_i\ha_i(T_\gamma)}\poly(1/\epsnm)
        \le \epsdir^{1/5}.
    \end{align*}

    Now move to bound $\Deltamax$. We have
    \begin{align*}
        \frac{\rd}{\rd t}\frac{\delta^2}{\norm{\ckvv}^2}
        =& \frac{\delta^2}{\norm{\ckvv}^2}\left(
            \frac{1}{\delta^2}\frac{\rd}{\rd t}\delta^2 
            -\frac{1}{\norm{\ckvv}^2}\frac{\rd}{\rd t}\norm{\ckvv}^2\right)\\
        =& \frac{\delta^2}{\norm{\ckvv}^2}\Bigg(
            \BigO{\frac{\omega^9}{\log m}+\bckv_i^2\ind_{\ckvv\in \Spoti}}
            + \BigO{\frac{\delta}{\norm{\ckvv_i}}} 
            + \BigO{\sqrt{r}\Delta+\epsdir^{1/4}}\frac{\norm{\vv}}{\delta}
            + \BigO{\sqrt{\frac{dr^2\log d}{n}}}\frac{\norm{\vv}}{\delta}
            \Bigg)\\
        =& \frac{\delta^2}{\norm{\ckvv}^2}\Bigg(
            \BigO{\frac{\omega^9}{\log m}+\bckv_i^2\ind_{\ckvv\in \Spoti}}
            \Bigg)
            + \BigO{\sqrt{r}\Delta(T_{2})}\frac{\delta}{\norm{\vv}}.
    \end{align*}
    Similar to $\Delta$ case, we get
    \[
        \Deltamax(t)
        \le \Deltamax(T_2)
        \lesssim \poly(1/\epsnm)\sqrt{r}\Delta(T_2)
        \le \epsdir^{1/6}.
    \]

    It remains to close the stopping-time argument. The above estimates, applied with the stopped stage endpoints, give for all $t\le \tau$ that $\Delta(t)\le\epsdir^{1/5}$ and $\Delta_{\max}(t)\le\epsdir^{1/6}$. Under our parameter choices, $\epsdir^{1/5}\ll 1/r^4$ and $\epsdir^{1/6}\ll \log m/r$. Hence $\tau<T_2$ is impossible: by the definition of $\tau$, either $\Delta(\tau)=1/r^4$ or $\Delta_{\max}(\tau)=\log m/r$, contradicting the improved bounds above. Therefore $\tau=T_2$, which proves the claimed bounds for all $t\le T_2$.
\end{proof}

\subsection{Bound of $A_t$}\label{sec: At bound proof}
First, we have the explicit form of $A_t$ from \cref{claim: tldw dynamic} as
    \begin{align*}
        A_t =
        -2 \sum_{j\ge 0}\langle 
            \grad_{2j,\vv} - \grad_{2j,\ckvv},\vv - \ckvv
            \rangle
        - \lambda \norm{\vv - \ckvv}^2
    \end{align*}
where $\grad_{2j,\vv}$ represents the gradient from $2j$-th order loss
    \begin{align*}
        -\grad_{0,\vv} 
        =& 2\hsigma_0^2 \left(
                \sum_{j=1}^r a_j^* \norm{\vw_j^*}_2
                - \E_\mu[ \norm{\vw}_2^2 ]\right) \vv,\\
        -\grad_{2j,\vv} 
        =& \hsigma_{2j}^2 \left(
            2j (\tT_k^* - \tT_k)(\bvv^{\otimes k-1})\norm{\vv}    
            - (2j-2) (\tT_k^* - \tT_k)(\bvv^{\otimes k})\vv\right).
    \end{align*}
Here 
    \begin{align*}
        \tT_k = \frac{1}{m}\sum_{i=1}^m \norm{\vw_i}_2^2 \bvw_i^{\otimes k}, \quad
        \tT_k^* = \sum_{i=1}^r a_i^*\norm{\vw_i^*}_2\bvw_i^{*\otimes k}
    \end{align*}
are $k$-th order tensors.

\begin{lemma}[Bound on $A_t$]\label{lem: At bound}
    Suppose $\norm{\vv-\ckvv}\le \norm{\ckvv}/2$. For Stage 1 that $t\le T_1$ we have for any neuron pair $\vv,\ckvv$
    \begin{align*}
        A_t 
        \le \left(\frac{\rd}{\rd t}\log \norm{\ckvv}^2
            + 4\hsigma_2^2\left(1-\frac{1}{r}\sum_{i\le r}\tldE w_i^2\right) 
            + \BigO{\frac{\log m}{r}}\right)\cdot \norm{\vv-\ckvv}^2
            + \BigO{1}\frac{\norm{\vv-\ckvv}^3}{\norm{\ckvv}} 
    \end{align*}
    For Stage 2 that $T_1\le t\le T_2$, for any small-and-dense neuron $\ckvv\in\Sdense$ we have
    \begin{align*}
        A_t 
        \le \left(\frac{\rd}{\rd t}\log \norm{\ckvv}^2+\err\right)\cdot \norm{\vv-\ckvv}^2
            + \BigO{1}\frac{\norm{\vv-\ckvv}^3}{\norm{\ckvv}} 
    \end{align*}
    where $\err=\BigO{\frac{\omega\log m}{r}}$ for $T_{1}\le t\le T_{21}$; $\err\lesssim 1$ for $T_{21}\le t\le T_\gamma$; $\err=\BigO{\frac{\omega^9}{\log m}}$ for $T_\gamma\le t\le T_2$.;

    For any potential neuron $\ckvv\in \Spoti\setminus G_i$ for any $t\le T_2$ or $\ckvv\in G_i$ for any $t\le T_{21}$ we have
    \begin{align*}
        A_t 
        \le \left(\frac{\rd}{\rd t}\log \norm{\ckvv}^2+\err+\BigO{\bckv_i^2}\right)\cdot \norm{\vv-\ckvv}^2
            + \BigO{1}\frac{\norm{\vv-\ckvv}^3}{\norm{\ckvv}} 
    \end{align*}
    where $\err$ is the same defined above.
    
    For any $\ckvv\in G_i$ for $t\ge T_{21}$ we have
    \begin{align*}
        A_t 
        \le \left(\frac{\rd}{\rd t}\log \norm{\ckvv}^2+\gamma_t\err\right)\cdot \norm{\vv-\ckvv}^2
            + \BigO{\gamma_t}\frac{\norm{\vv-\ckvv}^3}{\norm{\ckvv}}, 
    \end{align*}
    where for $T_{21}\le t\le T_\gamma$ $\err=\BigO{\frac{\omega^4}{r}}$  before $\max_i \ha_i =1/r^2$ and $\err\lesssim 1$ otherwise (which only takes at most $o(\frac{1}{\gamma}\log m)$ time as the spread of $\ha_i$ is $m^{o(1)}$); $\err=\BigO{\frac{\omega^9}{\log m}}$ for $T_\gamma\le t\le T_2$.
\end{lemma}
\begin{proof}
    From \cref{lem: smooth activation jacob} and \cref{lem: relu jacob} we know
    \begin{equation}\label{eq: At proof}
        \grad_\vv \L(\tldvtheta) - \grad_{\ckvv} \L(\tldvtheta)
        = \tJ(\bckvv) (\vv - \ckvv) + \BigO{1} \frac{\norm{\vv-\ckvv}^2}{\norm{\ckvv}}
    \end{equation}
    $\tJ(\bckvv)$ is the Jacobian given in \cref{lem: hermite of g and J} that has the form of
    \[
        \tJ(\bvv):=\sum_{k\ge 0}\hat\sigma_k^2\,\tJ_k(\bvv),
    \]
    where
    \[
    \tJ_k(\bvv)
    =
    (2-k)\tDelta_k[\bvv^{\otimes k}]\,I
    +
    k(2-k)\Big(
    \bvv\,\tDelta_k[\bvv^{\otimes k-1}]^\top
    +
    \tDelta_k[\bvv^{\otimes k-1}]\,\bvv^\top
    -
    \tDelta_k[\bvv^{\otimes k-1}]\bvv\bvv^\top
    \Big)
    +
    k(k-1)\tDelta_k[\bvv^{\otimes k-2}].
    \]
    As an example, we know all odd order $\tJ_k=0$ due to symmetry and for even order, we have the first few $\tJ_k$ as
    \begin{align*}
    \tJ_0(\bckvv)=&2\tDelta_0 I
               = 2 (\tldalpha -\alpha),\\
    \tJ_1(\bckvv)=&2\tDelta_1 = 0 \quad\text{(due to symmetry)},\\
    \tJ_2(\bckvv)=&2\tDelta_2
               = 2\left(\diag{
                   \tldE w_1^2,\ldots,\tldE w_d^2} - \sum_{i=1}^r \vw_i^*\vw_i^{*\top}\right).
    \end{align*}

    Recall 
    \[
        A_t = -2\gamma_t\langle \grad_{\vv}\L(\tldvtheta) - \grad_{\ckvv}\L(\tldvtheta),\vv - \ckvv\rangle
        = - 2\gamma_t\langle \tJ(\bckvv), (\vv-\ckvv)(\vv-\ckvv)^\top\rangle
        \pm \BigO{\gamma_t}\frac{\norm{\vv-\ckvv}^3}{\norm{\ckvv}}.
    \]
    where we use \eqref{eq: At proof}. In below, we go though the stages to simplify $\tJ$ to get the bound for $A_t$. We focus on ReLU in below. The same argument applies to smooth activation by \cref{lem: tDelta and J bound}, and it is often easier as the decay of $\hsigma_k^2$ is much faster.

    \paragraph{Stage 1}
    From \cref{lem: stage 1 jacob bound}, we know
    \begin{align*}
        &\langle \tJ(\bckvv), (\vv-\ckvv)(\vv-\ckvv)^\top\rangle
        - \sum_{k=0,2} \hsigma_k^2 \langle \tJ_k(\bckvv), (\vv-\ckvv)(\vv-\ckvv)^\top\rangle\\
        =& \sum_{k\ge4} \hsigma_k^2 \langle \tJ_k(\bckvv), (\vv-\ckvv)(\vv-\ckvv)^\top\rangle\\
        \myeq{a}& \sum_{k\ge4} \hsigma_k^2 \langle \mK_k, (\vv-\ckvv)(\vv-\ckvv)^\top\rangle
        \pm \BigO{\sum_{k\ge4} \frac{k^{-3/2}}{r} + k^{-1/2}\left(\frac{\log m}{r}\right)^{(k-2)/2}}\norm{\vv-\ckvv}^2\\
        \myge{b}& \left(\sum_{k\ge4} 2\hsigma_k^2 \tDelta_k[\bckvv^k]-\BigO{\frac{\log m}{r}}\right)\cdot \norm{\vv-\ckvv}^2
    \end{align*}
    where (a)(b) we use \cref{lem: stage 1 jacob bound} and $\hsigma_k^2=\Theta(k^{-5/2})$ from \cref{claim: hermite coeff}.

    Noticing that $\tldE w_i^2$ in $\tJ_2$ are balanced from \cref{lem: stage 1.1 main} and \cref{lem: stage 1.2 0th 2nd order dynamic} we have
    \begin{align*}
        A_t 
        \le& -2\gamma_t\left(\sum_{k\ge0} 2\hsigma_k^2 \tDelta_k[\bckvv^k]-\BigO{\frac{\log m}{r}}\right)\cdot \norm{\vv-\ckvv}^2
            + \BigO{1}\frac{\norm{\vv-\ckvv}^3}{\norm{\ckvv}}\\
        =& \gamma_t\left(\frac{\rd}{\rd t}\log \norm{\ckvv}^2
            +4\hsigma_2^2\tDelta_2[\bckvv^2] - 4\hsigma_2^2\tDelta_2[\overline{\vv-\ckvv}^2]
            +\BigO{\frac{\log m}{r}}
            \right)\cdot \norm{\vv-\ckvv}^2
            + \BigO{1}\frac{\norm{\vv-\ckvv}^3}{\norm{\ckvv}}
    \end{align*}
    where we use \cref{lem: ckvv dynamic}.

    Note that using the explict form of $\tDelta_2$ above and $\tldE w_i^2$ are balanced, we have
    \begin{align}\label{eq: J2 bound}
        4\hsigma_2^2\tDelta_2[\bckvv^2] - 4\hsigma_2^2\tDelta_2[\overline{\vv-\ckvv}^2]
        =& 4\hsigma_2^2\left(1-\frac{1}{r}\sum_{i\le r}\tldE w_i^2\right) \left( 
            \frac{\norm{\vvler-\ckvvler}^2}{\norm{\vv-\ckvv}^2} 
            - \frac{\norm{\ckvvler}^2}{\norm{\ckvv}^2}
            \right)
            \pm \BigO{\frac{\log d}{r}}\\
        \le& 4\hsigma_2^2\left(1-\frac{1}{r}\sum_{i\le r}\tldE w_i^2\right)    
            \pm \BigO{\frac{\log d}{r}}.
    \end{align}
    Combine the above, we get the estimate for Stage 1.

    \paragraph{Stage 2}
    For small-and-dense neuron $\ckvv\in \Sdense$, by \cref{lem: stage 2 jacob bound} we know
    \begin{align*}
        &\langle \tJ(\bckvv), (\vv-\ckvv)(\vv-\ckvv)^\top\rangle
        - \sum_{k=0,2} \hsigma_k^2 \langle \tJ_k(\bckvv), (\vv-\ckvv)(\vv-\ckvv)^\top\rangle\\
        =& \sum_{k\ge4} \hsigma_k^2 \langle \tJ_k(\bckvv), (\vv-\ckvv)(\vv-\ckvv)^\top\rangle\\
        \myge{a}& \sum_{k\ge4} 2\hsigma_k^2 \tDelta_k[\bckvv^k]\norm{\vv-\ckvv}^2
        \pm \BigO{\sum_{k\ge4} \frac{k^{-3/2}\omega^4}{r} + k^{-1/2}\left(\frac{\omega\log m}{r}\right)^{(k-2)/2}}\norm{\vv-\ckvv}^2\\
        \ge& \left(\sum_{k\ge4} 2\hsigma_k^2 \tDelta_k[\bckvv^k]-\BigO{\frac{\omega\log m}{r}}\right)\cdot \norm{\vv-\ckvv}^2
    \end{align*}
    where (a) we use \cref{lem: stage 2 jacob bound} and $\hsigma_k^2=\Theta(k^{-5/2})$ from \cref{claim: hermite coeff}.

    Using the balance of $\tldE w_i^2$ in \cref{lem: stage 2.1 IH}\ref{item: stage 2.1 IH population 2nd order balance} and \cref{lem: stage 2.2 log spread}, and $\norm{\ckvvler}/\norm{\ckvv}=1-\BigO{\omega^5/r}$ in \cref{lem: stage 2 IH}\ref{item: stage 2 IH reg} to bound $\tJ_2$ term as in \eqref{eq: J2 bound}, we know the contribution of $\tJ_2$ term is upper bounded by $\err$ term
    \begin{align*}
        A_t 
        \le& -2\left(\sum_{k\ge0} 2\hsigma_k^2 \tDelta_k[\bckvv^k]-\err\right)\cdot \norm{\vv-\ckvv}^2
            + \BigO{1}\frac{\norm{\vv-\ckvv}^3}{\norm{\ckvv}}
        = \left(\frac{\rd}{\rd t}\log \norm{\ckvv}^2+\err\right)\cdot \norm{\vv-\ckvv}^2
            + \BigO{1}\frac{\norm{\vv-\ckvv}^3}{\norm{\ckvv}},
    \end{align*}
    where $\err=\BigO{\frac{\omega\log m}{r}}$ for $T_{1}\le t\le T_{21}$; $\err\lesssim 1$ for $T_{21}\le t\le T_\gamma$; $\err=\BigO{\frac{\omega^9}{\log m}}$ for $T_\gamma\le t\le T_2$.

    For potential neuron $\ckvv\in\Spoti\setminus G_i$ or $\ckvv\in G_i$ before $T_{21}$, similarly by \cref{lem: stage 2 jacob bound} we know
    \begin{align*}
        &\langle \tJ(\bckvv), (\vv-\ckvv)(\vv-\ckvv)^\top\rangle
        - \sum_{k=0,2} \hsigma_k^2 \langle \tJ_k(\bckvv), (\vv-\ckvv)(\vv-\ckvv)^\top\rangle\\
        \myge{a}& \sum_{k\ge4} 2\hsigma_k^2 \tDelta_k[\bckvv^k]\norm{\vv-\ckvv}^2
        \pm \BigO{
            \sum_{k\ge4} \frac{k^{-3/2}\omega^4}{r} 
            + k^{-1/2}\left(\frac{\omega\log m}{r}\right)^{(k-2)/2}
            + k^{-3/2}\bckv_i^2
            + k^{-1/2} \bckv_i^{k-2}
            }
            \norm{\vv-\ckvv}^2\\
        \ge& \left(\sum_{k\ge4} 2\hsigma_k^2 \tDelta_k[\bckvv^k]-\BigO{\frac{\omega\log m}{r}+\bckv_i^2}\right)\cdot \norm{\vv-\ckvv}^2
    \end{align*}
    where (a) we use \cref{lem: stage 2 jacob bound} and $\hsigma_k^2=\Theta(k^{-5/2})$ from \cref{claim: hermite coeff}.

    Again similar to small-and-dense case above, we have
    \begin{align*}
        A_t 
        \le& \left(\frac{\rd}{\rd t}\log \norm{\ckvv}^2+\err+\BigO{\bv_i^2}\right)\cdot \norm{\vv-\ckvv}^2
            + \BigO{1}\frac{\norm{\vv-\ckvv}^3}{\norm{\ckvv}},
    \end{align*}
    where $\err$ is defined in the above small-and-dense case.

    Lastly, for $\ckvv\in G_i$ and $t\ge T_{21}$, we use the improved bound in \cref{lem: stage 2 jacob bound} to get
    \begin{align*}
        &\langle \tJ(\bckvv), (\vv-\ckvv)(\vv-\ckvv)^\top\rangle
        - \sum_{k=0,2} \hsigma_k^2 \langle \tJ_k(\bckvv), (\vv-\ckvv)(\vv-\ckvv)^\top\rangle\\
        \myge{a}& \sum_{k\ge4} 2\hsigma_k^2 \tDelta_k[\bckvv^k]\norm{\vv-\ckvv}^2
        \pm \BigO{
            \sum_{k\ge4} \frac{k^{-3/2}\omega^4}{r} 
            + k^{-3/2}\min\{k\sqrt{\epsdir},1\}
            }
            \norm{\vv-\ckvv}^2\\
        \ge& \left(\sum_{k\ge4} 2\hsigma_k^2 \tDelta_k[\bckvv^k]-\BigO{\frac{\omega^4}{r}}\right)\cdot \norm{\vv-\ckvv}^2
    \end{align*}
    where (a) we use \cref{lem: stage 2 jacob bound} and $\hsigma_k^2=\Theta(k^{-5/2})$ from \cref{claim: hermite coeff}.

    Again similar to the small-and-dense case, we know
    \begin{align*}
        A_t 
        \le& \gamma_t\left(\frac{\rd}{\rd t}\log \norm{\ckvv}^2+\err\right)\cdot \norm{\vv-\ckvv}^2
            + \BigO{\gamma_t}\frac{\norm{\vv-\ckvv}^3}{\norm{\ckvv}},
    \end{align*}
    where $\err=\BigO{\frac{\omega^4}{r}}$ for $T_{21}\le t\le T_\gamma$ and before $\max_i \ha_i =1/r^2$ and $\err\lesssim 1$ otherwise (which only takes at most $o(\frac{1}{\gamma}\log m)$ time as the spread of $\ha_i$ is $m^{o(1)}$); $\err=\BigO{\frac{\omega^9}{\log m}}$ for $T_\gamma\le t\le T_2$.
\end{proof}

\begin{lemma}\label{lem: ckvv dynamic}
    Dynamic of $\norm{\ckvv}$ is given by
    \begin{align*}
        \frac{\rd}{\rd t}\norm{\ckvv}^2
        = -\gamma_t\sum_k 2\hsigma_k^2 \tDelta_k[\bckvv^{\otimes k}]\norm{\ckvv}^2.
    \end{align*}
\end{lemma}
\begin{proof}
    We give a direct proof based on Jacobian $\tJ$. Note that due to the 2-homogeneous of the activation $\norm{\vv}^2\sigma(\bvv^\top\vx)$ we have
    \begin{align*}
        \frac{\rd}{\rd t}\norm{\ckvv}^2
        = -\gamma_t\langle\grad_{\ckvv}\L(\tldvtheta),\ckvv\rangle
        = -\gamma_t\langle\tJ(\bckvv),\ckvv\ckvv^\top\rangle
    \end{align*}
    Then using the exact form of $\tJ_k$ given in \cref{lem: hermite of g and J} we have
    \begin{align*}
        \frac{\rd}{\rd t}\norm{\ckvv}^2
        = -\gamma_t\sum_k 2\hsigma_k^2 \tDelta_k[\bckvv^{\otimes k}]\norm{\ckvv}^2.
    \end{align*}
\end{proof}

The next two lemma below gives the bound on Jacobian $\tJ_k$ in Stage 1 and Stage 2. Note that the bound here is likely not tight, but it is enough give our results. 
\begin{lemma}[Stage 1 Jacobian bound]\label{lem: stage 1 jacob bound}
    For Stage 1 that $t\le T_1$, we have for any even $k\ge 4$ and neuron $\tldvv$,
    \[
        \norm{\tJ_k(\btldvv)-\mK_k}_{\op}
        \lesssim \frac{k}{r} + k^2 \left(\frac{\log m}{r}\right)^{(k-2)/2}
    \]
    where $\mK_k:=2\tDelta_k[\btldvv^{\otimes k}]\mI
          + k(k-2)\left(\mI-\btldvv\btldvv^\top\right) 
            \tT_k[\btldvv^{\otimes k-2}]
            \left(\mI-\btldvv\btldvv^\top\right)
            \succ 2\tDelta_k[\btldvv^{\otimes k}]\mI$.
\end{lemma}
\begin{proof}
    We will drop the tilde to use $\vv$ instead of $\tldvv$ during the proof.
    
    By \cref{lem: hermite of g and J} we know
    \begin{align*}
    \tJ_k(\bvv)
    =& (2-k) \tDelta_k[\bvv^{\otimes k}]\mI
          + k(k-2)\left(\mI-\bvv\bvv^\top\right) 
            \tDelta_k[\bvv^{\otimes k-2}]
            \left(\mI-\bvv\bvv^\top\right)
          +k\tDelta_k[\bvv^{\otimes k-2}]\\
    =& 2\tDelta_k[\bvv^{\otimes k}]\mI
          + k(k-2)\left(\mI-\bvv\bvv^\top\right) 
            \tT_k[\bvv^{\otimes k-2}]
            \left(\mI-\bvv\bvv^\top\right)\\
         &+k\tT_k[\bvv^{\otimes k-2}]-k\tT_k[\bvv^{\otimes k}]\mI\\
         &+ k\tT_k^*[\bvv^{\otimes k}]\mI-k(k-2)\left(\mI-\bvv\bvv^\top\right) 
            \tT_k^*[\bvv^{\otimes k-2}]
            \left(\mI-\bvv\bvv^\top\right)
          -k\tT_k^*[\bvv^{\otimes k-2}].
    \end{align*}
    where recall $\tDelta=\tT_k - \tT_k^*$. Note that the second term in first line is always p.s.d..

    For the second line, we have for $k\ge 6$
    \begin{align*}
        \norm{k\tT_k[\bvv^{\otimes k-2}]-k\tT_k[\bvv^{\otimes k}]\mI}_{\op}
        \le 2k \tldE[\norm{\vw}^2\langle\bvw,\bvv\rangle^{k-2}]
        \lesssim \frac{k}{r}
    \end{align*}
    where we use \cref{lem: high order moment bound}, \cref{lem: stage 1.1 IH implication} and \cref{lem: stage 1.2 IH implication}.

    For $k=4$, one can explicitly write it out as
    \begin{align*}
        \norm{k\tT_k[\bvv^{\otimes k-2}]-k\tT_k[\bvv^{\otimes k}]\mI}
        \lesssim \norm{\tldE[\norm{\vw}^2\langle\bvw,\bvv\rangle^2\bvw\bvw^\top]}+\tldE[\norm{\vw}^2\langle\bvw,\bvv\rangle^4]
        \lesssim \frac{1}{r}
    \end{align*}
    where we use \cref{lem: high order moment bound}, \cref{lem: stage 1.1 IH implication} and \cref{lem: stage 1.2 IH implication}, and a direct bound on $\tldE[\norm{\vw}^2\langle\bvw,\bvv\rangle^2\bvw\bvw^\top]$ when $\vw$ is Gaussian.
    
    We bound the third line term below as
    \begin{align*}
        &k\tT_k^*[\bvv^{\otimes k}]\mI-k(k-2)\left(\mI-\bvv\bvv^\top\right) 
            \tT_k^*[\bvv^{\otimes k-2}]
            \left(\mI-\bvv\bvv^\top\right)
          -k\tT_k^*[\bvv^{\otimes k-2}]\\
        =& k\sum_i\bv_i^k \mI
         -k(k-2)\left(\mI-\bvv\bvv^\top\right) 
            \diag{\bv_1^{k-2},\ldots,\bv_d^{k-2}}
            \left(\mI-\bvv\bvv^\top\right)
          -k\diag{\bv_1^{k-2},\ldots,\bv_d^{k-2}}.
    \end{align*}
    Thus, a direct bound is
    \begin{align*}
        &\norm{k\tT_k^*[\bvv^{\otimes k}]\mI-k(k-2)\left(\mI-\bvv\bvv^\top\right) 
            \tT_k^*[\bvv^{\otimes k-2}]
            \left(\mI-\bvv\bvv^\top\right)
          -k\tT_k^*[\bvv^{\otimes k-2}]}_{\op}\\
        \lesssim& k \sum_i \bv_i^k + k^2 \max_i \bv_i^{k-2}
        \lesssim k \left(\frac{\log m}{r}\right)^{k/2}+k^2 \left(\frac{\log m}{r}\right)^{(k-2)/2},
    \end{align*}
    where we use $k\ge 4$ and the bound in \cref{lem: stage 1.1 IH implication} and \cref{lem: stage 1.2 IH implication}. In particular, $\sum_i\bv_i^k\le \max_i\bv_i^{k-4}\sum_i\bv_i^4$.

\end{proof}

\begin{lemma}[Stage 2 Jacobian bound]\label{lem: stage 2 jacob bound}
    For Stage 2 that $T_1\le t\le T_2$, we have for any even $k\ge 4$ and neuron $\tldvv$,
    \[
        \tJ_k(\btldvv)
        \succ 2\tDelta_k[\btldvv^{\otimes k}]\mI
            \pm \BigO{\frac{k\omega^4}{r} + k^2 \left(\frac{\omega\log m}{r}\right)^{(k-2)/2} + (k\ha_i\bv_i^2 + k^2\bv_i^{k-2})\ind_{\tldvv\in \Spoti}}
    \]
    Moreover, for $\tldvv\in G_i$ and $t\ge T_{21}$, we can further improved to
    \[
        \tJ_k(\btldvv)
        \succ 2\tDelta_k[\btldvv^{\otimes k}]\mI
            \pm \BigO{\frac{k\omega^4}{r}+\min\{k^2\sqrt{\epsdir},k\}}
    \]
\end{lemma}
\begin{proof}
    The proof is overall similar to \cref{lem: stage 1 jacob bound}. The only difference is that now in $\tT_k$ part we separate the $\Sdense$ part and $\Spot$ part. The $\Sdense$ part follow the simlar argument, and $\Spot$ part is the new part we need to bound.

    We will drop the tilde to use $\vv$ instead of $\tldvv$ during the proof.
    
    By \cref{lem: hermite of g and J} we know
    \begin{align*}
    \tJ_k(\bvv)
    =& (2-k) \tDelta_k[\bvv^{\otimes k}]\mI
          + k(k-2)\left(\mI-\bvv\bvv^\top\right) 
            \tDelta_k[\bvv^{\otimes k-2}]
            \left(\mI-\bvv\bvv^\top\right)
          +k\tDelta_k[\bvv^{\otimes k-2}]\\
    =& 2\tDelta_k[\bvv^{\otimes k}]\mI
          + k(k-2)\left(\mI-\bvv\bvv^\top\right) 
            \tT_k[\bvv^{\otimes k-2}]
            \left(\mI-\bvv\bvv^\top\right)\\
         &+k\tT_k[\bvv^{\otimes k-2}]-k\tT_k[\bvv^{\otimes k}]\mI\\
         &+ k\tT_k^*[\bvv^{\otimes k}]\mI-k(k-2)\left(\mI-\bvv\bvv^\top\right) 
            \tT_k^*[\bvv^{\otimes k-2}]
            \left(\mI-\bvv\bvv^\top\right)
          -k\tT_k^*[\bvv^{\otimes k-2}].
    \end{align*}
    where recall $\tDelta=\tT_k - \tT_k^*$. Note that the second term in first line is always p.s.d..

    For the second line, we have for $k\ge 6$
    \begin{align*}
        \norm{k\tT_k[\bvv^{\otimes k-2}]-k\tT_k[\bvv^{\otimes k}]\mI}_{\op}
        \le& 2k \tldE_{\Sdense}[\norm{\vw}^2\langle\bvw,\bvv\rangle^{k-2}]
            + 2k\sum_{i\le r} \tldE_{\Spoti}[\norm{\vw}^2\langle\bvw,\bvv\rangle^{k-2}]\\
        \lesssim& \frac{k\omega^4}{r} + k\sum_{i\le r}\ha_i \bv_i^4
        \lesssim \frac{k\omega^4}{r} + k\ha_i\bv_i^4\ind_{\vv\in \Spoti}
    \end{align*}
    where we use \cref{lem: high order moment bound}, \cref{lem: stage 2 IH implication}.

    For $k=4$, one can explicitly write it out as
    \begin{align*}
        \norm{k\tT_k[\bvv^{\otimes k-2}]-k\tT_k[\bvv^{\otimes k}]\mI}
        \lesssim \norm{\tldE_{\Sdense}[\norm{\vw}^2\langle\bvw,\bvv\rangle^2\bvw\bvw^\top]}
            +\sum_i\norm{\tldE_{\Spoti}[\norm{\vw}^2\langle\bvw,\bvv\rangle^2\bvw\bvw^\top]}\lesssim \frac{\omega^4}{r}+\ha_i\bv_i^2\ind_{\vv\in \Spoti}
    \end{align*}
    where we use \cref{lem: high order moment bound}, \cref{lem: stage 2 IH implication}, and a direct bound on $\tldE[\norm{\vw}^2\langle\bvw,\bvv\rangle^2\bvw\bvw^\top]$ when $\vw$ is Gaussian.
    
    We bound the third line term below as
    \begin{align*}
        &k\tT_k^*[\bvv^{\otimes k}]\mI-k(k-2)\left(\mI-\bvv\bvv^\top\right) 
            \tT_k^*[\bvv^{\otimes k-2}]
            \left(\mI-\bvv\bvv^\top\right)
          -k\tT_k^*[\bvv^{\otimes k-2}]\\
        =& k\sum_i\bv_i^k \mI
         -k(k-2)\left(\mI-\bvv\bvv^\top\right) 
            \diag{\bv_1^{k-2},\ldots,\bv_d^{k-2}}
            \left(\mI-\bvv\bvv^\top\right)
          -k\diag{\bv_1^{k-2},\ldots,\bv_d^{k-2}}.
    \end{align*}
    Thus, a direct bound is
    \begin{align*}
        &\norm{k\tT_k^*[\bvv^{\otimes k}]\mI-k(k-2)\left(\mI-\bvv\bvv^\top\right) 
            \tT_k^*[\bvv^{\otimes k-2}]
            \left(\mI-\bvv\bvv^\top\right)
          -k\tT_k^*[\bvv^{\otimes k-2}]}_{\op}\\
        \lesssim& k \sum_i \bv_i^k + k^2 \max_i \bv_i^{k-2}
        \lesssim k \left(\frac{\omega\log m}{r}\right)^{(k-2)/2}+k^2 \left(\frac{\omega\log m}{r}\right)^{(k-2)/2}+k^2 \bv_i^{k-2}\ind_{\vv\in \Spoti},
    \end{align*}
    where we use $k\ge 4$ and the bound in \cref{lem: stage 1.1 IH implication} and \cref{lem: stage 1.2 IH implication}. In particular, $\sum_i\bv_i^k\le \max_i\bv_i^{k-4}\sum_i\bv_i^4$.

    \paragraph{Bound for $\vv\in G_i$.} Still using the form of $\tJ_k$ above, but we recombine the term as below, especially the $\Sdense$ part in $\tT_k$ will still be bounded in the same way and we now combine the $\Spot$ part with $\tT_k^*$ term.

    We have
    \begin{align*}
        &k\tT_k[\bvv^{\otimes k-2}]-k\tT_k^*[\bvv^{\otimes k-2}]\\
        =& k\tldE_{\Sdense}[\norm{\vw}^2\langle\bvw,\bvv\rangle^{k-2}\bvw\bvw^\top] 
            +k\sum_i \tldE_{\Spoti}[\norm{\vw}^2\langle\bvw,\bvv\rangle^{k-2}\bvw\bvw^\top]
            -k\bv_i^{k-2}\ve_i\ve_i^\top - \BigO{k\epsdir}\\
        =&  -k(1-\ha_i)\bv_i^{k-2}\ve_i\ve_i^\top
            \pm \BigO{\frac{k\omega^4}{r}+\min\{k^2\sqrt{\epsdir},k\}}
    \end{align*}
    and
    \begin{align*}
        &k\tT_k^*[\bvv^{\otimes k}]\mI-k\tT_k[\bvv^{\otimes k}]\mI\\
        =& k\bv_i^k\mI 
            - k\tldE_{\Sdense}[\norm{\vw}^2\langle\bvw,\bvv\rangle^{k}]\mI 
            - k\sum_i \tldE_{\Spoti}[\norm{\vw}^2\langle\bvw,\bvv\rangle^{k}]\mI
            \pm \BigO{k\epsdir}\\
        =& k(1-\ha_i)\bv_i^k\mI 
            \pm \BigO{\frac{k\omega^4}{r}+\min\{k^2\sqrt{\epsdir},k\}}
    \end{align*}     
    and 
    \begin{align*}
        &-k(k-2)\left(\mI-\bvv\bvv^\top\right) 
            \tT_k^*[\bvv^{\otimes k-2}]
            \left(\mI-\bvv\bvv^\top\right)
        =k(k-2)\left(\mI-\bvv\bvv^\top\right) 
            \diag{\bv_1^{k-2},\ldots,\bv_d^{k-2}}
            \left(\mI-\bvv\bvv^\top\right)
        =\pm\BigO{k^2\epsdir}.
    \end{align*}

    Thus, we know 
    \begin{align*}
        \tJ_k(\bvv)-2\tDelta_k[\bvv^{\otimes k}]\mI
        \succ k(1-\ha_i)(\mI-\ve_i\ve_i^\top) 
            \pm \BigO{\frac{k\omega^4}{r}+\min\{k^2\sqrt{\epsdir},k\}}
        \succ \pm \BigO{\frac{k\omega^4}{r}+\min\{k^2\sqrt{\epsdir},k\}}.
    \end{align*}
\end{proof}
\subsection{Bound of $B_t$ and $\E[B_t]$}\label{sec: Bt bound proof}
In this section, we give the bound on $B_t$ term and an improved bounod of $\E[B_t]$ where $\E$ is empirical expectation over all neuron pair $(\vv,\ckvv)$. Recall the definition of $B_t$ as
\begin{align*}
    B_t &= -2\langle \grad_{\vv}\L(\vtheta) - \grad_{\vv}\L(\tldvtheta),\vv - \ckvv\rangle
\end{align*}

\begin{lemma}[Bound on $B_t$]\label{lem: Bt bound}
    Assume $\Delta<1$. We have for $t\le T_2$
    \begin{align*}
        B_t
        \lesssim& \gamma_t\left(\sqrt{r}\Delta + \err\right)\norm{\vv}\norm{\vv - \ckvv},\\
        \E[B_t]
        \lesssim& \err\sqrt{r}\Delta
            + r\Delta^{5/2}, \text{ excluding $T_{21}\le t\le T_\gamma$,}
    \end{align*}
    where $\err\lesssim \frac{1}{m^{1/2-o(1)}}$ for $t\le T_{21}$ and $\err \lesssim r\epsdir^{1/4}\log m$ for $T_{21}\le t\le T_{2}$.

    For $T_{21}\le t\le T_\gamma$, denote $G=\cup_k G_k$ and $\ha=\sum_i\ha_i$. Suppose $\Delta_G\le \ha^{1/2}$, we have
    \[
        \E_G[B_t]=\frac{1}{m}\sum_{i\in G} B_{t,i}
        \lesssim \gamma\left(\sqrt{r\ha}\Delta_{-G}+\epsdir^{1/4}\ha\log m\right)\Delta_G + \gamma\ha^{3/4}\Delta_G^{5/2}.
    \]
\end{lemma}
\begin{proof}
    We have
    \begin{align*}
    B_t 
    &= -2\gamma_t\langle \grad_{\vv}\L(\vtheta) - \grad_{\vv}\L(\tldvtheta),\vv - \ckvv\rangle
    \end{align*}
    Note that we can split it into 2 terms
    \begin{align*}
        \grad_{\vv}\L(\vtheta) - \grad_{\vv}\L(\tldvtheta)
        =& \E_\vx\left[
            (f(\vx)-\tldf(\vx))\nabla_\vv(\norm{\vv}^2\sigma(\bvv^\top\vx))
            \right]\\
        =& \E_\vx\left[
            (f(\vx)-\ckf(\vx))\nabla_\vv(\norm{\vv}^2\sigma(\bvv^\top\vx))
            \right]
            + \E_\vx\left[
            (\ckf(\vx)-\tldf(\vx))\nabla_\vv(\norm{\vv}^2\sigma(\bvv^\top\vx))
            \right].
    \end{align*}
    Thus,
    \begin{align*}
        B_t
        =& -2\gamma_t\E_\vx\left[
            (f(\vx)-\ckf(\vx)) \langle\nabla_\vv(\norm{\vv}^2\sigma(\bvv^\top\vx)), \vv - \ckvv\rangle
            \right]
            - 2\gamma_t\E_\vx\left[
            (\ckf(\vx)-\tldf(\vx)) \langle\nabla_\vv(\norm{\vv}^2\sigma(\bvv^\top\vx)),\vv - \ckvv\rangle
            \right]\\
        \lesssim& \gamma_t(\sqrt{r}\Delta+\Delta^2)\norm{\vv}\norm{\vv - \ckvv}
            + \gamma_t\sqrt{\frac{r B_\infty\log m}{m}}\norm{\vv}\norm{\vv - \ckvv},
    \end{align*}
    where we use \cref{lem: tldf ckf bound,lem: f ckf bound} that bound $\norm{f-\ckf}$ and $\norm{\ckf-\tldf}$, and $B_\infty := \max_{i\in[m]} \norm{\ckvv_i}_2^2
    \left(
        \avgrho 
        \norm{
            \btldvv_{i,\rho}
            -
            \bckvv_i
        }_2
    \right) $.

    For $B_\infty$ above, before Stage 2.1 that $t\le T_{21}$, we know $\max_i\norm{\ckvv_i}^2\le \sigma_1^4=m^c$ with any small enough constant $c$ (from \cref{lem: stage 1.1 IH implication} for Stage 1.1, \cref{lem: stage 1.2 IH implication} for Stage 1.2 and  \cref{lem: stage 2 IH implication}, \cref{lem: stage 2 IH}\ref{item: stage 2 IH norm bound or basis like} for Stage 2.1), so $B_\infty\lesssim m^{o(1)}$. After that for $t\ge T_{21}$, we know (1) for $\ckvv\not\in\cup_i G_i$ that not aligned, $\max_i\norm{\ckvv_i}^2\lesssim \sigma_1^2=m^{c}$ from \cref{lem: stage 2 IH implication}, \cref{lem: stage 2 IH}\ref{item: stage 2 IH norm bound or basis like}; 
    (2) for $\ckvv\in G_i$, $\norm{\ckvv_i}_2^2\le m$ and $\norm{\btldvv_{i,\rho}-\bckvv_i}_2=\norm{\rho(\btldvv_{i})-\bckvv_i}_2\lesssim\sqrt{\epsdir}$ for any $\rho\in\Rho$. Thus, we know $B_\infty\lesssim \sqrt{\epsdir} m$ for $t\ge T_{21}$.

    Combine this with the bound of $B_t$ above gives the result.

    \paragraph{Bound on $\E[B_t]$.}
    Let $\phi(\vv;\vx):=\norm{\vv}^2\sigma(\bvv^\top\vx)$. To get a better bound than $B_t$, it is important to notice that
    \begin{align*}
        \langle\nabla_\vv(\norm{\vv}^2\sigma(\bvv^\top\vx),\vv-\ckvv\rangle
        =\langle\nabla_\vv\phi(\vv;\vx),\vv-\ckvv\rangle
        = \phi(\vv;\vx)-\phi(\ckvv;\vx) + \tau(\vv,\ckvv;\vx)
    \end{align*}
    where $\tau(\vv,\ckvv;\vx)=\phi(\ckvv;\vx) - \phi(\vv;\vx) - \langle\nabla_\vv\phi(\vv;\vx),\vv-\ckvv\rangle$ should be viewed as second-order remainder error of activation $\phi$.

    Recall we have above that
    \begin{align*}
        B_t
        =& -2\gamma_t\E_\vx\left[
            (f(\vx)-\ckf(\vx)) \langle\nabla_\vv(\norm{\vv}^2\sigma(\bvv^\top\vx)), \vv - \ckvv\rangle
            \right]
            + \gamma_t\E_\vx\left[
            (\ckf(\vx)-\tldf(\vx)) \langle\nabla_\vv(\norm{\vv}^2\sigma(\bvv^\top\vx)),\vv - \ckvv\rangle
            \right].
    \end{align*}
    For the first term, we have
    \begin{equation}\label{eq: Bt f ckf}
    \begin{aligned}
        &\E_\vx\left[
            (f(\vx)-\ckf(\vx)) \langle\nabla_\vv(\norm{\vv}^2\sigma(\bvv^\top\vx)), \vv - \ckvv\rangle
            \right]\\
        &= \E_\vx\left[
            (f(\vx)-\ckf(\vx)) (\phi(\vv;\vx)-\phi(\ckvv;\vx))
            \right]
          + \E_\vx\left[
            (f(\vx)-\ckf(\vx)) \tau(\vv,\ckvv;\vx)
            \right]
    \end{aligned}    
    \end{equation}
    When taking $\E$ over all neuron, it becomes
    \begin{align*}
        \E_{(\vv,\ckvv)}\E_\vx\left[
            (f(\vx)-\ckf(\vx)) \langle\nabla_\vv(\norm{\vv}^2\sigma(\bvv^\top\vx)), \vv - \ckvv\rangle
            \right]
        =& \E_\vx\left[
            (f(\vx)-\ckf(\vx))^2
            \right]
          + \E_\vx\left[
            (f(\vx)-\ckf(\vx)) \E_{(\vv,\ckvv)}\tau(\vv,\ckvv;\vx)
            \right]\\
        \gtrsim& - \left(\sqrt{r}\Delta + \Delta^2\right)\E_{(\vv,\ckvv)}\norm{\tau},
    \end{align*}
    where the first term is non-negative, the second term follows from \cref{lem: f ckf bound} that bound $\norm{f-\ckf}$ and \cref{lem: sigma 2nd remainder bound} that bound $\norm{\tau}\lesssim \norm{\vv-\ckvv}^2$ if $\sigma$ is smooth activation and $\norm{\tau}\lesssim\norm{\ckvv}^{1/2}\norm{\vv-\ckvv}^{3/2}$ if $\sigma$ is ReLU/LeakyReLU/Absolute. Note that for the later case $\E_{(\vv,\ckvv)}\norm{\tau}\le (\E \norm{\ckvv}^2)^{1/4}(\E\norm{\vv-\ckvv}^2)^{3/4}\lesssim r^{1/4}\Delta^{3/2}$. Thus,
    \begin{align*}
        \E_{(\vv,\ckvv)}\E_\vx\left[
            (f(\vx)-\ckf(\vx)) \langle\nabla_\vv(\norm{\vv}^2\sigma(\bvv^\top\vx)), \vv - \ckvv\rangle
            \right]
        \gtrsim -r \Delta^{5/2}.
    \end{align*}

    For the second term, we use the same bound as $B_t$, we get
    \begin{equation}\label{eq: Bt ckf tldf}
    \begin{aligned}
        \E_\vx\left[
            (\ckf(\vx)-\tldf(\vx)) \langle\nabla_\vv(\norm{\vv}^2\sigma(\bvv^\top\vx)),\vv - \ckvv\rangle
            \right]
        \gtrsim& -\sqrt{\frac{rB_\infty\log m}{m}} \E_{(\vv,\ckvv)}\norm{\vv}\norm{\vv - \ckvv}\\
        \ge& -\sqrt{\frac{rB_\infty\log m}{m}}(\norm{\ckvv}\norm{\vv - \ckvv}+\norm{\vv - \ckvv}^2)
    \end{aligned}    
    \end{equation}
    When taking $\E$ over all neuron, it becomes
    \begin{align*}
        \E_{(\vv,\ckvv)}\E_\vx\left[
            (\ckf(\vx)-\tldf(\vx)) \langle\nabla_\vv(\norm{\vv}^2\sigma(\bvv^\top\vx)),\vv - \ckvv\rangle
            \right]
        \gtrsim-\sqrt{\frac{rB_\infty\log m}{m}}(\sqrt{r}\Delta+\Delta^2),   
    \end{align*}
    where we use $\tldE \norm{\ckvv}^2\lesssim r$ from \cref{lem: stage 1.1 IH implication} for Stage 1.1, \cref{lem: stage 1.2 IH}\ref{item: stage 1.2 IH reg} for Stage 1.2, \cref{lem: stage 2 IH}\ref{item: stage 2 IH reg} for Stage~2.

    Thus, we finally get
    \begin{align*}
        \E[B_t]\lesssim
        \sqrt{\frac{r^2 B_\infty\log m}{m}}\Delta
        + r\Delta^{5/2}.
    \end{align*}
    The $B_\infty$ parts follow the same as in $B_t$ part.

    Lastly, for $T_{21}\le t\le T_\gamma$, since we increase the stepsize $\gamma$ for those in $G$. we need to slightly update the above bound. Denote $\ha=\frac{1}{m}\sum_i\ha_i$

    For \eqref{eq: Bt f ckf}, we only have neurons $G$ to do taylor expansion so we directly bound the rest neuron and get
    \begin{align*}
        &\E_{(\vv,\ckvv)\in G}\E_\vx\left[
            (f(\vx)-\ckf(\vx)) \langle\nabla_\vv(\norm{\vv}^2\sigma(\bvv^\top\vx)), \vv - \ckvv\rangle
            \right]\\
        =& \E_\vx\left[
            (f_G(\vx)-\ckf_G(\vx))^2
            \right]
          + \E_\vx\left[
            (f_G(\vx)-\ckf_G(\vx)) \E_{(\vv,\ckvv)\in G}\tau(\vv,\ckvv;\vx)
            \right]\\
          &+ \E_\vx\left[
            (f_{-G}(\vx)-\ckf_{-G}(\vx)) \langle\nabla_\vv(\norm{\vv}^2\sigma(\bvv^\top\vx)), \vv - \ckvv\rangle
            \right]\\
        \gtrsim& - \left(\sqrt{\ha}\Delta_G + \Delta_G^2\right)\E_{(\vv,\ckvv)\in G}\norm{\tau}
        -  \left(\sqrt{r}\Delta_{-G} + \Delta_{-G}^2\right)\E_{(\vv,\ckvv)\in G}\norm{\vv}\norm{\vv-\ckvv}\\
        \gtrsim& - \ha^{3/4}\Delta_G^{5/2} - \ha^{1/4}\Delta_G^{7/2}- (r\ha)^{1/2}\Delta_{-G}\Delta_G,
    \end{align*}
    where we use the same bound of $\tau$ as above and note that $\E_G \norm{\ckvv}^2=\sum_i\ha_i=\ha$ when using \cref{lem: f ckf bound}. Here $f_G$ represents the network with only neurons in $G$ and $f_{-G}$ be the network excluding the neurons in $G$.

    For \eqref{eq: Bt ckf tldf}, follow the same way we get
    \begin{align*}
        \E_{(\vv,\ckvv)\in G}\E_\vx\left[
            (\ckf(\vx)-\tldf(\vx)) \langle\nabla_\vv(\norm{\vv}^2\sigma(\bvv^\top\vx)),\vv - \ckvv\rangle
            \right]
        \gtrsim-\sqrt{\ha}\epsdir^{1/4}\log m \left(\ha^{1/2}\Delta_G+\Delta_G^2\right),   
    \end{align*}
    Combine above we get bound for $\E_G B_t$.
\end{proof}

\begin{lemma}\label{lem: sigma 2nd remainder bound}
Recall $\phi(\vv;\vx):=\norm{\vv}^2\sigma(\bv^\top \vx)$.
For \(\vv,\vu\neq 0\), define
\[
\tau(\vv,\vu;\vx)
:=
\phi(\vu;\vx)-\phi(\vv;\vx)
-
\left\langle \nabla_\vv \phi(\vv;\vx),\vu-\vv\right\rangle .
\]
Assume $\norm{\vu-\vv}\le \norm{\vu}/2$.
Then
\begin{align*}
    \mathbb E_{\vx}\left[\tau(\vv,\vu;\vx)^2\right]
    \lesssim\left\{\begin{array}{ll}
           \norm{\vu-\vv}^4  & \text{ if $\sigma$ is smooth activation} \\
            \norm{\vu}\norm{\vu-\vv}^3 & \text{ if $\sigma\in\{\mathrm{ReLU}, \mathrm{LeakyReLU}, \mathrm{Absolute}\}$}
    \end{array}\right.
\end{align*}
\end{lemma}

\begin{proof}
We show one by one.
\paragraph{Smooth activation.} Let $\vh=\vu-\vv$, $\vw_t:=\vv+t(\vu-\vv)$, $F(t):=\phi(\vw_t;\vx)$.
Then
\[
\tau(\vv,\vu;\vx)
=
F(1)-F(0)-F'(0)
=
\int_0^1 (1-t)F''(t)\,dt .
\]
It is not hard to show
$|F''(t)|
\lesssim
\norm{\Pi\vx}^2\norm{\vh}^2$,
where $\Pi$ is the projection on the $\operatorname{span}\{\vv,\vu\}$.
Therefore
\[
\E_{\vx}\left[\tau(\vv,\vu;\vx)^2\right]
\lesssim
\norm{\vh}^4
\E_{\vx}\left[\norm{\Pi\vx}^4\right]
\lesssim \norm{\vu-\vv}^4 .
\]

\paragraph{ReLU, LeakyReLU, Absolute activation.}
We focus on ReLU case. Other follows similarly.
Let $\vh:=\vu-\vv$.
We can compute $\nabla\phi$ and have 
\begin{align*}    
    \tau(\vv,\vu;\vx)
    =& \norm{\vu}\sigma(\vu^\top \vx)
        - \norm{\vv}\sigma(\vv^\top \vx)
        - \left\langle \bvv,\vh\right\rangle\sigma(\vv^\top \vx)
        - \norm{\vv}\sigma'(\vv^\top \vx)\vh^\top \vx\\
    =& \left(\norm{\vu}-\norm{\vv}
        -\left\langle \bvv,\vh\right\rangle
        \right)\sigma(\vv^\top \vx)
        +\left(\norm{\vu}-\norm{\vv}\right)
        \sigma'(\vv^\top \vx)\vh^\top \vx
        + \norm{\vu}\left[
            \sigma(\vu^\top \vx)
            - \sigma(\vv^\top \vx)
            - \sigma'(\vv^\top \vx)\vh^\top \vx
            \right]\\
    \lesssim& \frac{\norm{\vh}^2}{\norm{\vv}}|\sigma(\vv^\top \vx)|
        + \norm{\vh}|\sigma'(\vv^\top \vx)\vh^\top \vx|
        + \norm{\vu}|\vh^\top \vx| \ind_{\{(\vv^\top \vx)(\vu^\top \vx)\le 0\}},
\end{align*}
where we use $\left|
\sigma(a+b)-\sigma(a)-\sigma'(a)b
\right|
\lesssim |b|\mathbf 1_{\{a(a+b)\le 0\}}$.

Hence, taking expectation we obtain
\[
\E_{\vx}\left[\tau(\vv,\vu;\vx)^2\right]
\lesssim
\norm{\vh}^4
+\norm{\vu}\norm{\vh}^3
\lesssim \norm{\vu}\norm{\vh}^3.
\]
\end{proof}

\subsection{Bound on $C_t$}\label{sec: Ct bound proof}
In this section, we give the bound on $C_t$. Intuitively, $C_t$ captures the error induced by finite samples compared to population dynamics, so mostly follows from the standard concentration results.
\begin{lemma}[Bound on $C_t$]\label{lem: Ct bound}
    Suppose sample size $n>d\log^3 d$ and $\Delta<1$. Then we have for all $t\le T_2$
    \begin{align*}
        C_t\lesssim \gamma_t\sqrt{\frac{dr^2\log d}{n}}\norm{\vv}\norm{\vv-\ckvv}.
    \end{align*}
\end{lemma}
\begin{proof}
    Recall
    \begin{align*}
        C_t &= -2\gamma_t\langle \grad_{\vv}\hat{\L}(\vtheta) - \grad_{\vv}\L(\vtheta),\vv - \ckvv\rangle.
    \end{align*}
    We have
    \begin{align*}
        \grad_{\vv}\hat{\L}(\vtheta) - \grad_{\vv}\L(\vtheta)
        = \E_{\vu}\left[
            \frac{1}{n}\sum_i G_{\vu,\vv}(\vx_i)
            - \E_\vx G_{\vu,\vv}(\vx_i)
            \right]
            - \sum_{i=1}^r \left(
                \frac{1}{n}\sum_i G_{\vw_i^*,\vv}(\vx_i)
                - \E_\vx G_{\vw_i^*,\vv}(\vx_i)
                \right),
    \end{align*}
    where
    \[
        G_{\vu,\vv}(\vx):=
        \norm{\vu}^2\sigma(\bvu^\top \vx)\,\nabla_{\vv}\!\big(\norm{\vv}^2 \sigma(\bvv^\top \vx)\big)
        = \norm{\vu}^2\norm{\vv}\Gamma_{\bvu,\bvv}(\vx),
    \]
    \[
        \Gamma_{\bvu,\bvv}(\vx)
        :=
        2\bvv\,\sigma(\bvu^\top \vx)\sigma(\bvv^\top \vx)
        +
        \sigma(\bvu^\top \vx)\sigma'(\bvv^\top \vx)\,(I-\bvv\bvv^\top)\vx,
    \]
    Using the concentration of $\Gamma$ from \cref{lem: concentration Gamma gradient} and $n>d\log^3 d$, we know with probability of $1-1/d$
    \begin{align*}
        \norm{\grad_{\vv}\hat{\L}(\vtheta) - \grad_{\vv}\L(\vtheta)}
        \lesssim \E_\vu\left[
            \norm{\vu}^2\norm{\vv}\sqrt{\frac{d\log d}{n}}
            \right]
            + r\norm{\vv}\sqrt{\frac{d\log d}{n}}
        \lesssim \sqrt{\frac{dr^2\log d}{n}}\norm{\vv}
            +\sqrt{\frac{d\log d}{n}}\norm{\vv}\Delta^2,
    \end{align*}
    where we use $\tldE \norm{\ckvv}^2\lesssim r$ from \cref{lem: stage 1.1 IH implication} for Stage 1.1, \cref{lem: stage 1.2 IH}\ref{item: stage 1.2 IH reg} for Stage 1.2, \cref{lem: stage 2 IH}\ref{item: stage 2 IH reg} for Stage~2.

    The bound on $C_t$ then follows.
\end{proof}
Recall
\[
G_{\vu,\vv}(\vx)
:=
\norm{\vu}^2\sigma(\bvu^\top \vx)\,\nabla_{\vv}\!\big(\norm{\vv}^2 \sigma(\bvv^\top \vx)\big)
= \norm{\vu}^2\norm{\vv}\Gamma_{\bvu,\bvv}(\vx),
\]
where
\[
\Gamma_{\bvu,\bvv}(\vx)
:=
2\bvv\,\sigma(\bvu^\top \vx)\sigma(\bvv^\top \vx)
+
\sigma(\bvu^\top \vx)\sigma'(\bvv^\top \vx)\,(I-\bvv\bvv^\top)\vx,
\]

Below we give the concentration of $\Gamma$.
\begin{lemma}\label{lem: concentration Gamma gradient}
Let $\vx_1,\ldots,\vx_n \overset{\mathrm{iid}}{\sim} N(0,I_p)$. For activation $\sigma$ under \cref{assum: activation}, there exists a universal constant $C>0$ such that for any $0<\delta<1$ we have with probability $1-\delta$
\[
    \sup_{\bvu,\bvv\in\mathbb S^{p-1}}
    \left\|
    \frac1n\sum_{i=1}^n
    \Big(
    \Gamma_{\bvu,\bvv}(\vx_i)-\E_\vx\Gamma_{\bvu,\bvv}(\vx)
    \Big)
    \right\|
    \le
    C\Big(\sqrt{\frac{p\log(en/\delta)+\log(1/\delta)}{n}}+\frac{p\log(n)\log(1/\delta)}{n}\Big)
\]
\end{lemma}
\begin{proof}
We know
\[
\sup_{\bvu,\bvv\in\mathbb S^{p-1}}
\left\|
\frac1n\sum_{i=1}^n
\Big(
\Gamma_{\bvu,\bvv}(\vx_i)-\E_\vx\Gamma_{\bvu,\bvv}(\vx)
\Big)
\right\|
=
\sup_{\bvu,\bvv,\bvw\in\mathbb S^{p-1}}
\left|
(P_n-P)f_{\bvu,\bvv,\bvw}
\right|,
\]
where
\[
f_{\bvu,\bvv,\bvw}(\vx):=\bvw^\top \Gamma_{\bvu,\bvv}(\vx),
\qquad
P_n:=\frac1n\sum_{i=1}^n \delta_{\vx_i},
\qquad
P:=N(\vzero,\mI_p).
\]
Also,
\[
f_{\bvu,\bvv,\bvw}(\vx)
=
2(\bvw^\top\bvv)\sigma(\bvu^\top\vx)\sigma(\bvv^\top\vx)
+
\sigma(\bvu^\top\vx)\sigma'(\bvv^\top\vx)\,\bvw^\top(I-\bvv\bvv^\top)\vx.
\]

\medskip
\noindent{\bf Case (a): smooth activations.}
Let $U=\bvu^\top \vx$, $V=\bvv^\top \vx$, and $Z=\bw^\top(I-\bvv\bvv^\top)\vx$,
then $U,V,Z$ are centered Gaussian linear forms with uniformly bounded $\psi_2$ norms.
Since $\sigma(U)$ and $\sigma(V)$ have linear growth and $\sigma'(V)$ is bounded, each term in $f_{\bvu,\bvv,\bvw}(\vx)$ is a product of at most two sub-Gaussian factors. Hence
\[
\|f_{\bvu,\bvv,\bvw}(\vx)-\E f_{\bvu,\bvv,\bvw}(\vx)\|_{\psi_1}\le C_\sigma
\]
uniformly in $(\bvu,\bvv,\bvw)$ (see, e.g., the standard product and Bernstein facts for sub-Gaussian/sub-exponential variables \cite{vershynin2018high}). Therefore, by the sub-exponential Bernstein inequality \cite[Cor.~2.9.2]{vershynin2018high}, for every fixed
$\theta=(\bvu,\bvv,\bvw)$ and every $s\ge 0$,
\[
\P\left(
\left|(P_n-P)f_\theta\right|
\ge
C_\sigma\Big(\sqrt{\frac{s}{n}}+\frac{s}{n}\Big)
\right)
\le 2e^{-s}.
\]

Let $\Theta:=(\mathbb S^{p-1})^3$ and $\Theta_\eps$ be an $\eps$-net of $\Theta$ in the product Euclidean metric. Standard bounds give
$\log |\Theta_\eps|\le Cp\log(C/\eps)$.
Applying the fixed-parameter bound and taking a union bound over $\Theta_\eps$, we obtain with probability at least $1-e^{-t}$,
\[
\sup_{\theta\in\Theta_\eps}|(P_n-P)f_\theta|
\le
C_\sigma\Big(
\sqrt{\frac{p\log(C/\eps)+t}{n}}
+
\frac{p\log(C/\eps)+t}{n}
\Big).
\]

Next, we extend from $\eps$-net to the whole suprema. Let
\[
\gE_t:=\left\{
\max_{1\le i\le n}\norm{\vx_i}\le C\sqrt{p+\log n+t}
\right\}.
\]
We know it happens with probability at least $1-e^{-t}$. In the following we condition on such event happen. We have
\[
|f_\theta(\vx)-f_{\theta'}(\vx)|
\le
C_\sigma(1+\norm{\vx}^3)\norm{\theta-\theta'}_2,
\qquad
\theta,\theta'\in\Theta.
\]
\[
\frac1n\sum_{i=1}^n (1+\norm{\vx_i}^3)\le C(1+(p+\log n+t)^{3/2}),
\]
\[
\E(1+\norm{\vx}^3)\le C(1+p^{3/2})\le C(1+(p+\log n+t)^{3/2}).
\]
Therefore, 
\[
|(P_n-P)f_\theta-(P_n-P)f_{\theta'}|
\le
C_\sigma (p+\log n+t)^{3/2}\norm{\theta-\theta'}_2.
\]

Now choose
$\eps:=\frac{1}{\sqrt n\, (p+\log n+t)}.$
Then
\[
C_\sigma (p+\log n+t)^{3/2}\eps
\le
C_\sigma \sqrt{\frac{p+\log n+t}{n}},
\]
which is absorbed by the main square-root term. 
Hence, 
\[
\sup_{\theta\in\Theta}|(P_n-P)f_\theta|
\le
C_\sigma\Big(
\sqrt{\frac{p\log(Cn(2+t))+t}{n}}
+
\frac{p\log(Cn(2+t))+t}{n}
\Big).
\]

\medskip
\noindent{\bf Case (b): ReLU/LeakyReLU/Absolute value.}
For fixed $(\bvu,\bvv,\bw)$, $f_{\bvu,\bvv,\bw}(\vx)$ is a sum of products of Gaussian linear forms, hence is sub-exponential. More precisely, by the standard fact that a product of sub-Gaussian random variables is sub-exponential and by Bernstein's inequality for sub-exponential sums \cite[Lem.~2.8.6 and Cor.~2.9.2]{vershynin2018high},
\[
\|f_{\bvu,\bvv,\bw}(\vx)-\E f_{\bvu,\bvv,\bw}(\vx)\|_{\psi_1}\le C.
\]
The class
$\gF:=\{f_{\bvu,\bvv,\bw}:\ \bvu,\bvv,\bw\in\mathbb S^{p-1}\}$
is piecewise polynomial of degree at most $2$, with parameter dimension $3p$ and
$O(1)$ pieces, hence it is a VC-subgraph class with $\mathrm{VC}(\gF)\le Cp$ by standard VC/entropy bounds for parametric piecewise-polynomial classes \cite[Sec.~2.6]{van1996weak}. Since the envelope is sub-exponential under $P=N(\vzero,\mI_p)$, applying Adamczak's tail inequality for suprema of unbounded empirical processes \cite[Thm.~4]{adamczak2008tail} gives
\[
    \P \left(
    \sup_{f\in\gF}|(P_n-P)f|
    \ge
    C\Big(\sqrt{\frac{p\log(en)+t}{n}}+\frac{p\log n}{n}t\Big)
    \right)
    \le e^{-t},
\]
where we use $|f_{\bvu,\bvv,\bvw}(\vx)|\lesssim \norm{\vx}^2$.
\end{proof}
\subsection{Technical lemma}
In this part, we give the bound of gradient and jacobian that used in previous analysis.

\paragraph{Notation Setup.}
Define the shorthand for gradient
\[
g(\vv):=\grad_\vv \L(\tldvtheta)
=\E\left[(f(\vx)-f_*(\vx))\nabla_\vv(\norm{\vv}^2\sigma(\bvv^\top \vx))\right],
\]
and the residual
\[
\tDelta_k
:= \tT_k - \tT_k^*
= \frac1m\sum_{i=1}^m \norm{\vw_i}^2 \bvw_i^{\otimes k}
-
\sum_{j=1}^r \norm{\vw_j^*}^2\bvw_j^{*\otimes k}.
\]

For a symmetric \(k\)-tensor \(\tT\), we write
\[
\tT[\bvv^k]\in\R,\qquad
\tT[\bvv^{k-1}]\in\R^d,\qquad
\tT[\bvv^{k-2}]\in\R^{d\times d}
\]
for the corresponding contractions.

\begin{lemma}[Hermite expansion of \(G\), \(g\), and \(\nabla g\)]\label{lem: hermite of g and J} 
Let
$G(\vv):=\E\big[(f(\vx)-f_*(\vx))\,\norm{\vv}^2\sigma(\bvv^\top \vx)\big]$.
Then for \(\vv\neq 0\)
\begin{align*}
G(\vv) &=\sum_{k\ge 0}\hat\sigma_k^2 \norm{\vv}^{2-k}\tDelta_k[\vv^k]
       =\sum_{k\ge 0}\hat\sigma_k^2 \norm{\vv}^2 \tDelta_k[\bvv^k],\\
g(\vv) &=
\sum_{k\ge 0}\hat\sigma_k^2
\Big(
(2-k)\tDelta_k[\bvv^k]\vv
+
k\norm{\vv}\,\tDelta_k[\bvv^{k-1}]
\Big).
\end{align*}
Finally, if the Hessian series converges absolutely, then
\[
\nabla g(\vv)
=
2\hat\sigma_0^2\tDelta_0\mI
+
\hat\sigma_1^2\tJ_1(\bvv)
+
\sum_{k\ge 2}\hat\sigma_k^2\,\tJ_k(\bvv),
\]
where
$\tJ_1(\vu)
=
\tDelta_1[\vu]\mI
+
\vu\tDelta_1^\top
+
\tDelta_1\vu^\top
-
\tDelta_1[\vu]\vu\vu^\top$,
and for \(k\ge 2\),
\[
\tJ_k(\vu)
=
(2-k)\tDelta_k[\vu^{\otimes k}]\,\mI
+
k(k-2)(\mI-\vu\vu^\top)\tDelta_k[\vu^{\otimes k-2}](\mI-\vu\vu^\top)
+
k\,\tDelta_k[\vu^{\otimes k-2}].
\]
\end{lemma}

\begin{proof}
    Since the hermite expansion
    \begin{align*}
        f(\vx)-f_*(\vx)
        =&
        \sum_{k\ge 0}\hat\sigma_k\,\langle \tDelta_k,\tH_k(\vx)\rangle,
        \qquad
        \norm{\vv}^2\sigma(\bvv^\top \vx)
        =
        \sum_{k\ge 0}\hat\sigma_k\,\norm{\vv}^{2-k}\langle \vv^{\otimes k},\tH_k(\vx)\rangle,    
    \end{align*}
    orthogonality of Hermite tensors gives
    \[
        G(\vv)=\sum_{k\ge 0}\hat\sigma_k^2 \norm{\vv}^{2-k}\tDelta_k[\vv^k]
        =\sum_{k\ge 0}\hat\sigma_k^2 \norm{\vv}^2\tDelta_k[\bvv^k].
    \]
    
    Now fix \(k\ge 0\) and define
    $F_k(\vv):=\norm{\vv}^{2-k}\tDelta_k[\vv^k]=\norm{\vv}^2\tDelta_k[\bvv^k]$.
    Let \(\vdelta\in\R^d\). Since
    $D(\norm{\vv}^2)[\vdelta]=2\vv^\top\vdelta$, 
    $D(\bvv)[\vdelta]=\frac{1}{\norm{\vv}}(\mI-\bvv\bvv^\top)\vdelta$,
    we have
    \[
        D\big(\tDelta_k[\bvv^k]\big)[\vdelta]
        =
        k\,\tDelta_k\big[\bvv^{k-1},D(\bvv)[\vdelta]\big]
        =
        \frac{k}{\norm{\vv}}
        \big(\tDelta_k[\bvv^{k-1}]-\tDelta_k[\bvv^k]\bvv\big)^\top \vdelta.
    \]
    Therefore
    \begin{align*}
        DF_k(\vv)[\vdelta]
        &=
        D(\norm{\vv}^2)[\vdelta]\tDelta_k[\bvv^k]
        +
        \norm{\vv}^2 D\big(\tDelta_k[\bvv^k]\big)[\vdelta]\\
        &=
        2(\vv^\top\vdelta)\tDelta_k[\bvv^k]
        +
        k\norm{\vv}
        \big(\tDelta_k[\bvv^{k-1}]-\tDelta_k[\bvv^k]\bvv\big)^\top \vdelta\\
        &=
        \Big(
        (2-k)\tDelta_k[\bvv^k]\vv
        +
        k\norm{\vv}\,\tDelta_k[\bvv^{k-1}]
        \Big)^\top \vdelta.
    \end{align*}
    Hence
    \[
    \nabla F_k(\vv)
    =
    (2-k)\tDelta_k[\bvv^k]\vv
    +
    k\norm{\vv}\,\tDelta_k[\bvv^{k-1}],
    \]
    and summing over \(k\) gives the formula for \(g(\vv)\).
    
    It remains to differentiate once more.
    
    For \(k=0\), one has
    $F_0(\vv)=\tDelta_0\norm{\vv}^2$, 
    $\nabla^2 F_0(\vv)=2\tDelta_0\mI$.
    
    For \(k=1\), one has
    $F_1(\vv)=\norm{\vv}\,\tDelta_1[\vv]=\norm{\vv}^2\tDelta_1[\bvv]$,
    hence
    \[
    \nabla F_1(\vv)=\tDelta_1[\bvv]\vv+\norm{\vv}\,\tDelta_1.
    \]
    Using
    $D\big(\tDelta_1[\bvv]\big)[\vdelta]
    =
    \frac{1}{\norm{\vv}}
    \big(\tDelta_1-\tDelta_1[\bvv]\bvv\big)^\top\vdelta$,
    we obtain
    \begin{align*}
    D(\nabla F_1(\vv))[\vdelta]
    &=
    \tDelta_1[\bvv]\vdelta
    +
    \bvv\big(\tDelta_1-\tDelta_1[\bvv]\bvv\big)^\top\vdelta
    +
    \tDelta_1\bvv^\top\vdelta\\
    &=
    \Big(
    \tDelta_1[\bvv]\mI
    +
    \bvv\tDelta_1^\top
    +
    \tDelta_1\bvv^\top
    -
    \tDelta_1[\bvv]\bvv\bvv^\top
    \Big)\vdelta.
    \end{align*}
    This is exactly \(\tJ_1(\bvv)\vdelta\).
    
    Now let \(k\ge 2\). From the gradient formula above,
    \[
    \nabla F_k(\vv)
    =
    (2-k)\tDelta_k[\bvv^k]\vv
    +
    k\norm{\vv}\,\tDelta_k[\bvv^{k-1}].
    \]
    For the first term,
    \[
    D\Big((2-k)\tDelta_k[\bvv^k]\vv\Big)[\vdelta]
    =
    (2-k)\tDelta_k[\bvv^k]\vdelta
    +
    (2-k)\,D\big(\tDelta_k[\bvv^k]\big)[\vdelta]\,\vv.
    \]
    For the second term,
    $D(\norm{\vv})[\vdelta]=\bvv^\top\vdelta$,
    and
    \begin{align*}
    D\big(\tDelta_k[\bvv^{k-1}]\big)[\vdelta]
    &=
    (k-1)\tDelta_k\big[\bvv^{k-2},D(\bvv)[\vdelta]\big]
    =
    \frac{k-1}{\norm{\vv}}
    \Big(
    \tDelta_k[\bvv^{k-2}]\vdelta
    -
    (\bvv^\top\vdelta)\tDelta_k[\bvv^{k-1}]
    \Big).
    \end{align*}
    Hence
    \begin{align*}
    D\Big(k\norm{\vv}\,\tDelta_k[\bvv^{k-1}]\Big)[\vdelta]
    &=
    k(\bvv^\top\vdelta)\tDelta_k[\bvv^{k-1}]
    +
    k\norm{\vv}\,D\big(\tDelta_k[\bvv^{k-1}]\big)[\vdelta]\\
    &=
    k(\bvv^\top\vdelta)\tDelta_k[\bvv^{k-1}]
    +
    k(k-1)
    \Big(
    \tDelta_k[\bvv^{k-2}]\vdelta
    -
    (\bvv^\top\vdelta)\tDelta_k[\bvv^{k-1}]
    \Big).
    \end{align*}
    
    Combining the two derivatives and using
    $D\big(\tDelta_k[\bvv^k]\big)[\vdelta]
    =
    \frac{k}{\norm{\vv}}
    \big(\tDelta_k[\bvv^{k-1}]-\tDelta_k[\bvv^k]\bvv\big)^\top \vdelta$, 
    $\vv=\norm{\vv}\bvv$,
    we obtain
    \begin{align*}
    D(\nabla F_k(\vv))[\vdelta]
    &=
    (2-k)\tDelta_k[\bvv^k]\vdelta
    +
    k(2-k)\,
    \bvv
    \big(\tDelta_k[\bvv^{k-1}]-\tDelta_k[\bvv^k]\bvv\big)^\top\vdelta\\
    &\quad
    +
    k(\bvv^\top\vdelta)\tDelta_k[\bvv^{k-1}]
    +
    k(k-1)\Big(
    \tDelta_k[\bvv^{k-2}]\vdelta
    -
    (\bvv^\top\vdelta)\tDelta_k[\bvv^{k-1}]
    \Big).
    \end{align*}
    Rearranging terms yields
    $D(\nabla F_k(\vv))[\vdelta]=\tJ_k(\bvv)\vdelta$
    with
    \[
    \tJ_k(\bvv)
    =
    (2-k)\tDelta_k[\bvv^{\otimes k}]\,\mI
    +
    k(k-2)(\mI-\bvv\bvv^\top)\tDelta_k[\bvv^{\otimes k-2}](\mI-\bvv\bvv^\top)
    +
    k\,\tDelta_k[\bvv^{\otimes k-2}].
    \]
    Summing over \(k\) proves the formula for \(\nabla g(\vv)\).
\end{proof}

Below we show $\tJ$ is bounded and Lipschitz.
\begin{lemma}\label{lem: tDelta and J bound}
    For every \(\vu,\vu'\in \sS^{d-1}\), we have
    \[
    \tJ_0(\vu)=2\tDelta_0\mI,\qquad \norm{\tJ_1(\vu)}_{\op}\le C\norm{\tDelta_1},
    \qquad
    \norm{\tJ_1(\vu)-\tJ_1(\vu')}_{\op}\le C\norm{\tDelta_1}\norm{\vu-\vu'}.
    \]
    
    For every integer \(k\ge 2\), if 
    \[
    \norm{\tT_2}_{\op}
    =
    \norm{\frac1m\sum_{i=1}^m \vw_i^{\otimes 2}}_{\op}\le M,
    \qquad
    \norm{\tT_2^*}_{\op}
    =
    \norm{\sum_{j=1}^r \vw_j^{*\otimes 2}}_{\op}\le M,
    \]
    then there is an absolute constant \(C\) such that
    \[
    \norm{\tJ_k(\vu)}_{\op}\le C M k,
    \qquad
    \norm{\tJ_k(\vu)-\tJ_k(\vu')}_{\op}\le C M k^2\norm{\vu-\vu'}.
    \]
\end{lemma}

\begin{proof}
    The formula for \(\tJ_0, \tJ_1\) are immediate from \cref{lem: hermite of g and J}.

    \paragraph{Bound on $\norm{\tJ_k}$.}
    For $k\ge2$, by linearity, it is enough to prove the corresponding estimate for a single atom.
    Let \(\vs\in\sS^{d-1}\), and let \(\tJ_k(\vu;\vs)\) denote the matrix obtained from the formula in \cref{lem: hermite of g and J} by replacing \(\tDelta_k\) with \(\vs^{\otimes k}\).
    Write
    $\rho:=\vu^\top\vs$, 
    $\vb:=\vs-\rho\vu$, 
    $P:=\mI-\vu\vu^\top$.
    Since
    \[
        \vs^{\otimes k}[\vu^{k-2}]
        =
        \rho^k\vu\vu^\top
        +\rho^{k-1}(\vu\vb^\top+\vb\vu^\top)
        +\rho^{k-2}\vb\vb^\top,
    \]
    and \(P\vu=0\), \(P\vb=\vb\), we get
    \[
        \tJ_k(\vu;\vs)
        =
        2\rho^k\vu\vu^\top
        +
        (2-k)\rho^k P
        +
        k\rho^{k-1}(\vu\vb^\top+\vb\vu^\top)
        +
        k(k-1)\rho^{k-2}\vb\vb^\top.
    \]
    
    Fix \(\vx\in\sS^{d-1}\), and define
    $a:=\vx^\top\vu$, $c:=\vx^\top\vs$, $d:=\vx^\top\vb=c-\rho a$.
    Then to bound $\norm{\tJ_k(\vu;\vs)}_\op$, it suffices to bound
    \[
        q_k(\vu):=x^\top \tJ_k(\vu;\vs)x
        =
        2\rho^k a^2
        +
        (2-k)\rho^k(1-a^2)
        +
        2k\rho^{k-1}ad
        +
        k(k-1)\rho^{k-2}d^2.
    \]
    
    For the last term, we claim that
    \begin{align}\label{eq: jacobian calculation bound}
        (k-1)|\rho|^{k-2}d^2\le C(c^2+\rho^2).
    \end{align}
    Indeed, if \((k-1)|\rho|^{k-2}\le 1\), then
    $
        d^2=(c-\rho a)^2\le 2c^2+2\rho^2.
    $
    If \((k-1)|\rho|^{k-2}>1\), then necessarily \(|\rho|\ge (k-1)^{-1/(k-2)}\ge \frac12\), and since \(d=x^\top\vb\),
    we know $
        d^2\le \norm{\vb}^2=1-\rho^2.
    $
    Hence
    \[
    (k-1)|\rho|^{k-2}d^2
    \le
    \sup_{0\le r\le 1}(k-1)r^{k-2}(1-r^2)
    \le 2
    \le 8\rho^2.
    \]
    This proves the claim on the last term.
    
    Now for the second term, we have
    \[
        2k|\rho|^{k-1}|a||d|
        \le
        k\rho^2+k|\rho|^{2k-4}d^2
        \le
        k\rho^2+k|\rho|^{k-2}d^2.
    \]
    
    Using \(|\rho|^k\le \rho^2\), we obtain
    \begin{align*}
    |q_k(\vu)|
    &\le
    2\rho^2+(k-2)\rho^2+k\rho^2+k|\rho|^{k-2}d^2+k(k-1)|\rho|^{k-2}d^2\\
    &\le
    Ck\big(\rho^2+(k-1)|\rho|^{k-2}d^2\big)
    \le
    Ck(c^2+\rho^2).
    \end{align*} Therefore
    \[
    |\vx^\top \tJ_k(\vu;\vs)\vx|
    \le
    Ck\Big((\vx^\top\vs)^2+(\vu^\top\vs)^2\Big).
    \]
    
    Now let
    \[
    \tJ_k^{\rm mdl}(\vu)
    :=
    \frac1m\sum_{i=1}^m \norm{\vw_i}^2 \tJ_k(\vu;\bvw_i).
    \]
    Applying the above bound on $\tJ_k$ and using \(\norm{\tT_2}_{\op}\le M\),
    \[
        \norm{\tJ_k^{\rm mdl}(\vu)}_{\op}
        \le |\vx^\top \tJ_k^{\rm mdl}(\vu)\vx|
        \le
        Ck\Big(\vx^\top \tT_2 \vx+\vu^\top \tT_2 \vu\Big)
        \le 2CMk.
    \]
    The same estimate holds for the teacher contribution, hence
    \[
    \norm{\tJ_k(\vu)}_{\op}\le CMk.
    \]

    \paragraph{Bound on Lipschitz of $\tJ_k$.}
    It remains to prove the Lipschitz bound.
    
    For \(k=2\), the formula in \cref{lem: hermite of g and J} gives
    $\tJ_2(\vu)=2\tDelta_2,$
    which is independent of \(\vu\). 
    
    For $k\ge3$, fix \(\vu\in \sS^{d-1}\), \(\vxi\perp \vu\) with \(\norm{\vxi}=1\), and keep \(\vx\in \sS^{d-1}\) fixed.
    Set
    \[
        \mu:=\vx^\top \vxi,
        \qquad
        \nu:=\vxi^\top \vs.
    \]
    Using the definition of $\rho,a,d$ before, one has
    \[
        D\rho[\vxi]=\nu,
        \qquad
        Da[\vxi]=\mu,
        \qquad
        Dd[\vxi]=-(\nu a+\rho\mu).
    \]
    Also, using \(c=\rho a+d\), $q_k$ defined above can be rewritten as
    \[
        q_k(\vu)
        =
        (2-k)\rho^k
        +
        k\rho^{k-2}\Big(c^2+(k-2)d^2\Big).
    \]
    Differentiating in the tangent direction \(\vxi\) gives
    \[
    Dq_k(\vu)[\vxi]
    =
    k(2-k)\rho^{k-1}\nu
    +
    k(k-2)\rho^{k-3}\nu\Big(c^2+(k-2)d^2\Big)
    -
    2k(k-2)\rho^{k-2}d(\nu a+\rho\mu).
    \]
    
    We bound each term one by one.
    
    For first term, $|k(2-k)\rho^{k-1}\nu|
    \le Ck^2(\rho^2+\nu^2)$.
    
    For second term, we have $\big|k(k-2)\rho^{k-3}\nu\,c^2\big|
    \le Ck^2(c^2+\nu^2)$ and claim that
    \[
    (k-2)^2|\rho|^{k-3}d^2\le Ck(c^2+\rho^2).
    \]
    If \(|\rho|\le \frac12\), then
    $
    (k-2)^2|\rho|^{k-3}\le \sup_{m\ge 0} m^2 2^{-m}\le C,
    $
    so the bound follows from \(d^2\le 2c^2+2\rho^2\).
    If \(|\rho|>\frac12\), then \(d^2\le 1-\rho^2\), and
    $
    (k-2)^2|\rho|^{k-3}d^2
    \le
    \frac{k-2}{|\rho|}\sup_{0\le r\le 1}(k-2)r^{k-2}(1-r^2)
    \le Ck.
    $
    Since \(|\rho|>\frac12\), we have \(Ck\le C'k\rho^2\), which also proves the bound.
    
    For the last term, we have using \eqref{eq: jacobian calculation bound}
    \begin{align*}
        2k(k-2)|\rho|^{k-2}|d||\nu||a|
        \le&
        k(k-2)|\rho|^{k-2}(d^2+\nu^2)
        \le
        Ck(c^2+\rho^2)+k^2\nu^2,\\
        2k(k-2)|\rho|^{k-1}|d||\mu|
        \le&
        k(k-2)\rho^2
        +
        k(k-2)|\rho|^{2k-4}d^2
        \le
        Ck^2(c^2+\rho^2)
    \end{align*}
    
    Combining these estimates yields
    \[
    |Dq_k(\vu)[\vxi]|
    \le
    Ck^2\Big(c^2+\rho^2+\nu^2\Big)
    =
    Ck^2\Big((\vx^\top\vs)^2+(\vu^\top\vs)^2+(\vxi^\top\vs)^2\Big).
    \tag{8}
    \]
    
    Now averaging over the model atoms:
    \begin{align*}
    \norm{D\tJ_k^{\rm mdl}(\vu)[\vxi]}_{\op}
    \le \big|x^\top D\tJ_k^{\rm mdl}(\vu)[\vxi]x\big|
    &\le
    Ck^2\Big(x^\top \tT_2 x+\vu^\top \tT_2 \vu+\vxi^\top \tT_2 \vxi\Big)
    \le 3CMk^2.
    \end{align*}
    The same bound holds for the teacher contribution, hence
    \[
    \norm{D\tJ_k(\vu)[\vxi]}_{\op}\le CMk^2
    \qquad
    (\vxi\perp \vu,\ \norm{\vxi}=1).
    \]
    Integrating along a shortest geodesic on \(\sS^{d-1}\) from \(\vu\) to \(\vu'\) gives
    \[
    \norm{\tJ_k(\vu)-\tJ_k(\vu')}_{\op}
    \le CMk^2 \,{\rm dist}_{\sS}(\vu,\vu')
    \le CMk^2\norm{\vu-\vu'}.
    \]
    This proves the lemma.
\end{proof}
    
\begin{lemma}[Jacobian expansion of smooth activation]\label{lem: smooth activation jacob}
    Assume
    $
        \sum_{k\ge 2} k^2 |\hat\sigma_k|^2 < \infty.
    $ for activation $\sigma$.
    Then:
    \begin{enumerate}
        \item The series
        \[
        \tJ(\vu)
        :=
        2\hat\sigma_0^2\tDelta_0\mI
        +
        \hat\sigma_1^2\tJ_1(\vu)
        +
        \sum_{k\ge 2}\hat\sigma_k^2 \tJ_k(\vu)
        \]
        converges absolutely and uniformly on \(\sS^{d-1}\).
        
        \item Let \(\vv_t:=(1-t)\vv'+t\vv\) for \(0\le t\le 1\), and assume
        $\norm{\vv_t}\ge r_0>0$ for all $t\in[0,1]$.
        Then
        \[
        g(\vv)-g(\vv')
        =
        \int_0^1 \tJ(\bvv_t)(\vv-\vv')\,\rd t
        =
        \tJ(\bvv)(\vv-\vv')+\vr(\vv,\vv'),
        \]
        with
        \[
        \norm{\vr(\vv,\vv')}
        \le
        C\frac{1}{r_0}
        \Big(
        |\hat\sigma_1|^2\norm{\tDelta_1}
        +
        M\sum_{k\ge 2}k^2|\hat\sigma_k|^2
        \Big)
        \norm{\vv-\vv'}^2.
        \]
    \end{enumerate}
\end{lemma}
    
\begin{proof}
    We show one by one
    \paragraph{Proof of (a).}
    By \cref{lem: tDelta and J bound}, we know for \(k\ge 2\)
    \[
    \sup_{\vu\in\sS^{d-1}}\norm{\hat\sigma_1^2\tJ_1(\vu)}_{\op}
    \le C|\hat\sigma_1|^2\norm{\tDelta_1},\qquad
    \norm{\hat\sigma_k^2\tJ_k(\vu)}_{\op}
    \le
    CMk|\hat\sigma_k|^2.
    \]
    Since \(\sum_{k\ge 2}k|\hat\sigma_k|^2<\infty\), the Weierstrass \(M\)-test implies that
    $
    \sum_{k\ge 2}\hat\sigma_k^2\tJ_k(\vu)
    $
    converges absolutely and uniformly on \(\sS^{d-1}\). Adding the explicit \(k=0,1\) terms proves (a).
    
    Moreover, by \cref{lem: tDelta and J bound}, we know $\tJ$ is Lipschitz on $\sS^{d-1}$
    \[
        \norm{\tJ(\vu)-\tJ(\vu')}_{\op}
        \le
        C\Big(
        |\hat\sigma_1|^2\norm{\tDelta_1}
        +
        M\sum_{k\ge 2}k^2|\hat\sigma_k|^2
        \Big)\norm{\vu-\vu'}.
    \]
    
    \paragraph{Proof of (b).}
    For each \(N\ge 2\), define
    \[
    G_N(\vv):=\sum_{k=0}^N \hat\sigma_k^2 \norm{\vv}^{2-k}\tDelta_k[\vv^k].
    \]
    By \cref{lem: hermite of g and J},
    \[
    \nabla^2 G_N(\vv)
    =
    2\hat\sigma_0^2\tDelta_0\mI
    +
    \hat\sigma_1^2\tJ_1(\bvv)
    +
    \sum_{k=2}^N\hat\sigma_k^2 \tJ_k(\bvv)
    =:\tJ_N(\bvv).
    \]
    Therefore,
    \[
    \nabla G_N(\vv)-\nabla G_N(\vv')
    =
    \int_0^1 \tJ_N(\bvv_t)(\vv-\vv')\,\rd t.
    \]
    Since \(\sum_{k\ge 2}k|\hat\sigma_k|^2<\infty\), the gradient series converges absolutely, so
    $
    \nabla G_N(\vv)\to g(\vv)$ and
    $\nabla G_N(\vv')\to g(\vv').
    $
    Also \(\tJ_N\to \tJ\) uniformly on \(\sS^{d-1}\) by part (a). Passing to the limit gives
    \[
    g(\vv)-g(\vv')
    =
    \int_0^1 \tJ(\bvv_t)(\vv-\vv')\,\rd t.
    \]
    
    Now write
    \[
    g(\vv)-g(\vv')
    =
    \tJ(\bvv)(\vv-\vv')
    +
    \underbrace{\int_0^1\big(\tJ(\bvv_t)-\tJ(\bvv)\big)(\vv-\vv')\,\rd t}_{:=\vr(\vv,\vv')}.
    \]
    Since the map \(\vz\mapsto \vz/\norm{\vz}\) is \(2/r_0\)-Lipschitz on \(\{\vz:\norm{\vz}\ge r_0\}\),
    $
    \norm{\bvv_t-\bvv}
    \le
    \frac{2}{r_0}\norm{\vv_t-\vv}
    =
    \frac{2(1-t)}{r_0}\norm{\vv-\vv'}.
    $
    Hence, using $\tJ$ is Lipschitz in part (a),
    \begin{align*}
    \norm{\vr(\vv,\vv')}
    &\le
    \int_0^1 \norm{\tJ(\bvv_t)-\tJ(\bvv)}_{\op}\,\norm{\vv-\vv'}\,\rd t\\
    &\le
    C\Big(
    |\hat\sigma_1|^2\norm{\tDelta_1}
    +
    M\sum_{k\ge 2}k^2|\hat\sigma_k|^2
    \Big)
    \int_0^1 \norm{\bvv_t-\bvv}\,\rd t\,\norm{\vv-\vv'}\\
    &\le
    C\frac{1}{r_0}
    \Big(
    |\hat\sigma_1|^2\norm{\tDelta_1}
    +
    M\sum_{k\ge 2}k^2|\hat\sigma_k|^2
    \Big)
    \norm{\vv-\vv'}^2.
    \end{align*}
\end{proof}

\begin{lemma}[Jacobian expansion of ReLU]\label{lem: relu jacob}
    Assume \(\sigma(t)=t_+\), and let \(\hat\sigma_k\) be its Hermite coefficients. Then:
    \begin{enumerate}
    \item The series
    \[
    \tJ(\vu)
    :=
    2\hat\sigma_0^2\tDelta_0\mI
    +
    \hat\sigma_1^2\tJ_1(\vu)
    +
    \sum_{k\ge 2}\hat\sigma_k^2\tJ_k(\vu)
    \]
    converges absolutely and uniformly on \(\sS^{d-1}\).
    Moreover, for every \(N\ge 2\),
    \[
    \sup_{\vu\in\sS^{d-1}}
    \norm{\sum_{k>N}\hat\sigma_k^2\tJ_k(\vu)}_{\op}
    \le
    CMN^{-1/2},
    \qquad
    \sup_{\vu\in\sS^{d-1}}\norm{\tJ_k(\vu)}_{\op}\le CMk
    \quad(k\ge 2).
    \]
    
    \item Let \(\vv_t:=(1-t)\vv'+t\vv\) for \(0\le t\le 1\), and assume
    $\norm{\vv_t}\ge r_0>0$ for all $t\in[0,1].$
    Then
    \[
        g(\vv)-g(\vv')
        =
        \int_0^1 \tJ(\bvv_t)(\vv-\vv')\,\rd t
        =
        \tJ(\bvv)(\vv-\vv')+\vr_{\mathrm{ReLU}}(\vv,\vv'),
    \]
    with
    \[
    \norm{\vr_{\mathrm{ReLU}}(\vv,\vv')}
    \le
    C\frac{\norm{\tDelta_1}+M}{r_0}\norm{\vv-\vv'}^2.
    \]
    \end{enumerate}
\end{lemma}

\begin{proof}
    We show one by one.
    \paragraph{Proof of (a).}
    By \cref{lem: tDelta and J bound}, for \(k\ge 2\),
    $
    \sup_{\vu\in\sS^{d-1}}\norm{\tJ_k(\vu)}_{\op}\le CMk.
    $
    For ReLU, one has
    $
    |\hat\sigma_k|\lesssim k^{-5/4},
    $
    hence
    $\sup_{\vu\in\sS^{d-1}}
    \norm{\hat\sigma_k^2\tJ_k(\vu)}_{\op}
    \lesssim
    CMk^{-3/2}.
    $
    Therefore \(\sum_{k\ge 2}\hat\sigma_k^2\tJ_k(\vu)\) converges absolutely and uniformly on \(\sS^{d-1}\), and
    \[
    \sup_{\vu\in\sS^{d-1}}
    \norm{\sum_{k>N}\hat\sigma_k^2\tJ_k(\vu)}_{\op}
    \le
    CM\sum_{k>N}k^{-3/2}
    \le
    CMN^{-1/2}.
    \]
    Adding the explicit \(k=0,1\) terms finishes the proof of (a).
    
    \paragraph{Proof of (b): exact identity.}
    Exactly as in the smooth case, we omit the details.
    
    \paragraph{Proof of (b): Lipschitz bound for \(\tJ\).}
    The \(k=0\) term is constant in \(\vu\), so it has no contribution to the Lipschitz norm.
    The \(k=1\) term satisfies
    \[
    \norm{\hat\sigma_1^2\tJ_1(\vu)-\hat\sigma_1^2\tJ_1(\vu')}_{\op}
    \le
    C\norm{\tDelta_1}\norm{\vu-\vu'}
    \]
    by \cref{lem: tDelta and J bound}.
    
    It remains to treat the nonlinear part
    \[
    \tJ_{\ge 2}(\vu):=\sum_{k\ge 2}\hat\sigma_k^2\tJ_k(\vu).
    \]
    For ReLU,
    \[
    \kappa(\rho):=\E[(z_1)_+(z_2)_+]
    =
    \frac1{2\pi}\Big(\sqrt{1-\rho^2}+(\pi-\arccos\rho)\rho\Big)
    =
    \sum_{k\ge 0}\hat\sigma_k^2\rho^k.
    \]
    Let
    \[
    \widetilde\kappa(\rho):=\kappa(\rho)-\hat\sigma_0^2-\hat\sigma_1^2\rho
    =
    \kappa(\rho)-\frac1{2\pi}-\frac14\rho
    =
    \sum_{k\ge 2}\hat\sigma_k^2\rho^k.
    \]
    For a single atom \(\vs\in\sS^{d-1}\), write
    $\rho:=\vu^\top\vs$, $\vb:=\vs-\rho\vu$.
    A direct differentiation of \(\norm{\vv}^2\widetilde\kappa(\bvv^\top\vs)\) gives
    \[
    \tJ_{\ge 2}(\vu;\vs)
    =
    2\widetilde\kappa(\rho)\vu\vu^\top
    +
    \bigl(2\widetilde\kappa(\rho)-\rho\widetilde\kappa'(\rho)\bigr)(\mI-\vu\vu^\top)
    +
    \widetilde\kappa'(\rho)(\vu\vb^\top+\vb\vu^\top)
    +
    \widetilde\kappa''(\rho)\vb\vb^\top.
    \]
    Moreover,
    \[
    \tJ_{\ge 2}(\vu)
    =
    \frac1m\sum_{i=1}^m \norm{\vw_i}^2\,\tJ_{\ge 2}(\vu;\bvw_i)
    -
    \sum_{j=1}^r \norm{\vw_j^*}^2\,\tJ_{\ge 2}(\vu;\bvw_j^*).
    \]
    
    Fix \(\vu\in\sS^{d-1}\), \(\vxi\perp\vu\) with \(\norm{\vxi}=1\), and \(\vx\in\sS^{d-1}\). Define
    \[
    a:=\vx^\top\vu,\qquad
    c:=\vx^\top\vs,\qquad
    d:=\vx^\top\vb=c-\rho a,
    \qquad
    \mu:=\vx^\top\vxi,\qquad
    \nu:=\vxi^\top\vs.
    \]
    Since \(D\rho[\vxi]=\nu\), \(Da[\vxi]=\mu\), and \(Dd[\vxi]=-(\nu a+\rho\mu)\), the quadratic form
    \[
    q(\vu):=x^\top\tJ_{\ge 2}(\vu;\vs)x
    =
    \bigl(2\widetilde\kappa(\rho)-\rho\widetilde\kappa'(\rho)\bigr)
    +
    \rho\widetilde\kappa'(\rho)a^2
    +
    2\widetilde\kappa'(\rho)ad
    +
    \widetilde\kappa''(\rho)d^2,
    \]
    and differentiating yields
    \[
    Dq(\vu)[\vxi]
    =
    A(\rho)\bigl(\nu(1-a^2)+2\mu d\bigr)
    +
    \kappa'''(\rho)\nu d^2,
    \]
    where
    $
    A(\rho):=\widetilde\kappa'(\rho)-\rho\widetilde\kappa''(\rho).
    $
    
    The goal is to bound $Dq(\vu)[\vxi]$. Now
    \[
    \widetilde\kappa'(\rho)=\frac{\arcsin\rho}{2\pi},
    \qquad
    \widetilde\kappa''(\rho)=\kappa''(\rho)=\frac{1}{2\pi\sqrt{1-\rho^2}},
    \qquad
    \kappa'''(\rho)=\frac{\rho}{2\pi(1-\rho^2)^{3/2}}.
    \]
    Hence
    \[
    A(\rho)=\frac1{2\pi}\Big(\arcsin\rho-\frac{\rho}{\sqrt{1-\rho^2}}\Big).
    \]
    For \(\rho\ge 0\), the function
    $
    h(\rho):=\frac{\rho}{\sqrt{1-\rho^2}}-\arcsin\rho
    $
    satisfies \(h(0)=0\) and
    $
    h'(\rho)=\frac{\rho^2}{(1-\rho^2)^{3/2}}\ge 0.
    $
    Therefore
    \[
    0\le h(\rho)=\int_0^\rho \frac{t^2}{(1-t^2)^{3/2}}\,\rd t
    \le
    \rho^2\int_0^\rho \frac{1}{(1-t^2)^{3/2}}\,\rd t
    =
    \frac{\rho^3}{\sqrt{1-\rho^2}}
    \le
    \frac{\rho^2}{\sqrt{1-\rho^2}}.
    \]
    By oddness, this implies
    $|A(\rho)|\le C\frac{\rho^2}{\sqrt{1-\rho^2}}.
    $
    Since \(|d|\le \norm{\vb}=\sqrt{1-\rho^2}\) and \(|\nu|\le \sqrt{1-\rho^2}\), the above gives
    \[
    |A(\rho)|\,|\nu|\le C\rho^2,
    \qquad
    |A(\rho)|\,|d|\le C\rho^2.
    \]
    
    We also claim that
    \[
    |\kappa'''(\rho)\nu d^2|\le C(c^2+\rho^2).
    \]
    Indeed, if \(|\rho|\le \frac12\), then \(|\kappa'''(\rho)|\le C|\rho|\), so
    $|\kappa'''(\rho)\nu d^2|
    \le
    C|\rho|d^2
    \le
    C(c^2+\rho^2)$,
    because \(d^2\le 2c^2+2\rho^2\).
    If \(|\rho|>\frac12\), then \(d^2\le 1-\rho^2\) and \(|\nu|\le \sqrt{1-\rho^2}\), hence
    \[
    |\kappa'''(\rho)\nu d^2|
    \le
    C\frac{|\rho|}{(1-\rho^2)^{3/2}}\cdot \sqrt{1-\rho^2}\cdot (1-\rho^2)
    \le
    C|\rho|
    \le
    C\rho^2.
    \]
    
    Combining all bounds above, we obtain
    \[
    |Dq(\vu)[\vxi]|
    \le
    C\bigl(c^2+\rho^2\bigr)
    =
    C\Big((\vx^\top\vs)^2+(\vu^\top\vs)^2\Big).
    \]
    
    Now averaging over the model atoms:
    \begin{align*}
    \norm{D\tJ_{\ge 2}^{\rm mdl}(\vu)[\vxi]}_{\op}
    \le \big|x^\top D\tJ_{\ge 2}^{\rm mdl}(\vu)[\vxi]x\big|
    &\le
    C\Big(x^\top \tT_2 x+\vu^\top \tT_2 \vu\Big)
    \le 2CM.
    \end{align*}
    The same bound holds for the teacher contribution, so
    \[
    \norm{D\tJ_{\ge 2}(\vu)[\vxi]}_{\op}\le CM
    \qquad
    (\vxi\perp\vu,\ \norm{\vxi}=1).
    \]
    Integrating along a shortest geodesic on \(\sS^{d-1}\) gives
    \[
    \norm{\tJ_{\ge 2}(\vu)-\tJ_{\ge 2}(\vu')}_{\op}
    \le
    CM\norm{\vu-\vu'}.
    \]
    Therefore
    \[
    \norm{\tJ(\vu)-\tJ(\vu')}_{\op}
    \le
    C(\norm{\tDelta_1}+M)\norm{\vu-\vu'}.
    \]
    
    Finally, follow the same argument as in \cref{lem: smooth activation jacob} we have
    \[
    g(\vv)-g(\vv')
    =
    \tJ(\bvv)(\vv-\vv')
    +
    \underbrace{\int_0^1\bigl(\tJ(\bvv_t)-\tJ(\bvv)\bigr)(\vv-\vv')\,\rd t}_{:=\vr_{\mathrm{ReLU}}(\vv,\vv')}.
    \]
    with
    \begin{align*}
    \norm{\vr_{\mathrm{ReLU}}(\vv,\vv')}
    &\le
    C\frac{\norm{\tDelta_1}+M}{r_0}\norm{\vv-\vv'}^2.
    \end{align*}
\end{proof}

\section{Technical lemma}
In this section, we collect some lemmas used in the analysis. They are quite technical and often involve tedious calculations.

We provide several results that bound the different terms appearing in the gradient.
\begin{lemma}\label{lem: high order moment bound}
    Recall the notation
    $\E_S f(\vw) = \frac{1}{|S|}\sum_{\vw\in S}f(\vw)$, 
    $\bbvvler:=\frac{\vvler}{\norm{\vvler}_2}$ and $\bbvvgtr:=\frac{\vvgtr}{\norm{\vvgtr}_2}$.
    Suppose for all neuron $\vv$ in the given set $S$ satisfy $\norm{\vv}\le b_1$,
    $(\bbvvler)_i^2 \le b_2 (\bbvvler^\btt{0})_i^2$ and $(\bbvvgtr)_i^2 \le b_2 (\bbvvgtr^\btt{0})_i^2$ where $\vv^\btt{0}$ is random gaussian (at initialization). 
    
    Then for $j< r$ and any vector $\vu$ (not necessary be one of neuron in set $S$), there exists a large enough universal constant $C$ such that
    \begin{align*}
        0\le \tldE_S \left[\norm{\vw}_2^2\langle \bvw,\bvu\rangle^{2j}\right]
        \le& 
            C b_1^2 \cdot \left(\frac{Cj b_2}{r}\right)^j,\\
        0\le \tldE_S [\norm{\vw}_2^2\langle\bvw,\bvu\rangle^{2j-1}\bw_i/\bu_i]
        \le&  
           Cb_1^2 \cdot \left(\frac{Cj b_2}{r}\right)^{j},
            \text{ for all $i\in[d]$}
    \end{align*}

    If we have a better bound of $ (\bvv_i)^2 \le b_2(\bvv_i^\btt{0})^2$ (e.g., in Stage 1.1), then the above can be improved into
    \begin{align*}
        0\le \tldE_S \left[\norm{\vw}_2^2\langle \bvw,\bvu\rangle^{2j}\right]
        \le& 
            C b_1^2 \cdot \left(\frac{C j b_2}{d}\right)^j,\\
        0\le \tldE_S [\norm{\vw}_2^2\langle\bvw,\bvu\rangle^{2j-1}\bw_i/\bu_i]
        \le&  
            C b_1^2 \cdot \left(\frac{C j b_2}{d}\right)^{j},
            \text{ for all $i\in[d]$}
    \end{align*}
\end{lemma}

We note here that the typical usage of this lemma is to set $b_1^2 = r$ and $b_2$ as a large constant, so all above terms is upper bounded as $r (j/r)^j$.

\begin{proof}
    We show one by one. For simplicity, we will omit the subscript $S$.
    
    \paragraph{Bound $\tldE \left[\norm{\vw}_2^2\langle \bvw,\bvu\rangle^{2j}\right]$} 
    We have (recall $\bbvuler = \frac{\vuler}{\norm{\vuler}}$ and similar for others)
    \begin{align*}
        &\tldE \left[\norm{\vw}_2^2\langle \bvw,\bvu\rangle^{2j}\right]\\
        \le& \tldE \left[\norm{\vw}_2^2\left(
            \langle \bvwler,\bvuler\rangle 
            + \langle \bvwgtr,\bvugtr\rangle
            \right)^{2j}\right]\\
        =& \tldE \left[\norm{\vw}_2^2
            \sum_{k=0}^{2j} \binom{2j}{k}  
            \left(\frac{\norm{\vwler}\norm{\vuler}}{\norm{\vw}\norm{\vu}}\right)^k
            \left(\frac{\norm{\vwgtr}\norm{\vugtr}}{\norm{\vw}\norm{\vu}}\right)^{2j-k}
            \left\langle \bbvwler,\bbvuler\right\rangle^{k}
            \left\langle \bbvwgtr,\bbvugtr\right\rangle^{2j-k}
            \right]\\
        =& \tldE \left[\norm{\vw}_2^2\sum_{k=0}^{2j} \binom{2j}{k}  
            \left(\frac{\norm{\vwler}\norm{\vuler}}{\norm{\vw}\norm{\vu}}\right)^k
            \left(\frac{\norm{\vwgtr}\norm{\vugtr}}{\norm{\vw}\norm{\vu}}\right)^{2j-k}
            \E_{\rho_\sgn}\left[
                \left\langle \rho_\sgn(\bbvwler),\bbvuler\right\rangle^{k}
                \left\langle \rho_\sgn(\bbvwgtr),\bbvugtr\right\rangle^{2j-k}
                \right]
            \right]\\
        \le& b_1^2 \max_{k} \tldE\left[\E_{\rho_\sgn}\left[
                \left\langle \rho_\sgn(\bbvwler),\bbvuler\right\rangle^{k}
                \left\langle \rho_\sgn(\bbvwgtr),\bbvugtr\right\rangle^{2j-k}
                \right]
            \right],
    \end{align*}
    where the second to last line we use $\tldmu$ is conditionally symmetric (hence invariant under any sign flipping) and $\E_{\rho_\sgn}g(\vw)=2^{-d}\sum_{\rho\in\Rho_\sgn}g(\rho(\vw))$ for any function $g$ with $\Rho_\sgn=\{\pm 1\}^d$ as the set of all possible sing flipping ($\rho(\vw)$ as element-wise product $\rho\odot \vw$);
    last line is because the term $\E_{\rho_\sgn}$ is non-negative as we will see next.

    Then we have
    \begin{align*}
        &\E_{\rho_\sgn}\left[
                \left\langle \rho_\sgn(\bbvwler),\bbvuler\right\rangle^{k}
                \left\langle \rho_\sgn(\bbvwgtr),\bbvugtr\right\rangle^{2j-k}
                \right]\\
        =& \E_{\rho_\sgn} \Bigg[
            \sum_{\alpha_1+\cdots+\alpha_r = k} \binom{k}{\alpha_1,\ldots,\alpha_r}
            \prod_{\ell\in[r]}(\rho_\sgn(\bbvwler)_{\ell}(\bbvuler)_{\ell})^{\alpha_\ell}
            \sum_{\alpha_{r+1}+\cdots+\alpha_d = 2j-k} \binom{2j-k}{\alpha_{r+1},\ldots,\alpha_d}
            \prod_{\ell\in[d]\setminus[r]}(\rho_\sgn(\bbvwgtr)_{\ell}(\bbvugtr)_{\ell})^{\alpha_\ell}
            \Bigg]\\
        \myeq{a}& \E_{\rho_\sgn} \Bigg[
            \sum_{\beta_1+\cdots+\beta_r = k/2} \binom{k}{2\beta_1,\ldots,2\beta_r}
            \prod_{\ell\in[r]}((\bbvwler)_{\ell}(\bbvuler)_{\ell})^{2\beta_\ell}
            \sum_{\beta_{r+1}+\cdots+\beta_d = j-k/2} \binom{2j-k}{2\beta_{r+1},\ldots,2\beta_d}
            \prod_{\ell\in[d]\setminus[r]}((\bbvwgtr)_{\ell}(\bbvugtr)_{\ell})^{2\beta_\ell} 
            \Bigg]\\
        \le& b_2^j\E_{\rho_\sgn} \Bigg[
            \sum_{\beta_1+\cdots+\beta_r = k/2} \binom{k}{2\beta_1,\ldots,2\beta_r}
            \prod_{\ell\in[r]}((\bbvwler^\btt{0})_{\ell}(\bbvuler)_{\ell})^{2\beta_\ell}\\ 
            &\qquad\sum_{\beta_{r+1}+\cdots+\beta_d = j-k/2} \binom{2j-k}{2\beta_{r+1},\ldots,2\beta_d}
            \prod_{\ell\in[d]\setminus[r]}((\bbvwgtr^\btt{0})_{\ell}(\bbvugtr)_{\ell})^{2\beta_\ell} 
            \Bigg]\\
        =& b_2^j \E_{\rho_\sgn}\left[
                \left\langle \rho_\sgn(\bbvwler^\btt{0}),\bbvuler\right\rangle^{k}
                \left\langle \rho_\sgn(\bbvwgtr^\btt{0}),\bbvugtr\right\rangle^{2j-k}
                \right],
    \end{align*}
    where (a) $k/2$ is an integer as otherwise the expectation is 0 (since the polynomial is an odd polynomial) and set $\alpha_i=2\beta_i$.
    
    Hence,
    \begin{align*}
        \tldE \left[\norm{\vw}_2^2\langle \bvw,\bvu\rangle^{2j}\right]
        \le& b_1^2 \max_{k} \tldE\left[
                \left\langle \bbvwler^\btt{0},\bbvuler\right\rangle^{k}
                \left\langle \bbvwgtr^\btt{0},\bbvugtr\right\rangle^{2j-k}
            \right]
        \lesssim b_1^2b_2^j \max_k \left(\frac{Ck}{r}\right)^{k/2} \left(\frac{C(2j-k)}{d-r}\right)^{j-\frac{k}{2}}\\
        \le& b_1^2 b_2^j \left(\frac{Cj}{r}\right)^j,
    \end{align*}
    where we use standard Gaussian moment concentration.

    \paragraph{Bound $\tldE [\norm{\vw}_2^2\langle\bvw,\bvu\rangle^{2j-1}\bw_1/\bu_1]$}
    For the second term, similarly we have (WLOG let $i=1$; we first treat the case $1\in[r]$, and the case $1\notin[r]$ is identical)
    \begin{align*}
        &\tldE [\norm{\vw}_2^2\langle\bvw,\bvu\rangle^{2j-1}\bw_1/\bu_1]\\
        =& \tldE \left[\norm{\vw}^2\left(
            \langle \bvwler,\bvuler\rangle 
            + \langle \bvwgtr,\bvugtr\rangle
            \right)^{2j-1}\bw_1/\bu_1\right]\\
        =& \tldE \Bigg[\norm{\vw}^2\sum_{k=0}^{2j-1} \binom{2j-1}{k} 
            \left(\frac{\norm{\vwler}\norm{\vuler}}{\norm{\vw}\norm{\vu}}\right)^k
            \left(\frac{\norm{\vwgtr}\norm{\vugtr}}{\norm{\vw}\norm{\vu}}\right)^{2j-k-1}\\
            &\qquad\E_{\rho_\sgn}\left[
                \left\langle \rho_\sgn(\bbvwler),\bbvuler\right\rangle^{k}
                \left\langle \rho_\sgn(\bbvwgtr),\bbvugtr\right\rangle^{2j-k-1}
                \rho_\sgn(\bvw)_1/\bu_1
                \right]
            \Bigg].
    \end{align*}
    When $k=0$, the above inner expectation is $0$ because the sign on the first coordinate appears only in $\rho_\sgn(\bvw)_1$. Hence it suffices to consider $k\ge 1$. For each $k\ge 1$, we always have at least a $\frac{\norm{\vwler}\norm{\vuler}}{\norm{\vw}\norm{\vu}}$ factor and define
    \begin{align*}
        \Phi_k(\vw):=
        \E_{\rho_\sgn}\left[
            \left\langle \rho_\sgn(\bbvwler),\bbvuler\right\rangle^{k}
            \left\langle \rho_\sgn(\bbvwgtr),\bbvugtr\right\rangle^{2j-k-1}
            \rho_\sgn(\bvw)_1/\bu_1
        \right].
    \end{align*}
    Thus
    \begin{align*}
        &\tldE [\norm{\vw}_2^2\langle\bvw,\bvu\rangle^{2j-1}\bw_1/\bu_1]
        \le
        b_1^2 2^{2j-1}\max_{1\le k\le 2j-1}
        \tldE\left[
            \frac{\norm{\vwler}\norm{\vuler}}{\norm{\vw}\norm{\vu}}|\Phi_k(\vw)|
        \right].
    \end{align*}
    
    Write $\rho_\sgn=(\sigma,\tau)$ where $\sigma\in\{\pm1\}$ is the sign on the first coordinate and
    $\tau\in\{\pm1\}^{d-1}$ is the sign pattern on the remaining coordinates. For each fixed $\tau$, let
    \begin{align*}
        A_\tau:=\sum_{\ell=2}^{r}\tau_\ell (\bbvwler)_\ell (\bbvuler)_\ell,\qquad
        B_\tau:=\sum_{\ell=r+1}^{d}\tau_\ell (\bbvwgtr)_\ell (\bbvugtr)_\ell,\qquad
        a:=(\bbvwler)_1(\bbvuler)_1.
    \end{align*}
    Then
    $\left\langle \rho_\sgn(\bbvwler),\bbvuler\right\rangle
        = \sigma a + A_\tau$ and 
    $\left\langle \rho_\sgn(\bbvwgtr),\bbvugtr\right\rangle
        = B_\tau$,
    and therefore
    \begin{align*}
        \Phi_k(\vw)
        =
        \frac{\bw_1}{\bu_1}
        \E_{\tau}\Big[
            B_\tau^{\,2j-k-1}\,
            \E_{\sigma}\big[(A_\tau+\sigma a)^k \sigma\big]
        \Big].
    \end{align*}
    
    By the mean value theorem,
    \begin{align*}
        \left|
        \E_{\sigma}\big[(A_\tau+\sigma a)^k \sigma\big]
        \right|
        &=
        \frac{|(A_\tau+a)^k-(A_\tau-a)^k|}{2}
        \le
        k|a|\max_{|t|\le 1}|A_\tau+t a|^{k-1}
        \le
        k|a|(|A_\tau|+|a|)^{k-1}.
    \end{align*}
    Thus
    \begin{align*}
        \frac{\norm{\vwler}\norm{\vuler}}{\norm{\vw}\norm{\vu}} |\Phi_k(\vw)|
        \le&
        \frac{\norm{\vwler}\norm{\vuler}}{\norm{\vw}\norm{\vu}}\cdot
        k\frac{|\bw_1|}{|\bu_1|}|a|\,
        \E_{\tau}\Big[
            (|A_\tau|+|a|)^{k-1}|B_\tau|^{2j-k-1}
        \Big]\\
        \le& k\E_{\tau}\Big[
            (|A_\tau|+|a|)^{k-1}|B_\tau|^{2j-k-1}
        \Big]\\
        \le& 2k\E_{\tau}\Big[
            |A_\tau+a|^{k-1}|B_\tau|^{2j-k-1}
        \Big],
    \end{align*}
    where we use $(|A_\tau|+|a|)^{k-1}\le |A_\tau+a|^{k-1}+|A_\tau-a|^{k-1}=2\E_\sigma|A_\tau+a|^{k-1}$.
    
    Hence, using
    $(\bbvwler)_\ell^2\le b_2(\bbvwler^\btt{0})_\ell^2$ coordinate-wise exactly as in the previous part, we obtain
    \begin{align*}
        \frac{\norm{\vwler}\norm{\vuler}}{\norm{\vw}\norm{\vu}} |\Phi_k(\vw)|
        \le&
        k(Cb_2)^j
        \E_{\rho_\sgn}\left[
            (\bbvwler^\btt{0})_1^2
            \left|
            \left\langle \rho_\sgn(\bbvwler^\btt{0}),\bbvuler\right\rangle
            \right|^{k-1}
            \left|
            \left\langle \rho_\sgn(\bbvwgtr^\btt{0}),\bbvugtr\right\rangle
            \right|^{2j-k-1}
        \right].
    \end{align*}
        Therefore
    \begin{align*}
        \tldE\left[\frac{\norm{\vwler}\norm{\vuler}}{\norm{\vw}\norm{\vu}}|\Phi_k(\vw)|\right]
        \le& k(Cb_2)^j
            \tldE\left[
            (\bbvwler^\btt{0})_1^2
            \left|
            \left\langle \bbvwler^\btt{0},\bbvuler\right\rangle
            \right|^{k-1}
            \left|
            \left\langle \bbvwgtr^\btt{0},\bbvugtr\right\rangle
            \right|^{2j-k-1}
        \right]\\
        \lesssim& k (C b_2)^j \cdot
        \frac{1}{r}
        \left(\frac{Ck}{r}\right)^{\frac{k-1}{2}}
        \left(\frac{C(2j-k-1)}{d-r}\right)^{\frac{2j-k-1}{2}}
        \le
        \left(\frac{Cjb_2}{r}\right)^j,
    \end{align*}
    where in the last step we use $k\le 2j-1$.

    Combining with the previous reduction gives
    \begin{align*}
        0\le
        \tldE [\norm{\vw}_2^2\langle\bvw,\bvu\rangle^{2j-1}\bw_1/\bu_1]
        \le
        C b_1^2 \left(\frac{Cjb_2}{r}\right)^j.
    \end{align*}
    This proves the claimed bound.

    When we have the better bound of $ (\bvv_i)^2 \le b_2(\bvv_i^\btt{0})^2$, we do not need to split $\langle \bvw,\bvv\rangle^{2j}$ into $\left(
            \langle \bvwler,\bvvler\rangle 
            + \langle \bvwgtr,\bvvgtr\rangle
            \right)^{2j}$. Then follow the same argument we can get the improved bound.
\end{proof}

\begin{lemma}\label{lem: <>2j bound}
Let $\vv,\vu \in \R^d$ with $\|\vv\|=\|\vu\|=1$. 
If $\vu_1^2 \ge 1-\varepsilon$ for some $\varepsilon \in [0,1]$, then for every integer $j \ge 1$,
\[
\langle \vv,\vu\rangle^{2j} \;=\; \vv_1^{2j} \;+\; O\!\bigl(j\sqrt{\varepsilon}\bigr).
\]
\end{lemma}

\begin{proof}
Since $\vu_1^2 \ge 1-\varepsilon$, we have $|\vu_1|\ge \sqrt{1-\varepsilon}$.  
Let $s=\mathrm{sign}(\vu_1)$. Then
\[
\|\vu - s e_1\|^2 \;=\; 2(1-|\vu_1|)\;\le\; 2\bigl(1-\sqrt{1-\varepsilon}\bigr)\;\le\;2\varepsilon,
\]
so $\|\vu - s e_1\|\le \sqrt{2\varepsilon}$.

By Cauchy--Schwarz,
\[
\bigl|\langle \vv,\vu\rangle- s\,\vv_1\bigr|
=\bigl|\langle \vv,\vu-se_1\rangle\bigr|
\le \|\vv\|\,\|\vu-se_1\|\;\le\;\sqrt{2\varepsilon}.
\]

Now consider $f(t)=t^{2j}$ on $[-1,1]$. Its derivative is $f'(t)=2j\,t^{2j-1}$, so
$\sup_{|t|\le 1}|f'(t)| \le 2j$. Thus $f$ is $2j$-Lipschitz, and
\[
\bigl|\langle \vv,\vu\rangle^{2j} - (s\vv_1)^{2j}\bigr|
\;\le\; 2j \cdot \bigl|\langle \vv,\vu\rangle - s\vv_1\bigr|
\;\le\; 2j\sqrt{2\varepsilon}.
\]

Since $(s\vv_1)^{2j} = \vv_1^{2j}$, the claim follows.
\end{proof}

\end{document}